\documentclass[11pt]{article}
\usepackage{authblk}
\usepackage{amsmath}
\usepackage{floatpag}
\usepackage{amssymb}
\usepackage{xcolor}
\usepackage{graphicx}
\usepackage{grffile}
\usepackage{caption}
\usepackage{subcaption}
\usepackage{enumerate}
\usepackage{setspace}
\usepackage{amsthm}
\usepackage{url} % not crucial - just used below for the URL 
\usepackage{mathtools}
\usepackage{bbm}
\usepackage[normalem]{ulem}
\usepackage{threeparttable}
\usepackage{footnote}
\usepackage[left=2.5cm, right=2.5cm, top=2cm, bottom=2.5cm]{geometry}
\usepackage{tablefootnote}
\usepackage{adjustbox}
\usepackage{makecell}
\usepackage{booktabs}
\usepackage{multirow}
\usepackage{tocloft}
\usepackage{dsfont}
\usepackage{amsfonts}
\usepackage[ruled,linesnumbered]{algorithm2e}
\SetKwInput{KwInput}{Input}       % Input keyword
\SetKwInput{KwOutput}{Output}     % Output keyword
\SetKw{KwRet}{Return}

\usepackage{tikz}
\usepackage{pgfplots}
\pgfplotsset{compat=1.18} % or latest available

\usepackage{bm}
\usepackage{tocloft}
\usepackage[title]{appendix}
\usepackage{textcomp}
\usepackage{eufrak}
\usepackage[mathscr]{eucal}
\usepackage{hyperref,cleveref}

\usepackage{natbib}
\mathtoolsset{showonlyrefs}

\hypersetup{colorlinks=true,linkcolor=blue, citecolor=blue, linktocpage}
\allowdisplaybreaks

\newtheorem{example}{Example} 
\newtheorem{theorem}{Theorem}
\newtheorem{lemma}{Lemma} 
\newtheorem{proposition}{Proposition} 
\newtheorem{remark}{Remark}
\newtheorem{corollary}{Corollary}
\newtheorem{definition}{Definition}

\newtheorem{assumption}{Assumption}

\makeatletter
\newcommand*{\Rom}[1]{\expandafter\@slowromancap\romannumeral #1@}
\makeatother

\newcommand*{\rom}[1]{\romannumeral #1}
\DeclareMathOperator*{\argmax}{arg\,max}
\DeclareMathOperator*{\argmin}{arg\,min}
\newcommand{\lambdamin}{\lambda_{\min}}
\newcommand{\lambdamax}{\lambda_{\max}}
\newcommand{\tP}{\mathbb{P}}
\newcommand{\tQ}{\mathbb{Q}}
\newcommand{\tE}{\mathbb{E}}

\newcommand{\cov}{\textup{Cov}}

\newcommand{\btheta}{\bm{\theta}}
\newcommand{\bthetas}{\bm{\theta}^*}
\newcommand{\bthetaks}[1]{\bm{\theta}^{(#1)*}}
\newcommand{\thetaks}[1]{\theta^{(#1)*}}
\newcommand{\thetak}[1]{\theta^{(#1)}}
\newcommand{\bthetak}[1]{\bm{\theta}^{(#1)}}
\newcommand{\hthetak}[1]{\hat{\bm{\theta}}^{(#1)}}

\newcommand{\twonorma}[1]{\left\|#1\right\|_{2}}

\newcommand{\twonorm}[1]{\|#1\|_{2}}

\newcommand{\fnorm}[1]{\|#1\|_{F}}

\newcommand{\bx}{\bm{x}}
\newcommand{\bg}{\bm{g}}

\newcommand{\bmu}{\bm{\mu}}
\newcommand{\bmuk}[1]{\bm{\mu}^{(#1)}}
\newcommand{\bz}{\bm{z}}
\newcommand{\bv}{\bm{v}}
\newcommand{\bxk}[1]{\bm{x}^{(#1)}}
\newcommand{\yk}[1]{y^{(#1)}}
\newcommand{\zk}[1]{z^{(#1)}}

\newcommand{\by}{\bm{y}}
\newcommand{\bX}{\bm{X}}
\newcommand{\bSigma}{\bm{\Sigma}}
\newcommand{\bSigmak}[1]{\bm{\Sigma}^{(#1)}}
\newcommand{\hSigma}{\widehat{\bm{\Sigma}}}
\newcommand{\hSigmak}[1]{\widehat{\bm{\Sigma}}^{(#1)}}
\newcommand{\barx}{\bar{\bm{x}}}
\newcommand{\barxk}[1]{\bar{\bm{x}}^{(#1)}}
\newcommand{\bu}{\bm{u}}

\newcommand{\htheta}{\hat{\bm{\theta}}}
\newcommand{\otheta}{\overline{\bm{\theta}}}
\newcommand{\hotheta}{\hat{\overline{\bm{\theta}}}}
\newcommand{\prox}{\textup{prox}}
\newcommand{\proxpl}{\prox_{p}}
\newcommand{\<}{\langle}
\renewcommand{\>}{\rangle}
\newcommand{\bTheta}{\bm{\Theta}}
\newcommand{\hTheta}{\widehat{\bm{\Theta}}}
\newcommand{\mL}{\mathcal{L}}
\newcommand{\mLk}[1]{\mathcal{L}^{(#1)}}
\newcommand{\hmL}{\hat{\mathcal{L}}}
\newcommand{\hmLk}[1]{\hat{\mathcal{L}}^{(#1)}}
\newcommand{\hk}[1]{h^{(#1)}}
\newcommand{\hS}{\widehat{S}}
\newcommand{\hmax}{h_{\max}}
\newcommand{\tildeOmega}{\tilde{\Omega}}
\newcommand{\tildeO}{\tilde{O}}

\begin{document}

\title{\bf Contaminated Multi-task Learning with Heterogeneity: Fundamental Limits and Optimal Algorithms}
\date{}
\author[1]{Ye Tian}
\author[2]{Mengchu Li}
\author[3]{Marco Avella Medina}
\affil[1]{Department of Statistics, Pennsylvania State University}
\affil[2]{School of Mathematics, University of Birmingham}
\affil[3]{Department of Statistics, Columbia University}
\maketitle

\vspace{-1cm}

\begin{center}
    % Current version: May 23, 2026
    Current version: July 2, 2026
\end{center}

\vspace{-0.2cm}

\begin{abstract}
    Integrating information across related tasks can substantially improve estimation and prediction in transfer, multi-task, and federated learning. However, contamination and heterogeneity make this problem fundamentally challenging. We study a contaminated multi-task empirical risk minimization (ERM) framework in which an $\epsilon$ fraction of $K$ tasks, each with sample size $n$, may be arbitrarily contaminated, while the uncontaminated tasks remain heterogeneous. Unlike much of the existing literature, we consider a general problem in which both the global minimizer of the average risk and the local minimizer for each task are of interest, with the goal of achieving robustness to contamination and personalization under heterogeneity. In the Gaussian mean model, we show that several popular paradigms, including adaptive and robust regularization around a shared center, global matrix regularization, decomposition-based regularization, and score-based outlier-task detection, all suffer from a worst-case contamination error of order $\epsilon\sqrt{d/n}$, which is suboptimal compared to the lower bound $\epsilon/\sqrt{n}$. This reveals a fundamental dimension-dependent barrier for these approaches. We then establish comprehensive minimax lower bounds for a general ERM heterogeneous setting and propose a computationally efficient robust multi-task gradient descent method based on filtering. Under local strong convexity, smoothness, and sub-Gaussian gradient assumptions, we prove high-probability upper bounds that match the minimax lower bounds up to logarithmic factors over a broad regime. These bounds remove the extra $\sqrt{d}$ contamination dependence that characterizes many existing regularization-based methods and score-based outlier detection, while achieving personalization to local tasks under strong heterogeneity. Simulations and a real-data analysis demonstrate strong robustness and personalization relative to a broad range of benchmark methods.
\end{abstract}


\addtocontents{toc}{\protect\setcounter{tocdepth}{0}}

% \setstretch{1.8}

% main text

\section{Introduction}\label{sec: introduction}
Integrating data from multiple related sources is an important theme in modern statistics and machine learning. When the underlying tasks are related, borrowing information across them can substantially improve the estimation and prediction performance compared to learning each task individually. This idea underlies a broad range of methods in transfer learning, multi-task learning, and federated learning \citep[e.g.][]{pan2009survey,weiss2016survey,zhang2021survey,mcmahan2017communication,smith2017federated,li2020federatedsurvey}. In many applications, however, these data sources are neither identical nor fully trustworthy. Different hospitals may serve different patient populations, different devices may generate systematically different usage patterns, and different studies may be conducted under different protocols. As a result, the clean tasks can be  heterogeneous, while a fraction of tasks may at the same time be corrupted, unreliable, or even adversarial.

Data contamination and heterogeneity make the data integration challenging. If one pools all tasks too aggressively, heterogeneity introduces bias. If one learns each task separately, one loses the statistical gains from data integration. If one borrows information through a non-robust method, a few contaminated tasks can destroy the entire procedure. Therefore, developing methods that are simultaneously robust to contamination and adaptive to heterogeneity is a fundamental problem with broad applications.

% -----------------------------------------------------
\subsection{Problem setup}\label{subsec: problem setup}
We consider a contaminated multi-task learning setup. Suppose we have $K$ tasks, each with its own dataset $\{\zk{k}_i\}_{i=1}^n \overset{\textup{i.i.d.}}{\sim} \tP^{(k)}, k \in [K]=\{1,\dotsc,K\}$. We write $\zk{k} \sim \tP^{(k)}$ as a generic notation for one observation from the $k$-th task that takes value in some space $\mathcal{Z}$.
Consider a loss function $\ell: \mathcal{Z}\times \mathbb{R}^d \rightarrow \mathbb{R}_+$ for all tasks. Denote the population-level task-specific risk function $\mLk{k}(\btheta) = \tE[\ell(\zk{k}, \btheta)]$ and the average risk $\mL(\btheta) = \frac{1}{K}\sum_{k=1}^K \mLk{k}(\btheta)$. We define our parameter of interest as $\bthetaks{k} = \argmin_{\btheta \in \Theta}\mLk{k}(\btheta)$ and $\bthetas = \argmin_{\btheta \in \Theta} \mL(\btheta)$, where $\Theta$ is the parameter space. This is a standard empirical risk minimization (ERM) setup for multi-task learning. It can either be a supervised learning problem where $\zk{k}_i$ contains both features and response, or an unsupervised learning problem where $\zk{k}_i$ only contains features.

We assume that there exists a contamination mechanism that can first pick a $S^c \subseteq [K]$ with $|S^c| \leq K\epsilon$ and $\epsilon\in(0,1/2)$. Then, it chooses a contamination function $M$ that maps the contaminated data to \textit{arbitrary} values while keeping the uncontaminated data unchanged. To better describe this procedure, we can define a collection of sets $\mathcal{S} = \{S \subseteq [K]: |S| \geq K(1-\epsilon)\}$, and a collection of contamination functions $\mathcal{M}_S = \{M: [K]\times [n]\times  \mathcal{Z} \to \mathcal{Z} \textup{ such that } M(k, i, \zk{k}_i) = \zk{k}_i, k \in S, i \in [n]\}$. When the context is clear, for fixed $S$ and $M \in \mathcal{M}_S$, we write $\tilde{z}_i^{(k)} \coloneqq M(k,i,\zk{k}_i)$ for the observed, possibly contaminated, data point.

Our goal here is to construct estimators $(\htheta, \{\hthetak{k}\}_{k=1}^K)$ based on contaminated data $\{M(k,\,i,\, \zk{k}_i)\}_{i\in[n], k\in[K]}$, that achieve a small estimation error $\twonorm{\htheta - \bthetas}$ for the global minimizer $\bthetas$ and, for each clean task $k \in S$, a small estimation error $\twonorm{\hthetak{k} - \bthetaks{k}}$ for the corresponding task-specific minimizer $\bthetaks{k}$.

As discussed later, some works in the literature study contaminated multi-round federated learning, where contamination occurs in the messages passed between tasks and the contamination mechanism may change across rounds. This is commonly known as the Byzantine attack model in the literature. Our setup is slightly different in that the contamination is directly on the data. Nevertheless, our algorithms proposed in Section \ref{sec: method and theory} operate through robust aggregation of task-level gradients and they can also be applied to such multi-round federated learning settings with Byzantine clients.

% ----------------------------------------------------
\subsection{A motivating dimension gap}
A wide range of existing works adopt regularization to obtain personalized estimators in multi-task learning settings \cite[e.g.][]{evgeniou2004regularized,argyriou2006multi,lounici2011oracle,gong2012robust,jalali2013dirty}.  In particular, \cite{duan2022adaptive} recently proposed an adaptive and robust multi-task learning framework based on regularization. Consider the Gaussian mean model $\zk{k}_i = \bx_i^{(k)} \sim N(\bthetaks{k}, \bm{I}_d)$.  \cite{duan2022adaptive} assumes that the clean tasks are similar to each other, in the sense that $\min_{\otheta}\max_{k \in [K]}\twonorm{\bthetaks{k}-\otheta} \leq \hmax$ and  studies the estimator
\begin{equation}\label{eq: ARMUL}
    (\{\hthetak{k}\}_{k=1}^K,\hotheta) \in \argmin_{\{\bthetak{k}\}_{k=1}^K,\otheta} \sum_{k=1}^K\left\{ \frac{1}{2n}\sum_{i=1}^n\twonorm{\bx_i^{(k)}-\bthetak{k}}^2 +\lambda\twonorm{\bthetak{k}-\otheta} \right\}.
\end{equation}
This type of regularization is appealing for several reasons. First, it shrinks similar tasks toward a common center which borrows the information across tasks while adapting to each task automatically. Second, it is shown to be robust against contamination, where its robustness is connected to the well-known relationship between penalization and robust M-estimation \cite{sardy:tseng:bruce2001,she2011outlier,witten2013penalized,donoho2016high,wang2025robust}. In Section \ref{subsec: examples regularizer supp} of the appendix, we also provide an equivalent explicit robust M-estimation formulation for this type of estimator, which corresponds to the estimator proposed in \cite{mathieu2022concentration} for the classical single-task robust statistics setting.

In terms of the worst-case per-task estimation error $\max_{k \in S}\twonorm{\hthetak{k} - \bthetaks{k}}$, by considering the worst-case parameter collection $\{\bthetaks{k}\}_{k=1}^K$ satisfying the heterogeneity condition, the worst contamination mechanism, and the best estimator $\{\hthetak{k}\}_{k=1}^K$,
\cite{duan2022adaptive} proves a minimax lower bound of order $\tilde{\Omega}\Big(\sqrt{\frac{d}{nK}} + \min\Big\{\hmax ,\sqrt{\frac{d}{n}}\Big\} + \frac{\epsilon}{\sqrt{n}}\Big)$ \footnote{In the final published version \cite{duan2023adaptive}, they consider the estimation error across all tasks, $(\frac{1}{K}\sum_{k=1}^K\twonorm{\hthetak{k} - \bthetaks{k}}^2)^{1/2}$, including the contaminated ones, and derive a lower bound containing $\sqrt{\frac{\epsilon d}{n}}$. This metric does not seem to be particularly meaningful to us, since it includes the contaminated tasks. We therefore focus on the more meaningful result presented in the earlier version of their paper \cite{duan2022adaptive}.}. However, the regularized estimator above satisfies a high-probability upper bound of order $\tilde{O}\Big(\sqrt{\frac{d}{nK}} + \min\Big\{\hmax ,\sqrt{\frac{d}{n}}\Big\} + \epsilon\sqrt{\frac{d}{n}}\Big)$. Thus, the contamination term in the upper bound suffers a dimension-dependent cost of order $\sqrt{d}$. 
% In low dimension this gap may be mild, but in high dimension it becomes the dominant obstacle. 

Similar dimension-dependent gaps also appear in subsequent penalization-based approaches for contaminated transfer and multi-task learning problems \cite{tian2022unsupervised,tian2023learning,tian2024towards, kim2026multi}. Related works \cite{chen2017distributed} and \cite{yin2018byzantine} use robust gradient descent methods for federated learning with geometric median and coordinate-wise median estimating the gradients, and their analysis reveals a similar $\sqrt{d}$ dependence in the contamination-related terms in their upper bounds. Note that as $\lambda \rightarrow 0+$, the regularized estimator $\hotheta$ in \eqref{eq: ARMUL} can be linked to the geometric median of the local empirical means \cite{ronchetti2009robust,maronna2019robust}.

The gap between the minimax lower bound and the upper bound raises a natural question: 
\begin{center}
    \textit{``Is this gap merely an artifact of a particular regularizer, or does it reflect a fundamental limitation of regularization-based robust multi-task learning?"}
\end{center}
The negative results in Section \ref{sec: negative results} show that the latter is closer to the truth. In the Gaussian mean setting, we prove that broad classes of regularization schemes, including richer penalization families than the one above, still suffer a dimension-dependent contamination barrier. In this sense, it is hard to hope that simply changing the regularizer within these paradigms will simultaneously deliver optimal robustness to contamination and adaptivity to heterogeneity. This motivates us to explore alternative approaches that can break this barrier, and we propose a filtering-based robust multi-task gradient descent method that achieves nearly-minimax optimal guarantees in a general contaminated multi-task ERM setup.

% ----------------------------------------------------
\subsection{Our contributions}

Our main contributions are as follows.

\begin{enumerate}[(i)]
    \item \textbf{Negative results for some common robust MTL paradigms.} In Section \ref{sec: negative results}, we show in the Gaussian mean model that several popular frameworks, including adaptive and robust regularization around a shared center, global regularization on the whole parameter matrix, decomposition-based regularization, and score-based outlier-task detection, all exhibit a dimension-dependent contamination error of order $\epsilon\sqrt{d/n}$ in the worst case. These results show that the gap observed in \cite{duan2022adaptive,tian2022unsupervised,tian2023learning,kim2026multi} persists across a much broader class of methods, and can help better understand the performance of many existing robust MTL algorithms.

    \item \textbf{Positive results via a filtering-based robust MTL algorithm and theory under contamination and heterogeneity.} In Section \ref{sec: method and theory}, we move beyond the motivating Gaussian mean example and formulate a general ERM-based MTL problem, under task heterogeneity and adversarial contamination. We first establish comprehensive minimax lower bounds for estimating both the global parameter $\bthetas$ and the clean local parameters $\bthetaks{k}$.
    We then propose a computationally efficient robust multi-task gradient descent method that combines joint robust gradient estimation with a filtering procedure and a robust covariance estimator built from single-task empirical covariances. Under local strong convexity, smoothness, and sub-Gaussian gradient assumptions, we prove high-probability upper bounds for both global and local estimation. These bounds match the minimax lower bounds up to logarithmic factors in a broad regime. In particular, our method avoids the $\epsilon\sqrt{d/n}$ contamination dependence that characterizes the regularization-based methods in Section \ref{sec: negative results}. 
\end{enumerate}

As we will argue in the next subsection, although there are many works studying MTL under heterogeneity and contamination, there is no clean minimax lower bound for parameter estimation errors in our generic ERM setting, and none of the existing algorithms match the existing lower bound even under simple statistical models such as the Gaussian location model and generalized linear models. Our work fills this gap in the literature.

% ----------------------------------------------------
\subsection{Related works}

\paragraph{Transfer learning, multi-task learning, and federated learning.} Borrowing information across related tasks is the core idea and has a long history in transfer learning and multi-task learning \citep[e.g.][]{pan2009survey,weiss2016survey,zhang2021survey}. Common strategies for parametric problems assume some shared structure that is exploited through regularization, encouraging either sparse or low-rank decompositions \cite[e.g.][]{evgeniou2004regularized,lounici2011oracle,jalali2013dirty,gong2012robust,bastani2021predicting,li2022transfer,tian2023transfer,duan2023adaptive, hector2024turning}. Another popular structure assumes a common latent representation \cite[e.g.][]{du2021few,tripuraneni2021provable}. Related ideas also appear in federated learning, where one aims to exploit cross-task similarity while respecting communication or privacy constraints \cite[e.g.][]{mcmahan2017communication,smith2017federated,li2020federatedsurvey,t2020personalized,li2024federated}. Our negative results show that, once adversarial task contamination is introduced, broad regularization families can be fundamentally suboptimal. 

\paragraph{Robust estimation in centralized settings.}
Classical robust statistics studies the estimation of model parameters in the presence of outliers or model misspecification, with the goal of limiting their effect on statistical procedures. Such effects are commonly quantified through notions such as the breakdown point and the influence function \citep{ronchetti2009robust}. More recently, increasing attention has been devoted to establishing optimal non-asymptotic guarantees under various contamination models, ranging from the Huber $\epsilon$-contamination model \citep{huber1964robust} to the $\epsilon$-replacement model of \citep{donoho:huber1983} and related variants. Recent work in algorithmic robust statistics, however, often considers the strong contamination model \citep[e.g.][]{diakonikolas2023algorithmic}, where the corrupted samples are allowed to be chosen in a more adaptive manner. The contamination setup described in \Cref{subsec: problem setup} belongs to the strong contamination model, where the contaminated data are not necessarily independent, and they can further depend on the realized uncontaminated data.

\citep{chen2018robust} establish minimax rates for Gaussian mean and covariance estimation under the Huber contamination model, and show that depth-based estimators, such as the Tukey median, can be statistically optimal despite their computational intractability. Filtering-based methods have subsequently played a particularly important role in modern robust estimation \citep[e.g.][]{diakonikolas2019robust,diakonikolas2023algorithmic}, providing computationally efficient and near-optimal guarantees for a wide range of statistical tasks, including stochastic optimization \citep{prasad2020robust,diakonikolas2019sever}, sparse mean and covariance estimation \citep{balakrishnan2017computationally,diakonikolas2022robust}, network analysis \citep{acharya2022robust}, and high-dimensional regression \citep{pensia2020robust,liu2020high}, among others. Structural or geometric constraints, such as symmetry, have also been exploited to obtain sharper robust estimation guarantees \citep{pensia:jog:loh2022,prasadan2026information,novikov2023robust}. 

More recently, \citep{pittas2025optimal} study a two-layer contamination model for robust mean estimation, in which one layer allows arbitrary contamination of a fraction of the samples, while the other permits a mean shift among the remaining uncontaminated samples. Closer to our ERM setting, \cite{prasad2020robust} develop a robust gradient descent method that aggregates sample gradients using a robust mean estimator, making the approach broadly applicable to $M$-estimation problems. See also \cite{loh2025} for a recent review of modern robust statistics.

\paragraph{Robustness in distributed settings.}
In parallel, the distributed and federated learning literature has studied task-level contamination, where some datasets or communicated messages from different tasks may be corrupted. This scenario is often formulated as Byzantine robustness in federated learning. Many proposed methods combine gradient descent with classical robust aggregation rules \citep[e.g.][]{chen2017distributed,blanchard2017machine,guerraoui2018hidden,yin2018byzantine,zhu2023byzantine}, or use outlier-detection strategies to identify corrupted tasks \citep[e.g.][]{konstantinov2019robust,konstantinov2020sample,tian2023transfer,li2024federated}. Beyond these generic approaches, \cite{zhang2026byzantine} develops a model-specific Byzantine-robust aggregation method for distributed finite mixture learning based on distance-filtered mixture reduction.

A related line of work studies robust learning from corrupted or heterogeneous batches, a setting closely connected to task-level contamination since each batch contains multiple observations. \cite{qiao2018learning} initiated the untrusted-batch model for discrete distribution learning, where an $\epsilon$ fraction of batches may be arbitrary while the remaining batches are drawn from distributions close to a target distribution. Subsequent works developed computationally efficient and near-optimal algorithms for discrete and structured distribution classes \citep[e.g.][]{chen2020efficiently,jain2020optimal}. 
More recent work considers heterogeneous batched data in linear regression without contamination \citep{jain2024linear}, and mean estimation with contamination both at the batch level and within each batch \citep{aliakbarpour2026highdimensional}. Overall, these works demonstrate that the batch structure can substantially mitigate the effect of corrupted sources, but they primarily focus on distributional learning or specific models, rather than the multi-task robust estimation setting considered here.

Recent studies have also explored the interaction between task heterogeneity and contamination. For example, \cite{karimireddy2022byzantine,allouah2023fixing,allouah2023robustdistributed} consider a gradient heterogeneity condition and focus on deterministic global optimization problems without data randomness or task-specific personalization. \cite{allouah2024finetuning} studies personalization in the presence of adversarial clients under simple settings, including mean estimation and binary classification with bounded loss functions. For the general ERM setting we consider, to the best of our knowledge, there are no complete minimax lower bounds and matching upper bounds for parameter estimation error under both task heterogeneity and adversarial contamination.

% ----------------------------------------------------
\subsection{Notation and organization}

Throughout the paper, $[K] = \{1,\ldots,K\}$, boldface letters denote vectors or matrices, and $\twonorm{\cdot}$ denotes the Euclidean norm for vectors and the spectral norm for matrices. For a finite set $S$, we denote its cardinality by $|S|$ or $\# S$. For two sequences $\{a_n\}_{n=1}^{\infty}$ and $\{b_n\}_{n=1}^{\infty}$, we write $a_n \lesssim b_n$ or $b_n \gtrsim a_n$ if there exists a universal constant $C > 0$ such that $a_n \leq Cb_n$. For many results presented in the main text, we use $\tilde{O}(\cdot)$ and $\tilde{\Omega}(\cdot)$ to suppress logarithmic factors. The detailed corresponding results with more explicit logarithmic dependence are provided in the appendix. For two quantities $a$ and $b$, $a \vee b$ and $a \wedge b$ represent $\max(a, b)$ and $\min(a, b)$, respectively. Constants denoted by  $C$, $\{C_i\}_{i=1}^{\infty}$, and $\{c_i\}_{i=1}^{\infty}$ may vary from line to line. 

Section \ref{sec: negative results} focuses on the Gaussian mean model and proves negative results for broad regularization and outlier-detection frameworks. Section \ref{sec: method and theory} formulates the general contaminated multi-task ERM problem, establishes minimax lower bounds, and presents our filtering-based robust multi-task gradient descent algorithm together with its theoretical guarantees. Section \ref{sec: experiment} reports simulation and real-data experiments. Section \ref{sec: discussion} concludes with a brief summary and discussion of future directions. All proofs, along with additional theoretical and numerical details, are deferred to the Appendix.

\section{Negative results on regularization and outlier detection}\label{sec: negative results}

In this section, we will present negative results for some common multi-task learning frameworks. More specifically, we show in the Gaussian mean model that several popular frameworks, including adaptive and robust regularization around a shared center, global regularization on the whole parameter matrix, decomposition-based regularization, and score-based outlier-task detection%{\color{red}[M: can you add references for them here as well?]} \yt{Is it necessary? There are too many references}
, all exhibit a dimension-dependent contamination error of order $\epsilon\sqrt{d/n}$ in the worst case. As a result, they fail to achieve minimax optimality in the presence of contamination.

For simplicity, throughout this section we consider the Gaussian location model $\zk{k}_i = \bx_i^{(k)} \sim N(\bthetaks{k}, \bm{I}_d)$ with  observed contaminated data $\tilde{z}^{(k)}_i = \tilde{\bx}_i^{(k)}$.

% ---------------------------------------------------
\subsection{Regularization frameworks}
We will review several popular MTL regularization frameworks including \eqref{eq: ARMUL} and show that all of them can suffer from a dimensionality gap which prevents them from achieving the minimax optimality.

% ----------------------------
\subsubsection{Adaptive and robust regularization}\label{subsubsec: adaptive regularization}
% Since we will focus on the contamination impact, we consider the distributed learning setting where $\bthetak{k} \equiv \bthetas$.

We consider a regularized MTL estimator that generalizes \eqref{eq: ARMUL} as follows:
\begin{equation}\label{eq: penalization form}
	\{\hthetak{k}\}_{k=1}^K, \hotheta \in \argmin_{\{\bthetak{k}\}_{k=1}^K, \otheta} \bigg\{\sum_{k=1}^K\bigg(\frac{1}{2n}\sum_{i=1}^n\twonorm{\bxk{k}_i - \bthetak{k}}^2 +  p(\twonorm{\bthetak{k} - \otheta})\bigg)\bigg\}, 
\end{equation}
where $p: [0, \infty) \rightarrow [0, \infty)$ is the regularizer. \cite{duan2022adaptive} considers $p(x) = \lambda x$ where $\lambda \geq 0$ is a tuning parameter.
By writing $\barx^{(k)} = n^{-1}\sum_{i=1}^n\tilde{\bx}_i^{(k)}$, we see that \eqref{eq: penalization form} is equivalent to
\begin{equation}\label{eq: penalization form single version}
	\{\hthetak{k}\}_{k=1}^K, \hotheta \in \argmin_{\{\bthetak{k}\}_{k=1}^K, \otheta} \bigg\{\sum_{k=1}^K\bigg(\frac{1}{2}\twonorm{\barx^{(k)} - \bthetak{k}}^2 + p(\twonorm{\bthetak{k} - \otheta})\bigg)\bigg\}.
\end{equation}
Equation \eqref{eq: penalization form single version} shows a connection of MTL to the classical single-task learning, where we may view the summary statistics $\barx^{(k)}$ as a sample. When $d = 1$, \cite{sardy:tseng:bruce2001} proved that \eqref{eq: penalization form single version} is equivalent to M-estimation with Huber loss function when $p(x) = x$ and \cite{she2011outlier, donoho2016high} generalized the conclusion to other regularizers. More specifically, \eqref{eq: penalization form} and \eqref{eq: penalization form single version} are equivalent to
\begin{align}
    \hotheta &\in \argmin_{\btheta}\bigg\{\sum_{k=1}^K\min_{\bm{\Delta}}\Big(\frac{1}{2}\twonorm{\barx^{(k)} - \btheta -\Delta}^2 + p(\twonorm{\Delta})\Big)\bigg\} = \argmin_{\btheta}\bigg\{\sum_{k=1}^K\rho(\barx^{(k)} - \btheta)\bigg\}, \label{eq: m-est form} \\
    \hthetak{k} &\in \argmin_{\btheta}\bigg\{\frac{1}{2}\twonorm{\barx^{(k)} - \btheta}^2 + p(\twonorm{\btheta - \hotheta})\bigg\}. \label{eq: m-est form individual}
\end{align}
where $\rho(\bx) \coloneqq \min_{\bz}\big[\frac{1}{2}\twonorm{\bz-\bx}^2 + p(\twonorm{\bz})\big]$.  Note that 
\[
\text{prox}_{p}(\bx) \coloneqq \argmin_{\bz}\Big[\frac{1}{2}\twonorm{\bz-\bx}^2 + p(\twonorm{\bz})\Big]
\]
is often called the proximal projection of $\bx$, and $\rho(\bx)$ is often called the Moreau envelope \cite{parikh2014proximal, polson2015proximal}. By choosing some common regularizers, the derived loss function $\rho$ becomes some robust loss functions in robust statistics, which is part of the motivation for adopting the regularization-based methods in MTL \cite[e.g.][]{duan2022adaptive,tian2022unsupervised,tian2024towards}.

To better analyze the behavior of \eqref{eq: m-est form} and \eqref{eq: m-est form individual}, we impose the following regularity assumptions on the regularizer $p(\cdot)$, which we will see later are satisfied for most commonly used regularizers.

\begin{assumption}\label{asmp: penalty}
	Define $L \coloneqq \inf_{x>0}\big\{\frac{1}{2}x + \frac{p(x)}{x}\big\}$ and $L_{\infty} \coloneqq \sup_{\twonorm{\bx} > L}p'(\twonorm{\proxpl(\bx)})$. The regularizer $p(\cdot): [0, +\infty) \rightarrow [0, +\infty)$ satisfies the following conditions:
	\begin{enumerate}[(i)]
		\item $L \neq 0$ or $L_{\infty} \neq 0$;
		\item $p(0) = 0$, $p$ is non-decreasing and continuous on $[0, +\infty)$;
		\item $p'$ exists and is continuous on $(0, +\infty)$, and $p''$ exists on $(0, +\infty)$ almost everywhere \footnote{``almost everywhere" means that it holds up to a zero-measure set w.r.t. the Lebesgue measure.};
		\item $\frac{p'(x)}{x}$ is non-increasing on $(0, +\infty)$;
		\item There exists a constant $\tau \in [0, 1)$ such that for all $\bm{x}$ with $\twonorm{\bx} > L$, $p''(\twonorm{\proxpl(\bx)}) \geq -\tau$, if $p''(\twonorm{\proxpl(\bx)})$ exists;
		\item There exist constants $c_0 > 0$ and $c_1 \geq 1$ such that $p''(t)t\geq -c_0p'(t)$ when $t \geq c_1(L\vee L_{\infty})$. \footnote{When $L \vee L_{\infty} = +\infty$, this condition is not needed.}
	\end{enumerate}
\end{assumption}

\begin{remark}
    Condition (\rom{1}) is imposed to rule out degenerate regularizers such as constant functions, which would make $\hotheta$ not well-defined. Conditions (\rom{2}) and (\rom{3}) are standard and ensure that $p$ is twice differentiable on $\mathbb{R}_+$ almost everywhere. Conditions (\rom{4})-(\rom{6}) jointly ensure the smoothness of the corresponding loss function $\rho$, in the sense that $|\lambdamax(\nabla^2 \rho(\bx))|$ and $|\lambdamin(\nabla^2 \rho(\bx))|$ remains bounded for all $\bx$. The quantity $L$ characterizes the radius of the shrinkage basin of $\proxpl(\bx)$, in the sense that $\proxpl(\bx) = 0$ if $\|\bx\| < L$, and $\proxpl(\bx) \neq 0$ if $\|\bx\| > L$.  It can be shown to be no smaller than the quantity $\tilde{L}=\inf_{x > 0}\{x + p'(x)\}$ defined in \cite{fan2001variable} which proved that when $\|\bx\| \leq \tilde{L}$, $\proxpl(\bx) = 0$.
\end{remark}

Next, we provide some examples of commonly used regularizers $p(\cdot)$ which satisfy Assumption \ref{asmp: penalty}. More regularizer examples (e.g., SCAD, MC+, and hard-thresholding) can be found in Example \ref{exp: penalties supp} in Appendix \ref{sec: negative results supp}, where we also verify that they satisfy Assumption \ref{asmp: penalty}.

\begin{example}\label{exp: penalties main text}
	\begin{enumerate}[(i)]
		\item (Lasso) \cite{duan2022adaptive, tian2022unsupervised, tian2024towards} $p(x) = \lambda x$, $L = L_{\infty} = \lambda$, $\tau = 0$, $\proxpl(\bx) = \begin{cases}
			\bm{0}, \quad &\textup{ if } \twonorm{\bx} \leq \lambda; \\
			(1-\frac{\lambda}{\twonorm{\bx}})\bx, \quad &\textup{ if } \twonorm{\bx} > \lambda.
		\end{cases}$, $\rho(\bx) = \begin{cases}
		    \frac{1}{2}\twonorm{\bx}^2, \quad &\textup{ if } \twonorm{\bx} \leq \lambda; \\
			\lambda\twonorm{\bx} - \frac{1}{2}\lambda^2, \quad &\textup{ if } \twonorm{\bx} > \lambda.
		\end{cases}$
		\item (Ridge) \cite{evgeniou2004regularized, chen2015data, t2020personalized} $p(x) = \lambda x^2$, $L = 0$,  $L_{\infty} = +\infty$, $\tau = 0$, $\proxpl(\bx) = \frac{1}{2\lambda + 1}\bx$, $\rho(\bx) = \frac{\lambda}{2\lambda+1}\twonorm{\bx}^2$.
		\item (Bridge, $0< q < 1$, $\lambda>0$) $p(x) = \lambda x^q$, $L = [2\lambda(1-q)]^{\frac{1}{2-q}}\cdot \frac{1}{2}\cdot (1+\frac{1}{1-q})$, $L_{\infty} = \lambda q r_L^{q-1}$, $\tau = 1-\frac{q}{2}$, where $r_L$ is the larger root $r$ of equation $r + \lambda q r^{q-1} = L$. $\proxpl(\bx) = \begin{cases}
		    \bm{0}, &\quad \textup{if } \twonorm{\bx} \leq L, \\
                r\frac{\bx}{\twonorm{\bx}}, &\quad \textup{if } \twonorm{\bx} > L,
		\end{cases}$, $\rho(\bx) = \begin{cases}
		    \frac{1}{2}\twonorm{\bx}^2, &\quad \textup{if } \twonorm{\bx} \leq L, \\
                \frac{1}{2}(\twonorm{\bx}-r)^2 + \lambda r^q, &\quad \textup{if } \twonorm{\bx} > L
		\end{cases}$, where $r$ is the solution of $r + q\lambda r^{q-1} = \twonorm{\bx}$.
	\end{enumerate}
\end{example}

The following theorem shows that for regularizers satisfying Assumption \ref{asmp: penalty}, the robust MTL estimator \eqref{eq: penalization form} always suffers from a $\epsilon\sqrt{\frac{d}{n}}$ lower bound.

\begin{theorem}\label{thm: finite sample lower bound}
	Consider the distributed learning case where $\bthetaks{k} = \bthetas$ for all $k \in [K]$. Let $\{C_i\}_{i=1}^6$ and $\{c_i\}_{i=1}^3$ be some positive absolute constants. Suppose the regularizer satisfies Assumption \ref{asmp: penalty}, the contamination proportion $\epsilon \leq 1/4$, $d \geq C_1$, and $C_2e^{C_3d} \geq K \geq C_4d/\epsilon^2$. The following conclusions hold for any $\btheta^* \in \mathbb{R}^d$. 
	There exist a subset $S^c \subseteq [K]$ with $|S^c|/K \leq \epsilon$ and a contamination mechanism $M \in \mathcal{M}_S$ such that:
    \begin{enumerate}[(i)]
	   \item with probability at least $3/16$, for all $(\hotheta, \{\hthetak{k}\}_{k=1}^K)$ in \eqref{eq: penalization form}, $\twonorm{\hotheta - \btheta^*} \geq C_5\sqrt{\frac{d}{n}}\epsilon$;
        \item with probability at least $1/16$, for all $(\hotheta, \{\hthetak{k}\}_{k=1}^K)$ in \eqref{eq: penalization form}, $\max\limits_{k \in S}\twonorm{\hthetak{k} - \btheta^*} \geq C_6\sqrt{\frac{d}{n}}\epsilon$.
    \end{enumerate}
\end{theorem}

% \begin{remark}
% 	We can verify Assumption \ref{asmp: penalty} for all regularizers in Example \ref{exp: penalties main text}. See Section \ref{sec: negative results supp} in the supplementary materials for details. 
% \end{remark}

Some works in the literature apply a similar regularizer with group structures \cite[e.g.][]{chen2015data, gross2016data, ollier2017regression}:
\begin{equation}\label{eq: group penalization}
	\{\hthetak{k}\}_{k=1}^K, \hotheta \in \argmin_{\{\bthetak{k}\}_{k=1}^K, \otheta} \bigg\{\sum_{k=1}^K\bigg(\frac{1}{2n}\sum_{i=1}^n\twonorm{\tilde{\bx}_i^{(k)} - \bthetak{k}}^2 +  \sum_{j=1}^G p_j(\twonorm{\bthetak{k}_{\mathcal{G}_j} - \otheta_{\mathcal{G}_j}})\bigg)\bigg\},
\end{equation}
where the groups $\mathcal{G}_j$'s are disjoint and $\cup_{j=1}^G\mathcal{G}_j = [d]$ and $p_j: [0, \infty) \rightarrow [0, \infty)$ is the regularizer for group $\mathcal{G}_j$. Since the square loss is also decomposable, the same proof arguments used in Theorem \ref{thm: finite sample lower bound} can also be used to show the same lower bound for estimators in \eqref{eq: group penalization} by proving for each $\mathcal{G}_j$ separately. 

\cite{chen2017distributed} and \cite{yin2018byzantine} use robust gradient descents for federated learning with geometric median and coordinate-wise median estimating the gradients. When $p(x)=\lambda x$ in \eqref{eq: penalization form}, the profiled loss function for $\hotheta$ can be viewed as a multivariate Huber loss on the local empirical means $\{\barxk{k}\}_{k=1}^K$. Thus, up to a positive rescaling that does not change the minimizer, the limit $\lambda \rightarrow 0$ connects \eqref{eq: penalization form} to geometric median. Similarly, when the groups in \eqref{eq: group penalization} are singletons and $p_j(x)=\lambda x$, the corresponding limit $\lambda \rightarrow 0$ connects to coordinate-wise median \cite{ronchetti2009robust,maronna2019robust}. Therefore, our results implicitly indicate that the estimation error using geometric median and coordinate-wise median would suffer from the same sub-optimal dependence on the dimension.

% -------------------------------------
\subsubsection{Global regularization}\label{subsubsec: global regularization}
In addition to the regularization in Section \ref{subsubsec: adaptive regularization}, some multi-task learning literature also considers the following global regularization form:
\begin{equation}\label{eq: global regularization}
    \{\hthetak{k}\}_{k=1}^K \in \argmin\limits_{\Theta=\{\bthetak{k}\}_{k=1}^K}\bigg\{\frac{1}{2n}\sum_{k=1}^K\sum_{i=1}^n\twonorm{\tilde{\bx}_i^{(k)} - \bthetak{k}}^2 + p(\bTheta)\bigg\},
\end{equation}
where $\bTheta$ is the parameter matrix with the $k$-th column being $\bthetak{k}$ and $p(\cdot): \mathbb{R}^{d \times K} \rightarrow [0, \infty)$ is a non-decreasing function of $|\theta^{(k)}_j|$ for all $k\in[K]$ and $j\in[d]$, when the remaining entries are fixed. We present some examples next. %{\color{red}[M: not all these references are presented in a MTL setting? I guess that is OK?]} \yt{I guess you are referring to Group Lasso. I think it is fine. Originally it's not for MTL, but two later theoretical papers by Van de Geer et al. highlight its application to MTL.}

\begin{example}
    \begin{enumerate}[(i)]
        \item \cite{argyriou2006multi}: $p(\bTheta) = \|\bTheta\|_{2,1}^2 = \Big(\sum_{j=1}^d \sqrt{\sum_{k=1}^K (\thetak{k}_j)^2}\Big)^2$
        \item Group Lasso \cite{obozinski2006multi, yuan2006model, lounici2011oracle}:  $p(\bTheta) = \|\bTheta\|_{2,1} = \sum_{j=1}^d \sqrt{\sum_{k=1}^K (\thetak{k}_j)^2}$
        \item \cite{liu2009blockwise}: $p(\bTheta) = \|\bTheta\|_{\infty,1} = \sum_{j=1}^d \max\limits_{k\in [K]} |\thetak{k}_j|$
        \item \cite{negahban2011simultaneous}: $p(\bTheta) = \|\bTheta\|_{q,1} = \sum_{j=1}^d \big(\sum_{k=1}^K|\thetak{k}_j|^q\big)^{1/q}$, where $1\leq q \leq \infty$
        \item \cite{zhou2010exclusive}: $p(\bTheta) = \|\bTheta\|_{1,2}^2 = \sum_{j=1}^d (\sum_{k=1}^K |\thetak{k}_j|)^2$
    \end{enumerate}
\end{example}

Intuitively, such a regularizer encourages sparsity of $\bTheta$ but does not lead to any robustness against contamination, and can lead to large bias if such sparsity does not hold in practice. This can be confirmed by the lower bound $\sqrt{\frac{d}{n}}$ of the estimation error presented in Theorem \ref{thm: lower bdd global regularization}.

\begin{theorem}\label{thm: lower bdd global regularization}
    Suppose the sub-gradient of $p$ exists everywhere on $\bTheta$ \footnote{This can be relaxed to Clarke sub-differential, but for simplicity we will work with the current ordinary sub-differential definition.} and $p(\bTheta)$ is a non-decreasing function of $|\theta^{(k)}_{j}|$ when $\{\theta^{(k')}_{j'}\}_{k'\neq k, j' \neq j}$ is fixed, for all $j \in [d]$ and $k \in [K]$. Then for any $C_0 \in (0, 1]$ and any $\{\bthetaks{k}\}_{k=1}^K$ with $\#\{(k,j) \in [K] \times [d]: |\thetaks{k}_j| \geq \frac{1}{\sqrt{2\pi n}}\} \geq C_0^2 d K$, without any contamination (i.e. $S = [K]$), with probability at least $1/4$, all minimizers $\{\hthetak{k}\}_{k=1}^K$ in \eqref{eq: global regularization} satisfy $\max\limits_{k \in [K]}\twonorm{\hthetak{k} - \bthetaks{k}} \geq \frac{1}{40}C_0 \sqrt{\frac{d}{n}}$.
\end{theorem}

% -------------------------------------
\subsubsection{Decomposition-based regularization}\label{subsubsec: decomp regularization}
There is another popular regularization framework based on decomposition, where the parameter matrix $\bTheta$ is decomposed into two components, $\bm{G}$ and $\bm{Q}$, in the sense that $\bthetak{k} = \bm{g}^{(k)}+\bm{q}^{(k)}$, for $k\in [K]$. This decomposition allows different regularizers to be used:
\begin{equation}
    \{\hat{\bm{g}}^{(k)}\}_{k=1}^K, \{\hat{\bm{q}}^{(k)}\}_{k=1}^K \in \argmin\limits_{\bm{G} = \{\bm{g}^{(k)}\}_{k=1}^K, \bm{Q} = \{\bm{q}^{(k)}\}_{k=1}^K}\bigg\{\frac{1}{2n}\sum_{k=1}^K\sum_{i=1}^n\twonorm{\tilde{\bx}_i^{(k)} - \bm{g}^{(k)}-\bm{q}^{(k)}}^2 + p_1(\bm{G}) + p_2(\bm{Q})\bigg\}, 
\end{equation}
where $\bm{G}$ and $\bm{Q}$ are $d \times K$ matrices, whose $k$-th columns are $\bm{g}^{(k)}$ and $\bm{q}^{(k)}$, respectively, and $p_1(\bm{G})$ and $p_2(\bm{Q})$ are non-decreasing functions in the magnitude of each entry of $\bm{G}$ and $\bm{Q}$ when the remaining entries are fixed. The final estimator for task $k$ is 
\begin{equation}\label{eq: decomp regularization}
	\hthetak{k} = \hat{\bm{g}}^{(k)} + \hat{\bm{q}}^{(k)}.
\end{equation}

Some examples are presented as follows.

\begin{example}
    \begin{enumerate}[(i)]
        \item Dirty model \cite{jalali2010dirty, jalali2013dirty, yang2013dirty, yang2017sparse+}: $p_1(\bm{G}) = \|\bm{G}\|_{1,1}$, $p_2(\bm{Q}) = \|\bm{Q}\|_{1,\infty}$
        \item Robust feature learning \cite{gong2012robust}: $p_1(\bm{G}) = \|\bm{G}\|_{2,1}$, $p_2(\bm{Q}) = \|\bm{Q}^\top\|_{2,1}$
    \end{enumerate}
\end{example}

This regularizer intuitively induces some robustness against contamination or outlier tasks, and there have been several discussions on this \cite[e.g.][]{chen2011integrating, zhou2011malsar, chen2012learning, gong2012robust}. %{\color{red}[M: is the second reference legit? it looks weird to me]} \yt{It is more like a paper introducing the software they developed for MTL, which used to be a famous MATLAB package.} 
However, as the following theorem shows, it may not effectively aggregate information and leads to a sub-optimal lower bound $\sqrt{\frac{d}{n}}$, even without contamination.

\begin{theorem}\label{thm: lower bdd decomp regularization}
    Suppose the sub-gradient exists for $p_1(\bm{G})$ and $p_2(\bm{Q})$ for all $\bm{G}$ and $\bm{Q}$, respectively. Also assume that $p_1(\bm{G})$ and $p_2(\bm{Q})$ are non-decreasing functions of $|g^{(k)}_{j}|$ and $|q^{(k)}_{j}|$ when $\{g^{(k')}_{j'}\}_{k'\neq k, j' \neq j}$ and $\{q^{(k')}_{j'}\}_{k'\neq k, j' \neq j}$ are fixed, respectively, for all $j \in [d]$ and $k \in [K]$. Then for any $C_0 \in (0, 1]$ and any $\{\bthetaks{k}\}_{k=1}^K$ with $\#\{(k,j) \in [K] \times [d]: |\thetaks{k}_j| \geq \frac{1}{\sqrt{2\pi n}}\} \geq C_0^2 d K$, without contamination (i.e. $S = [K]$), with probability at least $1/4$, the estimators $\{\hthetak{k}\}_{k=1}^K$ obtained in \eqref{eq: decomp regularization} satisfy $\max\limits_{k = 1:K}\twonorm{\hthetak{k} - \bthetaks{k}} \geq \frac{1}{40}C_0\sqrt{\frac{d}{n}}$.
\end{theorem}

% --------------------------------------------------
\subsection{Outlier detection framework}\label{subsec: detection}
In addition to the popular regularization frameworks, some works in the literature also consider the outlier task detection framework, which is motivated by the outlier or anomaly detection in classical robust statistics \cite{hawkins1980identification}. The framework first detects the outlier tasks and then aggregates the data from the remaining tasks for estimation. Denote $D^{(k)} = \{\tilde{\bx}_i^{(k)}\}_{i=1}^n$ as the observed $k$-th dataset, $k\in [K]$. The following definition defines a class of score-based outlier-task detection algorithms for the Gaussian mean estimation problem.

\begin{definition}\label{def: outlier detection alg}
    We call a method an outlier task detection algorithm if it takes the data as the input, selects task index set $\widehat{S}$, then outputs $\frac{1}{|\widehat{S}|}\sum_{k \in \widehat{S}}\barxk{k}$. We call an outlier task detection algorithm as a score-based algorithm, if the index set $\widehat{S}$ is selected via a score function $f: D^{(k)} \mapsto f(D^{(k)})\in \mathbb{R}$ which satisfies that if $f(D^{(j)}) \leq f(D^{(k)})$ and $k \in \widehat{S}$, then $j \in \widehat{S}$.
\end{definition}

Some examples are presented as follows.

\begin{example} 
	Suppose that we are also given a clean dataset $\{\bxk{0}_i\}_{i=1}^n \overset{\textup{i.i.d.}}{\sim} \mathcal{N}(\bthetas, \sigma^2\bm{I}_d)$ which will not be contaminated, in addition to the $K$ datasets subject to contamination. 
	\begin{enumerate}[(i)]
		\item \cite{tian2023transfer} $f(D^{(k)}) = \frac{1}{n}\sum_{i=1}^n\twonorm{\bxk{0}_i - \barx^{(k)}}^2 - \frac{1}{n}\sum_{i=1}^n\twonorm{\bxk{0}_i - \barx^{(0)}}^2$, $\widehat{S} = \{k:  f(D^{(k)}) \leq \lambda\}$, then output $\frac{1}{|\widehat{S}|}\sum_{k \in \widehat{S}}\barx^{(k)}$ if $\widehat{S} \neq \emptyset$ and output $\bm{0}$ otherwise.
    		\item \cite{li2024federated} $f(D^{(k)}) = \twonorm{\barx^{(k)} - \barx^{(0)}}$, $\widehat{S} = \{k: f(D^{(k)}) \leq \lambda\}$, then output $\frac{1}{|\widehat{S}|}\sum_{k \in \widehat{S}}\barx^{(k)}$ if $\widehat{S} \neq \emptyset$ and output $\bm{0}$ otherwise.
    		\item \cite{konstantinov2020sample} $f(D^{(k)}) = \#\{j \in [K]: \sup\limits_{\btheta \in \Theta}|\frac{1}{n}\sum_{i=1}^n\twonorm{\tilde{\bx}_i^{(k)} - \btheta}^2 - \frac{1}{n}\sum_{i=1}^n\twonorm{\tilde{\bx}_i^{(j)} - \btheta}^2| > \lambda \}$, $\widehat{S} = \Big\{k \in [K]: f(D^{(k)}) \leq \lfloor K/2\rfloor\Big\}$, where $\Theta$ is a user-defined subset in $\mathbb{R}^d$.
    		\item \cite{zhang2022fldetector} $k$-means clustering on $\{f(D^{(k)})\}_{k=1}^K$ with 2 clusters, and $\widehat{S}$ is set to the cluster with the smaller average score. \footnote{This example is not based on thresholding the score function, but it can be shown to satisfy Definition \ref{def: outlier detection alg}.}
	\end{enumerate}
\end{example}

The following theorem shows that for any score-based outlier task detection algorithm, there exists a contamination mechanism such that the estimation error suffers from a $\sqrt{\frac{d}{n}}\epsilon$ sub-optimal lower bound. For simplicity, we consider the distributed learning case where $\bthetaks{k} = \bthetas$ for all $k \in [K]$. 

\begin{theorem}\label{thm: lower bdd detection}
	Suppose $d \geq K$ and $2(2/e)^K + 2Ke^{-d/64} \leq 3/4$. For any $S \subseteq [K]$ with $|S^c|/K = \epsilon$ and $\btheta^*$ with $\twonorm{\bthetas} \geq \frac{1}{8\sqrt{5}}\epsilon\sqrt{\frac{d}{n}}$, for any score-based outlier task detection algorithm, there exists a contamination mechanism $M \in \mathcal{M}_S$ such that the estimator $\htheta$ satisfies
	\begin{equation}
		\twonorm{\htheta - \btheta^*} \geq \frac{1}{8\sqrt{5}}\epsilon\sqrt{\frac{d}{n}},
	\end{equation}
	with probability at least $1/4$.
\end{theorem}

% Algorithm \ref{algo: detection} is a special case of the trusted source detection algorithm in \cite{konstantinov2020sample} for the mean estimation problem, which can also be seen as a MTL version of the transferable source detection algorithm used by \cite{tian2023transfer, li2024federated} in the context of TL.

% \begin{algorithm}[!h]
% \caption{MTL based on the detection of trusted sources}
% \label{algo: detection}
% \KwIn{Sample means $\{\barx^{(k)}\}_{k=1}^K$, threshold $\xi \geq 0$}
% Initialize $\widehat{S} = \emptyset$\\
% \For{$k = 1:K$}{
% 	\If{$\twonorm{\barxk{k} - \barxk{k'}} \leq \xi$ for at least $\lfloor K/2 \rfloor$ values of $k' \neq k$}{
% 		$\widehat{S} = \widehat{S} \cup \{k\}$
% 	}
% }
% \KwOut{Estimator $\htheta = |\widehat{S}|^{-1}\sum_{k \in \widehat{S}}\barx^{(k)}$ if $\widehat{S} \neq \emptyset$, otherwise output $\htheta = \bm{0}$}
% \end{algorithm}

% \begin{theorem}
% 	Suppose $d \geq 8K$ and $(K^2+2K)e^{-d/128} + Ke^{-d/8} \leq 3/4$. For any $S \subseteq [K]$ with $|S^c|/K \leq \epsilon$, any threshold $\xi \geq 0$, there exists $\btheta^* \in \mathbb{R}^d$ and $\tQ$ a distribution of $\{\bxk{k}_i\}_{i \in [n], k \in S^c}$ such that the output $\htheta$ from Algorithm \ref{algo: detection} satisfies
% 	\begin{equation}
% 		\twonorm{\htheta - \btheta^*} \geq \frac{\sqrt{29}-5}{64}\sigma\epsilon\sqrt{\frac{d}{n}},
% 	\end{equation}
% 	with probability at least $1/4$.
% \end{theorem}

\section{Robust multi-task learning through filtering}\label{sec: method and theory}

In this section, we consider the general problem setup introduced in Section \ref{subsec: problem setup}, which covers the mean estimation problem discussed in Section \ref{sec: negative results} as a special case. We will present a minimax lower bound and a robust multi-task learning algorithm that can achieve the optimal error rate under contamination and data heterogeneity over a broad regime. In particular, this algorithm eliminates the additional $\sqrt{d}$ factor that arises for many existing methods as shown in Section \ref{sec: negative results}.

Recall that in Section \ref{subsec: problem setup}, we introduced a general contaminated multi-task ERM setting with $K$ related tasks, where an $\epsilon$ fraction of tasks may be adversarially contaminated. Recall also that $\mLk{k}(\btheta) = \tE[\ell(\zk{k}, \btheta)]$ and $\mL(\btheta) = \frac{1}{K}\sum_{k=1}^K \mLk{k}(\btheta)$ are the population task-specific and average risk functions, respectively. 

Our goal is to estimate both the global minimizer of the average risk and the clean task-specific minimizers. We first introduce some necessary conditions for the problem.

\begin{assumption}[Local strong convexity and smoothness]\label{asmp: risk function}
    There exist constants $L \geq 1$, $R_0 > 0$, such that for all $\btheta$, $\btheta' \in \mathbb{R}^d$ with $\twonorm{\btheta - \bthetaks{k}}, \twonorm{\btheta' - \bthetaks{k}} \leq R_0$:
    \begin{equation}
        \frac{1}{2L}\twonorm{\btheta - \btheta'}^2 \leq \mLk{k}(\btheta) - \mLk{k}(\btheta') - \nabla \mLk{k}(\btheta')^\top (\btheta - \btheta') \leq \frac{L}{2}\twonorm{\btheta - \btheta'}^2, \quad \forall k \in [K].
    \end{equation}
\end{assumption}

\begin{remark}
    Assumption \ref{asmp: risk function} essentially requires $\mLk{k}$ to be $L$-smooth and $L^{-1}$-strongly convex on the ball $B(\bthetaks{k}; R_0)$. The smoothness of $\mLk{k}$ implies that $\twonorm{\nabla \mLk{k}(\btheta) - \nabla \mLk{k}(\btheta')} \leq L\twonorm{\btheta - \btheta'}$.  Note that Assumption \ref{asmp: risk function} only requires local strong convexity and smoothness, which is much weaker than the global version commonly used in the literature \cite[e.g.,][]{yin2018byzantine, zhu2023byzantine, allouah2023fixing}. 
To ensure that there is a non-empty region where the risk functions of all tasks are strongly convex and smooth, we assume $\max_{k \in [K]}\twonorm{\bthetaks{k}-\bthetas} \leq R_0/2$ and define $\Theta \coloneqq B(\bthetas; R_0/2) \subseteq B(\bthetas; R_0) \cap (\cap_{k=1}^K B(\bthetaks{k}; R_0))$.
\end{remark}

In addition to the local strong convexity and smoothness for the risk function of each task, we consider the following heterogeneity conditions across tasks.
\begin{assumption}[Task heterogeneity]\label{asmp: task heterogeneity}
    Suppose
    \begin{align}
    \frac{1}{K}\sum_{k=1}^K\twonorm{\nabla \mLk{k}(\btheta) - \nabla \mL(\btheta)}^2 &\leq h^2, \label{eq: global heterogeneity} \\
    \twonorm{\nabla \mLk{k}(\btheta) - \nabla \mL(\btheta)}^2 &\leq (\hk{k})^2, \quad \forall k \in [K], \label{eq: local heterogeneity}
\end{align}
for all $\btheta \in \Theta$, where $h^2 \leq \frac{1}{K}\sum_{k=1}^K (\hk{k})^2$. 
\end{assumption}

Note that \eqref{eq: global heterogeneity} has been widely used in the  heterogeneous multi-task and federated learning literature \cite[e.g.,][]{el2021collaborative, karimireddy2022byzantine,allouah2023fixing}. We are not only interested in $\btheta^*$, but also the local parameter $\bthetaks{k}$ for each task, whose estimation error depends on another local heterogeneity parameter $\hk{k}$ defined in \eqref{eq: local heterogeneity}. This is a natural formalization of task heterogeneity for  MTL settings.

We also require a sub-Gaussian tail assumption for the gradient as in previous works \citep[e.g.][]{duan2023adaptive}.

\begin{assumption}\label{asmp: subG gradient}
    For any $\btheta \in B(\bthetaks{k}; R_0)$, $i \in [n]$, and $k \in [K]$, $\nabla \ell(z_i^{(k)},\btheta)$ is a sub-Gaussian vector with $\|\nabla \ell(z_i^{(k)},\btheta)\|_{\psi_2} \leq C$ for some constant $C > 0$ \footnote{We define the $\psi_2$-norm of a sub-Gaussian variable $X$ as $\|X\|_{\psi_2} = \inf_{t>0}\{\tE\exp(X^2/t^2) \leq 2\}$.}, where the gradient is taken with respect to $\btheta$.
\end{assumption}

% --------------------------------------------------------
\subsection{Minimax lower bounds}\label{subsec: lower bound}
In this subsection, we present minimax lower bounds for the estimation error of both $\bthetas$ and $\bthetaks{k}$ under the setting in Section \ref{subsec: problem setup}. By comparing the lower bound here and the upper bound provided later for our method, we can claim minimax optimality up to logarithmic factors in a large regime. 

Although some existing works \cite[e.g.,][]{karimireddy2022byzantine} have studied lower bounds for the average excess risk associated with $\bthetas$, these results typically only show a lower bound that depends on $\sqrt{\epsilon}h$, while ignoring the statistical costs related to other important parameters $n$, $d$, and $K$. Moreover, we are also interested in the local parameter $\bthetaks{k}$ and how $h$, $\hk{k}$, $\epsilon$ affect its estimation error, which has not been covered in the literature. Therefore, we first establish a comprehensive lower bound for the estimation error of both $\bthetas$ and $\bthetaks{k}$ that explicitly depends on these model parameters, which complements the existing lower-bound results in the literature.

For a given loss function $\ell$ and the associated population-level loss functions, we define the following heterogeneity constrained sets of distributions:
\begin{align}
    \mathcal{P}
    = \Bigg\{\{\tP^{(k)}\}_{k=1}^K :
    &\sup_{\btheta \in \Theta}
    \frac{1}{K}\sum_{k=1}^K
    \big\|\nabla \mLk{k}(\btheta)-\nabla \mL(\btheta)\big\|_2^2
    \le h^2
    \Bigg\}, \\
    \mathcal{P}'
    = \Bigg\{\{\tP^{(k)}\}_{k=1}^K :
    &\sup_{\btheta \in \Theta}
    \frac{1}{K}\sum_{k=1}^K
    \big\|\nabla \mLk{k}(\btheta)-\nabla \mL(\btheta)\big\|_2^2
    \le h^2, \sup_{\btheta \in \Theta}
    \big\|\nabla \mLk{k}(\btheta)-\nabla \mL(\btheta)\big\|_2^2
    \le (\hk{k})^2, \, k \in [K]
    \Bigg\}, 
\end{align}
\vspace{-0.2cm}

\noindent where we assume $h^2 \leq \frac{1}{K}\sum_{k=1}^K(\hk{k})^2$ in $ \mathcal{P}'$.

Consider the Gaussian mean estimation problem, where $\zk{k}_i = \bxk{k}_i$ are i.i.d. $d$-dimensional Gaussian vectors with the identity covariance for $i \in [n]$, together with the squared loss function $\ell(z, \btheta) = \ell(\bx, \btheta) = \frac{1}{2}\twonorm{\bx-\btheta}^2$.  It is straightforward to verify that Assumptions \ref{asmp: risk function} and \ref{asmp: subG gradient} are satisfied here (More details can be found in Section \ref{subsubsec: mean estimation}). Assumption \ref{asmp: task heterogeneity} reduces to $\bthetas = \frac{1}{K}\sum_{k=1}^K \bthetaks{k}$, $\frac{1}{K}\sum_{k=1}^K\twonorm{\bthetaks{k}-\bthetas}^2 \leq h^2$, and $\twonorm{\bthetaks{k} - \bthetas} \leq \hk{k}$ for $k \in [K]$. Note that this defines a more fine-grained parameter space compared to \cite{duan2022adaptive, tian2022unsupervised,tian2024towards, kim2026multi} in this setting, where the latter assumes a much stronger condition $\max_{k \in [K]}\twonorm{\bthetaks{k} - \bthetas} \leq \hmax$.

Recall the notations $\mathcal{S} = \{S \subseteq [K]: |S| \geq K(1-\epsilon)\}$  and $\mathcal{M}_S = \{M: [K] \times [n] \times \mathcal{Z} \to \mathcal{Z} \textup{ such that } M(k, i, \zk{k}_i) = \zk{k}_i, k \in S, i \in [n]\}$. When the context is clear, for fixed $S$ and $M \in \mathcal{M}_S$, we write $\tilde{z}_i^{(k)} \coloneqq M(k,i,\zk{k}_i)$ for the observed, possibly contaminated, data point. 

The following theorem presents minimax lower bounds for the estimation error of $\bthetas$ and $\bthetaks{k}$.

\begin{theorem}\label{thm: lower bound}
    There exist constants $C > 0$ and $c \in (0, 1)$ such that
    \begin{align}
        &\inf_{\htheta}\sup_{\tP \in \mathcal{P}, S \in \mathcal{S}}\sup_{M \in \mathcal{M}_S} \tP\bigg(\|\htheta - \bthetas\|_2 \geq C\bigg(\sqrt{\frac{d}{nK}} + \sqrt{\epsilon}h + \frac{\epsilon}{\sqrt{n}}\bigg)\bigg) \geq c,\\
        &\inf_{\{\hthetak{k}\}_{k=1}^K}\sup_{\tP \in \mathcal{P}', S \in \mathcal{S}}\sup_{M \in \mathcal{M}_S} \tP\bigg(\bigcup_{k \in S}\bigg\{\|\hthetak{k} - \bthetaks{k}\|_2  \geq C\bigg[\bigg(\sqrt{\frac{d}{nK}} + 
    \sqrt{\epsilon}h + \hk{k} + \frac{\epsilon}{\sqrt{n}}\bigg)\wedge \sqrt{\frac{d}{n}}\bigg]\bigg\}\bigg) \geq c.
    \end{align}
\end{theorem}

Compared to the lower bound $\tildeOmega\Big((\sqrt{\frac{d}{nK}} + \max_{k \in [K]}\hk{k} + \frac{\epsilon}{\sqrt{n}}) \wedge \sqrt{\frac{d}{n}}\Big)$ in \cite{duan2022adaptive} and \cite{tian2022unsupervised}, the bounds in Theorem \ref{thm: lower bound} are tighter and more sophisticated in the sense that $\sqrt{\epsilon}h + \hk{k}$ reflects the interaction between contamination and heterogeneity. It may seem counter-intuitive at first that the impact of heterogeneity on estimating the global parameter $\bthetas$ vanishes when there is no contamination ($\epsilon = 0$). However, this is reasonable because the global parameter $\bthetas$ is defined as the minimizer of the average risk across tasks,  rather than a quantity defined with respect to any single task.

% ---------------------------------------------------------
\subsection{Robust multi-task gradient descent}\label{subsec: MTL}
In this subsection, we introduce a robust multi-task gradient descent algorithm for estimating the parameters of interest $\bthetas$ and $\bthetaks{k}$.

Our algorithm is summarized in Algorithm \ref{algo: robust federated gradient descent}. The main idea is to first define a robust gradient aggregation algorithm $g(\btheta)$, which is introduced in Section \ref{subsec: JRGE}, and a personalized local gradient calculation algorithm $g^{(k)}(\btheta)$, and then run gradient descent on the global risk $\mL(\btheta)$ and the local risks $\mLk{k}(\btheta)$ to estimate $\bthetas$ and $\bthetaks{k}$ for each task. To implement this idea, we first need to replace the population risk functions by their empirical counterparts $\hmLk{k}(\btheta) = \frac{1}{n}\sum_{i=1}^n \ell(\tilde{z}^{(k)}_i, \btheta)$ and $\hmL(\btheta) = \frac{1}{K}\sum_{k=1}^K\hmLk{k}(\btheta)$. Moreover, while most of the empirical risks $\hmLk{k}$ are trustworthy, the averaged version $\hmL(\btheta)$ is certainly not due to the presence of adversarially contaminated tasks, which calls for a robust way of aggregating the gradients from the tasks.  

% At each iteration $t$ with the current estimate $\htheta_t$ of $\btheta^*$, we receive the gradient $\nabla\hmLk{k}(\htheta_t)$ from each task, then aggregate them by a joint robust gradient estimation algorithm (JRGE), due to the presence of contaminated tasks. 
Specifically, by viewing the $k$-th task gradient in iteration $t$, $\nabla \hmLk{k}(\htheta_t)$, as a ``sample" and the population-level gradient $\nabla \mL(\htheta_t) = \frac{1}{K}\sum_{k=1}^K\nabla \mLk{k}(\htheta_t)$ as the corresponding ``mean" value \footnote{Here $\nabla \hmLk{k}(\htheta_t)$ and $\nabla \mL(\htheta_t)$ represent the gradient of $\hmLk{k}$ and $\mL$ evaluated at $\htheta_t$.}, we adopt a robust mean estimation procedure to robustly aggregate the gradients. Our robust mean estimator is the joint robust gradient estimation (JRGE) algorithm discussed in the next subsection. After receiving the global aggregated gradient from the JRGE algorithm, each task can update its estimator by a similar gradient descent step. To better borrow information from other tasks, we add a soft-thresholding step to the local gradient to encourage the personalized estimator to be close to the global aggregated gradient. By iterating this process, we can obtain the final global estimator $\htheta_T$ and personalized estimators $\{\hthetak{k}_T\}_{k=1}^K$ after $T$ iterations.

\begin{algorithm}
\caption{Robust multi-task gradient descent}
\label{algo: robust federated gradient descent}
\KwInput{Observed possibly contaminated data $\{\tilde{z}^{(k)}_i\}_{i \in [n], k \in [K]}$, initial estimators $\htheta_0$ and $\{\hthetak{k}_0\}_{k=1}^K$, step sizes $\eta$ and $\{\eta^{(k)}\}_{k=1}^K$, number of iterations $T$, threshold $\lambda$}
\KwOutput{Global estimator $\htheta_T$ and personalized estimators $\{\hthetak{k}_T\}_{k=1}^K$}
Define $g(\btheta) \coloneqq \text{JRGE}(\{\nabla \hmLk{k}(\btheta)\}_{k=1}^K)$, $\forall \btheta \in \mathbb{R}^d$ \tcp*{Joint robust gradient estimation}
Define $g^{(k)}(\btheta) \coloneqq  g(\btheta) + \textup{soft-thresholding}(\nabla \hmLk{k}(\btheta) - g(\btheta), \lambda)$ \footnotemark  \tcp*{Personalized local gradient computation}
\For{$t = 0$ \KwTo $T-1$}{
Calculate $\nabla \hmLk{k}(\htheta_t) = \frac{1}{n}\sum_{i=1}^n \nabla \ell(\tilde{z}^{(k)}_i, \htheta_t), k \in [K]$;

$\htheta_{t+1} = \htheta_t - \eta\times g(\htheta_t)$ \tcp*{Joint gradient descent for global parameter}
$\hthetak{k}_{t+1} = \hthetak{k}_{t} - \eta^{(k)}g^{(k)}(\hthetak{k}_t)$, for $k \in [K]$ \tcp*{Local gradient descent}
}
\KwRet{$\htheta_T$, $\{\hthetak{k}_T\}_{k=1}^K$}
\end{algorithm}
\footnotetext{Here the soft-thresholding function is a generalized version for vectors: $\textup{soft-thresholding}(\bx, \lambda) = (\bx - \frac{\lambda}{\twonorm{\bx}}\bx)\mathds{1}(\twonorm{\bx} > \lambda)$ for $\bx \in \mathbb{R}^d$ and $\lambda \geq 0$.}

Next, we introduce the general theory for our robust multi-task gradient descent algorithm (Algorithm \ref{algo: robust federated gradient descent}).  In addition to the conditions introduced in Section \ref{subsec: problem setup}, the gradient descent algorithm also relies on an accurate gradient estimation, which is quantified in the following assumption.

\begin{assumption}[Gradient estimation error]\label{asmp: gradient est error}
    Denote $H = \{h, \hk{1}, \ldots, \hk{K}\}$. With probability at least $1-\delta$, for all subsets $S \subseteq [K]$ with $|S^c|/K \leq \epsilon$, all contamination mechanism $M \in \mathcal{M}_S$, and for all $\btheta \in \Theta$, we have:
    \begin{align}
        \twonorm{g(\btheta) - \nabla \mL(\btheta)} &\leq \alpha(n, K, d, \epsilon, \delta, H),\\
        \twonorm{g^{(k)}(\btheta) - \nabla \mLk{k}(\btheta)} &\leq \alpha^{(k)}(n, K, d, \epsilon, \delta, H), \quad \forall k \in S. 
    \end{align}
    We shall write $\alpha$ and $\alpha^{(k)}$ as shorthand notation for $\alpha(n, K, d, \epsilon, \delta, H)$ and $\alpha^{(k)}(n, K, d, \epsilon, \delta, H)$, respectively.
\end{assumption}

Now we are ready to present the main result for our robust multi-task gradient descent algorithm regarding the estimation error of $\bthetas$ and $\bthetaks{k}$'s.

\begin{theorem}\label{thm: federated gradient descent}
    Let $\kappa \coloneqq 2\eta/L - L^2\eta^2$ and $\kappa^{(k)} \coloneqq2\eta^{(k)}/L - L^2(\eta^{(k)})^2$, respectively. 
    Under Assumptions \ref{asmp: risk function} and \ref{asmp: gradient est error}, if the initializations $\htheta_0$ and $\hthetak{k}_0$, the step sizes $\eta$ and $\{\eta^{(k)}\}$ satisfy $\twonorm{\htheta_0 - \btheta^*} + \eta\alpha\sqrt{\frac{2(2-\kappa)}{\kappa^2}} \leq R_0$, $\max_{k \in [K]}\Big\{\twonorm{\hthetak{k}_0 - \bthetaks{k}}+\eta^{(k)}\alpha^{(k)}\sqrt{\frac{2(2-\kappa^{(k)})}{(\kappa^{(k)})^2}}\Big\} \leq R_0$, then for all subset $S \subseteq [K]$ with $|S^c|/K \leq \epsilon$, all contamination mechanisms $M \in \mathcal{M}_S$, with probability at least $1-\delta$, we have
    \begin{align}
        \twonorm{\htheta_T - \btheta^*} &\leq (1-\kappa/2)^{T/2} \twonorm{\htheta_0 - \bthetas} + \eta\alpha\sqrt{\frac{2(2-\kappa)}{\kappa^2}}, \\
        \twonorm{\hthetak{k}_T - \bthetaks{k}} &\leq (1-\kappa^{(k)}/2)^{T/2} \twonorm{\hthetak{k}_0 - \bthetaks{k}} + \eta^{(k)}\alpha^{(k)}\sqrt{\frac{2(2-\kappa^{(k)})}{(\kappa^{(k)})^2}}, \quad \forall k \in S.
    \end{align}
\end{theorem} 

Theorem \ref{thm: federated gradient descent} shows that provided the tuning parameters are suitably chosen and the number of iterations $T$ is sufficiently large that the initialization errors are dominated by the gradient estimation error terms, the parameter estimation errors are essentially of the same order as the gradient estimation errors. From the next subsection, our main focus will be on the JRGE algorithm used in Algorithm \ref{algo: robust federated gradient descent} and analyzing the gradient estimation error $\alpha$ and $\alpha^{(k)}$ in Assumption \ref{asmp: gradient est error}.

% ---------------------------------------------------------
\subsection{Joint robust gradient estimation}\label{subsec: JRGE}

In this subsection, we will describe the joint robust gradient estimation (JRGE) algorithm that we propose to use in Step 1 of Algorithm \ref{algo: robust federated gradient descent}. As mentioned in Section \ref{subsec: MTL}, we view the $k$-th gradient $\nabla \hmLk{k}(\htheta_t)$ as a ``sample" and the averaged population-level gradient $\nabla \mL(\htheta_t) = \frac{1}{K}\sum_{k=1}^K\nabla \mLk{k}(\htheta_t)$ as the corresponding ``mean" value that is to be estimated. This allows us to adapt robust mean estimation methods to estimate $\nabla \mL(\htheta_t)$ under task contamination. Here, we adapt a filtering-based algorithm from algorithmic robust statistics \citep[e.g.][]{diakonikolas2019robust,diakonikolas2023algorithmic} to our context. The filtering algorithm was originally proposed to estimate the population mean using a set of contaminated data whose uncontaminated versions are generated in an i.i.d.\ fashion. Compared to many other robust mean estimators such as coordinate-wise median, geometric median, and Tukey median, the filtering algorithm can achieve nearly optimal estimation error with polynomial computational time. 

% The basic idea of it is to identify possible outliers by checking the projection along the direction of the eigenvector corresponding to the largest eigenvalue of the sample covariance matrix. In each iteration, possible outliers will be removed, and the sample covariance will be updated. The algorithm will terminate when the spectral norm of the difference between the sample covariance and true covariance falls below a certain threshold, and the mean of the remaining sample can be shown to have a nearly-optimal estimation error for sub-Gaussian sample.

One common issue of the existing filtering-type algorithms is that the true covariance of the population is required to be known in advance, which is impractical in most cases. A natural solution is to replace the population covariance by an estimated covariance. However, it is unclear how the covariance estimation error propagates to the final mean estimation error, and this approach requires solving a statistically harder problem (covariance estimation) in order to address an easier one (mean estimation). But as we will see in the next subsection, in the contaminated MTL context, even with task heterogeneity, this idea works well and there exist some simple robust covariance estimators based on the single-task covariance matrices that are good enough for our purposes. The main reason is that, in contrast to the single-task problem, we have multiple observations from each task, and the single-task covariance matrices are easy to compute. 

% Theoretically, we quantified how the accuracy of estimated covariance impacts the mean estimation error, and proposed a simple estimator which helps achieve nearly-minimax optimal mean estimation performance in our context. 

We summarize the robust mean estimation algorithm in Algorithm \ref{algo: robust mean estimation}, which requires an estimator of the true covariance. When we use this algorithm as the JRGE in Algorithm \ref{algo: robust federated gradient descent}, given a parameter value $\btheta$, we consider the gradient of each sample risk function $\nabla \hmLk{k}(\btheta)$ as the data $\bxk{k}$,
and the corresponding covariance matrix can be defined as $\bSigma_{\btheta} = \frac{1}{K}\sum_{k=1}^K \tE[(\nabla \hmLk{k}(\btheta) - \nabla \mL(\btheta))(\nabla \hmLk{k}(\btheta) - \nabla \mL(\btheta))^\top]$. 
In this subsection, we first consider a black-box estimator $\hSigma_{\btheta}$ for $\bSigma_{\btheta}$ and provide black-box estimation error rates for the gradients and the parameter, which contain the covariance estimation error $\twonorm{\hSigma_{\btheta} - \bSigma_{\btheta}}$. In the next subsection, we will propose a covariance estimator, obtain the corresponding covariance estimation error, then plug it in the black-box results and obtain the final explicit error rates.

\begin{algorithm}
    \caption{Robust mean estimation algorithm (used as the joint robust gradient estimation method in Algorithm \ref{algo: robust federated gradient descent})}
    \label{algo: robust mean estimation}
    \KwInput{Observed possibly contaminated data $\{\tilde{\bx}^{(k)}\}_{k=1}^K$, contamination proportion $\epsilon$, threshold value $\lambda_{\Sigma}$ and the covariance estimate $\widehat{\bSigma}$}
    \KwOutput{Estimated mean value}
    Initialize $\tilde{S} = [K]$\\
    Compute the empirical mean $\bm{\mu}_{\tilde{S}} = |\tilde{S}|^{-1}\sum_{k \in \tilde{S}}\tilde{\bx}^{(k)}$, the empirical covariance $\bSigma_{\tilde{S}} = |\tilde{S}|^{-1}\sum_{k \in \tilde{S}}(\tilde{\bx}^{(k)}-\bm{\mu}_{\tilde{S}})(\tilde{\bx}^{(k)}-\bm{\mu}_{\tilde{S}})^\top$, and the top eigenvector $\bm{v}$ of $\bSigma_{\tilde{S}} - \widehat{\bSigma}$\\
    \While{$\lambdamax(\bSigma_{\tilde{S}} - \widehat{\bSigma}) > \lambda_{\bSigma}$}{
        Remove one task index $k$ from $\tilde{S}$ with probability $\frac{f(\tilde{\bx}^{(k)})}{\sum_{k \in \tilde{S}}f(\tilde{\bx}^{(k)})}$ for $k \in \tilde{S}$, and $f(\tilde{\bx}^{(k)}) =\begin{cases}
            0, & k \notin L;\\
            (\bm{v}^\top(\tilde{\bx}^{(k)} - \bm{\mu}_{\tilde{S}}))^2, & k\in L.
        \end{cases}$
        
        where $L\subseteq \tilde{S}$ contains the top $\epsilon|\tilde{S}|$ task indices corresponding to the largest values of $|\bm{v}^\top (\tilde{\bx}^{(k)} - \bm{\mu}_{\tilde{S}})|^2$. \footnotemark
        
        Update $\bm{\mu}_{\tilde{S}}$, $\bSigma_{\tilde{S}}$, and $\bv$
    }
    \KwRet{$\bm{\mu}_{\tilde{S}}$}
\end{algorithm}
\footnotetext{The set $\tilde{S}$ remains unchanged when calculating this probability and it is updated after this step is completed.}

Next, we describe one more assumption under which we will present the estimation error.

\begin{assumption}[Local smoothness of loss function]\label{asmp: lipschitz loss}
    With probability at least $1- (nK)^{-C_1d}$, $\twonorm{\nabla \ell(\zk{k}_i, \btheta) - \nabla \ell(\zk{k}_i, \btheta')} \leq L'\twonorm{\btheta - \btheta'}$, for all $i \in [n]$, $k \in [K]$, and $\btheta, \btheta'$ satisfying $\twonorm{\btheta - \bthetas}, \twonorm{\btheta' - \bthetas} \leq R_0$, where $L' \lesssim (nKd)^{C_2}$, and $\{C_i\}_{i=1}^2$ are some positive constants.
\end{assumption}

\begin{remark}
    This condition is required because Algorithm \ref{algo: robust mean estimation} is run on the current estimate $\htheta_t$ and $\hthetak{k}_t$ in each iteration of Algorithm \ref{algo: robust federated gradient descent}. To make the multi-task gradient descent work well, a uniform convergence result as in Assumption \ref{asmp: gradient est error} is needed. We use a covering argument to prove it, which requires the Lipschitzness of the gradient of loss function. Similar assumptions are made in other robust gradient descent papers \cite[e.g.][]{yin2018byzantine, su2019securing}. Note that this assumption can be viewed as a high-probability strengthening of Assumption \ref{asmp: risk function}, which only requires smoothness of the population-level risk function.
\end{remark}

Under \Cref{asmp: subG gradient} and \Cref{asmp: lipschitz loss}, together with appropriate choice of tuning parameters and appropriate conditions, we show in \Cref{thm: gradident est error appendix} (\Cref{app:proof-of-detailedresults} of the appendix) that the gradient estimation errors satisfy Assumption \ref{asmp: gradient est error} with 
\[
\alpha(n, K, d, \epsilon, \delta, H) = \tildeO \Big( \sqrt{\frac{d}{nK}} + \epsilon\sqrt{\frac{1}{n}} + \sqrt{\epsilon\sup\limits_{\btheta \in \Theta}\twonorm{\hSigma_{\btheta} - \bSigma_{\btheta}}} + \sqrt{\epsilon}h \Big)
\]
and 
\[
\alpha^{(k)}(n, K, d, \epsilon, \delta, H) = \tildeO\Big( \min \Big\{\sqrt{\frac{d}{nK}} + \epsilon\sqrt{\frac{1}{n}} + \sqrt{\epsilon\sup\limits_{\btheta \in \Theta}\twonorm{\hSigma_{\btheta} - \bSigma_{\btheta}}} + \sqrt{\epsilon}h + h^{(k)}, \sqrt{\frac{d}{n}}\Big\}\Big).
\]
By plugging the error rates above into Theorem \ref{thm: federated gradient descent}, we obtain the following high-probability upper bounds for the parameter estimation error.

\begin{theorem}\label{thm: parameter est error}
    Let $\{C_i\}_{i=1}^9$ be some positive constants. Let $\lambda = C_1\sqrt{\frac{d\log(nK)}{n}}$ in Algorithm \ref{algo: robust federated gradient descent} and $\lambda_{\bSigma} = C_2 \bigg[\sup\limits_{\btheta \in \Theta}\twonorm{\hSigma_{\btheta} - \bSigma_{\btheta}} + \frac{1}{n} \Big(\sqrt{\frac{d\log(nK)}{K}} + \frac{d\log(nK)}{K}\Big)  + \epsilon \frac{\log(1/\epsilon)}{n}   + \epsilon h^2\bigg]$ in Algorithm \ref{algo: robust mean estimation}. Under Assumptions \ref{asmp: risk function}, \ref{asmp: subG gradient} and \ref{asmp: lipschitz loss}, if $\kappa = 2\eta/L - L^2\eta^2 \in (0, 1)$, $nK \geq C_3R_0^{-2}d\log(nK)$, $n \geq C_4R_0^{-2}\epsilon^2\log(1/\epsilon)$, $\sqrt{\epsilon}\sup\limits_{\btheta \in \Theta}\twonorm{\hSigma_{\btheta} - \bSigma_{\btheta}}^{1/2} \leq C_5 R_0$, $\sqrt{\epsilon}h \leq C_6 R_0$, $\max\limits_{k \in [K]}\hk{k} \leq C_7 R_0$, then with probability at least $1- (nK)^{-C_8d} - e^{-C_9K\epsilon}$, for all subset $S \subseteq [K]$ with $|S^c|/K \leq \epsilon$, all contamination mechanism $M \in \mathcal{M}_S$, we have for all $k \in S$,
    \begin{align}
        \twonorm{\htheta_T - \btheta^*} &\lesssim (1-\kappa/2)^{T/2} \twonorm{\htheta_0 - \bthetas} + \tildeO\bigg(\sqrt{\frac{d}{nK}} + \epsilon\sqrt{\frac{1}{n}} + \sqrt{\epsilon}\sup_{\btheta \in \Theta}\twonorm{\hSigma_{\btheta} - \bSigma_{\btheta}}^{1/2} + \sqrt{\epsilon}h\bigg),\\
        \twonorm{\hthetak{k}_T - \bthetaks{k}} &\lesssim (1-\kappa^{(k)}/2)^{T/2} \twonorm{\hthetak{k}_0 - \bthetaks{k}} \\
        &\quad + \tildeO\bigg(\min \bigg\{\sqrt{\frac{d}{nK}} + \epsilon\sqrt{\frac{1}{n}} + \sqrt{\epsilon}\sup_{\btheta \in \Theta}\twonorm{\hSigma_{\btheta} - \bSigma_{\btheta}}^{1/2}  + \sqrt{\epsilon}h + h^{(k)}, \,\, \sqrt{\frac{d}{n}}\bigg\}\bigg).
    \end{align}
\end{theorem}

\begin{remark}
    All the conditions related to $R_0$ are required to guarantee that the optimization trajectory stays in the region where local convexity and smoothness hold. Similar conditions on sample size and heterogeneity also appear in the literature \cite[e.g.,][]{chen2023minimax, duan2022adaptive,tian2024towards}. 
\end{remark}

% ---------------------------------------------------------
\subsection{Gradient covariance estimation}\label{subsec: cov estimation}
In this subsection, we propose a covariance estimator $\hSigma$ to use as the input to Algorithm \ref{algo: robust mean estimation}, which is then used as a sub-routine in Algorithm \ref{algo: robust federated gradient descent}. 

Recall that our goal is to accurately estimate $\bSigma_{\btheta} = \frac{1}{K}\sum_{k=1}^K \tE[(\nabla \hmLk{k}(\btheta) - \nabla \mL(\btheta))(\nabla \hmLk{k}(\btheta) - \nabla \mL(\btheta))^\top]$ in Algorithm \ref{algo: robust mean estimation}. Generally speaking, this is a challenging robust covariance estimation problem due to the presence of both heterogeneity across different tasks and adversarial contamination. This makes most existing robust covariance estimators not directly applicable. While it may be possible to modify the analysis of some existing estimators to accommodate our setting, we note that many of the optimal robust covariance estimators are computationally inefficient with complexity scaling exponentially as the dimension increases \citep[e.g.][]{abdalla2024covariance,minasyan2025statistically,cherapanamjeri2020algorithms,chen2018robust}. Therefore, we propose a simple and computationally tractable estimator here.

Let us start from the homogeneous case, where the distributions of gradients from all uncontaminated tasks are the same. In this case, we have $\nabla \mL(\btheta) = \nabla \mLk{k}(\btheta)$, and the target covariance $\bSigma_{\btheta}$ would become 
$$\bSigma_{\btheta} = \frac{1}{K}\sum_{k=1}^K \cov(\nabla \hmLk{k}(\btheta)) = \frac{1}{K}\sum_{k=1}^K \frac{1}{n}\cov(\nabla \ell(\zk{k}, \btheta)),$$ which is the average of single-task covariance matrices $\cov(\nabla \ell(\zk{k}, \btheta))$ scaled by $1/n$.
%$\cov(\nabla \ell(\zk{k}, \btheta)) = \tE[(\nabla \ell(\zk{k}, \btheta) -\nabla\mLk{k}(\btheta))(\nabla \ell(\zk{k}, \btheta) -\nabla\mLk{k}(\btheta))^\top]$. 
Therefore, in the homogeneous case, if we can identify a subset of tasks that are unlikely to be contaminated, we can simply average their single-task covariance estimators to estimate $\bSigma_{\btheta}$. We will see that this idea also works well in the heterogeneous case, where the gradient distributions may differ across tasks. In that setting, the estimator incurs an additional bias, but this bias can be effectively controlled. 

The idea above is formalized in Algorithm \ref{algo: naive est of Sigma heterogeneous}. The key step is to identify a subset of tasks that are unlikely to be contaminated, which is done by checking the pairwise distance between the single-task covariance estimators. We keep those single-task covariance estimators that are close to each other, and take the average of them as the final estimator $\hSigma_{\btheta}$. The intuition is that contaminated tasks are expected to produce covariance estimators that are less compatible with the bulk of clean tasks.

\begin{algorithm}
    \caption{Gradient covariance estimation}
    \label{algo: naive est of Sigma heterogeneous}
    \KwInput{Single-task covariance estimators $\widehat{\bSigma}^{(k)}_{\btheta} =  \frac{1}{n}\times \frac{1}{n}\sum_{i=1}^n \big[\nabla \ell(\tilde{z}^{(k)}_i, \btheta) - \frac{1}{n}\sum_{i=1}^n\nabla \ell(\tilde{z}^{(k)}_i, \btheta)\big]\big[\nabla \ell(\tilde{z}^{(k)}_i, \btheta) - \frac{1}{n}\sum_{i=1}^n\nabla \ell(\tilde{z}^{(k)}_i, \btheta)\big]^\top$, $k \in [K]$, the contamination proportion $\epsilon$}
    \KwOutput{Estimator $\hSigma_{\btheta}$}
    $\widehat{S}_{\textup{safe}} = \big\{k \in [K]: \twonorm{\hSigmak{k}_{\btheta} - \hSigmak{k'}_{\btheta}} \leq \textup{quantile}_{\binom{K(1-\epsilon)}{2}/\binom{K}{2}}(\{\twonorm{\hSigmak{k_1}_{\btheta} - \hSigmak{k_2}_{\btheta}}\}_{k_1 \neq k_2}) \textup{ for at least } K/4 \textup{ indices } k' \in [K] \big\}$\\
    $\hSigma_{\btheta} = \frac{1}{|\widehat{S}_{\textup{safe}}|}\sum_{k \in \widehat{S}_{\textup{safe}}}\hSigmak{k}_{\btheta}$\\
    \KwRet{$\hSigma_{\btheta}$}
\end{algorithm}

The following theorem provides the estimation error of $\hSigma_{\btheta}$ for $\bSigma_{\btheta}$, which can be plugged into Theorem \ref{thm: parameter est error} to obtain the final error rate.

\begin{theorem}\label{thm: est error cov}
Under Assumptions \ref{asmp: subG gradient} and \ref{asmp: lipschitz loss}, with probability at least $1-(nK)^{-Cd}$, the output from Algorithm \ref{algo: naive est of Sigma heterogeneous} satisfies
\begin{equation}
    \sup_{\btheta \in \Theta}\twonorm{\hSigma_{\btheta} - \bSigma_{\btheta}} = \tildeO\Bigg(\frac{1}{n} \bigg(\sqrt{\frac{d}{K}} + \frac{d}{K}\bigg)  + \epsilon \frac{1}{n} + h^2 + \frac{\epsilon}{n}\bigg(\sqrt{\frac{d}{n}} + \frac{d}{n}\bigg)\Bigg),
\end{equation}
where $C > 0$ is a constant.
\end{theorem}

The following corollary is a direct consequence after plugging the covariance estimation error obtained in Theorem \ref{thm: est error cov} into the  parameter estimation error (Theorem \ref{thm: parameter est error}).

\begin{corollary}\label{cor: parameter alg error}
    Set $\lambda = C\sqrt{\frac{d\log(nK)}{n}}$ with some constant $C> 0$. Under Assumptions \ref{asmp: risk function}, \ref{asmp: subG gradient} and \ref{asmp: lipschitz loss}, $nK \geq C_1R_0^{-2}d\log(nK)$, $n \geq C_2 R_0^{-2}\epsilon^2\log(1/\epsilon)$, $\frac{\epsilon}{\sqrt{n}}\Big[\big(\frac{d\log(nK)}{n}\big)^{1/4} + \big(\frac{d\log(nK)}{n}\big)^{1/2}\Big] \leq C_3R_0$, $\sqrt{\epsilon} h \leq C_4 R_0$, $\max\limits_{k \in [K]}\hk{k} \leq C_5R_0$, with probability at least $1-(nK)^{-C_6d} - e^{-C_7K\epsilon}$, for all subsets $S \subseteq [K]$ with $|S^c|/K \leq \epsilon$, all contamination mechanism $M \in \mathcal{M}_S$, we have for all $k \in S$
    \begin{align}
        \twonorm{\htheta_T - \btheta^*} &\leq (1-\kappa/2)^{T/2} \twonorm{\htheta_0 - \bthetas} + \tildeO\bigg( \sqrt{\frac{d}{nK}} + \epsilon\sqrt{\frac{1}{n}} + \sqrt{\epsilon}h + \frac{\epsilon}{\sqrt{n}}\bigg[\bigg(\frac{d}{n}\bigg)^{1/4} + \bigg(\frac{d}{n}\bigg)^{1/2}\bigg]\bigg)\\
        \twonorm{\hthetak{k}_T - \bthetaks{k}} &\leq (1-\kappa^{(k)}/2)^{T/2} \twonorm{\hthetak{k}_0 - \bthetaks{k}} \\
        &\quad + \tildeO\bigg(\min \bigg\{\sqrt{\frac{d}{nK}} + \epsilon\sqrt{\frac{1}{n}}  + \sqrt{\epsilon}h + h^{(k)}   + \frac{\epsilon}{\sqrt{n}}\bigg[\bigg(\frac{d}{n}\bigg)^{1/4} + \bigg(\frac{d}{n}\bigg)^{1/2}\bigg], \sqrt{\frac{d}{n}}\bigg\}, 
    \end{align}
    where $\{C_i\}_{i=1}^7$ are some positive constants.
\end{corollary}

\begin{remark}\label{rem:optimality}
    By Corollary \ref{cor: parameter alg error}, when $T \gtrsim \log(nK)$, $\max_{k \in [K]} \twonorm{\hthetak{k}_0 - \bthetaks{k}} \vee \twonorm{\htheta_0 - \bthetas} \lesssim 1$, we have
\begin{align}
    \twonorm{\htheta_T - \btheta^*} 
    &= \tildeO\bigg(\sqrt{\frac{d}{nK}} + \epsilon\sqrt{\frac{1}{n}}  + \sqrt{\epsilon}h  + \frac{\epsilon}{\sqrt{n}}\bigg[\bigg(\frac{d}{n}\bigg)^{1/4} +  \bigg(\frac{d}{n}\bigg)^{1/2}\bigg]\bigg),\\
    \twonorm{\hthetak{k}_T - \bthetaks{k}} &= \tildeO\bigg(\min \bigg\{\sqrt{\frac{d}{nK}} + \epsilon\sqrt{\frac{1}{n}}  + \sqrt{\epsilon}h + h^{(k)}   + \frac{\epsilon}{\sqrt{n}}\bigg[\bigg(\frac{d}{n}\bigg)^{1/4} + \bigg(\frac{d}{n}\bigg)^{1/2}\bigg], \sqrt{\frac{d}{n}}\bigg\}, \,\forall k \in S,
\end{align}
with probability at least $1-(nK)^{-C_6d} - e^{-C_7K\epsilon}$. Comparing with the lower bound $\sqrt{\frac{d}{nK}} + \frac{\epsilon}{\sqrt{n}} + \sqrt{\epsilon}h$ for the estimation of $\btheta^*$ and $\min\Big\{\sqrt{\frac{d}{nK}} + \frac{\epsilon}{\sqrt{n}} + \sqrt{\epsilon}h + \hk{k}, \sqrt{\frac{d}{n}}\Big\}$ for $\bthetaks{k}$, it is clear that when $n \gtrsim d$ or $\epsilon^2 \big(1\vee \sqrt{\frac{n}{d}}\big) \lesssim \frac{n}{K}$, the upper bounds are minimax optimal up to logarithmic factors. Note that this minimax optimality regime includes $n \gtrsim \min\{d, \epsilon^2K\}$, which is easy to satisfy in practice. 
\end{remark}

We summarize the minimax optimality region of Remark \ref{rem:optimality} in Figure \ref{fig:minimax-region} for illustration, where the shaded area represents the regime in which the upper bounds in Corollary \ref{cor: parameter alg error} are minimax optimal up to logarithmic factors. In contrast, the methods discussed in Section \ref{sec: negative results} incur an additional $\sqrt{d}$ factor in the $\epsilon\sqrt{\frac{d}{n}}$ term. This leads to suboptimal performance unless stringent conditions, such as $d \asymp 1$, are satisfied, making these methods unsuitable for settings where the dimensionality is large.
% Similar discussions can be made for gradient estimation as well, which we omit here to avoid redundancy.

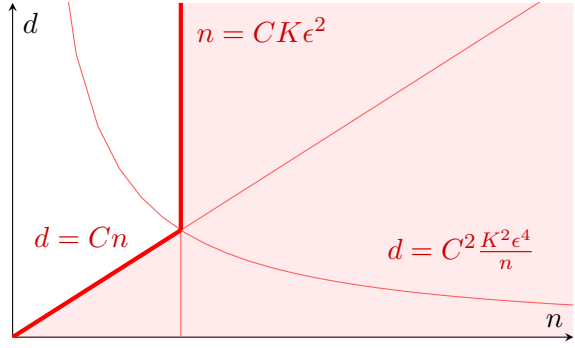
\begin{figure}[t]
\centering
\begin{tikzpicture}
    \pgfmathsetmacro{\nstar}{1.2}
    \pgfmathsetmacro{\dstar}{2.4}
    \pgfmathsetmacro{\hyperconst}{2.88}
    \begin{axis}[
    axis lines=middle,
    xlabel={$n$},
    ylabel={$d$},
    xmin=0, xmax=4,
    ymin=0, ymax=7.5,
    xtick=\empty,
    ytick=\empty,
    width=9cm,
    height=6cm,
    ]

    % --- shaded region: right of d=C_1n and n=C_2K\epsilon^2 ---
    \addplot[
        draw=none,
        domain=0:\dstar,
        samples=100,
        fill=red,
        fill opacity=0.08
    ] ({x/2}, {x}) -- (axis cs:4,\dstar) -- (axis cs:4,0) -- cycle;

    \addplot[
        draw=none,
        domain=\dstar:7.5,
        samples=2,
        fill=red,
        fill opacity=0.08
    ] ({\nstar}, {x}) -- (axis cs:4,7.5) -- (axis cs:4,\dstar) -- cycle;

    % --- thin full curves ---
    \addplot[domain=0.3:4, red!65, thin] {\hyperconst/x};
    \addplot[domain=0:4, red!65, thin] {2*x};
    \addplot[domain=0:7.5, red!65, thin] ({\nstar}, x);

    % --- thick actual boundary ---
    \addplot[domain=0:\nstar, red, line width=1.6pt] {2*x};
    \addplot[domain=\dstar:7.5, red, line width=1.6pt] ({\nstar}, x);

    % --- curve labels ---
    \node[red!80!black, anchor=south east] at (axis cs:0.9,1.8) {$d=Cn$};
    \node[red!80!black, anchor=west] at (axis cs:1.24,6.9) {$n=CK\epsilon^2$};
    \node[red!80!black, anchor=south west] at (axis cs:2.6,1.3) {$d=C^2\frac{K^2\epsilon^4}{n}$};

    \end{axis}
\end{tikzpicture}
\caption{Diagram of the minimax optimality region for the estimation of $\btheta^*$ and $\bthetaks{k}$'s, where the shaded region corresponds to the regime where the upper bound in Corollary \ref{cor: parameter alg error} is minimax optimal up to logarithmic factors. }
\label{fig:minimax-region}
\end{figure}

% --------------------------------------------------------
\subsection{Examples}\label{subsec: examples}
In this subsection, we consider two specific model examples to illustrate how our general algorithm and theory can be applied. 

% As noted in Section \ref{subsec: cov estimation}, in principle, any robust covariance estimation method can be used to estimate the gradient covariance and be incorporated into our filtering-based Algorithms \ref{algo: robust federated gradient descent} and \ref{algo: robust mean estimation}. Here, we want to highlight the potential of applying certain existing robust covariance estimators in specific settings, although most of these methods may be computationally expensive.

% --------------------------------------
\subsubsection{Mean estimation}\label{subsubsec: mean estimation}
The first example is the mean estimation problem, where $z_i^{(k)}$ are i.i.d $d$-dimensional sub-Gaussian vectors with the mean $\bthetaks{k}$ and $\|\zk{k}_i\|_{\psi_2}\lesssim 1$, for $k \in [K]$. We consider the squared loss $\ell(\btheta,z) = \|\btheta-z\|_2^2/2$, so that $\bthetaks{k} = \argmin_{\btheta \in \mathbb{R}^d}\mLk{k}(\btheta) = \argmin_{\btheta \in \mathbb{R}^d}\tE[\ell(\btheta, \zk{k})]$. Assumption \ref{asmp: subG gradient} automatically holds. Moreover, since $\nabla\ell(\btheta,z) = \btheta-z$ and $\nabla^2\ell(\btheta,z) = \bm{I}_d$, we have that $\mLk{k}(\btheta)$ is strongly convex and smooth on $\mathbb{R}^d$ with parameter $L = L' =1$ in Assumptions \ref{asmp: risk function} and \ref{asmp: lipschitz loss}. Therefore, the theory presented in previous sections generally applies to this setting. By Corollary \ref{cor: parameter alg error}, when $T \gtrsim \log(nK)$, $\max_{k \in [K]} \twonorm{\hthetak{k}_0 - \bthetaks{k}} \vee \twonorm{\htheta_0 - \bthetas} \lesssim 1$, we have
\begin{align}
    \twonorm{\htheta_T - \btheta^*} 
    &= \tildeO\bigg(\sqrt{\frac{d}{nK}} + \epsilon\sqrt{\frac{1}{n}}  + \sqrt{\epsilon}h  + \frac{\epsilon}{\sqrt{n}}\bigg[\bigg(\frac{d}{n}\bigg)^{1/4} +  \bigg(\frac{d}{n}\bigg)^{1/2}\bigg]\bigg),\\
    \twonorm{\hthetak{k}_T - \bthetaks{k}} &= \tildeO\bigg(\min \bigg\{\sqrt{\frac{d}{nK}} + \epsilon\sqrt{\frac{1}{n}}  + \sqrt{\epsilon}h + h^{(k)}   + \frac{\epsilon}{\sqrt{n}}\bigg[\bigg(\frac{d}{n}\bigg)^{1/4} + \bigg(\frac{d}{n}\bigg)^{1/2}\bigg], \sqrt{\frac{d}{n}}\bigg\}, \quad \forall k \in S,
\end{align}
with probability at least $1-(nK)^{-Cd}- e^{-C'K\epsilon}$, where $C, C'$ are some constants. 

Similar to our comment when discussing the minimax lower bound in Section \ref{subsec: lower bound}, Assumption \ref{asmp: task heterogeneity} reduces to $\bthetas = \frac{1}{K}\sum_{k=1}^K \bthetaks{k}$, $\frac{1}{K}\sum_{k=1}^K\twonorm{\bthetaks{k}-\bthetas}^2 \leq h^2$, and $\twonorm{\bthetaks{k} - \bthetas} \leq \hk{k}$ for $k \in [K]$. This provides a more sophisticated characterization of the task relationship compared to \cite{duan2022adaptive} and \cite{tian2022unsupervised} in this setting, where the latter assumes a stronger condition $\max_{k \in [K]}\min_{\otheta}\twonorm{\bthetaks{k} - \otheta} \leq \hmax$.

% --------------------------------------
\subsubsection{Generalized linear models}
The second example is a generalized linear model (GLM), where the data $z^{(k)}_i =(\bx^{(k)}_i, y^{(k)}_i)$ satisfies that the conditional density of $y^{(k)}_i$ given $\bx^{(k)}_i = \bx$, w.r.t.\ a proper base measure $\mu$, is proportional to
\begin{equation}
    \exp\left\{y^{(k)}_i\langle \bx^{(k)}_i, \bthetaks{k}\rangle - \varphi(\langle \bx^{(k)}_i, \bthetaks{k}\rangle)\right\},
\end{equation}
where $\varphi: \mathbb{R} \rightarrow \mathbb{R}$ is a known function. Different $\varphi$ functions induce different GLMs. Canonical examples include the linear regression model with $\varphi(u) = u^2/2$ and the logistic regression model with $\varphi(u) = \log(1 + e^u)$. We use the negative conditional log-likelihood as the loss function $\ell(\btheta,z) = -y\langle \bx,\btheta\rangle+ \varphi(\langle \btheta, \bx\rangle)$, therefore $\bthetaks{k} = \argmin_{\btheta} \mLk{k}(\btheta) = \argmin_{\btheta} \mathbb{E}[\ell(\btheta,z^{(k)})]$. 

The following lemma shows that under some mild conditions, Assumptions \ref{asmp: risk function}, \ref{asmp: subG gradient}, and \ref{asmp: lipschitz loss} are satisfied.

\begin{lemma}\label{lem: glm verify asmp}
    Suppose $\varphi''$ is positive, $R_0$ is any constant, and $\bxk{k}_i$'s are i.i.d. zero-mean sub-Gaussian vectors with $\|\bxk{k}_i\|_{\psi_2} \lesssim 1$ \footnote{For a sub-Gaussian vector $X$, its $\psi_2$-norm is defined by $\|X\|_{\psi_2}=\sup_{\bu \in \mathcal{S}^{d-1}}\inf_{t>0}\{\mathbb{E}\exp\left((\bu^\top X)^2/t^2\right)\leq 2\}$.} and $\underline{\lambda} \leq \lambdamin(\tE[\bxk{k}(\bxk{k})^\top]) \leq \lambda_{\max}(\tE[\bxk{k}(\bxk{k})^\top]) \leq \bar{\lambda}$, for $k \in [K]$, where $\underline{\lambda}$ and $\bar{\lambda}$ are some positive constants. Then we have the following conclusions:
    \begin{enumerate}[(i)]
        \item If $\sup_u \varphi''(u) \leq C$ for some constant $C > 0$, then Assumption \ref{asmp: risk function} holds with some constant $L$, and Assumption \ref{asmp: lipschitz loss} holds with some $L' \lesssim d + \log(nK)$.
        \item Assumption \ref{asmp: subG gradient} holds if either of the following conditions holds:
            \begin{enumerate}
                \item $\max_{i, k}\sup_{\btheta \in B(\bthetaks{k}; R_0)} |\varphi'(\langle \bxk{k}_i, \btheta \rangle)| \leq C_1$, and $\max_{i,k}|\yk{k}_i - \varphi'(\langle \bxk{k}_i, \bthetaks{k} \rangle)| \leq C_2$ a.s., where $C_1$ and $C_2$ are some positive constants.
                \item $\max_{i,k}\twonorm{\bxk{k}_i} \leq C_1$ a.s., and $\varphi'(\langle \bxk{k}, \btheta \rangle)$, $\yk{k}_i - \varphi'(\langle \bxk{k}_i, \bthetaks{k} \rangle)$ are sub-Gaussian variables with $\sup_{\btheta \in B(\bthetaks{k}; R_0)}\|\varphi'(\langle \bxk{k}_i, \btheta \rangle)\|_{\psi_2} \leq C_2$, $\|\yk{k}_i - \varphi'(\langle \bxk{k}_i, \bthetaks{k} \rangle)\|_{\psi_2} \leq C_3$, where $\{C_i\}_{i=1}^3$ are some positive constants.
            \end{enumerate}
    \end{enumerate}
\end{lemma}

To provide some intuition on the heterogeneity condition in Assumption \ref{asmp: task heterogeneity}, let us consider the case where the covariate distributions are the same across different tasks and $\sup_u \varphi''(u) \lesssim 1$. Then the following lemma shows that the heterogeneity condition in Assumption \ref{asmp: task heterogeneity} can be reduced to the condition on the parameters $\bthetaks{k}$'s. Specifically, if $\max_{k \in [K]}\min_{\otheta}\twonorm{\bthetaks{k} - \otheta} \leq \hmax$ holds, then Assumption \ref{asmp: task heterogeneity} is satisfied with $\hk{k}, h \lesssim \hmax$.

\begin{lemma}\label{lem: glm verify asmp heterogeneity}
    Suppose the covariate distributions are the same across different tasks and $\sup_u \varphi''(u) \leq C$ for some constant $C > 0$. Then Assumption \ref{asmp: task heterogeneity} is satisfied with $\hk{k} \leq \frac{C_1}{K}\sum_{k'=1}^K\twonorm{\bthetaks{k}-\bthetaks{k'}}$ and $h^2 \leq \frac{C_2}{K^2}\sum_{k, k'}\twonorm{\bthetaks{k}-\bthetaks{k'}}^2$, for $k \in [K]$.
\end{lemma}

Then by Corollary \ref{cor: parameter alg error}, when $T \gtrsim \log(nK)$, $\max_{k \in [K]} \twonorm{\hthetak{k}_0 - \bthetaks{k}} \vee \twonorm{\htheta_0 - \bthetas} \lesssim 1$, we have
\begin{align}
    \twonorm{\htheta_T - \btheta^*} 
    &= \tildeO\bigg(\sqrt{\frac{d}{nK}} + \epsilon\sqrt{\frac{1}{n}}  + \sqrt{\epsilon}h  + \frac{\epsilon}{\sqrt{n}}\bigg[\bigg(\frac{d}{n}\bigg)^{1/4} +  \bigg(\frac{d}{n}\bigg)^{1/2}\bigg]\bigg),\\
    \twonorm{\hthetak{k}_T - \bthetaks{k}} &= \tildeO\bigg(\min \bigg\{\sqrt{\frac{d}{nK}} + \epsilon\sqrt{\frac{1}{n}}  + \sqrt{\epsilon}h + h^{(k)}   + \frac{\epsilon}{\sqrt{n}}\bigg[\bigg(\frac{d}{n}\bigg)^{1/4} + \bigg(\frac{d}{n}\bigg)^{1/2}\bigg], \sqrt{\frac{d}{n}}\bigg\}\bigg), \quad \forall k \in S,
\end{align}
with probability at least $1-(nK)^{-Cd}- e^{-C'K\epsilon}$, where $C, C'$ are some constants. 

\section{Numerical experiments}\label{sec: experiment}

We evaluate the filtering-based robust multi-task gradient descent procedure from Section \ref{subsec: MTL}. We compare against the following benchmarks in the literature:
\begin{itemize}\setlength{\itemsep}{1pt}
     \item Global data-pooling average (Average)
     \item Single-task local training (Single-task)
     \item Coordinate-wise median (Median) \citep{yin2018byzantine}
     \item Trimmed mean \citep{yin2018byzantine}
     \item Krum \citep{blanchard2017machine}
     \item Bulyan \citep{mhamdi2018hidden}
     \item Filtering without covariance estimation (Filtering) \citep{zhu2023byzantine}
     \item Median-of-means variants of Filtering and Krum (MoM-Filtering and MoM-Krum) \citep{zhu2023byzantine}
     \item Adaptive and robust penalized multi-task estimator (ARMUL) \citep{duan2022adaptive}
     \item History-based robust aggregation (History) \citep{karimireddy2021history}
     \item Bucketing (Bucketing) \citep{karimireddy2022byzantine}
     \item Mean-regularized MTL (Mean-reg) \citep{evgeniou2004regularized}
     \item The dirty model (Dirty) \citep{jalali2010dirty,jalali2013dirty}
     \item Robust multi-task feature learning (RMTFL) \citep{gong2012robust}
     \item Robust low-rank MTL (RLRMTL) \citep{chen2011integrating,chen2012learning}
\end{itemize}

Note that Median, Trimmed mean, Krum, Bulyan, Filtering, Single-task, Average, the MoM variants, History, and Bucketing are aggregation rules that we combine with iterative gradient-based fitting, in the same spirit as Algorithm \ref{algo: robust federated gradient descent}. Most of these methods only return a global estimator; when we report local error for such methods, we evaluate the same global estimator on each uncontaminated task. Exceptions include ARMUL, Mean-reg, Dirty, RMTFL, RLRMTL, and Single-task, which directly produce task-specific estimators. ARMUL is implemented using the code provided in \cite{duan2023adaptive}. Median, Trimmed mean, Krum, Bulyan, Filtering, Single-task, Average, MoM-Filtering, MoM-Krum, History, and Bucketing are implemented using the codebase of \cite{zhu2023byzantine}. Mean-reg, Dirty, RMTFL, and RLRMTL are implemented from the MATLAB package \texttt{MALSAR} \cite{zhou2011malsar} and migrated to Python. More details of the implementation and tuning for each method can be found in Appendix \ref{subsec: implementation tuning}. Additional numerical results are summarized in Appendices \ref{subsec: simulation supp} and \ref{subsec: real data supp}. 
 
% ------------------------------------------------------
\subsection{Simulation}

We consider the following simulation setting, where the clean tasks are generated from linear regression model with
\begin{equation}
     \bthetaks{k} \sim N(\bthetas, \sigma^2 \bm{I}_d/d), \quad y_i^{(k)} = \langle \bx_i^{(k)}, \bthetaks{k}\rangle + \xi_i^{(k)},\qquad \bx_i^{(k)} \sim N(0, \bm{I}_d),
\end{equation}
where $\bthetas = 3d^{-1/2}\bm{1}_d$ and $\xi_i^{(k)} \sim N(0,1)$. Unless stated otherwise, we take $\sigma^2 = 1.5$. Contaminated tasks use shifted covariates $\bx_i^{(k)} \sim N(2\times \bm{1}_d,\bm{I}_d)$, a sign-reversed coefficient vector $-3\bthetas$, and shifted noise $\xi_i^{(k)} \sim N(1,1)$. The index set $S^c$ of contaminated tasks is randomly selected from $[K]$ with size $\epsilon K$. Note that $\bthetas$ may not be the exact global minimizer of the average risk, but it is close to the global minimizer across multiple replications and therefore serves as a reasonable ground truth for evaluating the global estimation error.

In the simulations, we report both the global estimation error $\twonorm{\htheta - \bthetas}$ and the average local estimation error $|S|^{-1}\sum_{k \in S}\twonorm{\hthetak{k} - \bthetaks{k}}$ over uncontaminated tasks, where $S$ denotes the clean-task index set. All simulation results are averages over $100$ replications. The standard deviations are much smaller for most methods, so we report only the average error to save space. Single-task is reported only for local error because it does not produce a pooled global estimator. In all numerical tables, boldface marks the smallest rounded error in each column, and italics mark the second and third smallest errors.

We consider four simulation settings, where we vary the heterogeneity level $\sigma^2$, the number of tasks $K$, the contamination level $\epsilon$, and the per-task sample size $n$. We focus in the main text on the varying-heterogeneity study, and additional results with varying $K$, $\epsilon$, and $n$ are deferred to Appendix \ref{subsec: simulation supp}.

We consider $n=d=50$, $K=40$, $\epsilon=0.2$, and varying heterogeneity variance $\sigma^2 \in  \{0, 0.5, 1, 1.5, 2,\\ 2.5, 3, 4, 5, 6, 7, 8\}$. Recall that the coefficient vector of each clean task is generated by $\bthetaks{k} \sim N(\bthetas, \sigma^2 \bm{I}_d/d)$, so larger $\sigma^2$ corresponds to more heterogeneous tasks. Tables \ref{tab:sim-variance-global} and \ref{tab:sim-variance-local} show that our method achieves the best performance or close to the best performance across all heterogeneity levels for both global and local errors. In most cases, the advantage of our method is substantial, in the sense that the gap between our method and the benchmarks is larger than twice the standard deviation of the error across replications.

\begin{table}[!ht]
\centering
\caption{Linear regression with $n=d=50$, $K=40$, $\epsilon=0.2$, and varying heterogeneity variance $\sigma^2$: global error $\twonorm{\htheta - \bthetas}$.}
\label{tab:sim-variance-global}
\scriptsize
\setlength{\tabcolsep}{2.5pt}
\begin{adjustbox}{width=\textwidth}
\begin{tabular}{lcccccccccccc}
\toprule
Method$\backslash \sigma^2$ & 0 & 0.5 & 1 & 1.5 & 2 & 2.5 & 3 & 4 & 5 & 6 & 7 & 8 \\
\midrule
Ours & \textbf{0.178} & \textbf{0.251} & \textbf{0.306} & \textbf{0.352} & \textbf{0.393} & \textbf{0.430} & \textbf{0.464} & \textbf{0.526} & \textbf{0.581} & \textbf{0.631} & \textbf{0.677} & \textbf{0.721} \\
Average & 13.061 & 13.066 & 13.067 & 13.068 & 13.068 & 13.069 & 13.070 & 13.071 & 13.073 & 13.074 & 13.076 & 13.077 \\
Median & 0.414 & 0.594 & 0.726 & 0.834 & 0.929 & 1.017 & 1.100 & 1.245 & 1.376 & 1.496 & 1.608 & 1.712 \\
Trimmed mean & 0.505 & 0.725 & 0.886 & 1.021 & 1.140 & 1.248 & 1.347 & 1.526 & 1.686 & 1.832 & 1.967 & 2.094 \\
Krum & 0.991 & 1.346 & 1.614 & 1.817 & 2.004 & 2.095 & 2.229 & 2.396 & 2.530 & 2.653 & 2.804 & 2.915 \\
Bulyan & 0.291 & 0.404 & 0.486 & 0.567 & 0.634 & 0.688 & 0.744 & 0.841 & 0.922 & 0.999 & 1.075 & 1.155 \\
Filtering & \textit{0.244} & \textit{0.343} & \textit{0.414} & \textit{0.478} & \textit{0.529} & \textit{0.576} & \textit{0.624} & \textit{0.708} & \textit{0.787} & \textit{0.845} & \textit{0.912} & \textit{0.962} \\
MoM-Filtering & \textit{0.211} & \textit{0.297} & \textit{0.363} & \textit{0.417} & \textit{0.466} & \textit{0.510} & \textit{0.548} & \textit{0.621} & \textit{0.686} & \textit{0.745} & \textit{0.801} & \textit{0.852} \\
MoM-Krum & 0.556 & 0.773 & 0.941 & 1.072 & 1.182 & 1.293 & 1.387 & 1.576 & 1.720 & 1.871 & 2.010 & 2.137 \\
ARMUL & 0.390 & 0.826 & 1.007 & 1.101 & 1.180 & 1.219 & 1.270 & 1.308 & 1.323 & 1.336 & 1.348 & 1.359 \\
History & 0.334 & 0.467 & 0.568 & 0.653 & 0.729 & 0.797 & 0.860 & 0.973 & 1.074 & 1.167 & 1.253 & 1.333 \\
Bucketing & 0.657 & 0.848 & 1.030 & 1.186 & 1.323 & 1.447 & 1.562 & 1.768 & 1.953 & 2.122 & 2.278 & 2.424 \\
Mean-reg & 3.135 & 3.114 & 3.183 & 3.187 & 3.253 & 3.262 & 3.248 & 3.120 & 3.119 & 3.083 & 3.045 & 3.084 \\
Dirty & 3.041 & 3.003 & 3.012 & 2.991 & 2.972 & 2.985 & 2.985 & 2.962 & 2.907 & 2.871 & 2.857 & 2.851 \\
RMTFL & 2.988 & 2.988 & 2.986 & 2.982 & 2.978 & 2.970 & 2.951 & 2.934 & 2.907 & 2.890 & 2.867 & 2.858 \\
RLRMTL & 3.252 & 3.258 & 3.261 & 3.248 & 3.250 & 3.265 & 3.214 & 2.948 & 2.940 & 2.943 & 2.942 & 2.941 \\
\bottomrule
\end{tabular}
\end{adjustbox}
\end{table}

\begin{table}[!ht]
\centering
\caption{Linear regression with $n=d=50$, $K=40$, $\epsilon=0.2$, and varying heterogeneity variance $\sigma^2$: average local error $|S|^{-1}\sum_{k \in S}\twonorm{\hthetak{k} - \bthetaks{k}}$.}
\label{tab:sim-variance-local}
\scriptsize
\setlength{\tabcolsep}{2.5pt}
\begin{adjustbox}{width=\textwidth}
\begin{tabular}{lcccccccccccc}
\toprule
Method$\backslash \sigma^2$ & 0 & 0.5 & 1 & 1.5 & 2 & 2.5 & 3 & 4 & 5 & 6 & 7 & 8 \\
\midrule
Ours & \textit{0.219} & \textbf{0.688} & \textbf{0.869} & \textbf{0.989} & \textbf{1.079} & \textbf{1.154} & \textbf{1.217} & \textbf{1.324} & \textbf{1.410} & \textbf{1.487} & \textbf{1.556} & \textbf{1.621} \\
Average & 13.061 & 13.082 & 13.100 & 13.118 & 13.136 & 13.154 & 13.172 & 13.209 & 13.245 & 13.281 & 13.317 & 13.353 \\
Single-task & 1.872 & 1.881 & 1.892 & 1.902 & 1.913 & 1.923 & 1.934 & 1.955 & 1.975 & 1.995 & 2.015 & 2.035 \\
Median & 0.414 & 0.902 & 1.205 & 1.443 & 1.646 & 1.828 & 1.995 & 2.289 & 2.551 & 2.789 & 3.008 & 3.212 \\
Trimmed mean & 0.505 & 0.993 & 1.306 & 1.557 & 1.772 & 1.964 & 2.139 & 2.452 & 2.729 & 2.980 & 3.212 & 3.428 \\
Krum & 0.991 & 1.508 & 1.877 & 2.165 & 2.423 & 2.590 & 2.780 & 3.071 & 3.317 & 3.543 & 3.783 & 3.985 \\
Bulyan & 0.291 & 0.792 & 1.079 & 1.309 & 1.502 & 1.671 & 1.827 & 2.101 & 2.342 & 2.562 & 2.766 & 2.957 \\
Filtering & \textit{0.244} & \textit{0.761} & \textit{1.047} & 1.270 & 1.459 & 1.625 & 1.778 & 2.048 & 2.288 & 2.501 & 2.701 & 2.883 \\
MoM-Filtering & \textbf{0.211} & \textit{0.742} & \textit{1.029} & \textit{1.251} & \textit{1.439} & \textit{1.606} & \textit{1.756} & 2.023 & 2.260 & 2.473 & 2.670 & 2.853 \\
MoM-Krum & 0.556 & 1.032 & 1.349 & 1.597 & 1.806 & 1.999 & 2.169 & 2.489 & 2.757 & 3.011 & 3.246 & 3.462 \\
ARMUL & 0.389 & 0.855 & 1.061 & \textit{1.197} & \textit{1.316} & \textit{1.397} & \textit{1.493} & \textit{1.603} & \textit{1.683} & \textit{1.758} & \textit{1.829} & \textit{1.897} \\
History & 0.334 & 0.824 & 1.116 & 1.346 & 1.543 & 1.716 & 1.874 & 2.155 & 2.404 & 2.629 & 2.836 & 3.029 \\
Bucketing & 0.657 & 1.086 & 1.408 & 1.670 & 1.896 & 2.097 & 2.281 & 2.611 & 2.903 & 3.168 & 3.412 & 3.641 \\
Mean-reg & 1.692 & 1.714 & 1.809 & 1.847 & 1.929 & 1.976 & 2.000 & 1.964 & 2.021 & 2.045 & 2.046 & 2.143 \\
Dirty & 1.703 & 1.721 & 1.760 & 1.780 & 1.813 & 1.849 & 1.869 & 1.911 & 1.940 & 1.960 & 1.991 & 2.026 \\
RMTFL & 1.631 & 1.666 & 1.695 & 1.722 & 1.749 & 1.777 & 1.802 & \textit{1.846} & \textit{1.883} & \textit{1.922} & \textit{1.953} & \textit{1.986} \\
RLRMTL & 1.103 & 1.324 & 1.506 & 1.651 & 1.796 & 1.945 & 2.065 & 2.157 & 2.271 & 2.379 & 2.475 & 2.566 \\
\bottomrule
\end{tabular}
\end{adjustbox}
\end{table}

% ------------------------------------------------------
\subsection{Real-data analysis}
For the real-data study, we consider the Human Activity Recognition (HAR) Dataset \cite{anguita2013public}, which has been used in other MTL papers \cite[e.g.,][]{duan2022adaptive,tian2022unsupervised, kim2026multi}. The data is collected from 30 volunteers when they performed six activities (walking, walking upstairs, walking downstairs, sitting, standing, and laying) wearing a smartphone. Motivated by \cite{duan2022adaptive}, we treat each subject as one task and consider a classification problem of classifying sitting against other activities. We first apply a Principal Component Analysis to reduce the dimension to 100, standardize the transformed covariates, and finally fit logistic regression classifiers with an intercept. For each task we split the subject-specific observations into training and testing sets with training proportion between $20\%$ and $60\%$. In the main text, we report the result with $20\%$ training data and defer the results of remaining training proportions to Appendix \ref{subsec: real data supp}. For a randomly selected $\epsilon$ fraction of tasks, we flip the labels and replace the feature vector $\bx$ by an affine Gaussian shift of the form $2\bx+\bm{\xi}$ with $\bm{\xi}$ generated entrywise from $N(5,1)$. Table \ref{tab:har-local-02} reports the mean clean-task local prediction error. When $\epsilon=0$, our method and Mean-reg perform similarly, but once contaminated tasks are introduced, our estimator becomes the most robust and achieves the smallest error for every nonzero contamination level shown. 

% Krum performs poorly mainly because it selects the gradient from only one task at each iteration to update the estimator, which may be ineffective in the presence of heterogeneity.

\begin{table}[!ht]
\centering
\caption{HAR local prediction error with $20\%$ training data.}
\label{tab:har-local-02}
\scriptsize
\setlength{\tabcolsep}{4pt}
\begin{adjustbox}{width=0.6\textwidth}
\begin{tabular}{lcccccc}
\toprule
Method$\backslash \epsilon$  & 0.00 & 0.05 & 0.10 & 0.15 & 0.20 & 0.25 \\
\midrule
Ours & \textit{0.021} & \textbf{0.021} & \textbf{0.022} & \textbf{0.022} & \textbf{0.023} & \textbf{0.023} \\
Average & 0.037 & 0.048 & 0.079 & 0.099 & 0.134 & 0.149 \\
Single-task & 0.058 & 0.058 & 0.058 & 0.058 & 0.058 & 0.059 \\
Median & 0.041 & 0.043 & 0.045 & \textit{0.046} & \textit{0.048} & \textit{0.050} \\
Trimmed mean & 0.037 & 0.039 & \textit{0.041} & \textit{0.043} & \textit{0.046} & \textit{0.047} \\
Krum & 0.306 & 0.306 & 0.305 & 0.305 & 0.303 & 0.303 \\
Bulyan & 0.037 & 0.042 & 0.053 & 0.057 & 0.061 & 0.062 \\
Filtering & 0.037 & \textit{0.038} & \textit{0.043} & 0.047 & 0.056 & 0.062 \\
MoM-Filtering & 0.037 & \textit{0.038} & \textit{0.043} & 0.047 & 0.056 & 0.062 \\
MoM-Krum & 0.191 & 0.191 & 0.193 & 0.193 & 0.202 & 0.206 \\
ARMUL & \textit{0.031} & 0.070 & 0.154 & 0.165 & 0.171 & 0.172 \\
History & 0.037 & 0.041 & 0.048 & 0.053 & 0.067 & 0.077 \\
Bucketing & 0.037 & 0.048 & 0.067 & 0.081 & 0.115 & 0.132 \\
Mean-reg & \textbf{0.019} & \textit{0.033} & 0.057 & 0.058 & 0.058 & 0.058 \\
Dirty & 0.046 & 0.056 & 0.055 & 0.056 & 0.055 & 0.055 \\
RMTFL & 0.052 & 0.056 & 0.056 & 0.056 & 0.056 & 0.056 \\
RLRMTL & 0.172 & 0.171 & 0.171 & 0.172 & 0.172 & 0.172 \\
\bottomrule
\end{tabular}
\end{adjustbox}
\end{table}

\section{Discussion}\label{sec: discussion}

This paper studies robust multi-task learning under the simultaneous presence of task heterogeneity and adversarial task-level contamination. In our setup, an $\epsilon$ fraction of tasks may be arbitrarily contaminated, while the remaining clean tasks are allowed to differ through the heterogeneity measures $h$ and $\hk{k}$. 
%Our goal is to estimate both the global minimizer $\bthetas$ of the average risk and the clean task-specific minimizers $\bthetaks{k}$.

Our first message is negative. In the Gaussian mean model, Section \ref{sec: negative results} shows that several widely used paradigms, including adaptive and robust regularization around a shared center, global matrix regularization, decomposition-based regularization, and score-based outlier-task detection, can all suffer a worst-case contamination error of order $\epsilon\sqrt{d/n}$. Thus, the extra $\sqrt{d}$ factor observed in earlier robust transfer and multi-task learning methods is not merely an artifact of a particular penalty or tuning choice, but reflects a broader limitation of these approaches in high dimension.

Our second message is positive. In the general contaminated multi-task ERM framework, Section \ref{sec: method and theory} establishes minimax lower bounds for estimating both $\bthetas$ and $\bthetaks{k}$, showing that the fundamental rates are
\[
\tildeOmega\bigg(\sqrt{\frac{d}{nK}} + \frac{\epsilon}{\sqrt{n}} + \sqrt{\epsilon}h\bigg)
\]
for the global parameter, and
\[
\tildeOmega\bigg(\min\bigg\{\sqrt{\frac{d}{nK}} + \frac{\epsilon}{\sqrt{n}} + \sqrt{\epsilon}h + \hk{k},\ \sqrt{\frac{d}{n}}\bigg\}\bigg)
\]
for the local parameter of task $k \in S$. Motivated by these lower bounds, we propose a filtering-based robust multi-task gradient descent method that combines robust gradient aggregation, filtering, and a simple covariance estimator constructed from single-task empirical covariances. Under local strong convexity, smoothness, and sub-Gaussian gradient assumptions, we prove high-probability upper bounds that match the minimax lower bounds up to logarithmic factors in a broad regime, thereby avoiding the dimension-dependent contamination barrier that characterizes the methods in Section \ref{sec: negative results}.

The numerical results in Section \ref{sec: experiment} support this picture. In the linear regression simulation with varying heterogeneity, our method remains highly competitive when tasks are nearly homogeneous and becomes increasingly advantageous as heterogeneity grows. In the HAR real-data analysis, it is also the most robust method once contaminated tasks are present.

There are several directions for future work. First, our gradient-estimation results suggest that filtering ideas may extend beyond the locally strongly convex setting studied here, including to broader classes of loss functions. Second, the multi-task gradient descent algorithm we analyze can be naturally extended to other contexts, such as federated differentially private learning \cite[e.g.][]{li2024federated,auddy2025minimax,hung2025optimal}, where gradients can be privatized before aggregation across tasks. Given the well-known connections between these two areas \cite[e.g.][]{liu2021robust,acharya2021robust,cheu2021manipulation,li2022robustness}, it would be interesting to understand how contamination interacts with privacy constraints and what algorithms are optimal under these constraints. More broadly, our results suggest that achieving robustness and adaptivity simultaneously may require moving beyond standard regularization paradigms toward more explicitly contamination-aware procedures.

% reference
\bibliography{reference}
\bibliographystyle{alpha}

% appendix
\newpage
\appendix

\renewcommand{\contentsname}{Appendices}
\tableofcontents
\addtocontents{toc}{\protect\setcounter{tocdepth}{5}}

\section{Technical details of Section \ref{sec: negative results}}\label{sec: negative results supp}

First, we summarize some useful lemmas here. Denote $\rho(\bx) = \min_{\bm{z}}\{\frac{1}{2}\twonorm{\bx-\bm{z}}^2 + p(\twonorm{\bm{z}})\}$. Lemma \ref{lem: prox} and \ref{lem: dev and hessian} provide explicit expressions of $\proxpl(\bx)$, $\rho(\bx)$, $\nabla \rho(\bx)$, and $\nabla \rho(\bx)$, which are very helpful for proving the lower bounds in Section \ref{subsubsec: adaptive regularization}. Lemma \ref{lem: minimax lower bound mean est} is a minimax lower bound in the classical single-task learning scenario, which can simplify the arguments in the proof of lower bounds in Section \ref{subsubsec: adaptive regularization}. Lemma \ref{lem: chi sq concentration} presents concentration bounds for the Chi-square distribution.

\begin{lemma}\label{lem: prox}
	The following conclusions hold:
	\begin{enumerate}[(i)]
		\item $\proxpl(\bx) = \bm{0}$ when $\twonorm{\bm{x}} < L$, and $\proxpl(\bx) \neq \bm{0}$ when $\twonorm{\bm{x}} > L$;
		\item When $\twonorm{\bm{x}} = L$, $\proxpl(\bx)$ may not be unique and $\bm{0} \in \proxpl(\bx)$;
		\item  $
				\proxpl(\bx) = \begin{cases}
				\bx - p'(\twonorm{\proxpl(\bx)})\cdot \frac{\proxpl(\bx)}{\twonorm{\proxpl(\bx)}}, \quad &\textup{if } \twonorm{\bx} > L, \\
				\bm{0}, \quad &\textup{if } \twonorm{\bm{x}} < L.
			\end{cases}$
	\end{enumerate}
\end{lemma}

\begin{proof}[Proof of Lemma \ref{lem: prox}]
With a fixed $\bx \in \mathbb{R}^d$, for any $\bz \in \mathbb{R}^d$, taking the difference between the objective value at $\bz$ and its value at $\bm{0}$ gives
\begin{equation}
\frac{1}{2}\twonorm{\bz-\bx}^2 + p(\twonorm{\bz}) - \frac{1}{2}\twonorm{\bx}^2  = \frac{1}{2}\twonorm{\bz}^2 + p(\twonorm{\bz}) - \bx^\top\bz.
\end{equation}
When $\bz \neq \bm{0}$ and $\twonorm{\bx} \leq L$, the definition of $L$ implies
\begin{align*}
\frac{1}{2}\twonorm{\bz}^2 + p(\twonorm{\bz})
&= \twonorm{\bz}\left\{\frac{1}{2}\twonorm{\bz} + \frac{p(\twonorm{\bz})}{\twonorm{\bz}}\right\}\\
&\geq L\twonorm{\bz} \\
&\geq \twonorm{\bx}\twonorm{\bz} \\
&\geq \bx^\top\bz .
\end{align*}
Together with the preceding display, this shows that $\bm{0} \in \proxpl(\bx)$ whenever $\twonorm{\bx} \leq L$. When $\twonorm{\bx}=L$, equality may also hold at nonzero values of $\bz$, so uniqueness is not guaranteed. If $\twonorm{\bx}<L$, then the above inequality is strict for every $\bz\neq \bm{0}$, and hence $\proxpl(\bx)=\bm{0}$ uniquely.

Now suppose $\twonorm{\bx}>L$. By the definition of the infimum in $L$, there exists $r>0$ such that
\[
\frac{1}{2}r+\frac{p(r)}{r}<\twonorm{\bx}.
\]
Taking $\bz=r\bx/\twonorm{\bx}$ in the same comparison gives
\[
\frac{1}{2}\twonorm{\bz}^2+p(\twonorm{\bz})
=r\left\{\frac{1}{2}r+\frac{p(r)}{r}\right\}
<r\twonorm{\bx}=\bx^\top\bz .
\]
Therefore the objective value at this $\bz$ is smaller than the objective value at $\bm{0}$, so $\bm{0}\notin \proxpl(\bx)$. Thus every element of $\proxpl(\bx)$ is nonzero when $\twonorm{\bx}>L$. Since $p$ is differentiable on $(0,+\infty)$, the first-order condition at any minimizer gives
\[
\proxpl(\bx)-\bx+p'(\twonorm{\proxpl(\bx)})\frac{\proxpl(\bx)}{\twonorm{\proxpl(\bx)}}=\bm{0},
\]
which is equivalent to the displayed formula in part (iii). This proves the lemma.
\end{proof}

\begin{lemma}\label{lem: dev and hessian}
	$\rho(\bx)$ and $\nabla \rho(\bx)$ are differentiable at $\bx$ where $\twonorm{\bx} \neq L$ and $p''(\twonorm{\proxpl(\bx)})$ exists, and
	\begin{enumerate}[(i)]
		\item $\nabla \rho(\bx) = \bx - \proxpl(\bx) = \begin{cases}
		p'(\twonorm{\proxpl(\bx)})\cdot \frac{\proxpl(\bx)}{\twonorm{\proxpl(\bx)}}, \quad &\textup{if } \twonorm{\bx} > L, \\
		\bx, \quad &\textup{if } \twonorm{\bx} < L.
	\end{cases}$;
		\item $\nabla^2 \rho(\bx) = 
			\frac{p'(\twonorm{\proxpl(\bx)})}{p'(\twonorm{\proxpl(\bx)}) + \twonorm{\proxpl(\bx)}}\bm{I}_d + \frac{\twonorm{\proxpl(\bx)}\cdot p''(\twonorm{\proxpl(\bx)}) - p'(\twonorm{\proxpl(\bx)})}{(p'(\twonorm{\proxpl(\bx)}) + \twonorm{\proxpl(\bx)})[1+p''(\twonorm{\proxpl(\bx)})]}\cdot \frac{\proxpl(\bx)(\proxpl(\bx))^\top}{\twonorm{\proxpl(\bx)}^2}$, if $\twonorm{\bx} > L$; and $\nabla^2 \rho(\bx) = \bm{I}_d$, if $\twonorm{\bx} < L$.
	\end{enumerate}
\end{lemma}

\begin{proof}[Proof of Lemma \ref{lem: dev and hessian}]
When $\twonorm{\bx}<L$, Lemma \ref{lem: prox} implies that $\proxpl(\bx)=\bm{0}$. Hence $\rho(\bx)=\frac{1}{2}\twonorm{\bx}^2$, so $\nabla\rho(\bx)=\bx$ and $\nabla^2\rho(\bx)=\bm{I}_d$.

It remains to consider the case $\twonorm{\bx}>L$. Let $r=\twonorm{\proxpl(\bx)}$. By Lemma \ref{lem: prox}, $r>0$, and the first-order condition is
\[
\proxpl(\bx)-\bx+p'(r)\frac{\proxpl(\bx)}{r}=\bm{0}.
\]
For $\bz\neq\bm{0}$, consider
\[
G(\bz,\by)=\bz-\by+p'(\twonorm{\bz})\frac{\bz}{\twonorm{\bz}} .
\]
At $(\bz,\by)=(\proxpl(\bx),\bx)$, we have $G(\bz,\by)=\bm{0}$, and
\begin{align*}
\frac{\partial G}{\partial \bz}\bigg|_{(\proxpl(\bx),\bx)}
&=\bm{I}_d+p''(r)\frac{\proxpl(\bx)\proxpl(\bx)^\top}{r^2}
+p'(r)\left\{\frac{1}{r}\bm{I}_d-\frac{\proxpl(\bx)\proxpl(\bx)^\top}{r^3}\right\} \\
&=\left\{1+\frac{p'(r)}{r}\right\}\bm{I}_d+\big\{r p''(r)-p'(r)\big\}\frac{\proxpl(\bx)\proxpl(\bx)^\top}{r^3}.
\end{align*}
The matrix in the last display has eigenvalues $1+p'(r)/r$ on the subspace orthogonal to $\proxpl(\bx)$ and eigenvalue $1+p''(r)$ in the direction of $\proxpl(\bx)$. The first eigenvalue is positive because $p$ is non-decreasing, and the second is positive by Assumption \ref{asmp: penalty}.(\rom{5}). Therefore the matrix is invertible. Since $\partial G/\partial\by=-\bm{I}_d$, the implicit function theorem gives that $\proxpl(\bx)$ is differentiable at $\bx$, with
\[
\frac{\partial \proxpl(\bx)}{\partial \bx}
=\left[\left\{1+\frac{p'(r)}{r}\right\}\bm{I}_d+\big\{r p''(r)-p'(r)\big\}\frac{\proxpl(\bx)\proxpl(\bx)^\top}{r^3}\right]^{-1}.
\]

Since $\rho(\bx) = \frac{1}{2}\twonorm{\bx - \proxpl(\bx)}^2 + p(\twonorm{\proxpl(\bx)})$, we have
\begin{align}
	\nabla \rho(\bx) &= (\bm{I} - \frac{\partial \proxpl(\bx)}{\partial \bx})(\bx - \proxpl(\bx)) + p'(\proxpl(\bx))\frac{\partial \proxpl(\bx)}{\partial \bx}\frac{\proxpl(\bx)}{\twonorm{\proxpl(\bx)}} \\
	&=\bx - \proxpl(\bx) + \frac{\partial \proxpl(\bx)}{\partial \bx}\cdot \left\{\proxpl(\bx) -\bx + p'(\twonorm{\proxpl(\bx)})\frac{\proxpl(\bx)}{\twonorm{\proxpl(\bx)}}\right\} \\
	&=\bx - \proxpl(\bx).
\end{align}
This implies
\begin{align}
	\nabla^2\rho(\bx) &= \bm{I}_d - \frac{\partial \proxpl(\bx)}{\partial \bx} \\
	&= \bm{I}_d - \left[\left\{1+\frac{p'(r)}{r}\right\}\bm{I}_d+\big\{r p''(r)-p'(r)\big\}\frac{\proxpl(\bx)\proxpl(\bx)^\top}{r^3}\right]^{-1} \\
	&=\frac{p'(r)}{p'(r)+r}\bm{I}_d + \frac{r p''(r)-p'(r)}{(p'(r)+r)[1+p''(r)]}\cdot \frac{\proxpl(\bx)\proxpl(\bx)^\top}{r^2},
\end{align}
where the last equality follows from the Sherman-Morrison formula. This completes the proof of the lemma.
\end{proof}

\begin{lemma}[\cite{duchi2013distance}, Corollary 5]\label{lem: minimax lower bound mean est}
	Consider a Gaussian mean estimation problem, where the data $\{\bx_i\}_{i=1}^n \overset{\textup{i.i.d.}}{\sim} N(\bthetas, \sigma^2\bm{I}_d)$. Then for any estimator $\htheta$, there exists $\btheta^* \in \mathbb{R}^d$ such that with probability at least $1/4$, $\twonorm{\htheta - \btheta^*} \geq \frac{\sigma}{4}\sqrt{\frac{d}{n}}$.
\end{lemma}

\begin{lemma}[Lemma 1 in \cite{laurent2000adaptive}]\label{lem: chi sq concentration}
	For $X \sim \chi_d^2$, we have
	\begin{equation}
		\tP(X -d \geq 2\sqrt{dx} + 2x) \leq e^{-x}, \quad \tP(X - d \leq -2\sqrt{dx}) \leq e^{-x}.
	\end{equation}
\end{lemma}

\subsection{Proof of Theorem \ref{thm: finite sample lower bound}}
It suffices to study the single-task problem where we have observations $\{\bx_i\}_{i=1}^n \overset{\textup{i.i.d.}}{\sim} N(\bthetas, \sigma^2\bm{I}_d)$ and finally replace $\sigma$ by $1/\sqrt{n}$. In the remaining proof, without special notice, we consider observations $\{\bx_i\}_{i=1}^n \overset{i.i.d.}{\sim} N(\bthetas, \sigma^2\bm{I}_d)$, an unknown set $S$, and a contamination mechanism $M$ such that $S \in \mathcal{S} = \{S \subseteq [n]: |S| \geq n(1-\epsilon)\}$, and $M \in \mathcal{M}_S = \{M: \mathcal{Z}^{\otimes n} \rightarrow \mathcal{Z}^{\otimes n} \textup{ such that } M(\bx_i) = \bx_i, i \in S\}$, and we are interested in estimating $\bthetas$ through 
    \begin{equation}\label{eq: penalization form 2}
    	\{\htheta_i\}_{i=1}^n, \hotheta \in \argmin_{\{\btheta_i\}_{i=1}^n, \otheta} \bigg\{\frac{1}{2}\sum_{i=1}^n\twonorm{\bx_i - \btheta_i}^2 + \sum_{i=1}^n p(\twonorm{\btheta_i - \otheta})\bigg\},
    \end{equation}
    where $p: [0, \infty) \rightarrow [0, \infty)$ is the regularizer which depends on $\lambda$. 
    
    Similar to the MTL scenario, \eqref{eq: penalization form 2} is equivalent to the following two-stage estimation process:
    \begin{align}
        \hotheta &\in \argmin_{\btheta}\bigg\{\sum_{i=1}^n\min_{\bm{\Delta}}\Big(\frac{1}{2}\twonorm{\bx_i - \btheta -\Delta}^2 + p(\twonorm{\Delta})\Big)\bigg\} = \argmin_{\btheta}\bigg\{\sum_{i=1}^n\rho(\bx_i - \btheta)\bigg\}, \label{eq: m-est form 2} \\
        \htheta_i &\in \argmin_{\btheta}\bigg\{\frac{1}{2}\twonorm{\bx_i - \btheta}^2 + p(\twonorm{\btheta - \hotheta})\bigg\}. \label{eq: m-est form individual 2}
    \end{align}
where $\rho(\bx) \coloneqq \min_{\bz}\big[\frac{1}{2}\twonorm{\bz-\bx}^2 + p(\twonorm{\bz})\big]$.

Regarding the lower bound of $\twonorm{\hotheta - \btheta^*}$ in the MTL setup, it suffices to prove the following theorem in the single-task learning context \eqref{eq: penalization form 2}.
\begin{theorem}\label{thm: finite sample lower bound penalty 2}
	Let Assumption \ref{asmp: penalty} hold and the contamination proportion satisfy  $\epsilon \leq 1/2$, $d \geq 5$, $\frac{1}{32}e^{d/64} \geq n \geq 96d/\epsilon^2$, $\frac{64\sqrt{2}}{\sqrt{n}}\sqrt{d+1}[\sqrt{\log(16Ce)} + \sqrt{2}(2-3e^{-1})] + 2\sqrt{\frac{\log 16+d\log 5}{n}} \leq \epsilon/4$, where $C > 0$ is some constant. 
    \begin{enumerate}[(i)]
        \item If $\frac{\epsilon\sigma\sqrt{d}}{4}\cdot \frac{1}{c_1+1} \leq L \vee L_{\infty}$, then there exist a subset $S^c \subseteq [n]$ with $|S^c|/n \leq \epsilon$ and a contamination mechanism $M \in \mathcal{M}_S$ such that with probability at least $3/16$, 
		\begin{equation}
			\twonorm{\hotheta - \btheta^*} \geq \frac{1-\tau}{(c_0 \vee 1)(4\sqrt{3}+1)}\cdot \frac{1}{72+48c_1}\epsilon\sigma\sqrt{d},
		\end{equation}
		for all stationary points $\hotheta$ of \eqref{eq: m-est form 2}.
        \item If $\frac{\epsilon\sigma\sqrt{d}}{4}\cdot \frac{1}{c_1+1} > L \vee L_{\infty}$, then there exist a subset $S^c \subseteq [n]$ with $|S^c|/n \leq \epsilon$ and a contamination mechanism $M \in \mathcal{M}_S$ such that with probability at least $3/8$, 
		\begin{equation}
			\twonorm{\hotheta - \btheta^*} \geq \frac{1-\tau}{(16c_0+6)\sqrt{2}}\cdot 2^{\frac{c_0}{1-\tau} \cdot \frac{c_1+1}{c_1}}\cdot \epsilon\sigma\sqrt{d},
		\end{equation}
		for all minimizers $\hotheta$ of \eqref{eq: m-est form 2}.
    \end{enumerate}
\end{theorem}

The following two propositions directly entail Theorem \ref{thm: finite sample lower bound penalty 2}.

\begin{proposition}\label{prop: finite sample lower bound large lambda}
	Suppose the conditions of $(n, d, \epsilon)$ stated at the beginning of Theorem \ref{thm: finite sample lower bound penalty 2} hold. And assume $\frac{\epsilon\sigma\sqrt{d}}{4}\cdot \frac{1}{c_1+1} \leq L \vee L_{\infty}$, where $C > 0$ is some constant. There exist a subset $S^c \subseteq [n]$ with $|S^c|/n \leq \epsilon$ and a contamination mechanism $M \in \mathcal{M}_S$ such that with probability at least $3/16$, $\twonorm{\hotheta - \btheta^*} \geq \frac{1-\tau}{(c_0 \vee 1)(4\sqrt{3}+1)}\cdot  \frac{1}{72+48c_1}\epsilon\sigma\sqrt{d}$ for all stationary points $\hotheta$ of \eqref{eq: m-est form 2}.
\end{proposition}

\begin{proposition}\label{prop: finite sample lower bound small lambda}
	Suppose $\frac{\epsilon\sigma\sqrt{d}}{4}\cdot \frac{1}{c_1+1} > L \vee L_{\infty}$ and the same remaining conditions in Proposition \ref{prop: finite sample lower bound large lambda} hold. There exist a subset $S^c \subseteq [n]$ with $|S^c|/n \leq \epsilon$ and a contamination mechanism $M \in \mathcal{M}_S$ such that with probability at least $3/8$, $\twonorm{\hotheta - \btheta^*} \geq \frac{1-\tau}{(16c_0+6)\sqrt{2}}\cdot 2^{-\frac{c_0}{1-\tau} \cdot \frac{c_1+1}{c_1}}\cdot \epsilon\sigma\sqrt{d}$ for all minimizers $\hotheta$ of \eqref{eq: m-est form 2}.
\end{proposition}

Regarding the lower bound of $\max_{k \in [K]}\twonorm{\hthetak{k} - \btheta^*}$ in the MTL setup, it suffices to prove the following theorem in the single-task learning context \eqref{eq: penalization form 2}.
\begin{theorem}\label{thm: finite sample lower bound supp}
	Suppose the same conditions in Proposition \ref{prop: finite sample lower bound large lambda} hold. There exists an absolute constant $C' > 0$, such that for any $\btheta^* \in \mathbb{R}^d$, any $\lambda \geq 0$, $\exists$ $S^c \subseteq [n]$ with $|S^c|/n \leq \epsilon$, and a contamination mechanism $M \in \mathcal{M}_S$, with probability at least $1/16$, we have
	\begin{align}
		&\max_{i \in S}\twonorm{\htheta_i - \btheta^*} \\
		&\geq \frac{1}{\sqrt{2}}\bigg(1-\frac{\sqrt{2}}{2}\bigg)^{1/2}\left\{\bigg[\frac{1-\tau}{(c_0 \vee 1)(4\sqrt{3}+1)}\cdot \frac{1}{72+48c_1}\bigg]\wedge \bigg[\frac{1-\tau}{(16c_0+6)\sqrt{2}}\cdot 2^{-\frac{c_0}{1-\tau} \cdot \frac{c_1+1}{c_1}}\bigg]\right\}\epsilon\sigma\sqrt{d},
	\end{align}
	where $\{\htheta_i\}_{i=1}^n$ are the estimators defined in \eqref{eq: m-est form individual 2}. 
\end{theorem}

We now present three key supporting proofs: those of Proposition \ref{prop: finite sample lower bound large lambda}, \ref{prop: finite sample lower bound small lambda}, and Theorem \ref{thm: finite sample lower bound supp}.

\begin{proof}[Proof of Proposition \ref{prop: finite sample lower bound large lambda}]
WLOG, assume $\btheta^* = \bm{0}$ and write the empirical distribution  $n^{-1}\sum_{i=1}^n\delta_{\tilde{\bx}_i}$ as $\tP_{\epsilon, n}$. Denote $G(\btheta) = \tP_{\epsilon, n}\rho(\bx - \btheta)$. Consider the contamination mechanism $\mathcal{M}$ such that $\mathcal{M}(\bx_i) = \tilde{\bx}_i = \bx_0$ for $i \in S^c$, where $\bx_0$ will be specified later. Denote $n_1 = |S^c|$ and $n_0 = |S|$. WLOG, let us fix an $S^c$ with $|S^c|/n = \epsilon$. 

The proof is divided into two cases. 

	\noindent \textbf{Case \Rom{1}: $\rho$ is differentiable}.
	
	We first consider the case that $\rho$ is differentiable, then we extend the proof to the non-differentiable case. Note that the differentiability of $\rho$ is equivalent to $\twonorm{\proxpl(\bx)} \rightarrow 0$ when $\twonorm{\bx} \rightarrow L$, as we mentioned in Remark \ref{rmk: rho differentiable}.
	
	When $\frac{\sigma}{4}\sqrt{\frac{d}{n}} > \frac{\sigma\sqrt{d}\epsilon}{2}$, if $\tilde{\bx}_i = \bx_i$ for all $i \in [n]$, i.e. there is no outlier observation and $\{\bx_i\}_{i=1}^n \overset{i.i.d.}{\sim} \tP_{\btheta^*}$, the lower bound $\frac{\sigma}{4}\sqrt{\frac{d}{n}}$ holds with probability at least $1/4$ by Lemma \ref{lem: minimax lower bound mean est}, hence Proposition \ref{prop: finite sample lower bound large lambda} automatically holds. Therefore, in the remaining part of this proof, we assume $\frac{\sigma}{4}\sqrt{\frac{d}{n}} \leq \frac{\sigma\sqrt{d}\epsilon}{2}$, i.e. $\frac{1}{2\sqrt{n}} \leq \epsilon $. 
	
	By Taylor expansion and the definition of $\htheta$,
	\begin{equation}
		\bm{0} = \nabla G(\htheta) = \nabla G(\btheta^*) + \underbrace{\bigg[\int_0^1 \tP_{\epsilon, n}\nabla^2 \rho(\bx - \btheta^* - \delta(\htheta-\btheta^*))d\delta\bigg]}_{M} (\htheta - \btheta^*),
	\end{equation}
	which implies
	\begin{equation}\label{eq: theta bound initial}
		\twonorm{M(\htheta - \btheta^*)} = \twonorm{\nabla G(\btheta^*)} \geq \frac{n_1}{n}\twonorm{\nabla \rho(\bx_0)} - \frac{n_0}{n}\twonorm{\tP_{n_0}\nabla \rho(\bx)}.
	\end{equation}
	This entails that
	\begin{equation}\label{eq: lower bound main}
		\max\{|\lambdamax(M)|, |\lambdamin(M)|\}\twonorm{\htheta - \bthetas} \geq \epsilon\twonorm{\nabla \rho(\bx_0)} - (1-\epsilon)\twonorm{\tP_{n_0}\nabla \rho(\bx)}.
	\end{equation}
    Next we will provide an upper bound for $\max\{|\lambdamax(M)|, |\lambdamin(M)|\}$. Note that using Lemma \ref{lem: dev and hessian} we have that %{\color{red}[M: in the next display I think you use $\btheta^*=0$ for one of the two $\btheta^*$...why not do it all the time until we really want to make $\btheta^*$ appear for $\|\hat\btheta-\btheta^*\|$?]} \yt{I guess I was trying to simplify the $\bx_i - \bthetas$ to $\bx_i$ by assuming $\bthetas$ as need to use the concentration inequality on it. But for $\htheta - \bthetas$ I wrote it as it was without dropping $\bthetas$. I think it doesn't matter?}{\color{red}[M: it doesn't matter but being consistent makes it more readable]}
    \begin{align}
        M &= \frac{1}{n}\int_0^1 \sum_{i=1}^{n} \Bigg[\frac{p'(\twonorm{\prox_{p}(\tilde{\bx}_i + \delta(\htheta-\btheta^*))})}{p'(\twonorm{\prox_{p}(\tilde{\bx}_i + \delta(\htheta-\btheta^*))}) + \twonorm{\prox_{p}(\tilde{\bx}_i + \delta(\htheta-\btheta^*))}}\bm{I}_d \\
		&\quad\quad + \frac{\twonorm{\proxpl(\tilde{\bx}_i + \delta(\htheta-\btheta^*))}\cdot p''(\twonorm{\proxpl(\tilde{\bx}_i + \delta(\htheta-\btheta^*))}) - p'(\twonorm{\proxpl(\tilde{\bx}_i + \delta(\htheta-\btheta^*))})}{[p'(\twonorm{\proxpl(\tilde{\bx}_i + \delta(\htheta-\btheta^*))}) + \twonorm{\proxpl(\tilde{\bx}_i + \delta(\htheta-\btheta^*))}][1+p''(\twonorm{\proxpl(\tilde{\bx}_i + \delta(\htheta-\btheta^*))})]}\\
        &\qquad\quad \cdot \frac{\proxpl(\tilde{\bx}_i + \delta(\htheta-\btheta^*))(\proxpl(\tilde{\bx}_i + \delta(\htheta-\btheta^*)))^\top}{\twonorm{\proxpl(\tilde{\bx}_i + \delta(\htheta-\btheta^*))}^2}\bigg]  \cdot \mathds{1}(\twonorm{\tilde{\bx}_i + \delta(\htheta-\btheta^*)} > L) \textup{d}\delta \\
        &\qquad\quad + \frac{1}{n}\int_0^1 \sum_{i=1}^n \mathds{1}(\twonorm{\tilde{\bx}_i + \delta(\htheta-\btheta^*)} \leq L) \bm{I}_d\textup{d}\delta.
        \label{eq: M}
    \end{align}
    Also, by Lemma \ref{lem: prox}.(\rom{3}), we have $(1 + p'(\twonorm{\proxpl(\bx)})/\twonorm{\proxpl(\bx)})\proxpl(\bx) = \bx$, implying that $\frac{\proxpl(\bx)(\proxpl(\bx))^\top}{\twonorm{\proxpl(\bx)}^2} = \frac{\bx\bx^\top}{\twonorm{\bx}^2}$ for any $\bx \neq \bm{0}$. 
	By Assumption \ref{asmp: penalty}.(\rom{3}), for those $x$ where $p''(x)$ exists, because $\frac{p'(x)}{x}$ is non-increasing, we must have $\frac{\textup{d}}{\textup{d}x}(\frac{p'(x)}{x}) = \frac{xp''(x) - p'(x)}{x^2} \leq 0$. In addition, $p$ is non-decreasing so $p'(x) \geq 0$ for any $x > 0$. Therefore, %{\color{red}[M: I found it a bit confusing that you find these upper and lower bounds on $\lambdamax(M)$ and $\lambdamin(M)$ here and then proceed to given other bounds without commenting on why the other ones or these ones were needed]} \yt{I moved the inequality involving $\lambdamin(M)$ and $\lambdamax(M)$ to earlier. I hope now it is clearer.}
    %{\color{red}[M: it doesn't matter but being consistent makes it much more readable]}
	\begin{align}
		\lambdamax(M) &\leq \frac{1}{n}\int_0^1 \sum_{i=1}^n \frac{p'(\twonorm{\prox_{p}(\tilde{\bx}_i + \delta(\htheta-\btheta^*))})\cdot \mathds{1}(\twonorm{\tilde{\bx}_i + \delta(\htheta-\btheta^*)} > L) }{p'(\twonorm{\prox_{p}(\tilde{\bx}_i + \delta(\htheta-\btheta^*))}) + \twonorm{\prox_{p}(\tilde{\bx}_i + \delta(\htheta-\btheta^*))}}\textup{d}\delta \\
		&\quad+ \frac{1}{n}\int_0^1\sum_{i=1}^n \mathds{1}(\twonorm{\tilde{\bx}_i + \delta(\htheta-\btheta^*)} \leq L)\textup{d}\delta \\
        & \leq \frac{1}{n}\int_0^1 \sum_{i=1}^n \frac{p'(\twonorm{\prox_{p}(\tilde{\bx}_i + \delta(\htheta-\btheta^*))})}{p'(\twonorm{\prox_{p}(\tilde{\bx}_i + \delta(\htheta-\btheta^*))}) + \twonorm{\prox_{p}(\tilde{\bx}_i + \delta(\htheta-\btheta^*))}}\textup{d}\delta \\
		&\leq 1.
	\end{align}

	On the other hand, because $p''_{\lambda}(\twonorm{\proxpl(\bx)}) \geq -\tau > -1$ when $p''_{\lambda}(\twonorm{\proxpl(\bx)})$ exists, a direct calculation implies that
	\begin{align}
		\lambdamin(M) &\geq \frac{1}{n}\int_0^1 \sum_{i=1}^n \frac{p''(\twonorm{\prox_{p}(\tilde{\bx}_i + \delta(\htheta-\btheta^*))})\cdot \mathds{1}(\twonorm{\tilde{\bx}_i + \delta(\htheta-\btheta^*)} > L) }{p''(\twonorm{\prox_{p}(\tilde{\bx}_i + \delta(\htheta-\btheta^*))}) + 1}\textup{d}\delta \\
		&\geq -\frac{\tau}{1-\tau}.
	\end{align}
	
Remember we require  $L \vee L_{\infty} \geq \epsilon\cdot \frac{\sigma\sqrt{d}}{4}\cdot \frac{1}{c_1+1}$. We will now consider the following two cases separately: 	$L \vee L_{\infty} \geq \frac{\sigma\sqrt{d}}{4}\cdot \frac{1}{c_1+1}$  and $\frac{\epsilon\sigma\sqrt{d}}{4}\cdot \frac{1}{c_1+1}\leq L \vee L_{\infty} \leq \frac{\sigma\sqrt{d}}{4}\cdot \frac{1}{c_1+1}$.

\noindent (\rom{1}) When $L \vee L_{\infty} \geq \frac{\sigma\sqrt{d}}{4}\cdot \frac{1}{c_1+1}$: 
	
	This implies 
	\begin{equation}
		\twonorm{\htheta - \bthetas} \geq \bigg(\frac{\tau}{1-\tau}\vee 1\bigg)^{-1}[\epsilon\twonorm{\nabla \rho(\bx_0)} - (1-\epsilon)\twonorm{\tP_{n_0}\nabla \rho(\bx)}].
	\end{equation}
	By standard symmetrization arguments, $\twonorm{\tP_{n_0}\nabla \rho(\bx)} \leq 2\max_{j=1:N}|\tP_{n_0}\<\nabla \rho(\bx_i), \bu_j\>|$, where $\{\bu_j\}_{j=1}^N$ is a $1/2$-cover of $B(\bm{0}, 1)$ under $\ell_2$-norm and $N \leq 5^d$. Then by Hoeffding's inequality and the union bound,
	\begin{equation}
		\tP(\max_{j=1:N}|\tP_{n_0}\<\nabla \rho(\bx_i), \bu_j\>| > t) \leq 2N\exp\bigg\{-\frac{n_0t^2}{8(L\vee L_{\infty})^2}\bigg\}.
	\end{equation}
	Therefore, if we define the event
	\begin{equation}\label{eq: event A}
		\mathcal{A} = \bigg\{\twonorm{\tP_{n_0}\nabla \rho(\bx)} \leq 2\sqrt{2}\cdot \frac{\sqrt{3d}}{\sqrt{n_0}}(L \vee L_{\infty})\bigg\},
	\end{equation}
	we must have $\tP(\mathcal{A}^c) \leq 2\times 5^d \exp\{-3d\} \leq 1/4$ when $d \geq 2$.
	If $L \leq L_{\infty}$, by taking $\bx_0$ with $\bx_0 \in \argmax_{\twonorm{\bx} > L} p'(\twonorm{\proxpl(\bx)})$, we have $\twonorm{\nabla \rho(\bx_0)} = L_{\infty}$. If $L > L_{\infty}$, by taking $\bx_0$ with $\twonorm{\bx_0} = L$, we have $\twonorm{\nabla \rho(\bx_0)} = L$. Therefore, on the event $\mathcal{A}$, we have
	\begin{align}
		\twonorm{\htheta - \bthetas} &\geq \bigg(\frac{\tau}{1-\tau}\vee 1\bigg)^{-1}\bigg[\epsilon (L\vee L_{\infty}) - (1-\epsilon)\cdot 2\sqrt{2}\cdot \frac{\sqrt{3d}}{\sqrt{n_0}}(L \vee L_{\infty})\bigg]\\
		&\geq \bigg(\frac{\tau}{1-\tau}\vee 1\bigg)^{-1}\bigg[\epsilon (L\vee L_{\infty}) - \sqrt{1-\epsilon}\cdot 2\sqrt{2}\cdot \frac{\sqrt{3d}}{\sqrt{n}}(L \vee L_{\infty})\bigg]\\
		&\geq \bigg(\frac{\tau}{1-\tau}\vee 1\bigg)^{-1}\cdot \frac{1}{2}\epsilon (L \vee L_{\infty})\\
		&\geq (1-\tau)\epsilon\frac{\sigma\sqrt{d}}{8}\cdot \frac{1}{c_1+1},
	\end{align}
	when $n \geq 96d/\epsilon^2$ %{\color{red}[M: I'm not sure of how you are using this]}. \yt{Now it be clearer. We used it to claim $\epsilon (L\vee L_{\infty}) - \sqrt{1-\epsilon}\cdot 2\sqrt{2}\cdot \frac{\sqrt{3d}}{\sqrt{n}}(L \vee L_{\infty}) \geq \frac{1}{2}\epsilon (L\vee L_{\infty})$}
    
 	\noindent (\rom{2}) When $\frac{\epsilon\sigma\sqrt{d}}{4}\cdot \frac{1}{c_1+1}\leq L \vee L_{\infty} \leq \frac{\sigma\sqrt{d}}{4}\cdot \frac{1}{c_1+1}$: 

	We define event 
	\begin{equation}\label{eq: A1}
		\mathcal{A}_1 = \bigcap_{i=1}^{n_0}\{|\twonorm{\bx_i}^2/\sigma^2 - d| \leq d/4+d/16\leq d/2\}.
	\end{equation}
	By Lemma \ref{lem: chi sq concentration} and the union bound, we have
     \begin{equation}
    		\tP(\mathcal{A}_1^c) \leq n\cdot 2\exp\{-d/64\} \leq \frac{1}{16}, 
     \end{equation}
     when $n \leq \frac{1}{32}e^{d/64}$. 

	Under $\mathcal{A}_1$, if $\twonorm{\htheta - \bthetas} \leq \frac{\sigma\sqrt{d}}{2\sqrt{2}}$, we have $\twonorm{\bx_i + \delta(\htheta - \bthetas)}\geq \frac{\sigma\sqrt{d}}{\sqrt{2}} - \twonorm{\htheta - \bthetas} \geq \frac{\sigma\sqrt{d}}{2\sqrt{2}} \geq (c_1 + 1)(L \vee L_{\infty})$. This entails that
	\begin{equation}
		\twonorm{\proxpl(\bx_i + \delta(\htheta - \bthetas))} = \twonorm{\bx_i + \delta(\htheta - \bthetas)} - p'(\twonorm{\proxpl(\bx_i + \delta(\htheta - \bthetas))}) \geq c_1(L \vee L_{\infty}).
	\end{equation}
	And $\twonorm{\bx_i + \delta(\htheta - \bthetas)} \leq \sqrt{\frac{3}{2}}\sigma\sqrt{d} + \twonorm{\htheta - \bthetas} \leq \Big(\sqrt{\frac{3}{2}}+\frac{1}{4\sqrt{2}}\Big)\sigma\sqrt{d}$, which implies that %{\color{red}[M: the first inequality below is not super clear to me]} \yt{Added one more step. Hopefully it is clearer now.}
	\begin{align}
		\twonorm{\proxpl(\bx_i + \delta(\htheta - \bthetas))} 
		&= \twonorm{\bx_i + \delta(\htheta - \bthetas)} - p'(\twonorm{\proxpl(\bx_i + \delta(\htheta - \bthetas))})\\ 
		&\geq  \frac{\sigma\sqrt{d}}{\sqrt{2}} - \frac{\sigma\sqrt{d}}{4\sqrt{2}} - L \vee L_{\infty}\\
		&\geq \frac{1}{4\sqrt{2}}\sigma\sqrt{d} \label{eq: prox lower bdd} \\
		&\geq \frac{1}{4\sqrt{3}+1}\twonorm{\bx_i + \delta(\htheta - \bthetas)},
	\end{align}
	for all $i = [n_0]$.
	Therefore, on the event $\mathcal{A}_1$, using Assumption \ref{asmp: penalty} $(v)-(vi)$, we have %{\color{red}[M: I'm not seeing the second inequality easily]} \yt{This is a trivial bound where the added first term is always negative! I added it to make the upper bound of $\lambdamax(M)$ and the lower bound of $\lambdamin(M)$ involve the same $\widetilde{U}$ term, so that we only need to prove an upper bound of $\widetilde{U}$ in the following.}
	\begin{equation}
		\resizebox{\linewidth}{!}{$\begin{aligned}
		&\lambdamin(M) \\
		&\geq \frac{1}{n}\int_0^1 \sum_{i=1}^n \frac{p''(\twonorm{\prox_{p}(\tilde{\bx}_i + \delta(\htheta-\btheta^*))})}{p''(\twonorm{\prox_{p}(\tilde{\bx}_i + \delta(\htheta-\btheta^*))}) + 1}\cdot \mathds{1}(\twonorm{\tilde{\bx}_i + \delta(\htheta-\btheta^*)} > L) \textup{d}\delta  \\
        &\geq -\frac{c_0}{1-\tau}(1-\epsilon)\cdot \frac{1}{n_0}\int_0^1 \sum_{i=1}^{n_0} \frac{p'(\twonorm{\prox_{p}(\tilde{\bx}_i + \delta(\htheta-\btheta^*))})}{\twonorm{\prox_{p}(\tilde{\bx}_i + \delta(\htheta-\btheta^*))}} \cdot \mathds{1}(\twonorm{\tilde{\bx}_i + \delta(\htheta-\btheta^*)} > L) \textup{d}\delta \\
	   &\quad -\frac{1}{n}\int_0^1 \sum_{i=n_0+1}^n \frac{p''(\twonorm{\prox_{p}(\tilde{\bx}_i + \delta(\htheta-\btheta^*))})}{p''(\twonorm{\prox_{p}(\tilde{\bx}_i + \delta(\htheta-\btheta^*))}) + 1}\cdot \mathds{1}(\twonorm{\tilde{\bx}_i + \delta(\htheta-\btheta^*)} > L) \textup{d}\delta \\
        &\geq -\frac{c_0}{1-\tau}(1-\epsilon)\cdot (4\sqrt{3}+1)\cdot  \frac{1}{n_0}\int_0^1 \sum_{i=1}^{n_0} \frac{p'(\twonorm{\prox_{p}(\tilde{\bx}_i + \delta(\htheta-\btheta^*))})}{\twonorm{\tilde{\bx}_i + \delta(\htheta-\btheta^*)}} \cdot \mathds{1}(\twonorm{\tilde{\bx}_i + \delta(\htheta-\btheta^*)} > L) \textup{d}\delta \\
        &\quad -\frac{1}{n}\int_0^1 \sum_{i=n_0+1}^n \frac{p''(\twonorm{\prox_{p}(\tilde{\bx}_i + \delta(\htheta-\btheta^*))})}{p''(\twonorm{\prox_{p}(\tilde{\bx}_i + \delta(\htheta-\btheta^*))}) + 1}\cdot \mathds{1}(\twonorm{\tilde{\bx}_i + \delta(\htheta-\btheta^*)} > L) \textup{d}\delta \\
	   &\geq -\frac{c_0}{1-\tau}(1-\epsilon)\cdot (4\sqrt{3}+1)\cdot\widetilde{U} - \epsilon\frac{\tau}{1-\tau},
	\end{aligned}$}
	\end{equation}
	where %{\color{red}[M: not obvious where $\tilde{U}$ is coming from]} \yt{It actually first poped up in the upper bound $\lambdamax(M)$ and I tried to use it to control $\lambdamin(M)$, so that I only need to analyze $\widetilde{U}$ in the following.}
    \begin{align}
    \widetilde{U} \coloneqq &\frac{1}{n_0}\int_0^1\sum_{i=1}^{n_0}\frac{p'(\twonorm{\prox_{p}(\bx_i + \delta(\htheta-\btheta^*))})}{p'(\twonorm{\prox_{p}(\bx_i + \delta(\htheta-\btheta^*))}) + \twonorm{\prox_{p}(\bx_i + \delta(\htheta-\btheta^*))}}\cdot \mathds{1}(\twonorm{\bx_i + \delta(\htheta-\btheta^*)} > L) \textup{d}\delta\\
    & \qquad+ \frac{1}{n_0}\int_0^1\sum_{i=1}^{n_0} \mathds{1}(\twonorm{\bx_i + \delta(\htheta-\btheta^*)} \leq L)\textup{d}\delta.
    \end{align}
    Note that the same $\widetilde{U}$ can be used to control $\lambdamax(M)$. More specifically,
	\begin{equation}
	\resizebox{\textwidth}{!}{$\begin{aligned}
		\lambdamax(M) &\leq \frac{1}{n}\int_0^1 \sum_{i=1}^n \frac{p'(\twonorm{\prox_{p}(\tilde{\bx}_i + \delta(\htheta-\btheta^*))})}{p'(\twonorm{\prox_{p}(\tilde{\bx}_i + \delta(\htheta-\btheta^*))}) + \twonorm{\prox_{p}(\tilde{\bx}_i + \delta(\htheta-\btheta^*))}}\cdot \mathds{1}(\twonorm{\tilde{\bx}_i + \delta(\htheta-\btheta^*)} > L) \textup{d}\delta \\
		&\quad+ \frac{1}{n}\int_0^1\sum_{i=1}^n \mathds{1}(\twonorm{\tilde{\bx}_i + \delta(\htheta-\btheta^*)} \leq L)\textup{d}\delta \\
		&\leq (1-\epsilon)\widetilde{U} + \epsilon,
	\end{aligned}$}
	\end{equation}

By \eqref{eq: lower bound main}:
\begin{equation}
	\bigg[\bigg(\frac{c_0(4\sqrt{3}+1)}{1-\tau}\vee 1\bigg)\widetilde{U}(1-\epsilon)+\bigg(\frac{\tau}{1-\tau}\vee 1\bigg)\epsilon\bigg]\twonorm{\htheta - \btheta^*} \geq \epsilon \twonorm{\nabla \rho(\bx_0)} - (1-\epsilon)\twonorm{\tP_{n_0}\nabla\rho(\bx)},
\end{equation}
which implies
\begin{equation}
\label{eq: lb_prop3}
	\twonorm{\htheta - \btheta^*} \geq \frac{1-\tau}{(c_0 \vee 1)(4\sqrt{3}+1)}\cdot \frac{\epsilon \twonorm{\nabla \rho(\bx_0)} - (1-\epsilon)\twonorm{\tP_{n_0}\nabla\rho(\bx)}}{\widetilde{U} + \epsilon}.
\end{equation}
When $\twonorm{\htheta - \btheta^*} \leq \frac{1}{2}\sigma\epsilon\sqrt{d} \leq \frac{\sigma}{4}\sqrt{d}$, with probability at least $1-\delta$:
\begin{align}
	\widetilde{U} &\leq \max_{\twonorm{\bu} \leq \sigma\sqrt{d}\epsilon/2}\Bigg\{\frac{1}{n_0}\sum_{i=1}^{n_0}\frac{p'(\twonorm{\proxpl(\bx_i+\bu)})}{\twonorm{\bx_i+\bu}}\mathds{1}(\twonorm{\bx_i+\bu} > L)  + \frac{1}{n_0}\sum_{i=1}^{n_0}\mathds{1}(\twonorm{\bx_i+\bu} \leq L)\Bigg\} \\
	&\leq \tE\max_{\twonorm{\bu} \leq \sigma\sqrt{d}\epsilon/2}\bigg\{\frac{p'(\twonorm{\proxpl(\bx+\bu)})}{\twonorm{\bx+\bu}}\mathds{1}(\twonorm{\bx+\bu} > L)+ \mathds{1}(\twonorm{\bx+\bu} \leq L)\bigg\} + \sqrt{\frac{\log (1/\delta)}{2n}} \\
	&\quad + (1-\epsilon)\tE \max_{\twonorm{\bu} \leq \sigma\sqrt{d}\epsilon/2}\bigg\{\frac{1}{n_0}\sum_{i=1}^{n_0}\epsilon_i\bigg[\frac{p'(\twonorm{\proxpl(\bx_i+\bu)})}{\twonorm{\bx_i+\bu}}\mathds{1}(\twonorm{\bx_i+\bu} > L)+ \mathds{1}(\twonorm{\bx_i+\bu} \leq L)\bigg]\bigg\}  \\
	&\leq 2(1-\epsilon)\tE\max_{\twonorm{\bu} \leq \sigma\sqrt{d}\epsilon/2}\bigg\{\frac{p'(\twonorm{\proxpl(\bx+\bu)})}{\twonorm{\bx+\bu}}\mathds{1}(\twonorm{\bx+\bu} > L)+ \mathds{1}(\twonorm{\bx+\bu} \leq L)\bigg\} \\ 
	&\quad + \sqrt{\frac{\log (1/\delta)}{2n}},  \label{eq: the first bound of U}
\end{align}
where the second inequality holds due to the bounded difference inequality and the standard symmetrization argument. Indeed, letting $\{\epsilon_i\}_{i=1}^n$ denote the i.i.d. Rademacher variables and $W_i = \max_{\twonorm{\bu} \leq \sigma\sqrt{d}\epsilon/2}\Big\{\frac{1}{n_0}\sum_{i=1}^{n_0}\frac{p'(\twonorm{\proxpl(\bx_i+\bu)})}{\twonorm{\bx_i+\bu}}\mathds{1}(\twonorm{\bx_i+\bu} > L)+ \mathds{1}(\twonorm{\bx_i+\bu} \leq L)\Big\}$, we see that
\begin{align}
	W_i &\leq |W_i - \tE W_i| + \tE W_i \\
	&\leq \tE |W_i - \tE W_i| + \sqrt{\frac{\log(1/\delta)}{2n_0}} + \tE W_i \\
	&\leq \tE\max_{\twonorm{\bu} \leq \sigma\sqrt{d}\epsilon/2}\bigg|\frac{1}{n_0}\sum_{i=1}^{n_0}\epsilon_i\bigg[\frac{p'(\twonorm{\proxpl(\bx_i+\bu)})}{\twonorm{\bx_i+\bu}}\mathds{1}(\twonorm{\bx_i+\bu} > L)+ \mathds{1}(\twonorm{\bx_i+\bu} \leq L)\bigg]\bigg| \\
	&\quad + \sqrt{\frac{\log(1/\delta)}{2n_0}} + \tE W_i \\
	&\leq 2\tE\max_{\twonorm{\bu} \leq \sigma\sqrt{d}\epsilon/2}\bigg\{\frac{p'(\twonorm{\proxpl(\bx+\bu)})}{\twonorm{\bx+\bu}}\mathds{1}(\twonorm{\bx+\bu} > L)+ \mathds{1}(\twonorm{\bx+\bu} \leq L)\bigg\} + \sqrt{\frac{\log(1/\delta)}{2n_0}},
\end{align}
with probability at least $1-\delta$. 

Defining an iid copy of $W_i$ by $W$ and event
\begin{equation}\label{eq: A2}
	\mathcal{A}_2 = \{|W - \tE W| \leq \tE |W - \tE W|+ \sqrt{\frac{\log(1/\delta)}{2n}}\},
\end{equation}
it is easy to see that we have
\begin{equation}
    \mathbb{P}(\mathcal{A}_2) =\mathbb{P}\left(|W - \tE W| \leq \tE |W - \tE W|+ \sqrt{\frac{\log(1/\delta)}{2n}}\right)\geq 1-\delta.
\end{equation}
%$\mathcal{A}_1 = \{|W - \tE W| \leq \tE |W - \tE W|+ \sqrt{\frac{\log(1/\delta)}{2n}}\}$. Then as analyzed above, $\tP(\mathcal{A}_1^c) \leq \delta$. 

Therefore, going back to the upper bound of $\widetilde{U}$, under event $\mathcal{A}_1 \cap \mathcal{A}_2$, since $\epsilon \leq 1/2$,  
\begin{equation}
\resizebox{\textwidth}{!}{$\begin{aligned}
	\widetilde{U} &\leq 2\tE\max_{\twonorm{\bu} \leq \sigma\sqrt{d}\epsilon/2}\bigg\{\bigg[\frac{p'(\twonorm{\proxpl(\bx+\bu)})}{\twonorm{\bx+\bu}}\mathds{1}(\twonorm{\bx+\bu} > L)+ \mathds{1}(\twonorm{\bx+\bu} \leq L)\bigg]\cdot \mathds{1}(\twonorm{\bx} \geq \sigma\sqrt{d}\epsilon)\bigg\} \\
	&\quad +2\tP(\twonorm{\bx} \leq \sigma\sqrt{d}\epsilon) + \sqrt{\frac{\log (1/\delta)}{2n}} \\
	&\leq 2\tE\max_{\twonorm{\bu} \leq \sigma\sqrt{d}\epsilon/2}\bigg\{\bigg[\frac{L_{\infty}}{\twonorm{\bx}-\sigma\sqrt{d}\epsilon/2}\mathds{1}(\twonorm{\bx}+\sigma\sqrt{d}\epsilon/2 > L)+ \mathds{1}(\twonorm{\bx} -  \sigma\sqrt{d}\epsilon/2 \leq L)\bigg] \mathds{1}(\twonorm{\bx} \geq \sigma\sqrt{d}\epsilon)\bigg\} \\
	&\quad +2\tP(\twonorm{\bx} \leq \sigma\sqrt{d}\epsilon) + \sqrt{\frac{\log (1/\delta)}{2n}} \\
	&\leq \tE\bigg\{\bigg[\frac{4L_{\infty}}{\twonorm{\bx}}+ 2\mathds{1}(\twonorm{\bx}\leq L + \sigma\sqrt{d}\epsilon/2)\bigg]\cdot \mathds{1}(\twonorm{\bx} \geq \sigma\sqrt{d}\epsilon)\bigg\}+ 2\tP(\twonorm{\bx} \leq \sigma\sqrt{d}\epsilon) + \sqrt{\frac{\log (1/\delta)}{2n}}  \\
	&\leq 4L_{\infty}\tE\bigg(\frac{1}{\twonorm{\bx}}\bigg) + 2\tP(\twonorm{\bx}\leq L + \sigma\sqrt{d}\epsilon/2) + 2\tP(\twonorm{\bx} \leq \sigma\sqrt{d}\epsilon) + \sqrt{\frac{\log (1/\delta)}{2n}}   \\
	&= 4L_{\infty}\frac{\Gamma((d-1)/2)}{\sqrt{2}\Gamma(d/2)}+ 2\tP(\twonorm{\bx}/\sigma \leq L/\sigma + \sqrt{d}\epsilon/2) + 2\tP(\twonorm{\bx}/\sigma \leq \sqrt{d}\epsilon) + \sqrt{\frac{\log (1/\delta)}{2n}}  \\
	&\leq \frac{4\sqrt{2}L_{\infty}}{\sigma\sqrt{d}} + 2\tP(\twonorm{\bx}/\sigma \leq L/\sigma + \sqrt{d}\epsilon/2) + 2\tP(\twonorm{\bx}/\sigma \leq \sqrt{d}\epsilon) + \sqrt{\frac{\log (1/\delta)}{2n}}. \label{eq: the second bound of U}
	\end{aligned}$}
	\end{equation}
	The last  inequality used that $d>2$.
\begin{enumerate}[(a)]
	\item When $L \leq L_{\infty}$:  
	Note that the density function $f(x)$ of $\chi^2_d$-distribution ($d\geq 1$) is given by $f(x) = \frac{1}{2^{d/2}\Gamma(d/2)}x^{d/2-1}e^{-x/2}$, which is increasing on $[0, d-2]$ and decreasing on $[d-2, +\infty)$. Therefore, since $L \leq \frac{\sigma}{4}\sqrt{d}$ implies that $L/\sigma + \sqrt{d}\epsilon/2 \leq \sqrt{d}/4 + \sqrt{d}/4 \leq d-2$ when $d \geq 4$, it can be shown that 
	\begin{align}
		2\tP(\twonorm{\bx}/\sigma \leq L/\sigma + \sqrt{d}\epsilon/2) &\leq 2\cdot \frac{(L/\sigma+\sqrt{d}\epsilon)^{d-1}}{2^{d/2}\Gamma(d/2)}\cdot (L/\sigma+\sqrt{d}\epsilon/2) \\
		&\leq \frac{(\sqrt{d}/2)^{d-1}}{2^{d/2}\Gamma(d/2)}\cdot 2\sqrt{d}\cdot \bigg(\frac{L_{\infty}}{2\sqrt{d}\sigma} + \frac{\epsilon}{4}\bigg) \\
		&\leq \frac{L_{\infty}}{2\sqrt{d}\sigma} + \frac{\epsilon}{4}, \label{eq: the third bound of U}\\
		2\tP(\twonorm{\bx}/\sigma \leq \sqrt{d}\epsilon) &\leq 2\cdot \frac{(\sqrt{d}\epsilon)^{d-1}}{2^{d/2}\Gamma(d/2)}\cdot \sqrt{d}\epsilon \leq \epsilon, \label{eq: the fourth bound of U}
	\end{align}
	where we used the fact that $\frac{(\sqrt{d}/2)^{d-1}}{2^{d/2}\Gamma(d/2)}\cdot 2\sqrt{d} \leq 1$ when $d \geq 4$.
	
	Let $\delta = 1/4$, then $\sqrt{\frac{\log (1/\delta)}{2n}} = \sqrt{\frac{\log 2}{n}} \leq 2\sqrt{\log 2}\epsilon$ since we now focus on the case $\epsilon \geq \frac{1}{2\sqrt{n}}$. The case $\epsilon < \frac{1}{2\sqrt{n}}$ was handled at the beginning of the proof, where it was shown to follow from Lemma \ref{lem: minimax lower bound mean est}. Take $\bx_0$ with $\bx_0 \in \argmax\limits_{\twonorm{\bx} > L}p'(\twonorm{\proxpl(\bx)})$, since $\frac{\epsilon\sigma\sqrt{d}}{4}\cdot \frac{1}{c_1+1} \leq L_{\infty}$, then \eqref{eq: lb_prop3}, \eqref{eq: the second bound of U}, \eqref{eq: the third bound of U} and \eqref{eq: the fourth bound of U} give
	\begin{align}
		\frac{\epsilon\twonorm{\nabla \rho(\bx_0)} - (1-\epsilon)\twonorm{\tP_{n_0}\nabla \rho(\bx)}}{\widetilde{U} + \epsilon} &\geq \epsilon\frac{L_{\infty}/2}{ \frac{4\sqrt{2}L_{\infty}}{\sigma\sqrt{d}} + \frac{L_{\infty}}{2\sqrt{d}\sigma}+ \frac{5}{4}\epsilon+ 2\sqrt{\log 2}\epsilon}  \\
		&\geq \epsilon\frac{L_{\infty}/2}{[4\sqrt{2} + \frac{1}{2} + (5 + 8\sqrt{\log 2})(c_1+1)]\frac{L_{\infty}}{\sqrt{d}\sigma}} \\
		&\geq \frac{1}{36+24c_1}\epsilon\sigma\sqrt{d},
	\end{align}
	with probability at least $1-\delta = 3/4$. 
	\item When $L > L_{\infty}$, by \eqref{eq: the second bound of U}, under $\mathcal{A}_1 \cap \mathcal{A}_2$: 
\begin{equation}
	\widetilde{U}+\epsilon \leq \frac{4\sqrt{2}L_{\infty}}{\sigma\sqrt{d}} + 2\tP(\twonorm{\bx}/\sigma \leq L/\sigma + \sqrt{d}\epsilon/2) + 2\tP(\twonorm{\bx}/\sigma \leq \sqrt{d}\epsilon) + \sqrt{\frac{\log (1/\delta)}{2n}} + \epsilon,
\end{equation}
 with probability at least $1-\delta$.

Since $L \leq \frac{1}{4}\sigma\sqrt{d}$, similar to the previous analysis, by \eqref{eq: the third bound of U} and \eqref{eq: the fourth bound of U} and considering $\bx_0$ with $\twonorm{\bx_0} = L$, on event $\mathcal{A} \cap \mathcal{A}_1 \cap \mathcal{A}_2$, we have
	\begin{align}
		\frac{\epsilon\twonorm{\nabla \rho(\bx_0)} - (1-\epsilon)\twonorm{\tP_{n_0}\nabla \rho(\bx)}}{\widetilde{U} + \epsilon} 
        &\geq \frac{L/2}{\frac{4\sqrt{2}L_{\infty}}{\sigma\sqrt{d}} + \frac{L_{\infty}}{2\sqrt{d}\sigma}+ \frac{5}{4}\epsilon+ 2\sqrt{\log 2}\epsilon}  \\
        &\geq \frac{L/2}{[4\sqrt{2} + \frac{1}{2} + (5 + 8\sqrt{\log 2})(c_1+1)]\frac{L}{\sqrt{d}\sigma}}\\
		&\geq \frac{1}{36+24c_1}\sigma\sqrt{d},
	\end{align}
	with probability at least $1-\delta= 3/4$.  
\end{enumerate}

Finally, putting everything together, on event $\mathcal{A} \cap \mathcal{A}_1 \cap \mathcal{A}_2$, we have
\begin{align}
	\twonorm{\htheta - \bthetas} &\geq \frac{1-\tau}{(c_0 \vee 1)(4\sqrt{3}+1)}\cdot \frac{\epsilon\twonorm{\nabla \rho(\bx_0)} - (1-\epsilon)\twonorm{\tP_{n_0}\nabla \rho(\bx)}}{\widetilde{U} + \epsilon} \\
	&\geq \frac{1-\tau}{(c_0 \vee 1)(4\sqrt{3}+1)}\cdot  \frac{1}{36+24c_1}\epsilon\sigma\sqrt{d}.
\end{align}
Therefore,
\begin{equation}
	\sup_{\mathcal{M}}\tP\bigg(\twonorm{\htheta - \btheta^*} > \frac{1-\tau}{(c_0 \vee 1)(4\sqrt{3}+1)}\cdot  \frac{1}{36+24c_1}\epsilon\sigma\sqrt{d}\bigg) \geq 1-\tP(\mathcal{A}^c) - \tP(\mathcal{A}_1^c) - \tP(\mathcal{A}_2^c) \geq \frac{7}{16}.
\end{equation}

\noindent \textbf{Case \Rom{2}: $\rho$ is non-differentiable}. 

Note that the non-differentiable points must fall into $L\mathcal{S}^{d-1}$. We can still use Taylor expansion but we need to be careful about the non-differentiable points of $\rho$. Specifically, we have
\begin{align}
	\nabla \rho(\bx_i - \htheta) &= \nabla \rho(\bx_i) + \mathds{1}(\twonorm{\bx_i} < L < \twonorm{\bx_i - \htheta})\bigg[p'(\twonorm{\proxpl(\bx_i - t_i\htheta)})\cdot \frac{\bx_i - t_i\htheta}{\twonorm{\bx_i - t_i\htheta}} - (\bx_i - t_i\htheta)\bigg] \\
	&\quad - \mathds{1}(\twonorm{\bx_i - \htheta} < L < \twonorm{\bx_i})\bigg[p'(\twonorm{\proxpl(\bx_i - t_i\htheta)})\cdot \frac{\bx_i - t_i\htheta}{\twonorm{\bx_i - t_i\htheta}} - (\bx_i - t_i\htheta)\bigg] \\
	&\quad + \int_0^1 \nabla^2 \rho(\bx_i - t\htheta)(\btheta^* - \htheta) \textup{d}t,
\end{align}
where $t_i \in [0, 1]$ such that $\twonorm{\bx_i - t_i\htheta} = L$. Denote
\begin{align}
	A_{i1} &= \mathds{1}(\twonorm{\bx_i} < L < \twonorm{\bx_i - \htheta})\bigg[p'(\twonorm{\proxpl(\bx_i - t_i\htheta)})\cdot \frac{\bx_i - t_i\htheta}{\twonorm{\bx_i - t_i\htheta}} - (\bx_i - t_i\htheta)\bigg],\\
	A_{i2} &= \mathds{1}(\twonorm{\bx_i - \htheta} < L < \twonorm{\bx_i})\bigg[p'(\twonorm{\proxpl(\bx_i - t_i\htheta)})\cdot \frac{\bx_i - t_i\htheta}{\twonorm{\bx_i - t_i\htheta}} - (\bx_i - t_i\htheta)\bigg]., \\
	A_i &= A_{i1} + A_{i2}.
\end{align}
Note that $p'(\twonorm{\proxpl(\bx_i - t_i\htheta)})$ is a deterministic function of $\twonorm{\bx_i - t_i\htheta} = L$, therefore it is a fixed number that depends on $L$.

Consider $\bx_0$ as a function of $\{\bx_i\}_{i=1}^{n_0}$ satisfying that $\bx_0(\{\bx_i\}_{i=1}^{n_0}) = -\bx_0(-\{\bx_i\}_{i=1}^{n_0})$. Then notice that $\htheta = \htheta(\{\bx_i\}_{i=1}^{n})$ is a symmetric function of $\{\bx_i\}_{i=1}^{n}$, in the sense that $\htheta(\{\bx_i\}_{i=1}^{n}) = -\htheta(-\{\bx_i\}_{i=1}^{n})$. If we let $\bx_i = \bx_0$ for $i = (n_0+1):n$, since $\bx_0$ is a symmetric function of $\{\bx_i\}_{i=1}^{n_0}$, $\htheta$ is a symmetric function of $\{\bx_i\}_{i=1}^{n_0}$. Note that for $\{\bx_i\}_{i=1}^{n_0}$ and $-\{\bx_i\}_{i=1}^{n_0}$, $t_i$ would stay the same.

Finally, note that since the distribution of $\{\bx_i\}_{i=1}^{n_0}$ is symmetric around $0$, we must have %{\color{red}[M: I found the explanation a bit confusing but it also seems to be a consequence of Theorem \ref{thm: psi-form}]} \yt{We used a simple fact here: If function $f(\bx)$ is odd, then $\tE f(\bx) = 0$ when the distribution of $\bx$ is symmetric around $0$. Thm 11 is nice but it only applies to the case when $\rho$ is differentiable (which is implied by the conditions on the penalty there).}
\begin{equation}\label{eq: A mean zero}
	\tE\bigg(\sum_{i=1}^n A_{i1}\bigg) = \tE\bigg(\sum_{i=1}^n A_{i2}\bigg) = 0.
\end{equation}
If we consider $\{\bu_j\}_{j=1}^N$ as a $1/2$-cover of $\mathcal{S}^{d-1}$ w.r.t. to the $\ell_2$-norm, then
\begin{equation}\label{eq: A finite cover}
	\Big\|\sum_{i=1}^n A_i \Big\|_2 = \sup_{\twonorm{\bu} = 1} \Big\<\sum_{i=1}^n A_i, \bu\Big\> \leq 2\max_{j=1:N} \Big\<\sum_{i=1}^n A_i, \bu_j\Big\>,
\end{equation}
where $N \leq 5^d$. Next, we use \eqref{eq: A mean zero} and \eqref{eq: A finite cover} to upper bound $\|\sum_{i=1}^n A_i\|_2$ with the events
\begin{align}
	\mathcal{A}'_j = &\bigg\{\frac{1}{n}\sup_{\twonorm{\btheta} \leq \sigma \epsilon \sqrt{d}} \bigg|\sum_{i=1}^n \Big(\<A_i, \bu_j\> - \tE \<A_i, \bu_j\> \Big)\bigg| \\ 
	&\quad\quad- \frac{1}{n}\tE \sup_{\twonorm{\btheta} \leq \sigma \epsilon \sqrt{d}} \bigg|\sum_{i=1}^n \Big(\<A_i, \bu_j\> - \tE \<A_i, \bu_j\> \Big)\bigg| \leq |L_{\infty}-L|\sqrt{\frac{2\log(1/\delta)}{n}}\bigg\}, \, j\in[N].
\end{align}
By the bounded differences inequality, we have $\tP(\cap_{j=1}^N\mathcal{A}'_j) \geq 1-N\delta$. Therefore, when $\twonorm{\htheta - \btheta^*} \leq \sigma \epsilon \sqrt{d}$, by standard symmetrization argument, under event $\cap_{j=1}^N\mathcal{A}'_j$, we have
\begin{align}
	\frac{1}{n}\max_{j=1:N} \Big\<\sum_{i=1}^n A_i, \bu_j\Big\> &= \frac{1}{n}\max_{j=1:N} \sum_{i=1}^n \Big(\<A_i, \bu_j\> - \tE \<A_i, \bu_j\>\Big) \\
	&\leq \frac{1}{n}\max_{j=1:N}\sup_{\twonorm{\btheta} \leq \sigma \epsilon \sqrt{d}} \Big(\<A_i, \bu_j\> - \tE \<A_i, \bu_j\>\Big)\\
	&\leq \frac{1}{n}\max_{j=1:N} \tE \sup_{\twonorm{\btheta} \leq \sigma \epsilon \sqrt{d}}\bigg|\sum_{i=1}^n \Big(\<A_i, \bu_j\> - \tE \<A_i, \bu_j\>\Big)\bigg| + |L_{\infty}-L|\sqrt{\frac{2\log(1/\delta)}{n}}\\
	&\leq \frac{2}{n}\max_{j=1:N} \tE_{\bx}\tE_{\bm{\epsilon}} \sup_{\twonorm{\btheta} \leq \sigma \epsilon \sqrt{d}}\bigg|\sum_{i=1}^n \epsilon_i\<A_i, \bu_j\>\bigg|+ |L_{\infty}-L|\sqrt{\frac{2\log(1/\delta)}{n}} \\
	&\leq \frac{2}{n}\max_{j=1:N} \tE_{\bx}\tE_{\bm{\epsilon}} \sup_{\twonorm{\btheta} \leq \sigma \epsilon \sqrt{d}}\bigg|\sum_{i=1}^n \epsilon_i\<A_{i1}, \bu_j\>\bigg|+ |L_{\infty}-L|\sqrt{\frac{2\log(1/\delta)}{n}} \\
	&\quad + \frac{2}{n}\max_{j=1:N} \tE_{\bx}\tE_{\bm{\epsilon}} \sup_{\twonorm{\btheta} \leq \sigma \epsilon \sqrt{d}}\bigg|\sum_{i=1}^n \epsilon_i\<A_{i2}, \bu_j\>\bigg|.
\end{align}
Denote $f_{\btheta}(\bx) = \mathds{1}(\twonorm{\bx} < L < \twonorm{\bx - \btheta})\cdot \<\bx - t(\btheta)\btheta, \bu\>$, where $t(\btheta)$ is defined to be the number $t \in (0, 1)$ satisfying $\twonorm{\bx - t\btheta} = L$. It is straightforward to see that the VC dimension of the subgraph $\{(r, \bx): r < f_{\btheta}(\bx)\}$ is no larger than $d+1$. Therefore by Theorem 2.6.7 in \cite{van1996weak}, the covering number of $\mathcal{F} = \{f_{\btheta}: \btheta \in \mathbb{R}^d\}$ under $L_2(\tP_{n, \epsilon})$ can be bounded as
\begin{equation}
	N(\delta |L_{\infty} - L|, \mathcal{F}, L_2(\tP_{n, \epsilon})) \leq K(16e)^{d+1}(1/\delta)^{2(d+1)},
\end{equation}
which implies that
\begin{equation}
	N(\delta, \mathcal{F}, L_2(\tP_{n, \epsilon})) \leq K(16e)^{d+1}\bigg(\frac{|L_{\infty} - L|}{\delta}\bigg)^{2(d+1)},
\end{equation}
where $K$ is an absolute constant and it does not depend on $n$, $p$, $\epsilon$, $L$, and $L_{\infty}$. By Dudley's entropy integral,
\begin{align}
	\tE_{\bm{\epsilon}} \sup_{\twonorm{\btheta} \leq \sigma \epsilon \sqrt{d}}\bigg|\frac{1}{\sqrt{n}}\sum_{i=1}^n \epsilon_i\<A_{i2}, \bu_j\>\bigg| &\leq 8\sqrt{2} \int_0^{|L_{\infty}-L|}\sqrt{\log N(\delta, \mathcal{F}, L_2(\tP_{n, \epsilon}))} \textup{d}\delta \\
	&\leq 8\sqrt{2}|L_{\infty}-L|\sqrt{d+1}\int_0^1 \sqrt{\log(16Ke) + 2\log(1/\delta)}\textup{d}\delta \\
	&\leq 8\sqrt{2}|L_{\infty}-L|\sqrt{d+1}\sqrt{\log(16Ke)} + \sqrt{2}\sqrt{d+1}\int_0^{\infty}\sqrt{\delta}e^{-\delta}\textup{d}\delta \\ 
	&\leq 8\sqrt{2}|L_{\infty}-L|\sqrt{d+1}\big[\sqrt{\log(16Ke)} + \sqrt{2}(2-3e^{-1})\big]. 
\end{align}
The same argument can be used the term that depends on $A_{i1}$, which implies that
\begin{equation}
	\bigg\|\frac{1}{n}\sum_{i=1}^n A_i \bigg\|_2 \leq \frac{64\sqrt{2}}{\sqrt{n}}|L_{\infty}-L|\sqrt{d+1}\big[\sqrt{\log(16Ke)} + \sqrt{2}(2-3e^{-1})\big] + 2|L_{\infty}-L|\sqrt{\frac{\log 16+d\log 5}{n}},
\end{equation}
by taking $\delta = (16\times 5^d)^{-1}$ in $\{\mathcal{A}'_j\}_{j=1}^N$. 

Therefore, when $\frac{64\sqrt{2}}{\sqrt{n}}\sqrt{d+1}[\sqrt{\log(16Ke)} + \sqrt{2}(2-3e^{-1})] + 2\sqrt{\frac{\log 16+d\log 5}{n}} \leq \epsilon/4$, by following the same proof given for the differentiable case in Case I, we can show that with probability at least $1-(\frac{1}{4}+\frac{1}{4}+ \frac{1}{4} + \frac{1}{16}) = 3/16$ (we need to condition on $\cap_{j=1}^N \mathcal{A}'_j$ in addition to the event $\mathcal{A}_1$ in Case I), $\twonorm{\hotheta - \btheta^*} \geq \frac{1-\tau}{(c_0 \vee 1)(4\sqrt{3}+1)}\cdot  \frac{1}{72+48c_1}\epsilon\sigma\sqrt{d}$ for all  points $\hotheta$ in the argmin \eqref{eq: m-est form 2}.
\end{proof}

\begin{proof}[Proof of Proposition \ref{prop: finite sample lower bound small lambda}]
	Fix the outlier set $S^c = [(n_0+1):n] \subseteq [n]$ with $|S^c|/n = \epsilon$. WLOG, consider the case $\btheta^* = 0$. 

    Recall the event $\mathcal{A}_1 = \bigcap_{i=1}^{n_0}\{|\twonorm{\bx_i}^2/\sigma^2 - d| \leq d/4+d/16\leq d/2\}$ we defined in \eqref{eq: A1} and $\tP(\mathcal{A}_1^c) \leq n\cdot 2\exp\{-d/64\} \leq \frac{1}{16}$.

    By following the same argument for \eqref{eq: prox lower bdd} in the proof of Proposition \ref{prop: finite sample lower bound large lambda}, we can show that $\forall i \in [n], \forall \delta \in [0, 1]$,
    \begin{equation}
	  \twonorm{\proxpl(\bx_i + \delta(\htheta - \bthetas))} \geq \frac{1}{4\sqrt{2}}\sigma\sqrt{d}.
    \end{equation}
    Let $\widetilde{\bx}_i = M(\bx_i) = \bx_0 = \argmax\limits_{\twonorm{\bx} \geq \frac{\sigma\sqrt{d}}{\sqrt{2}}\epsilon} p'(\twonorm{\prox(\bx)})$ for $i = [n]\setminus [n_0]$. 

    \noindent \textbf{Case \Rom{1}: $p'(\twonorm{\prox(\bx_0)}) > 0$:}
    Note that 
    \begin{equation}
\twonorm{\bx_0+\delta(\htheta - \bthetas)} \geq \twonorm{\bx_0} - \twonorm{\htheta - \bthetas} \geq \frac{3\sigma\sqrt{d}}{4\sqrt{2}}\epsilon \geq  (c_1+1)(L \vee L_{\infty}),
    \end{equation}
    when $\twonorm{\htheta - \bthetas} \leq \frac{\sigma\sqrt{d}}{4\sqrt{2}}\epsilon$. This and Lemma \ref{lem: prox} imply that
    \begin{equation}
	   \twonorm{\proxpl(\bx_0+\delta(\htheta - \bthetas))} \geq \twonorm{\bx_0+\delta(\htheta - \bthetas)} - p'(\twonorm{\proxpl(\bx_0+\delta(\htheta - \bthetas))}) \geq c_1(L \vee L_{\infty}),
    \end{equation}
    and
    \begin{align}
	\twonorm{\proxpl(\bx_0+\delta(\htheta - \bthetas))} &\geq \twonorm{\bx_0+\delta(\htheta - \bthetas)} - L_{\infty} \\
	&\geq \frac{3\sigma\sqrt{d}}{4\sqrt{2}}\epsilon - \frac{\sigma\sqrt{d}}{2\sqrt{2}}\epsilon  \\
	&\geq \frac{\sigma\sqrt{d}}{4\sqrt{2}}\epsilon.
    \end{align}
    Further note that $\twonorm{\bx_0} \geq \frac{\sigma\sqrt{d}}{\sqrt{2}}\epsilon$ by definition and $\frac{\sigma\sqrt{d}}{\sqrt{2}}\epsilon \geq L$ by assumption. It follows that $\|\bx_0\|>L$ and by Lemma \ref{lem: prox} 
    \begin{equation}
	   \twonorm{\proxpl(\bx_0)} = \twonorm{\bx_0} - p'(\twonorm{\proxpl(\bx_0)}) \geq \frac{\sigma\sqrt{d}}{\sqrt{2}}\epsilon - L_{\infty} \geq 2(c_1+1)(L \vee L_{\infty}) - L_{\infty} \geq (2c_1+1)(L \vee L_{\infty}).
    \end{equation}
    Define $\psi(t) = p'(\twonorm{\proxpl(\bx)})$ with $t = \twonorm{\bx}$. Note that $\psi$ can be defined as a function of $t$ because $\twonorm{\proxpl(\bx)}$ is a deterministic function of $\twonorm{\bx}$ when $\twonorm{\bx} \geq L$. In this case $p'(\twonorm{\proxpl(\bx)}) = \twonorm{\bx} - \twonorm{\proxpl(\bx)}$, and we have
    \begin{equation}
	\frac{\textrm{d} \twonorm{\proxpl(\bx)}}{\textrm{d}t} = \frac{1}{1 + p''(\twonorm{\proxpl(\bx)})}.
    \end{equation}
    Therefore,
    \begin{equation}
	  \psi'(t) = p''(\twonorm{\proxpl(\bx)})\cdot \frac{\textrm{d} \twonorm{\proxpl(\bx)}}{\textrm{d}t} = \frac{p''(\twonorm{\proxpl(\bx)})}{1 + p''(\twonorm{\proxpl(\bx)})}.
    \end{equation}
    Then it follows from Assumption \ref{asmp: penalty}.(\rom{6})  that when $t = \twonorm{\bx} \geq (c_1+1)(L \vee L_{\infty})$, $\twonorm{\proxpl(\bx)} \geq \twonorm{\bx} - p'(\twonorm{\proxpl(\bx)}) \geq c_1(L \vee L_{\infty})$, $\twonorm{\proxpl(\bx)} \geq \twonorm{\bx} - L_{\infty} \geq \frac{c_1}{c_1+1}t$, and
    \begin{align}
	\psi'(t) &\geq -\frac{c_0}{1-\tau} p'(\twonorm{\proxpl(\bx)})\cdot \frac{1}{\twonorm{\proxpl(\bx)}} \\
	&\geq -\frac{c_0}{1-\tau} \cdot \frac{c_1+1}{c_1}\cdot \frac{p'(\twonorm{\proxpl(\bx)})}{\twonorm{\bx}}\\
	&= -\frac{c_0}{1-\tau} \cdot \frac{c_1+1}{c_1}\cdot p'(\twonorm{\proxpl(\bx)})\cdot \frac{1}{t}.
    \end{align}
    This implies
    \begin{align}
	  &\frac{\textrm{d}\log \psi(t)}{d\log t} \geq -\frac{c_0}{1-\tau} \cdot \frac{c_1+1}{c_1} \\
	  &\Rightarrow \log \psi(2t) - \log \psi(t) \geq -\frac{c_0}{1-\tau} \cdot \frac{c_1+1}{c_1}\cdot \log 2\\
	  &\Rightarrow \psi(2t) \geq 2^{-\frac{c_0}{1-\tau} \cdot \frac{c_1+1}{c_1}}\psi(t).
    \end{align}
    Therefore,
    \begin{align}
	  p'(\twonorm{\proxpl(\bx_0)}) &= \max\limits_{\twonorm{\bx} \geq \frac{\sigma\sqrt{d}}{\sqrt{2}}\epsilon} p'(\twonorm{\prox(\bx)}) \\
	  &= \sup_{t \geq \frac{\sigma\sqrt{d}}{\sqrt{2}}\epsilon}\psi(t) \\
	  &\geq 2^{-\frac{c_0}{1-\tau} \cdot \frac{c_1+1}{c_1}}\cdot \sup_{t \geq \frac{\sigma\sqrt{d}}{2\sqrt{2}}\epsilon}\psi(t) \\
	  &= 2^{-\frac{c_0}{1-\tau} \cdot \frac{c_1+1}{c_1}}\cdot \sup_{\twonorm{\bx} \geq \frac{\sigma \sqrt{d}}{2\sqrt{2}}\epsilon}p'(\twonorm{\proxpl(\bx)}).\label{eq: p prime bound 1}
    \end{align}

    Since $L < \frac{\sigma \sqrt{d}\epsilon}{4} \leq \frac{\sigma\sqrt{d}}{8}$, under $\mathcal{A}_2$, when $\twonorm{\htheta - \btheta^*} \leq \frac{\sigma\epsilon\sqrt{d}}{2\sqrt{2}}$, $\twonorm{\bx_i} \geq \frac{\sigma\sqrt{d}}{\sqrt{2}} > L$ and $\twonorm{\bx_i - \delta \htheta} \geq \twonorm{\bx_i} - \twonorm{\htheta} \geq \frac{\sigma\sqrt{d}}{\sqrt{2}} -  \frac{\sigma\sqrt{d}}{4\sqrt{2}} \geq \frac{\sigma\sqrt{d}}{2\sqrt{2}} > L$ for all $i \in [n_0]$. Similarly, $\twonorm{\bx_0 - \delta \htheta} \geq \twonorm{\bx_0} - \twonorm{\htheta} \geq \frac{\epsilon\sigma\sqrt{d}}{\sqrt{2}} -  \frac{\epsilon\sigma\sqrt{d}}{4\sqrt{2}} \geq \frac{\epsilon\sigma\sqrt{d}}{2\sqrt{2}} > L$. Therefore, under $\mathcal{A}_1$, for all $i \in [n]$, $\nabla^2 \rho(\widetilde{\bx}_i-\btheta^*-\delta(\htheta - \btheta^*))$ exists for $\delta \in (0, 1)$, because by Lemma \ref{lem: dev and hessian}, non-differentiable points of $\rho$ must be on $L\mathcal{S}^{d-1}$. This helps us avoid the arguments in case (\Rom{2}) in the proof of Proposition \ref{prop: finite sample lower bound large lambda}. 

	Similar to the case (\Rom{1}) in the proof of Proposition \ref{prop: finite sample lower bound large lambda}, by Taylor expansion:
	\begin{equation}
		0 = \tP_n \nabla \rho(\tilde{\bx} -\htheta) = \tP_n \rho(\tilde{\bx} -\btheta^*) + \tP_n \bigg[\int_0^1 \nabla^2 \rho(\tilde{\bx}-\btheta^*-t(\htheta - \btheta^*)) \textup{d}t \bigg] (\htheta - \btheta^*).
	\end{equation}
	Denote $M = \tP_n \big[\int_0^1 \nabla^2 \rho(\bx-\btheta^*-t(\htheta - \btheta^*)) \textup{d}t \big]$. Following the same argument in the proof of Proposition \ref{prop: finite sample lower bound large lambda}, under $\mathcal{A}\cap \mathcal{A}_1 \cap \mathcal{A}_2$, we can get
	\begin{equation}
		\max\{|\lambdamax(M)|, |\lambdamin(M)|\}\twonorm{\htheta - \bthetas} \geq \epsilon\twonorm{\nabla \rho(\bx_0)} - (1-\epsilon)\twonorm{\tP_{n_0}\nabla \rho(\bx)},
	\end{equation}
	where by \eqref{eq: M},
	\begin{equation}
	\resizebox{\textwidth}{!}{$\begin{aligned}
		\lambdamax(M) &\leq \frac{1}{n}\int_0^1 \sum_{i=1}^n \frac{p'(\twonorm{\prox_{p}(\tilde{\bx}_i + \delta(\htheta-\btheta^*))})}{p'(\twonorm{\prox_{p}(\tilde{\bx}_i + \delta(\htheta-\btheta^*))}) + \twonorm{\prox_{p}(\tilde{\bx}_i + \delta(\htheta-\btheta^*))}}\cdot \mathds{1}(\twonorm{\tilde{\bx}_i + \delta(\htheta-\btheta^*)} > L) \textup{d}\delta \\
		&\quad+ \frac{1}{n}\int_0^1\sum_{i=1}^n \mathds{1}(\twonorm{\tilde{\bx}_i + \delta(\htheta-\btheta^*)} \leq L)\textup{d}\delta \\
		&= \frac{1}{n}\int_0^1 \sum_{i=1}^n \frac{p'(\twonorm{\prox_{p}(\tilde{\bx}_i + \delta(\htheta-\btheta^*))})}{\twonorm{\tilde{\bx}_i + \delta(\htheta-\btheta^*)}}\cdot \mathds{1}(\twonorm{\tilde{\bx}_i + \delta(\htheta-\btheta^*)} > L) \textup{d}\delta\\
		&= \frac{1}{n}\int_0^1 \sum_{i=1}^{n_0} \frac{p'(\twonorm{\prox_{p}(\bx_i + \delta(\htheta-\btheta^*))})}{\twonorm{\bx_i + \delta(\htheta-\btheta^*)}}\cdot \mathds{1}(\twonorm{\bx_i + \delta(\htheta-\btheta^*)} > L) \textup{d}\delta \\
		&\quad + \epsilon \int_0^1\frac{p'(\twonorm{\prox_{p}(\bx_0 + \delta(\htheta-\btheta^*))})}{\twonorm{\bx_0 + \delta(\htheta-\btheta^*)}}\cdot \mathds{1}(\twonorm{\bx_0 + \delta(\htheta-\btheta^*)} > L) \textup{d}\delta \\
		&\leq (1-\epsilon)\cdot \frac{1}{\sigma\sqrt{d}/(2\sqrt{2})}\cdot \sup_{\twonorm{\bx} \geq \frac{\sigma \sqrt{d}}{2\sqrt{2}}\epsilon}p'(\twonorm{\proxpl(\bx)}) + \epsilon \cdot \frac{1}{\epsilon\sigma\sqrt{d}/(2\sqrt{2})}\cdot \sup_{\twonorm{\bx} \geq \frac{\sigma \sqrt{d}}{2\sqrt{2}}\epsilon}p'(\twonorm{\proxpl(\bx)}) \\
		&\leq \frac{3\sqrt{2}}{\sigma \sqrt{d}}\cdot \sup_{\twonorm{\bx} \geq \frac{\sigma \sqrt{d}}{2\sqrt{2}}\epsilon}p'(\twonorm{\proxpl(\bx)}). \label{eq: lambdamax M}
	\end{aligned}$}
	\end{equation}
	 and
	\begin{align}
	    &\lambdamin(M) \\
		&\geq \frac{1}{n}\int_0^1 \sum_{i=1}^n \frac{p''(\twonorm{\prox_{p}(\tilde{\bx}_i + \delta(\htheta-\btheta^*))})}{p''(\twonorm{\prox_{p}(\tilde{\bx}_i + \delta(\htheta-\btheta^*))}) + 1}\cdot \mathds{1}(\twonorm{\tilde{\bx}_i + \delta(\htheta-\btheta^*)} > L) \textup{d}\delta  \\
        &\geq -\frac{c_0}{1-\tau}\cdot \frac{1}{n}\int_0^1 \sum_{i=1}^{n_0} \frac{p'(\twonorm{\prox_{p}(\bx_i + \delta(\htheta-\btheta^*))})}{\twonorm{\prox_{p}(\bx_i + \delta(\htheta-\btheta^*))}} \cdot \mathds{1}(\twonorm{\bx_i + \delta(\htheta-\btheta^*)} > L) \textup{d}\delta \\
	   &\quad -\frac{c_0}{1-\tau}\cdot \frac{1}{n}\int_0^1 \sum_{i=n_0+1}^n \frac{p'(\twonorm{\prox_{p}(\bx_0 + \delta(\htheta-\btheta^*))})}{\twonorm{\prox_{p}(\bx_0 + \delta(\htheta-\btheta^*))}} \cdot \mathds{1}(\twonorm{\bx_0 + \delta(\htheta-\btheta^*)} > L) \textup{d}\delta \\
	   &\geq -\frac{c_0}{1-\tau}(1-\epsilon) \cdot \frac{1}{\sigma\sqrt{d}/4\sqrt{2}}\cdot \sup_{\twonorm{\bx} \geq \frac{\sigma \sqrt{d}}{2\sqrt{2}}\epsilon}p'(\twonorm{\proxpl(\bx)}) \\
	   &\quad -\frac{c_0}{1-\tau}\epsilon \cdot \frac{1}{\epsilon\sigma\sqrt{d}/4\sqrt{2}}\cdot \sup_{\twonorm{\bx} \geq \frac{\sigma \sqrt{d}}{2\sqrt{2}}\epsilon}p'(\twonorm{\proxpl(\bx)}) \\
        &\geq -\frac{c_0}{1-\tau}\cdot \frac{8\sqrt{2}}{\sigma\sqrt{d}}\cdot \sup_{\twonorm{\bx} \geq \frac{\sigma \sqrt{d}}{2\sqrt{2}}\epsilon}p'(\twonorm{\proxpl(\bx)}). \label{eq: lambdamin M}
	\end{align}
	Moreover, by Lemma \ref{lem: prox}, under $\mathcal{A}_1$, 
	\begin{align}
		\twonorm{\tP_{n_0}\nabla \rho(\bx)} &\leq \twonorma{\frac{1}{n_0}\sum_{i=1}^{n_0} \frac{p'(\twonorm{\prox_{p}(\bx_i)})}{\twonorm{\bx_i}}\bx_i} \\
		&\leq \twonorma{\frac{1}{n_0}\sum_{i=1}^{n_0} \frac{p'(\twonorm{\prox_{p}(\bx_i)})}{\twonorm{\bx_i}}\bx_i\cdot \mathds{1}\Big(\twonorm{\bx_i} \geq \frac{\sigma\sqrt{d}}{\sqrt{2}}\epsilon\Big)}.
	\end{align}
	Define
	\begin{align}
		\mathcal{A}_3 &= \Bigg\{\twonorma{\frac{1}{n_0}\sum_{i=1}^{n_0} \frac{p'(\twonorm{\prox_{p}(\bx_i)})}{\twonorm{\bx_i}}\bx_i\cdot \mathds{1}\Big(\twonorm{\bx_i} \geq \frac{\sigma\sqrt{d}}{\sqrt{2}}\epsilon\Big)} \leq 2\sqrt{6}\cdot \frac{\sqrt{d}}{\sqrt{n_0}}L_{\infty}'\Bigg\},\\
		L_{\infty}' &= \sup_{\twonorm{\bx} \geq \frac{\sigma \sqrt{d}}{\sqrt{2}}\epsilon}p'(\twonorm{\proxpl(\bx)}).
	\end{align}
	Then \eqref{eq: lambdamax M} and \eqref{eq: lambdamin M} together with \eqref{eq: p prime bound 1} imply that
	\begin{align}
		\lambdamax(M) &\leq \frac{3\sqrt{2}}{\sigma \sqrt{d}}\cdot 2^{\frac{c_0}{1-\tau} \cdot \frac{c_1+1}{c_1}} \cdot L_{\infty}', \label{eq: lambdamax M bound}\\
		\lambdamin(M) &\geq -\frac{8\sqrt{2}c_0}{1-\tau}\cdot \frac{1}{\sigma\sqrt{d}}\cdot 2^{\frac{c_0}{1-\tau} \cdot \frac{c_1+1}{c_1}}\cdot L_{\infty}'. \label{eq: lambdamin M bound}
	\end{align}

	Similar to the arguments in the proof of Proposition \ref{prop: finite sample lower bound large lambda}, by bounded difference inequality, $\tP(\mathcal{A}^c_3) \leq 1/4$. Recall the events $\mathcal{A}_1$ and $\mathcal{A}_2$ defined in \eqref{eq: A1} and \eqref{eq: A2}. On event $\mathcal{A}_1 \cap \mathcal{A}_2 \cap \mathcal{A}_3$, we have
	\begin{align}
		\twonorm{\htheta - \bthetas} &\geq \frac{\epsilon\twonorm{\nabla \rho(\bx_0)} - (1-\epsilon)\twonorm{\tP_{n_0}\nabla \rho(\bx)}}{\max\{|\lambdamax(M)|, |\lambdamin(M)|\}} \\ 
		&\geq \frac{\epsilon L_{\infty}' - (1-\epsilon)2\sqrt{6}\cdot \frac{\sqrt{d}}{\sqrt{n_0}}L_{\infty}'}{\frac{8\sqrt{2}c_0}{1-\tau}\cdot \frac{1}{\sigma\sqrt{d}}\cdot 2^{\frac{c_0}{1-\tau} \cdot \frac{c_1+1}{c_1}}\cdot L_{\infty}' + \frac{3\sqrt{2}}{\sigma \sqrt{d}}\cdot 2^{\frac{c_0}{1-\tau} \cdot \frac{c_1+1}{c_1}} \cdot L_{\infty}'} \\
		&\geq \frac{\frac{1}{2}\epsilon L_{\infty}'}{\frac{8\sqrt{2}c_0}{1-\tau}\cdot \frac{1}{\sigma\sqrt{d}}\cdot 2^{\frac{c_0}{1-\tau} \cdot \frac{c_1+1}{c_1}}\cdot L_{\infty}' + \frac{3\sqrt{2}}{\sigma \sqrt{d}}\cdot 2^{\frac{c_0}{1-\tau} \cdot \frac{c_1+1}{c_1}} \cdot L_{\infty}'} \\ 
		&\geq \frac{1-\tau}{(16c_0+6)\sqrt{2}}\cdot 2^{-\frac{c_0}{1-\tau} \cdot \frac{c_1+1}{c_1}}\cdot \epsilon\sigma\sqrt{d},
	\end{align}
	where the second inequality is due to \eqref{eq: lambdamax M bound} and \eqref{eq: lambdamin M bound}.

	Therefore,
	\begin{equation}
		\sup_{\mathcal{M}}\tP\bigg(\twonorm{\htheta - \btheta^*} > \frac{1-\tau}{(16c_0+6)\sqrt{2}}\cdot 2^{-\frac{c_0}{1-\tau} \cdot \frac{c_1+1}{c_1}}\cdot \epsilon\sigma\sqrt{d}\bigg) \geq 1-\tP(\mathcal{A}^c) - \tP(\mathcal{A}_1^c) - \tP(\mathcal{A}_2^c) \geq \frac{7}{16}.
	\end{equation}

    \noindent  \textbf{Case \Rom{2}:  $p'(\twonorm{\prox(\bx_0)}) = 0$:} Then $\nabla \rho(\bx) = \bm{0}$ when $\twonorm{\bx} \geq \frac{\sigma\sqrt{d}\epsilon}{\sqrt{2}} \geq L$. This implies that $\rho(\bx) \equiv \bar{\rho}$ for some $\bar{\rho} \geq \frac{L^2}{2}$ when $\twonorm{\bx} \geq \frac{\sigma\sqrt{d}\epsilon}{\sqrt{2}}$. Consider the contamination mechanism $M$ such that $M(\bx_i) = \bx_i$ for all $i \in [n]$. Define the event $\mathcal{A}_3 = \cap_{i=2}^{n_0}\{|\twonorm{\bx_i - \bx_1}^2/2\sigma^2 - d| \leq d/2\}$. By Lemma \ref{lem: chi sq concentration}, $\tP(\mathcal{A}_3^c) \leq 2ne^{-d/64} \leq 1/16$. 

    Under $\mathcal{A} \cap \mathcal{A}_1 \cap \mathcal{A}_2\cap \mathcal{A}_3$, $\twonorm{\bx_i - \btheta} \geq \twonorm{\bx_i} - \twonorm{\btheta} \geq \frac{\sqrt{2}}{2}\sigma\sqrt{d} - \frac{\sqrt{2}}{4}\sigma\sqrt{d} \geq \frac{\sqrt{2}}{4}\sigma\sqrt{d}$ for all $i \in [n]$ and $\btheta$ with $\twonorm{\btheta} \leq \frac{\sqrt{2}}{4}\sigma\sqrt{d}$. Therefore, $G(\btheta) = n\bar{\rho} > G(\bx_1) = \sum_{i=2}^{n}\rho(\bx_i - \bx_1) = (n-1)\bar{\rho}$  for any $\btheta$ with $\twonorm{\btheta} \leq \frac{\sqrt{2}}{4}\sigma\sqrt{d}$, which means that such $\btheta$ cannot be a minimizer of the empirical risk function $G$. This implies the desired conclusion because $\tP(\mathcal{A} \cap \mathcal{A}_1 \cap \mathcal{A}_2\cap \mathcal{A}_3)\geq 3/8$, which completes the proof.
\end{proof}

\begin{proof}[Proof of Theorem \ref{thm: finite sample lower bound supp}]
    We split our proof into two cases. Define %{\color{red} [M: I can see how such an ugly $\zeta$ is suggested by the 2 preceeding propositions. However, that does not seem necessary in this proof. At some point you just use $\zeta<\sigma\sqrt{d}/8$]} \yt{It matters as we will use Theorem 1's conclusion}
    \begin{equation}
	\zeta = \left\{\bigg[\frac{1-\tau}{(c_0 \vee 1)(4\sqrt{3}+1)}\cdot \frac{1}{72+48c_1}\bigg]\wedge \bigg[\frac{1-\tau}{(16c_0+6)\sqrt{2}}\cdot 2^{-\frac{c_0}{1-\tau} \cdot \frac{c_1+1}{c_1}}\bigg]\right\}\epsilon\sigma\sqrt{d}.
    \end{equation}

    \noindent\textbf{Case \Rom{1}: $\frac{\epsilon\sigma\sqrt{d}}{4}\cdot \frac{1}{c_1+1} \leq \max\{L, L_{\infty}\}$}.
    
    First we present a key lemma:
    \begin{lemma}\label{lem: event A}
	%{\color{red}[M: this lemma made me wonder if we should have collected all the high probability events $\mathcal{A}_j$ in on lemma to make them easier to spot. Feel free to ignore this comment.]} \yt{I finally did not collect all high-probability events in one place, because the definition of each event may involve some notations introduced later only when they are used. It might be hard to state the definition of these events clearly at the beginning of the proof. But I did add equation numbers for each event and try to recall these numbers every time we use them.} 
    Define event
	\begin{equation}
		\mathcal{A}_4 = \bigg\{\forall \btheta \in \mathbb{R}^d, \exists i = i(\btheta) \in [n_0], \textup{s.t. } \btheta^{\top}(\bx_i-\btheta^*) > -\frac{\sqrt{2}}{2}\twonorm{\btheta}\twonorm{\bx_i-\btheta^*}\bigg\}.
	\end{equation} 
	When $d \geq 2$, we have
	\begin{equation}
		\tP(\mathcal{A}_4) \geq 1-\bigg(\frac{1}{2}\bigg)^{n_0-1}.
	\end{equation}
    \end{lemma}

%     Define the event $\mathcal{A} = \{\forall \btheta \in \mathbb{R}^d, \exists i = i(\btheta) \in S, \textup{s.t. } \btheta^{\top}(\bx_i-\btheta^*) > 0\}$. 
    
    By \eqref{eq: penalization form}, given $\hotheta$, we have
	\begin{equation}
		\htheta_i = \argmin_{\btheta}\Big\{\frac{1}{2}\twonorm{\btheta - \bx_i}^2 + p(\twonorm{\btheta - \hotheta})\Big\}, \quad i = 1:n.
	\end{equation}
	When $\htheta_i \neq \hotheta$, by the first-order condition, we must have
	\begin{equation}
		\htheta_i - \bx_i + p'(\twonorm{\htheta_i - \hotheta})\times \frac{\htheta_i - \hotheta}{\twonorm{\htheta_i - \hotheta}} = \bm{0}.
	\end{equation}
	Reorganizing it, we have
	\begin{equation}
		\Bigg[1+\frac{p'(\twonorm{\htheta_i - \hotheta})}{\twonorm{\htheta_i - \hotheta}}\Bigg]\htheta_i = \bx_i + \frac{p'(\twonorm{\htheta_i - \hotheta})}{\twonorm{\htheta_i - \hotheta}}\times \hotheta.
	\end{equation}
	Denote $a = \frac{p'(\twonorm{\htheta_i - \hotheta})}{p'(\twonorm{\htheta_i - \hotheta}) + \twonorm{\htheta_i - \hotheta}} \in [0, 1]$, then $\htheta_i - \btheta^* = (1-a)(\bx_i - \btheta^*) + a(\hotheta - \btheta^*)$, which together with $\mathcal{A}_4$ implies that  
	\begin{align}
		\twonorm{\htheta_i - \btheta^*}^2 &= (1-a)^2\twonorm{\bx_i - \btheta^*}^2 + a^2\twonorm{\hotheta - \btheta^*}^2 + 2a(1-a)(\hotheta - \btheta^*)^\top(\bx_i - \btheta^*)\\
		&\geq (1-a)^2\twonorm{\bx_i - \btheta^*}^2 + a^2\twonorm{\hotheta - \btheta^*}^2 - \sqrt{2}a(1-a)\twonorm{\hotheta - \btheta^*}\twonorm{\bx_i - \btheta^*} \\
		&\geq \bigg(1-\frac{\sqrt{2}}{2}\bigg)\big[(1-a)^2\twonorm{\bx_i - \btheta^*}^2 + a^2\twonorm{\hotheta - \btheta^*}^2\big]\\
		&\geq \frac{1}{2}\bigg(1-\frac{\sqrt{2}}{2}\bigg)(\twonorm{\bx_i - \btheta^*}^2 \wedge \twonorm{\hotheta - \btheta^*}^2).
	\end{align}
	Therefore, we have
	\begin{align}
		&\tP\bigg(\max_{i \in S}\twonorm{\htheta_i - \btheta^*} \geq  \frac{1}{\sqrt{2}}\bigg(1-\frac{\sqrt{2}}{2}\bigg)^{1/2}\zeta\bigg) \\
		&\geq \tP\bigg(\max_{i \in S}\twonorm{\htheta_i - \btheta^*} \geq \frac{1}{\sqrt{2}}\bigg(1-\frac{\sqrt{2}}{2}\bigg)^{1/2}\zeta, \bigcup_{i \in S}\{\htheta_i = \hotheta\}\bigg) \\
        &\quad + \tP\bigg(\max_{i \in S}\twonorm{\htheta_i - \btheta^*} \geq \frac{1}{\sqrt{2}}\bigg(1-\frac{\sqrt{2}}{2}\bigg)^{1/2}\zeta, \bigcap_{i \in S}\{\htheta_i \neq \hotheta\}, \mathcal{A}_4\bigg) \\
		&\geq \tP\bigg(\max_{i \in S}\twonorm{\htheta_i - \btheta^*} \geq \frac{1}{\sqrt{2}}\bigg(1-\frac{\sqrt{2}}{2}\bigg)^{1/2}\zeta, \bigcup_{i \in S}\{\htheta_i = \hotheta\}\bigg) \\
		&\quad + \tP\bigg(\frac{1}{\sqrt{2}}\bigg(1-\frac{\sqrt{2}}{2}\bigg)^{1/2}\cdot \Big(\twonorm{\hotheta - \btheta^*}\wedge\min_{i \in S}\twonorm{\bx_i - \btheta^*}\Big) \geq \frac{1}{\sqrt{2}}\bigg(1-\frac{\sqrt{2}}{2}\bigg)^{1/2}\zeta, \bigcap_{i \in S}\{\htheta_i \neq \hotheta\}, \mathcal{A}_4\bigg) \\
		&\geq \tP\bigg(\twonorm{\hotheta - \btheta^*} \geq \zeta, \bigcup_{i \in S}\{\htheta_i = \hotheta\}\bigg) + \tP\bigg(\twonorm{\hotheta - \btheta^*} \geq \zeta, \bigcap_{i \in S}\{\htheta_i \neq \hotheta\}\bigg) \\
		&\quad - \tP\bigg(\min_{i \in S}\twonorm{\bx_i - \btheta^*} < \zeta\bigg) - \tP(\mathcal{A}^c_4) \\
		&\geq \tP\bigg(\twonorm{\hotheta - \btheta^*} \geq \zeta\bigg) - \tP\bigg(\min_{i \in S}\twonorm{\bx_i - \btheta^*} < \zeta\bigg) - \tP(\mathcal{A}^c_4)\\
		&\geq \tP\bigg(\twonorm{\hotheta - \btheta^*} \geq \zeta\bigg) - \tP\bigg(\min_{i \in S}\twonorm{\bx_i - \btheta^*} < \frac{\sigma\sqrt{d}}{8}\bigg) - \tP(\mathcal{A}^c_4). \label{eq: main lower bound}
	\end{align}
	By Proposition \ref{prop: finite sample lower bound large lambda}, the first term $\tP\Big(\twonorm{\hotheta - \btheta^*} \geq \zeta\Big) \geq 3/16$. By Lemma \ref{lem: event A}, the last term 
	\begin{equation}
		\tP(\mathcal{A}_4^c) \leq \bigg(\frac{1}{2}\bigg)^{n_0-1} \leq \frac{1}{16},
	\end{equation}
	when $n_0 = n(1-\epsilon) \geq 5$. Regarding the second term, we can bound it by Lemma \ref{lem: chi sq concentration} and the fact that $\twonorm{\bx_i - \btheta^*}^2/\sigma^2 \sim \chi_d^2$ as follows:
	\begin{align}
		\tP\bigg(\min_{i \in S}\twonorm{\bx_i - \btheta^*} < \frac{\sigma\sqrt{d}}{8}\bigg) &\leq n\tP(\twonorm{\bx_1 - \btheta^*}^2/\sigma^2 < \sigma^2 d/64) \\
		&\leq n \tP(\twonorm{\bx_1 - \btheta^*}^2/\sigma^2 < d - 2\cdot \sqrt{d\cdot d/16})\\
		&\leq ne^{-d/16} \\
        &\leq \frac{1}{16}.
	\end{align}
	where the last inequality is due to our assumption that $\frac{1}{32}e^{d/64} \geq n$.
	
	Finally, putting all pieces together, we have
	\begin{equation}
		\tP\bigg(\max_{i \in S}\twonorm{\htheta_i - \btheta^*} \geq \frac{1}{2}\bigg(1-\frac{\sqrt{2}}{2}\bigg)\zeta\bigg) \geq \frac{3}{16}-\frac{1}{16}-\frac{1}{16} = \frac{1}{16}.
	\end{equation}
    
\noindent\textbf{Case \Rom{2}: $\frac{\epsilon\sigma\sqrt{d}}{4}\cdot \frac{1}{c_1+1} > \max\{L, L_{\infty}\}$}
	
	Consider any $S \subseteq [n]$ satisfying $|S| = \lceil n(1-\epsilon)\rceil$. By \eqref{eq: penalization form}, given $\hotheta$, we have
	\begin{equation}
		\htheta_i = \argmin_{\btheta}\Big\{\frac{1}{2}\twonorm{\btheta - \bx_i}^2 + p(\twonorm{\btheta - \hotheta})\Big\}, \quad i = 1:n.
	\end{equation}
	When $\twonorm{\hotheta - \bx_i}\leq L$ for some $i \in S$, by Lemma \ref{lem: prox}.(\rom{1}), for this $i \in S$, we must have $\htheta_i = \hotheta$ 
    and
	\begin{align}
		\twonorm{\htheta_i - \btheta^*} = \twonorm{\hotheta - \btheta^*} \geq \twonorm{\bx_i - \btheta^*} - \twonorm{\hotheta - \bx_i} \geq \twonorm{\bx_i - \btheta^*} - L \geq \twonorm{\bx_i - \btheta^*} - \frac{1}{4}\sigma\sqrt{d}\epsilon.
	\end{align}
	When $\twonorm{\hotheta - \bx_i} > L$ for all $i \in S$, by the first-order condition, for all $i \in S$, we must have
	\begin{equation}
		\htheta_i - \bx_i + p'(\twonorm{\htheta_i - \hotheta})\times \frac{\htheta_i - \hotheta}{\twonorm{\htheta_i - \hotheta}} = \bm{0}.
	\end{equation}
	Reorganizing it implies that $\twonorm{\htheta_i - \bx_i} = p'(\twonorm{\htheta_i - \hotheta}) \leq L_{\infty} = \sup_{\twonorm{\bx}>L}p'(\proxpl(\bx))$, where we used the fact that $\htheta_i - \hotheta = \text{prox}_{p}(\bx_i - \hotheta)$. Therefore,
	\begin{equation}
		\twonorm{\htheta_i - \btheta^*}  \geq \twonorm{\bx_i - \btheta^*} - \twonorm{\htheta_i - \bx_i} \geq \twonorm{\bx_i - \btheta^*} - L_{\infty} \geq \twonorm{\bx_i - \btheta^*} - \frac{1}{4}\sigma\sqrt{d}\epsilon.
	\end{equation}
    Hence 
	\begin{align}
		\tP\Big(\max_{i \in S}\twonorm{\htheta_i - \btheta^*} \geq \frac{3}{4}\sigma\sqrt{d}\epsilon\Big) 
		&= \tP\bigg(\max_{i \in S}\twonorm{\htheta_i - \btheta^*} \geq \frac{3}{4}\sigma\sqrt{d}\epsilon, \bigcup_{i \in S}\{\twonorm{\hotheta - \bx_i}\leq L\}\bigg) \\
		&\quad + \tP\bigg(\max_{i \in S}\twonorm{\htheta_i - \btheta^*} \geq \frac{3}{4}\sigma\sqrt{d}\epsilon, \bigcap_{i \in S}\{\twonorm{\hotheta - \bx_i}> L\}\bigg) \\
		&\geq \tP\bigg(\min_{i \in S}\twonorm{\bx_i - \btheta^*} \geq \sigma\sqrt{d}\epsilon, \bigcup_{i \in S}\{\twonorm{\hotheta - \bx_i}\leq L\}\bigg) \\
		&\quad + \tP\bigg(\min_{i \in S}\twonorm{\bx_i - \btheta^*} \geq \sigma\sqrt{d}\epsilon, \bigcap_{i \in S}\{\twonorm{\hotheta - \bx_i}> L\}\bigg) \\
		&\geq \tP\Big(\min_{i \in S}\twonorm{\bx_i - \btheta^*} \geq \sigma\sqrt{d}\epsilon\Big)\\
		&\geq \tP\Big(\min_{i \in S}\twonorm{\bx_i - \btheta^*} \geq \frac{1}{2}\sigma\sqrt{d}\Big). \label{eq: second case max}
	\end{align} 
	By Lemma \ref{lem: chi sq concentration} and the fact that $\twonorm{\bx_i - \btheta^*}^2/\sigma^2 \overset{\textup{i.i.d.}}{\sim} \chi_d^2$, we have
	\begin{align}
		\tP\Big(\min_{i \in S}\twonorm{\bx_i - \btheta^*} < \frac{1}{2}\sigma\sqrt{d}\Big) 
		&\leq n\tP\Big(\twonorm{\bx_1 - \btheta^*}^2/\sigma^2 \leq \frac{1}{4}d\Big)\\
		&\leq n\cdot \tP(\twonorm{\bx_1 - \btheta^*}^2/\sigma^2 < d - 2\cdot \sqrt{d\cdot d/16})\\
		&\leq ne^{-d/16} \\
        &\leq \frac{1}{16} \label{eq: second case min},
	\end{align}
	where the last inequality is due to our assumption that $\frac{1}{32}e^{d/64} \geq n$. Plugging \eqref{eq: second case min} back into \eqref{eq: second case max}, we obtain the desired conclusion.

    Finally, we conclude the proof of Theorem \ref{thm: finite sample lower bound supp} with the proof of Lemma \ref{lem: event A}.

    \begin{proof}[Proof of Lemma \ref{lem: event A}]
	Note that $\bx_i - \btheta^* \sim N(\bm{0}, \sigma^2\bm{I}_d)$. Hence $\frac{\bx_i - \btheta^*}{\twonorm{\bx_i - \btheta^*}} \overset{\textup{i.i.d.}}{\sim} \textup{Unif}(\mathcal{S}^{d-1})$. Denote $\bx_i = \frac{\bx_i - \btheta^*}{\twonorm{\bx_i - \btheta^*}}$. We will condition on $\bx_1$ in the following analysis, so WLOG, let us assume $\bx_1 = (1, \bm{0}_{d-1}^\top)^\top$. Consider a hyperspherical cap $\mathcal{R}_1$ and the half sphere $\mathcal{R}_2$ as
	\begin{align}
		\mathcal{R}_1 &= \bigg\{\bx \in \mathcal{S}^{d-1}: \bx^\top \bx_1 \leq -\frac{\sqrt{2}}{2}\bigg\},\\
		\mathcal{R}_2 &= \{\bx \in \mathcal{S}^{d-1}: \bx^\top \bx_1 \leq 0\}.
	\end{align}
	We claim that if $\bx_2 \in \mathcal{R}_2$, then for any $\btheta \in \mathbb{R}^d$, we must have $\btheta^\top \bx_1 \geq -\frac{\sqrt{2}}{2}$ or $\btheta^\top \bx_2 \geq -\frac{\sqrt{2}}{2}$. To prove this, it suffices to consider $\btheta \in \mathcal{S}^{d-1}$, as we can always normalize $\btheta$ by $\twonorm{\btheta}$ when $\btheta \neq \bm{0}$, and $\btheta^\top \bx_1\geq -\frac{\sqrt{2}}{2}$ automatically holds when $\btheta = \bm{0}$. For any $\btheta \in \mathcal{S}^{d-1}$, if $\btheta \notin \mathcal{R}_1$, then we must have  $\btheta^\top \bx_1 \geq -\frac{\sqrt{2}}{2}$. On the other hand, if $\btheta \in \mathcal{R}_1$, then $\theta_1 \leq -\frac{\sqrt{2}}{2}$, which implies that $\sum_{j=2}^d \theta_j^2 \leq 1/2$. Since $\bx_2 \in \mathcal{R}_2$, we must have $x_{21} \leq 0$. Therefore,
	\begin{equation}
		\btheta^\top \bx_2 = \theta_1x_{21} + \sum_{j=2}^d \theta_jx_{2j} \geq 0  - \sqrt{\sum_{j=2}^d \theta_j^2}\sqrt{\sum_{j=2}^d \bx_{2j}^2} \geq -\frac{\sqrt{2}}{2}.
	\end{equation}
	Hence our claim is correct. 
	
	With the claim, we have
	\begin{align}
		\tP(\mathcal{A}_4) &\geq \tE_{\bx_1}\tP(\exists i \in  2:n_0, \text{ s.t. } \bx_i \in \mathcal{R}_2|\bx_1) \\
		&= 1-\tE_{\bx_1}\tP(\forall i = 2:n_0, \bx_i \in \mathcal{R}_2|\bx_1) \\
		&\geq 1-\tE_{\bx_1}\prod_{i=2}^{n_0}\tP(\bx_i \in \mathcal{R}_2|\bx_1) \\
		&\geq 1-\bigg(\frac{1}{2}\bigg)^{n_0-1},
	\end{align}
	which completes the proof.
\end{proof}
\end{proof}

% ------------------------------------------------------------
\subsection{Proof of Theorem \ref{thm: lower bdd global regularization}}
First, note that
\begin{equation}
    \{\hthetak{k}\}_{k=1}^K \in \argmin\limits_{\bTheta=\{\bthetak{k}\}_{k=1}^K}\bigg\{\frac{1}{2n}\sum_{k=1}^K\sum_{i=1}^n\twonorm{\bxk{k}_i - \bthetak{k}}^2 + p(\bTheta)\bigg\}
\end{equation}
is equivalent to
\begin{equation}
    \{\hthetak{k}\}_{k=1}^K \in \argmin\limits_{\bTheta=\{\bthetak{k}\}_{k=1}^K}\bigg\{\frac{1}{2}\sum_{k=1}^K\twonorm{\barx^{(k)} - \bthetak{k}}^2 + p(\bTheta)\bigg\}.
\end{equation}
 Similar to the argument in the proof of Theorem \ref{thm: finite sample lower bound}, it suffices to prove that when $X_{ji} \overset{\textup{ind}}{\sim} N(\Theta^*_{ji}, \sigma^2)$, $\sigma^2 = \sigma_0^2/n$ for all $i \in [K]$ and $j \in [d]$ with $\#\{(j, i): |\Theta^*_{ji}| \geq \frac{\sigma}{\sqrt{2\pi}}\} \geq C_0^2 nd$, with probability at least $1/4$, all minimizers $\{\hthetak{k}\}_{k=1}^K$ in \eqref{eq: global regularization single-task} satisfy $\max\limits_{k = 1:K}\twonorm{\hthetak{k} - \bthetak{k}} \geq \frac{1}{40}\sigma C_0\sqrt{d}$, where
\begin{equation}\label{eq: global regularization single-task}
    \widehat{\bm{\Theta}}_{d \times K} = \{\hthetak{k}\}_{k=1}^K \in \argmin\limits_{\bm{\Theta}=\{\bthetak{k}\}_{k=1}^K}\bigg\{\frac{1}{2}\fnorm{\bm{\Theta} - \bX}^2 + p(\bm{\Theta})\bigg\}.
\end{equation}
Without loss of generality, let $\bm{\Theta}^*=(\Theta^*_{ji})_{j\in[d],\,i\in[K]}$, denote $\tilde{S}=\{(j,i): |\Theta^*_{ji}| \geq \frac{\sigma}{\sqrt{2\pi}}\}$, and assume $\Theta^*_{ji} \geq \frac{\sigma}{\sqrt{2\pi}}$ for all $(j,i)\in\tilde{S}$.

Since $\hTheta$ is a minimizer in \eqref{eq: global regularization single-task}, by the first-order optimality condition,
\begin{equation}
    \bm{0} = \hTheta - \bX + \frac{\nabla p}{\nabla \bTheta}|_{\hTheta},
\end{equation}
where $\frac{\partial p}{\partial \bTheta}|_{\hTheta}$ can be any sub-gradient of $p$ at $\hTheta$.

When $C_0 \leq \frac{\sqrt{72}}{\sqrt{n}}$, the conclusion directly follows from Lemma \ref{lem: minimax lower bound mean est} because $\frac{\sigma}{4}\sqrt{\frac{d}{n}} \geq \frac{\sigma}{40}\sqrt{d}C_0$. Therefore, it suffices to prove the case $C_0 > \frac{\sqrt{72}}{\sqrt{n}}$. Hence we assume $C_0 > \frac{\sqrt{72}}{\sqrt{n}}$ in the following part of the proof.

Note that by bounding the density of standard Gaussian variable, we have $\tP(X_{ji} - \Theta^*_{ji} \geq -\frac{\sigma}{\sqrt{2\pi}}) \leq \frac{1}{2} + (\frac{1}{\sqrt{2\pi}})^2 \leq \frac{2}{3}$. Then by bounded difference inequality, 
\begin{equation}
    \tP\bigg(\#\Big\{(j,i) \in \tilde{S}: X_{ji} - \Theta^*_{ji} < -\frac{\sigma}{\sqrt{2\pi}}\Big\} - \frac{1}{3}|\tilde{S}| \leq -|\tilde{S}|x\bigg) \leq \exp\bigg\{-\frac{|\tilde{S}|x^2}{2}\bigg\}.
\end{equation}
Let $x = 1/6$, $\mathcal{A} = \{\#\{(j,i) \in \tilde{S}: X_{ji} - \Theta^*_{ji} < -\frac{\sigma}{\sqrt{2\pi}}\} > \frac{1}{6}|\tilde{S}|\}$, we have

\begin{equation}
    \tP(\mathcal{A}^c) \leq \exp\bigg\{-\frac{|\tilde{S}|}{72}\bigg\} \leq \exp\bigg\{-\frac{C_0^2dn}{72}\bigg\} \leq e^{-d}.
\end{equation}
\noindent (\Rom{1}) If there is a set $\tilde{S}' \subseteq \tilde{S}$ with $|\tilde{S}'| \geq \frac{1}{12}|\tilde{S}|$ such that $\widehat{\Theta}_{ji} \leq 0$ for all $(j, i) \in \tilde{S}'$, then since $\Theta_{ji}^* \geq \frac{\sigma}{\sqrt{2\pi}}$ for all $(j, i) \in \tilde{S}$, we have
\begin{equation}\label{eq: fnorm bound}
    \fnorm{\hTheta - \bTheta^*}  \geq \frac{\sqrt{C_0^2dn}}{12}\times \frac{\sigma}{\sqrt{2\pi}},
\end{equation}
which implies $\max_{i=1:K}\twonorm{\hthetak{k} - \bthetak{k}} \geq \frac{\sigma C_0 \sqrt{d}}{12\sqrt{2\pi}}$.

\noindent (\Rom{2}) Otherwise, $\#\{(j,i) \in \tilde{S}:  \widehat{\Theta}_{ji} > 0\} \geq \frac{11}{12}|\tilde{S}|$. Denote $\tilde{S}'' = \{(j,i) \in \tilde{S}:  \widehat{\Theta}_{ji} > 0, X_{ji}^* - \Theta^*_{ji} < -\frac{\sigma}{\sqrt{2\pi}}\}$. Under $\mathcal{A}$, we must have $|\tilde{S}''| \geq \frac{1}{12}|\tilde{S}|$. For all $(j, i) \in \tilde{S}''$, 
\begin{equation}
    \widehat{\Theta}_{ji} - \Theta^*_{ji} = X_{ji}^* - \Theta^*_{ji} - \Big(\frac{\partial p}{\partial  \bTheta}\Big|_{\hTheta}\Big)_{ji} \leq X_{ji}^* - \Theta^*_{ji} \leq -\frac{\sigma}{\sqrt{2\pi}},
\end{equation}
where the first inequality is due to monotonicity of the penalty. This entails that
\begin{equation}
    \fnorm{\hTheta - \bTheta^*} \geq  \frac{\sqrt{C_0^2dn}}{12}\times \frac{\sigma}{\sqrt{2\pi}},
\end{equation}
implying that $\max_{i=1:K}\twonorm{\hthetak{k} - \bthetak{k}} \geq \frac{\sigma C_0 \sqrt{d}}{12\sqrt{2\pi}}$.

Combining (\Rom{1}) and (\Rom{2}), we have $\max_{i=1:K}\twonorm{\hthetak{k} - \bthetak{k}} \geq \frac{\sigma C_0 \sqrt{d}}{12\sqrt{2\pi}}$, with probability at least $1-\tP(\mathcal{A}^c) \geq 1-e^{-d}$. Lastly, the desired conclusion comes from the argument for the cases $C_0 \leq \frac{\sqrt{72}}{\sqrt{n}}$ and $C_0 > \frac{\sqrt{72}}{\sqrt{n}}$.

% -----------------------------------------------
\subsection{Proof of Theorem \ref{thm: lower bdd decomp regularization}}
Similar to the argument in the proof of Theorem \ref{thm: lower bdd global regularization}, it suffices to prove that when $X_{ij} \overset{\textup{ind}}{\sim} N(\Theta^*_{ji}, \sigma^2)$ for all $i \in [n]$ and $j \in [p]$ with $\#\{(j, i): |\Theta^*_{ji}| \geq \frac{\sigma}{\sqrt{2\pi}}\} \geq c_0^2 nd$, with probability at least $1/4$, all minimizers $\{\hthetak{k}\}_{k=1}^K$ in \eqref{eq: global regularization single-task} satisfy $\max\limits_{k = 1:K}\twonorm{\hthetak{k} - \bthetak{k}} \geq \frac{1}{40}\sigma\sqrt{\frac{d}{n}}$, where
\begin{align}
    \widehat{\bm{G}}_{d \times K}, \widehat{\bm{Q}}_{d \times K} &= \{\htheta_i\}_{i=1}^n \in \argmin\limits_{\bm{\Theta}=\{\btheta_i\}_{i=1}^n}\bigg\{\frac{1}{2n}\fnorm{\bm{G} + \bm{Q} - \bX}^2 + p_1(\bm{G}) + p_2(\bm{Q})\bigg\}, \label{eq: decomp regu def}\\
    \widehat{\bm{\Theta}} &= \widehat{\bm{G}} + \widehat{\bm{Q}}.
\end{align}
First, for any solution $\widehat{\bm{\Theta}}$, there exist $\widehat{\bm{G}}$, $\widehat{\bm{Q}}$ such that $\widehat{\bm{\Theta}} = \widehat{\bm{G}} + \widehat{\bm{Q}} > 0$ and $\widehat{G}_{ji}, \widehat{Q}_{ji} \geq 0$ or $\widehat{G}_{ji}, \widehat{Q}_{ji} \leq 0$ for all $i$ and $j$, i.e., the corresponding entries of $\widehat{\bm{G}}$ and $\widehat{\bm{Q}}$ have the same sign. To see this, WLOG, suppose that $\widehat{\bm{\Theta}} = \widehat{\bm{G}} + \widehat{\bm{Q}}$ where there exist $(j_0, i_0)$ such that $\widehat{G}_{j_0i_0} \geq 0$, $\widehat{Q}_{j_0i_0} < 0$. Consider $\widehat{\bm{G}}'$ and $\widehat{\bm{Q}}'$ with 
\begin{align}
    0\leq \widehat{G}'_{j_0i_0} &= \widehat{G}_{j_0i_0} + \widehat{Q}_{j_0i_0} < \widehat{G}_{j_0i_0}, \,\, \widehat{Q}'_{ji} = 0 > -\widehat{Q}_{j_0i_0},\\
    \widehat{G}'_{ji} &= \widehat{G}_{ji}, \widehat{Q}'_{ji} = \widehat{Q}_{ji}, \,\, \textup{all the other } (j, i) \textup{'s.}
\end{align}
Since $p_1$ and $p_2$ are entry-wise non-decreasing when all other entries are fixed, it follows that $\frac{1}{2n}\fnorm{\widehat{\bm{G}} + \widehat{\bm{Q}} - \bX}^2 + p_1(\widehat{\bm{G}}) + p_2(\widehat{\bm{Q}}) \geq \frac{1}{2n}\fnorm{\widehat{\bm{G}}' + \widehat{\bm{Q}} - \bX}^2 + p_1(\widehat{\bm{G}}') + p_2(\widehat{\bm{Q}}')$, therefore $(\widehat{\bm{G}}', \widehat{\bm{Q}}')$ is also a solution of \eqref{eq: decomp regu def} satisfying $\widehat{\bTheta} = \widehat{\bm{G}}' + \widehat{\bm{Q}}'$. This proves the claim.

The remainder of the proof then follows directly from the proof of Theorem \ref{thm: lower bdd global regularization}.

% ------------------------------------------------------------
\subsection{Proof of Theorem \ref{thm: lower bdd detection}}
Consider $\twonorm{\btheta^*} \geq \frac{1}{8\sqrt{5}}\sigma\epsilon\sqrt{\frac{d}{n}}$. When $\hat{S} \cap S = \emptyset$, $\twonorm{\htheta - \btheta^*} = \twonorm{\btheta^*} \geq \frac{1}{8\sqrt{5}}\sigma\epsilon\sqrt{\frac{d}{n}}$. When $\hat{S}  \cap S \neq \emptyset$, let $k_0 = \argmin_{k\in S} \textup{score}(k)$. Consider the case that $\bxk{k}_i = \bxk{k_0}_i$ for all $k \in S^c$ and $i \in [n]$. Then $\textup{score}(k) = \textup{score}(k_0)$ for all $k \in S^c$. By Definition \ref{def: outlier detection alg}, $\hat{S} \supseteq \{k_0\}\cup S^c$. Denote $\tilde{S} = \hat{S}\cap S$ and define two events
\begin{align}
    \mathcal{A}_1 &= \Bigg\{\bigg|\frac{n}{\sigma^2 |S'|}\bigg\|\sum_{k \in S'}\barxk{k} - \bthetas\bigg\|^2 -d \bigg| \leq 2\sqrt{d\delta} + 2\delta, \,\,  \forall S' \subseteq S\Bigg\}, \\
    \mathcal{A}_2 &= \bigcap_{k \in S} \big\{|n\twonorm{\barxk{k}-\bthetas}^2/\sigma^2 - d| \leq d/2\big\}.
\end{align}
Note that $\tP(\mathcal{A}_1^c) \leq 2e^{-\delta}\cdot 2^K$ and $\tP(\mathcal{A}_2^c) \leq 2K\exp\{-d/64\}$. Let $\delta = K$, then under $\mathcal{A}_1$, we have
\begin{equation}
    \frac{1}{|S'|\sigma^2}\bigg\|\sum_{k \in S'}\barxk{k} - \bthetas\bigg\|^2 \leq \frac{d + 2\sqrt{d\delta} + 2\delta}{n} \leq \frac{2d + 3K}{n}, \quad \forall S' \subseteq S.
\end{equation}
Then when $\epsilon \geq 2\sqrt{5/K}$ and $d \geq K$, under $\mathcal{A}_1 \cap \mathcal{A}_2$:
\begin{align}
    \twonorm{\htheta - \btheta^*} &\geq \frac{\epsilon K\twonorm{\barxk{k_0}-\bthetas} - |\tilde{S}|\sigma \sqrt{\frac{2d+3K}{|\tilde{S}|n}}}{|\tilde{S}|+\epsilon K} \\
&\geq \frac{\epsilon K \cdot \sigma \sqrt{\tfrac{d}{2n}} - \sqrt{K} \cdot \sqrt{\tfrac{2d+3K}{n}}}{|\tilde{S}| + \epsilon K} \\[0.5em]
&\geq \frac{(1 - \tfrac{\sqrt{2}}{2}) \, \epsilon K \, \sqrt{\tfrac{d}{2n}} \, \sigma}{|\tilde{S}| + \epsilon K} \\[0.5em]
&\geq \left( \tfrac{\sqrt{2}}{2} - \tfrac{1}{2} \right) \, \epsilon \, \sqrt{\tfrac{d}{n}} \, \sigma.
\end{align}
When $\epsilon < 2\sqrt{5/K}$, by Lemma \ref{lem: minimax lower bound mean est}, with probability at least $1/4$, we have $\twonorm{\htheta - \bthetas} \geq \frac{\sigma}{4}\sqrt{\frac{d}{nK}} \geq \frac{\sigma}{8\sqrt{5}}\epsilon\sqrt{\frac{d}{n}}$.
% -------------------------------------------------
\subsection{More examples of the regularizers and an equivalent formulation}\label{subsec: examples regularizer supp}

\begin{example}\label{exp: penalties supp}
	We list some commonly used regularizers 
	$p(\cdot)$ which satisfy  Assumption \ref{asmp: penalty}.
	\begin{enumerate}[(i)]
		\item (SCAD, $a > 2$, $\lambda>0$) $p(x) = \begin{cases}
			\lambda x, &\quad \textup{if } 0 \leq x \leq \lambda; \\
			-\frac{x^2 - 2a\lambda x + \lambda^2}{2(a-1)}, &\quad \textup{if } \lambda < x \leq a\lambda; \\
			\frac{a+1}{2}\lambda^2, &\quad \textup{if }x > a \lambda
		\end{cases}$, $L = L_{\infty} = \lambda$, $\tau = \frac{1}{a-1}$, \\ $\proxpl(\bx) = \begin{cases}
			\bm{0}, &\quad \textup{if } \twonorm{\bx} \leq \lambda;\\
			\frac{\twonorm{\bx} - \lambda}{\twonorm{\bx}}\cdot \bx, &\quad \textup{if } \lambda < \twonorm{\bx} \leq 2\lambda; \\
			\frac{(a-1)\twonorm{\bx} - a\lambda}{(a-2)\twonorm{\bx}}\cdot \bx, &\quad \textup{if } 2\lambda < \twonorm{\bx} \leq a\lambda;\\
			\bx, &\quad \textup{if } \twonorm{\bx} > a\lambda
		\end{cases}$,  \\
        $\rho(\bx) = \begin{cases}
			\frac{1}{2}\twonorm{\bx}^2, &\quad \textup{if } \twonorm{\bx} \leq \lambda;\\
			\frac{1}{2}\lambda^2 + \lambda(\twonorm{\bx} - \lambda), &\quad \textup{if } \lambda < \twonorm{\bx} \leq 2\lambda; \\
			\frac{-\twonorm{\bx}^2 + 2a\lambda\twonorm{\bx} - \lambda^2(a+2)}{2(a-2)}, &\quad \textup{if } 2\lambda < \twonorm{\bx} \leq a\lambda;\\
			\frac{a+1}{2}\lambda^2, &\quad \textup{if } \twonorm{\bx} > a\lambda
		\end{cases}$.
		\item (MC+, $b > 0$, $\lambda>0$) $p(x) = \begin{cases}
			\lambda x-\frac{x^2}{2b}, &\quad \textup{if } 0 \leq x \leq b\lambda; \\
			\frac{1}{2}b\lambda^2, \quad &\textup{if } x > b\lambda
		\end{cases}$, $L = \begin{cases}
			\lambda, &\quad \textup{if } b \geq 1; \\
			\sqrt{b}\lambda, \quad &\textup{if } 0<b<1
		\end{cases}$, $L_{\infty} = \begin{cases}
			\lambda, &\quad \textup{if } b > 1; \\
			0, \quad &\textup{if } 0<b\leq 1
		\end{cases}$. \\
        When $b > 1$, $\rho(\bx) = \begin{cases}
                \frac{1}{2}\twonorm{\bx}^2 &\quad \textup{if } 0 \leq \twonorm{\bx} \leq \lambda; \\
			\frac{-\twonorm{\bx}^2 + 2\lambda b \twonorm{\bx} - \lambda^2 b}{2(b-1)}, &\quad \textup{if } \lambda < \twonorm{\bx} \leq \lambda b; \\
			\frac{1}{2}b\lambda^2, \quad &\textup{if } \twonorm{\bx} > \lambda b
		\end{cases}$,  $\tau = \frac{1}{b}$, \\ $\proxpl(\bx) = \begin{cases}
                \bm{0} &\quad \textup{if } 0 \leq \twonorm{\bx} \leq \lambda; \\
			\frac{b}{b-1}\frac{\twonorm{\bx} - \lambda}{\twonorm{\bx}}\bx, &\quad \textup{if } \lambda < \twonorm{\bx} \leq \lambda b; \\
			\bx, \quad &\textup{if } \twonorm{\bx} > \lambda b
		\end{cases}$ \\
        When $0<b \leq 1$, $\rho(\bx) = \begin{cases}
                \frac{1}{2}\twonorm{\bx}^2 &\quad \textup{if } 0 \leq \twonorm{\bx} \leq \lambda \sqrt{b}; \\
			\frac{1}{2}b\lambda^2, \quad &\textup{if } \twonorm{\bx} > \lambda \sqrt{b}
		\end{cases}$, $\tau = 0$, \\ $\proxpl(\bx) = \begin{cases}
                \bm{0} &\quad \textup{if } 0 \leq \twonorm{\bx} \leq \lambda \sqrt{b}; \\
			\bx, \quad &\textup{if } \twonorm{\bx} > \lambda \sqrt{b}
		\end{cases}$
		\item (Hard-thresholding, $\lambda>0$) $p(x) = \begin{cases}
			\frac{1}{2}\lambda^2 - \frac{1}{2}(x - \lambda)^2, &\quad \textup{if } 0 \leq x \leq \lambda; \\
			\frac{1}{2}\lambda^2, \quad &\textup{if } x > \lambda
		\end{cases}$, $L = \lambda$, $L_{\infty} = 0$, $\tau = 0$, 
        $\proxpl(\bx) = \begin{cases}
                \bm{0} &\quad \textup{if } 0 \leq \twonorm{\bx} \leq \lambda; \\
			\bx, \quad &\textup{if } \twonorm{\bx} > \lambda
		\end{cases}$, $\rho(\bx) = \begin{cases}
                \frac{1}{2}\twonorm{\bx}^2 &\quad \textup{if } 0 \leq \twonorm{\bx} \leq \lambda; \\
			\frac{1}{2}\lambda^2, \quad &\textup{if } \twonorm{\bx} > \lambda
		\end{cases}$.
		\item (Bridge, $1< q < 2$) $p(x) = \lambda x^q$, $L = 0$, $L_{\infty} = +\infty$, $\tau = 0$, $\proxpl(\bx) = C_{q,\lambda}(\bx)\cdot \frac{\bx}{\twonorm{\bx}}$ with $C_{q,\lambda}(\bx) > 0$ satisfying $C_{q,\lambda}(\bx) + \lambda q[C_{q,\lambda}(\bx)]^{q-1} = \twonorm{\bx}$, $\rho(\bx) = \frac{1}{2}[\twonorm{\bx}-C_{q, \lambda}(\bx)]^2 + \lambda[C_{q, \lambda}(\bx)]^q$.
	\end{enumerate}
\end{example}

In addition to \eqref{eq: m-est form}, we can also derive the following $\psi$-estimator form for $\hotheta$. This connects the adaptive and robust regularized MTL estimator to a family of robust $M$-estimators. 

\begin{theorem}\label{thm: psi-form}
	Under Assumption \ref{asmp: penalty}, if $\twonorm{\proxpl(\bx)} \rightarrow 0$ when $\twonorm{\bx} \rightarrow L$, then the M-estimator $\hotheta$ in \eqref{eq: m-est form} can be shown to satisfy the following estimating equation:
	\begin{equation}\label{eq: psi-form}
		\sum_{k=1}^K \psi\big(\twonorm{\bxk{k}-\hotheta}\big)\frac{\bxk{k}-\hotheta}{\twonorm{\bxk{k}-\hotheta}} = \bm{0},
	\end{equation}
	where $\psi: [0, \infty) \rightarrow \mathbb{R}$ is uniquely defined by $p$ function and satisfies the following properties:
	\begin{enumerate}[(i)]
		\item $\psi(x) \geq 0$ for all $x \geq 0$, and $\psi(0) = 0$;
		\item $\psi$ is continuous and is differentiable almost everywhere on $(0, \infty)$, and $-\frac{\tau}{1-\tau}\leq \psi'(x) \leq 1$ for $x$ where $\psi'(x)$ exists.
	\end{enumerate}
\end{theorem} 

\begin{remark}\label{rmk: rho differentiable}
	$\twonorm{\proxpl(\bx)} \rightarrow 0$ when $\twonorm{\bx} \rightarrow L$ is a sufficient and necessary condition for the loss function $\rho$ in \eqref{eq: m-est form} to be differentiable, which can be directly verified by the explicit formula of $\nabla \rho(\bx)$ when $\twonorm{\bx} \neq L$ presented in Lemma \ref{lem: prox}.(\rom{3}) and Lemma \ref{lem: dev and hessian}.(\rom{1}). %{\color{red}[M: can you give a reference for this fact?]} \yt{This is directly by Lemmas 5.(iii) and 6.(i) which provides an explicit formula for $\nabla \rho(\bx)$ when $\twonorm{\bx} \neq L$.} 
    Without this condition, the solution of \eqref{eq: m-est form} does not necessarily satisfy \eqref{eq: psi-form}.
\end{remark}

In the classical single-task learning setup, \cite{mathieu2022concentration} showed that the M-estimator in \eqref{eq: psi-form} with stronger conditions ($\psi$ is concave, and there exist $C_1, C_2 > 0$ such that $\psi'(x) \geq C_1 \mathds{1}(x \leq C_2)$ for all $x > 0$) has robustness against adversarial contamination. However, their discussions do not apply to redescending $\psi$-functions such as Hampel's $\psi$ and Tukey's biweight $\psi$ which are non-concave and can have negative derivatives. In addition, they do not provide the finite-sample algorithmic lower bound. Compared to their analysis, we take a similar route by Taylor expansion, but our analysis is much more comprehensive and the removal of their strong conditions on $\psi$ is non-trivial. In fact, our proof of the algorithmic lower bound for \eqref{eq: m-est form} and \eqref{eq: m-est form individual} first transforms the problem to a classical single-task learning setting. Therefore, our analysis is also applicable to the classical single-task learning setting and of independent interest for robust statistics.

\begin{proof}[Proof of \Cref{thm: psi-form}]
    First, note that  when $\|\bx\|_2 < L$, by Lemma \ref{lem: dev and hessian} we have that $\nabla \rho(\bx) = \bx$ and hence $\rho(\bx) = \tfrac{1}{2}\|\bx\|_2^2$.  Furthermore, by the formula of $\proxpl(\bx)$ in Lemma \ref{lem: prox}.(\rom{3}), $\|\proxpl(\bx)\|_2 + p'(\|\proxpl(\bx)\|_2) = \|\bx\|_2$ when $\|\bx\|_2 > L$.  

Consider $G(u, v) = v + p'(v) - u$. Then by Lemma \ref{lem: prox}.(\rom{3}), $G(\twonorm{\bx}, \twonorm{\proxpl(\bx)}) = 0$.  
Since $\frac{\partial G}{\partial z}\big|_{v=\|\proxpl(\bx)\|_2,\, u=\|\bx\|_2} = 1 + p''(\|\proxpl(\bx)\|_2) > 0$, by the implicit function theorem,  $\exists\, g$ differentiable such that $\|\proxpl(\bx)\|_2 = g(\|\bx\|_2)$ when $\|\bx\|_2 > L$.

Note that by Lemma \ref{lem: dev and hessian}.(\rom{1}),
\begin{equation}
\nabla \rho(\bx) = p'(\|\proxpl(\bx)\|_2) \frac{\proxpl(\bx)}{\|\proxpl(\bx)\|_2} 
= p'(g(\|\bx\|_2)) \frac{\bx}{\|\bx\|_2}.
\end{equation}

Consider a function $\tilde{g}$ with $\tilde{g}'(x) = p'(g(x))$.  
Then 
\begin{equation}
\nabla \rho(\bx) = \tilde{g}'(\|\bx\|_2)\frac{\bx}{\|\bx\|_2}, \quad \text{if } \|\bx\|_2 > L,
\end{equation}
implying that 
\begin{equation}
\rho(\bx) = \tilde{g}(\|\bx\|_2) + C, \quad \text{if } \|\bx\|_2 > L,
\end{equation}
where $C$ is independent of $\bx$.

For $\|\bx\|_2 = L$, by Lemma \ref{lem: prox}.(\rom{2}), we know that $\bm{0} \in \arg\min_{\bz} \big\{\tfrac{1}{2}\|\bz - \bx\|_2^2 + \rho(\|\bz\|_2)\big\}$.  Therefore, $\rho(\bx) = \frac{1}{2}\twonorm{\bx}^2$. Hence $\rho(\bx)$ is actually a function of $\|\bx\|_2$, and $\nabla \rho(\bx)$ can be written as  $\psi(\|\bx\|_2)\frac{\bx}{\|\bx\|_2}$ with
\begin{equation}
\psi(x) =
\begin{cases}
x, & 0 \le x \le L, \\
\tilde{g}'(x), & x > L,
\end{cases}
\end{equation}
(where $\tilde{g}'(x) = p'(g(x))$).

Since $\rho(\bm{0}) = 0$ and  $p$ is increasing and continuous on $[0,+\infty)$, we have that  $ \psi(z) \ge 0$ for all $z \geq 0$. This shows  part 
(\rom{1}) of the theorem.

Furthermore, since $\|\proxpl(\bx)\|_2 + p'(\|\proxpl(\bx)\|_2) = \|\bx\|_2$ when $\|\bx\|_2 > L$ by Lemma \ref{lem: prox}.(\rom{3}), we have $g(v) + p'(g(v)) = v$, which implies $g'(v) + p''(g(v))g'(v) = 1$. Then because $p'$ is continuous on $(0,+\infty)$, and $p''$ exists on $(0,+\infty)$ almost everywhere, we have that   
$\psi$ is continuous and $\psi'$ exists almost everywhere.  
Moreover, since 
\begin{equation}
\frac{\partial g}{\partial v}(v) = \frac{1}{1 + p''(g(v))},
\qquad
-\frac{\tau}{1-\tau}\leq \psi'(x) = p''(g(x)) \frac{\partial g}{\partial x}(x)
= \frac{p''(g(x))}{1 + p''(g(x))} \leq 1,
\end{equation}
since $1 + p_{\lambda}''(\twonorm{\proxpl(\bx)}) > 0$ when $\twonorm{\bx} > L$.
This shows part
(\rom{2}).

\end{proof}

% -----------------------------------------------------
\subsection{Verification of assumptions for the regularizer examples}\label{subsec: revification penalty supp}
For a radial regularizer $p$, define
\begin{equation}
    F_u(r) \coloneqq \frac{1}{2}(u-r)^2 + p(r), \qquad u,r \geq 0.
\end{equation}
For every $\bx \neq \bm{0}$ and every $r \geq 0$,
\begin{equation}
    \min_{\twonorm{\bz}=r}\Big\{\frac{1}{2}\twonorm{\bx-\bz}^2 + p(\twonorm{\bz})\Big\} = \frac{1}{2}(\twonorm{\bx}-r)^2 + p(r),
\end{equation}
with equality attained at $\bz = r\bx/\twonorm{\bx}$. Hence every proximal point is collinear with $\bx$, and if we denote by $r_p(u)$ any minimizer of $F_u(r)$ over $r \geq 0$, then
\begin{equation}
    \proxpl(\bx) = r_p(\twonorm{\bx})\frac{\bx}{\twonorm{\bx}}, \qquad \twonorm{\proxpl(\bx)} = r_p(\twonorm{\bx}).
\end{equation}
Whenever $r_p(u) > 0$, the first-order condition gives
\begin{equation}\label{eq: appendix scalar prox equation}
    u = r_p(u) + p'(r_p(u)).
\end{equation}

\begin{proposition}\label{prop: verify penalty examples}
For every non-degenerate regularizer in Examples \ref{exp: penalties main text} and \ref{exp: penalties supp}, namely with $\lambda>0$, Assumption \ref{asmp: penalty} holds.
\end{proposition}

\begin{proof}
We verify the claims case by case. Throughout, $\lambda>0$, so the degenerate case ruled out by Assumption \ref{asmp: penalty}.(\rom{1}) does not occur.

\begin{enumerate}[(i)]
    \item \textbf{Lasso.}
    Here $p(r) = \lambda r$. Then $p(0)=0$, $p$ is continuous and non-decreasing on $[0,+\infty)$, and for $r>0$,
    \begin{equation}
        p'(r) = \lambda, \qquad p''(r) = 0, \qquad \frac{p'(r)}{r} = \frac{\lambda}{r}.
    \end{equation}
    Thus Assumption \ref{asmp: penalty}.(\rom{2})--(\rom{4}) hold. Also,
    \begin{equation}
        \frac{r}{2} + \frac{p(r)}{r} = \lambda + \frac{r}{2},
    \end{equation}
    hence $L=\lambda$. Minimizing $F_u(r)$ gives $r_p(u)=0$ for $u\leq \lambda$ and $r_p(u)=u-\lambda$ for $u>\lambda$, which is the proximal map in Example \ref{exp: penalties main text}. Therefore $L_{\infty}=\lambda$, so Assumption \ref{asmp: penalty}.(\rom{1}) holds. Moreover, $p''(\twonorm{\proxpl(\bx)})=0$ for every $\twonorm{\bx}>L$, so Assumption \ref{asmp: penalty}.(\rom{5}) holds with $\tau=0$. Finally, Assumption \ref{asmp: penalty}.(\rom{6}) holds with $c_0=1$ and $c_1=1$, because $p''(t)t=0\geq -p'(t)$ for all $t>0$.

    \item \textbf{Ridge.}
    Here $p(r)=\lambda r^2$. Then $p(0)=0$, $p$ is continuous and non-decreasing on $[0,+\infty)$, and for $r>0$,
    \begin{equation}
        p'(r)=2\lambda r, \qquad p''(r)=2\lambda, \qquad \frac{p'(r)}{r}=2\lambda.
    \end{equation}
    Hence Assumption \ref{asmp: penalty}.(\rom{2})--(\rom{4}) hold. Since
    \begin{equation}
        \frac{r}{2} + \frac{p(r)}{r} = \Big(\frac{1}{2}+\lambda\Big)r,
    \end{equation}
    we have $L=0$. The minimizer of $F_u(r)$ is $r_p(u)=u/(2\lambda+1)$, so $\proxpl(\bx)=\bx/(2\lambda+1)$. Since $p''(r)=2\lambda\geq 0$, Assumption \ref{asmp: penalty}.(\rom{5}) holds with $\tau=0$. Moreover,
    \begin{equation}
        p'(r_p(u))=\frac{2\lambda}{2\lambda+1}u \to +\infty \qquad \text{as } u\to +\infty,
    \end{equation}
    and therefore $L_{\infty}=+\infty$. Thus Assumption \ref{asmp: penalty}.(\rom{1}) holds, and Assumption \ref{asmp: penalty}.(\rom{6}) is not needed by the convention in Assumption \ref{asmp: penalty}.

    \item \textbf{Bridge penalty with $1<q<2$.}
    Here $p(r)=\lambda r^q$. Then $p(0)=0$, $p$ is continuous and non-decreasing, and for $r>0$,
    \begin{equation}
        p'(r)=\lambda q r^{q-1}, \qquad p''(r)=\lambda q(q-1)r^{q-2}>0, \qquad \frac{p'(r)}{r}=\lambda q r^{q-2}.
    \end{equation}
    Since $q-2<0$, the map $r\mapsto p'(r)/r$ is non-increasing, so Assumption \ref{asmp: penalty}.(\rom{2})--(\rom{4}) hold. Also,
    \begin{equation}
        \frac{r}{2} + \frac{p(r)}{r} = \frac{r}{2} + \lambda r^{q-1} \downarrow 0 \qquad \text{as } r\downarrow 0,
    \end{equation}
    hence $L=0$. The map $r\mapsto r + \lambda q r^{q-1}$ is strictly increasing from $0$ to $+\infty$ on $(0,+\infty)$, so for each $u>0$ there is a unique $C_{q,\lambda}(u)>0$ such that
    \begin{equation}
        C_{q,\lambda}(u) + \lambda q[C_{q,\lambda}(u)]^{q-1}=u,
    \end{equation}
    which gives the proximal formula in Example \ref{exp: penalties main text}. Since $p''(r)>0$, Assumption \ref{asmp: penalty}.(\rom{5}) holds with $\tau=0$. Finally, $C_{q,\lambda}(u)\to +\infty$ as $u\to +\infty$, so $L_{\infty}=+\infty$. Thus Assumption \ref{asmp: penalty}.(\rom{1}) holds, and Assumption \ref{asmp: penalty}.(\rom{6}) is not needed by the convention in Assumption \ref{asmp: penalty}.

    \item \textbf{SCAD with $a>2$.}
    For $r>0$,
    \begin{equation}
        p'(r)=\begin{cases}
            \lambda, & 0<r\leq \lambda, \\
            \dfrac{a\lambda-r}{a-1}, & \lambda<r\leq a\lambda, \\
            0, & r>a\lambda,
        \end{cases}
        \qquad
        p''(r)=\begin{cases}
            0, & 0<r<\lambda, \\
            -\dfrac{1}{a-1}, & \lambda<r<a\lambda, \\
            0, & r>a\lambda.
        \end{cases}
    \end{equation}
    These formulas show that $p(0)=0$, that $p$ is continuous and non-decreasing, that $p'$ is continuous on $(0,+\infty)$, and that $p''$ exists almost everywhere. Moreover,
    \begin{equation}
        \frac{p'(r)}{r}=\begin{cases}
            \dfrac{\lambda}{r}, & 0<r\leq \lambda, \\
            \dfrac{a\lambda-r}{(a-1)r}, & \lambda<r\leq a\lambda, \\
            0, & r>a\lambda,
        \end{cases}
    \end{equation}
    which is non-increasing on $(0,+\infty)$; the values match at $r=\lambda$ and $r=a\lambda$. Thus Assumption \ref{asmp: penalty}.(\rom{2})--(\rom{4}) hold.

    Next,
    \begin{equation}
        \frac{r}{2}+\frac{p(r)}{r} = \begin{cases}
            \lambda + \dfrac{r}{2}, & 0<r\leq \lambda, \\
            \dfrac{(a-2)r^2 + 2a\lambda r - \lambda^2}{2(a-1)r}, & \lambda<r\leq a\lambda, \\
            \dfrac{r}{2} + \dfrac{a+1}{2}\dfrac{\lambda^2}{r}, & r>a\lambda.
        \end{cases}
    \end{equation}
    The first piece is minimized at $r\downarrow 0$ with value $\lambda$. For the second piece,
    \begin{equation}
        \frac{\textup{d}}{\textup{d}r}\Big(\frac{r}{2}+\frac{p(r)}{r}\Big)=\frac{(a-2)r^2+\lambda^2}{2(a-1)r^2}>0,
    \end{equation}
    and for $r\geq a\lambda$,
    \begin{equation}
        \frac{\textup{d}}{\textup{d}r}\Big(\frac{r}{2}+\frac{p(r)}{r}\Big)=\frac{1}{2}-\frac{a+1}{2}\frac{\lambda^2}{r^2} \geq \frac{1}{2}-\frac{a+1}{2a^2}>0.
    \end{equation}
    Therefore $L=\lambda$. Solving \eqref{eq: appendix scalar prox equation} on each region gives the proximal map stated in Example \ref{exp: penalties supp}. Consequently,
    \begin{equation}
        p'(\twonorm{\proxpl(\bx)})=\begin{cases}
            \lambda, & \lambda<\twonorm{\bx}\leq 2\lambda, \\
            \dfrac{a\lambda-\twonorm{\bx}}{a-2}, & 2\lambda<\twonorm{\bx}\leq a\lambda, \\
            0, & \twonorm{\bx}>a\lambda,
        \end{cases}
    \end{equation}
    so $L_{\infty}=\lambda$. Thus Assumption \ref{asmp: penalty}.(\rom{1}) holds. When $\twonorm{\bx}>L$, the quantity $\twonorm{\proxpl(\bx)}$ belongs either to $(0,\lambda]$, $(\lambda,a\lambda]$, or $(a\lambda,+\infty)$, hence $p''(\twonorm{\proxpl(\bx)})\in\{0,-1/(a-1)\}$ wherever it exists. Thus Assumption \ref{asmp: penalty}.(\rom{5}) holds with $\tau=1/(a-1)$. Finally, since $L\vee L_{\infty}=\lambda$, Assumption \ref{asmp: penalty}.(\rom{6}) holds with $c_0=1$ and $c_1=a$: whenever $t\geq a\lambda=c_1(L\vee L_{\infty})$ and $p''(t)$ exists, we have $p''(t)t=0\geq -p'(t)$.

    \item \textbf{MC+ with $b>0$.}
    Here
    \begin{equation}
        p(r)=\begin{cases}
            \lambda r-\dfrac{r^2}{2b}, & 0\leq r\leq b\lambda, \\
            \dfrac{1}{2}b\lambda^2, & r>b\lambda.
        \end{cases}
    \end{equation}
    For $r>0$,
    \begin{equation}
        p'(r)=\begin{cases}
            \lambda-\dfrac{r}{b}, & 0<r\leq b\lambda, \\
            0, & r>b\lambda,
        \end{cases}
        \qquad
        p''(r)=\begin{cases}
            -\dfrac{1}{b}, & 0<r<b\lambda, \\
            0, & r>b\lambda.
        \end{cases}
    \end{equation}
    Hence $p(0)=0$, $p$ is continuous and non-decreasing, $p'$ is continuous on $(0,+\infty)$, and $p''$ exists almost everywhere. Also,
    \begin{equation}
        \frac{p'(r)}{r}=\begin{cases}
            \dfrac{\lambda}{r}-\dfrac{1}{b}, & 0<r\leq b\lambda, \\
            0, & r>b\lambda,
        \end{cases}
    \end{equation}
    which is non-increasing on $(0,+\infty)$ and continuous at $r=b\lambda$. Thus Assumption \ref{asmp: penalty}.(\rom{2})--(\rom{4}) hold.

    Moreover,
    \begin{equation}
        \frac{r}{2}+\frac{p(r)}{r}=\begin{cases}
            \lambda + \dfrac{b-1}{2b}r, & 0<r\leq b\lambda, \\
            \dfrac{r}{2} + \dfrac{b\lambda^2}{2r}, & r>b\lambda.
        \end{cases}
    \end{equation}
    If $b\geq 1$, both pieces are bounded below by $\lambda$, so $L=\lambda$. If $0<b<1$, the first piece is decreasing on $(0,b\lambda]$ and the second piece is minimized at $r=\sqrt{b}\lambda>b\lambda$, which gives $L=\sqrt{b}\lambda$.

    If $b>1$, solving \eqref{eq: appendix scalar prox equation} on $(0,b\lambda]$ gives
    \begin{equation}
        r_p(u)=\frac{b}{b-1}(u-\lambda), \qquad \lambda<u\leq b\lambda,
    \end{equation}
    while $r_p(u)=0$ for $u\leq \lambda$ and $r_p(u)=u$ for $u>b\lambda$. Consequently $L_{\infty}=\lambda$. If $0<b\leq 1$, comparing the values of $F_u$ at $r=0$ and at $r=u$ gives $r_p(u)=0$ for $u\leq \sqrt{b}\lambda$ and $r_p(u)=u$ for $u>\sqrt{b}\lambda$, and then $L_{\infty}=0$. Since $\lambda>0$, Assumption \ref{asmp: penalty}.(\rom{1}) holds in both regimes.

    If $b>1$, then for every $\twonorm{\bx}>L$, we have $p''(\twonorm{\proxpl(\bx)})\in\{-1/b,0\}$ wherever it exists, so Assumption \ref{asmp: penalty}.(\rom{5}) holds with $\tau=1/b$. If $0<b\leq 1$, then $\twonorm{\bx}>L$ implies $\twonorm{\proxpl(\bx)}>b\lambda$, hence $p''(\twonorm{\proxpl(\bx)})=0$, so Assumption \ref{asmp: penalty}.(\rom{5}) holds with $\tau=0$. Finally, Assumption \ref{asmp: penalty}.(\rom{6}) holds with $c_0=1$ and $c_1=\max\{1,b\}$. Indeed, if $b>1$, then $L\vee L_{\infty}=\lambda$ and $t\geq c_1(L\vee L_{\infty})=b\lambda$ implies $p''(t)t=0\geq -p'(t)$ wherever $p''(t)$ exists. If $0<b\leq 1$, then $L\vee L_{\infty}=\sqrt{b}\lambda$ and $t\geq L\vee L_{\infty}\geq b\lambda$ again implies $p''(t)t=0\geq -p'(t)$ wherever $p''(t)$ exists.

    \item \textbf{Hard-thresholding.}
    Here
    \begin{equation}
        p(r)=\begin{cases}
            \lambda r-\dfrac{r^2}{2}, & 0\leq r\leq \lambda, \\
            \dfrac{1}{2}\lambda^2, & r>\lambda.
        \end{cases}
    \end{equation}
    Therefore, for $r>0$,
    \begin{equation}
        p'(r)=\begin{cases}
            \lambda-r, & 0<r\leq \lambda, \\
            0, & r>\lambda,
        \end{cases}
        \qquad
        p''(r)=\begin{cases}
            -1, & 0<r<\lambda, \\
            0, & r>\lambda.
        \end{cases}
    \end{equation}
    Thus $p(0)=0$, $p$ is continuous and non-decreasing, $p'$ is continuous on $(0,+\infty)$, and $p''$ exists almost everywhere. Also,
    \begin{equation}
        \frac{p'(r)}{r}=\begin{cases}
            \dfrac{\lambda}{r}-1, & 0<r\leq \lambda, \\
            0, & r>\lambda,
        \end{cases}
    \end{equation}
    which is non-increasing on $(0,+\infty)$. Therefore Assumption \ref{asmp: penalty}.(\rom{2})--(\rom{4}) hold.

    Moreover, $r/2+p(r)/r$ equals $\lambda$ on $(0,\lambda]$ and equals $r/2+\lambda^2/(2r)\geq \lambda$ on $(\lambda,+\infty)$, so $L=\lambda$. The proximal map is the one in Example \ref{exp: penalties supp}, and $L_{\infty}=0$. Since $L=\lambda>0$, Assumption \ref{asmp: penalty}.(\rom{1}) holds. Since $\twonorm{\bx}>L$ implies $\twonorm{\proxpl(\bx)}=\twonorm{\bx}>\lambda$, we have $p''(\twonorm{\proxpl(\bx)})=0$, so Assumption \ref{asmp: penalty}.(\rom{5}) holds with $\tau=0$. Finally, Assumption \ref{asmp: penalty}.(\rom{6}) holds with $c_0=1$ and $c_1=1$, because $t\geq L\vee L_{\infty}=\lambda$ implies $p''(t)t=0\geq -p'(t)$ wherever $p''(t)$ exists.

    \item \textbf{Bridge penalty with $0<q<1$.}
    Let
    \begin{equation}
        g(r) \coloneqq \frac{r}{2}+\lambda r^{q-1}, \qquad h(r) \coloneqq r+\lambda q r^{q-1}, \qquad r>0.
    \end{equation}
    Then $g''(r)=\lambda(q-1)(q-2)r^{q-3}>0$, so $g$ has the unique minimizer
    \begin{equation}
        r_L = [2\lambda(1-q)]^{\frac{1}{2-q}},
    \end{equation}
    and
    \begin{equation}
        L = g(r_L)= [2\lambda(1-q)]^{\frac{1}{2-q}}\cdot\frac{1}{2}\Big(1+\frac{1}{1-q}\Big).
    \end{equation}
    Also, $\lambda(1-q)r_L^{q-2}=1/2$ implies
    \begin{equation}\label{eq: bridge threshold identity}
        h(r_L)=r_L+\lambda q r_L^{q-1}=L.
    \end{equation}

    The function $p$ satisfies $p(0)=0$, is continuous and non-decreasing, and for $r>0$,
    \begin{equation}
        p'(r)=\lambda q r^{q-1}, \qquad p''(r)=\lambda q(q-1)r^{q-2}, \qquad \frac{p'(r)}{r}=\lambda q r^{q-2}.
    \end{equation}
    Since $q-2<0$, Assumption \ref{asmp: penalty}.(\rom{2})--(\rom{4}) hold.

    As in the display above, $F_u(r)-F_u(0)=r(g(r)-u)$. Hence $r=0$ is globally optimal when $u\leq L$, and not globally optimal when $u>L$. For $r\geq r_L$,
    \begin{equation}
        h'(r)=1-\lambda q(1-q)r^{q-2} \geq 1-\lambda q(1-q)r_L^{q-2}=1-\frac{q}{2}>0.
    \end{equation}
    Thus for every $u>L$ there is a unique $r(u)>r_L$ with $h(r(u))=u$, and this is the positive minimizer of $F_u$. Therefore $r_p(u)=0$ for $u\leq L$ and $r_p(u)=r(u)$ for $u>L$, as in Example \ref{exp: penalties supp}. Since $r(u)$ is continuous and strictly increasing on $(L,+\infty)$,
    \begin{equation}
        L_{\infty}=\sup_{u>L}p'(r_p(u))=\lambda q r_L^{q-1}.
    \end{equation}
    Since $L=r_L+\lambda q r_L^{q-1}$, we have $L>L_{\infty}$, so Assumption \ref{asmp: penalty}.(\rom{1}) holds.

    Finally, for every $u>L$, we have $r_p(u)=r(u)\geq r_L$, and hence
    \begin{equation}
        p''(r_p(u)) \geq -\lambda q(1-q)r_L^{q-2} = -\frac{q}{2} \geq -\Big(1-\frac{q}{2}\Big).
    \end{equation}
    Hence Assumption \ref{asmp: penalty}.(\rom{5}) holds with the choice $\tau=1-q/2$ used in Example \ref{exp: penalties supp}. Moreover,
    \begin{equation}
        p''(t)t=\lambda q(q-1)t^{q-1}=-(1-q)p'(t), \qquad t>0,
    \end{equation}
    so Assumption \ref{asmp: penalty}.(\rom{6}) holds with $c_0=1-q$ and $c_1=1$.
\end{enumerate}
This completes the proof.
\end{proof}

% -----------------------------------------------------
% Appendix: Proofs of results in Section 3
% -----------------------------------------------------

\section{Technical details of Section \ref{sec: method and theory}}

% {\color{red}[M: maybe call the section ``auxiliary results for filtering algorithm'' and just keep the first subsection. See my comment about the second subsection.]}

% Proofs of results in this section are collected in Section \ref{app:proof-of-detailedresults}.
% -------------------------------------

\subsection{Proofs of results in Section \ref{subsec: lower bound}}

For Gaussian mean estimation with squared loss, as shown in the proof, the gradient estimation error is equivalent to the parameter estimation error. Therefore, in addition to lower bounds on the estimation errors of the parameters $\bthetas$ and $\bthetaks{k}$, we can also obtain lower bounds on the estimation errors of the gradients $\nabla \mL(\btheta)$ and $\nabla \mLk{k}(\btheta)$ uniformly over $\btheta \in \Theta$. We summarize these lower bounds in the following theorem, which includes Theorem \ref{thm: lower bound} as a special case.

\begin{theorem}\label{thm: lower bound supp}
    There exist constants $C > 0$ and $c \in (0, 1)$ such that
    \begin{align}
        &\inf_{\hat{\bmu}}\sup_{\tP \in \mathcal{P}, S \in \mathcal{S}}\sup_{M \in \mathcal{M}_S} \tP\bigg(\sup_{\btheta \in \Theta}\Big\|\hat{\bmu}(\btheta) - \nabla \mL(\btheta)\Big\|_2 \geq C\bigg(\sqrt{\frac{d}{nK}} + \sqrt{\epsilon}h + \frac{\epsilon}{\sqrt{n}}\bigg)\bigg) \geq c,\\
        &\inf_{\{\hat{\bmu}^{(k)}\}_{k=1}^K}\sup_{\tP \in \mathcal{P}', S \in \mathcal{S}}\sup_{M \in \mathcal{M}_S} \tP\bigg(\bigcup_{k \in S}\bigg\{\sup_{\btheta \in \Theta}\|\hat{\bmu}^{(k)}(\btheta) - \nabla \mLk{k}(\btheta)\|_2 \\
        &\hspace{6cm}\geq C\bigg[\bigg(\sqrt{\frac{d}{nK}} + \sqrt{\epsilon}h + \hk{k} + \frac{\epsilon}{\sqrt{n}}\bigg)\wedge \sqrt{\frac{d}{n}}\bigg]\bigg\}\bigg) \geq c, \\
        &\inf_{\htheta}\sup_{\tP \in \mathcal{P}, S \in \mathcal{S}}\sup_{M \in \mathcal{M}_S} \tP\bigg(\|\htheta - \bthetas\|_2 \geq C\bigg(\sqrt{\frac{d}{nK}} + \sqrt{\epsilon}h + \frac{\epsilon}{\sqrt{n}}\bigg)\bigg) \geq c,\\
        &\inf_{\{\hthetak{k}\}_{k=1}^K}\sup_{\tP \in \mathcal{P}', S \in \mathcal{S}}\sup_{M \in \mathcal{M}_S} \tP\bigg(\bigcup_{k \in S}\bigg\{\|\hthetak{k} - \bthetaks{k}\|_2  \geq C\bigg[\bigg(\sqrt{\frac{d}{nK}} + 
    \sqrt{\epsilon}h + \hk{k} + \frac{\epsilon}{\sqrt{n}}\bigg)\wedge \sqrt{\frac{d}{n}}\bigg]\bigg\}\bigg) \geq c.
    \end{align}
\end{theorem}

\subsubsection{Proof of Theorem \ref{thm: lower bound supp}}
We work under the Gaussian mean estimation model with squared loss. Then
\begin{equation}
    \ell(\btheta, \bx)=\frac{1}{2}\twonorm{\btheta-\bx}^2, \qquad \nabla \ell(\btheta, \bx)=\btheta-\bx,
\end{equation}
and therefore
\begin{equation}
    \nabla \mL(\btheta)=\btheta-\bthetas, \qquad \nabla \mLk{k}(\btheta)=\btheta-\bthetaks{k}.
\end{equation}

We first reduce the gradient estimation problem to the parameter estimation problem. For any estimator $\hat{\bmu}(\btheta)$ of $\nabla \mL(\btheta)$, define $\hat{\bm{\vartheta}}(\btheta)=\btheta-\hat{\bmu}(\btheta)$. Then
\begin{equation}
    \hat{\bmu}(\btheta)-\nabla \mL(\btheta)=\bthetas-\hat{\bm{\vartheta}}(\btheta),
\end{equation}
which implies
\begin{equation}
    \sup_{\btheta \in \Theta}\twonorm{\hat{\bmu}(\btheta)-\nabla \mL(\btheta)}
    = \sup_{\btheta \in \Theta}\twonorm{\hat{\bm{\vartheta}}(\btheta)-\bthetas}.
\end{equation}
Fix any $\btheta_0 \in \Theta$ and define $\tilde{\bm{\vartheta}}(\btheta)\equiv \hat{\bm{\vartheta}}(\btheta_0)$ for all $\btheta \in \Theta$. Then
\begin{equation}
    \sup_{\btheta \in \Theta}\twonorm{\tilde{\bm{\vartheta}}(\btheta)-\bthetas}
    = \twonorm{\hat{\bm{\vartheta}}(\btheta_0)-\bthetas}
    \leq \sup_{\btheta \in \Theta}\twonorm{\hat{\bm{\vartheta}}(\btheta)-\bthetas}.
\end{equation}
Hence, for the minimax lower bound, it suffices to restrict attention to estimators of the form $\hat{\bmu}(\btheta)=\btheta-\hat{\bm{\vartheta}}$, where $\hat{\bm{\vartheta}}$ does not depend on $\btheta$. In that case,
\begin{equation}
    \sup_{\btheta \in \Theta}\twonorm{\hat{\bmu}(\btheta)-\nabla \mL(\btheta)} = \twonorm{\hat{\bm{\vartheta}}-\bthetas}.
\end{equation}
The same argument applies to $\hat{\bmu}^{(k)}(\btheta)$. Therefore, it suffices to prove the lower bounds for estimating $\bthetas$ and $\bthetaks{k}$.

We first consider the lower bound for $\bthetas$. When $\sqrt{\frac{d}{nK}} + \frac{\epsilon}{\sqrt{n}} \gtrsim \sqrt{\epsilon}h$, the same construction used in the proof of Theorem 4.3 of \cite{duan2022adaptive} (which in turn is based on Theorem 2.2 of \cite{chen2018robust}) yields the lower bound $\sqrt{\frac{d}{nK}} + \frac{\epsilon}{\sqrt{n}}$. More specifically, if we consider $\bthetaks{k} = \bthetas$ for all $k \in S$, then $\frac{\epsilon}{\sqrt{n}}$ part comes from a construction of two mixture distributions $\tP_1 = (1-\epsilon)\tP_{\btheta_1}^{\otimes n} + \epsilon \tQ_1$ and $\tP_2 = (1-\epsilon)\tP_{\btheta_2}^{\otimes n} + \epsilon \tQ_2$ where $\tP_1 = \tP_2$ but $\twonorm{\btheta_1 - \btheta_2} \asymp \epsilon/\sqrt{n}$. Then the lower bound is a direct consequence of Le Cam's lemma. The $\sqrt{\frac{d}{nK}}$ part comes from the case where there is no contamination,% (which is within our parameter set because the contamination proportion $0$ is smaller than $\epsilon$), which is 
as a variant of a classical lower bound for mean estimation under the Gaussian model, which can be proved by Fano's lemma.

Hence it remains to consider the regime
\begin{equation}
    \sqrt{\epsilon}h \gtrsim \sqrt{\frac{d}{nK}} + \frac{\epsilon}{\sqrt{n}}.
\end{equation}
Fix any $\bm{v}\in \mathbb{S}^{d-1}$, and for simplicity assume that $K\epsilon$ is an integer. Consider the following two parameter setups.

\smallskip
(\rom{1}) Let the uncontaminated task set be $S=[K]$ and the contaminated set be $S^c = \emptyset$, and define
\begin{equation}
    \bthetaks{k}=
    \begin{cases}
        \epsilon^{-1/2}h\bm{v}, & k=1,\ldots, K\epsilon,\\
        \bm{0}, & k=K\epsilon+1,\ldots,K.
    \end{cases}
\end{equation}
Then
\begin{equation}
    \bthetas=\frac{1}{K}\sum_{k=1}^K \bthetaks{k}=\sqrt{\epsilon}h\bm{v},
\end{equation}
and
\begin{align}
    \frac{1}{K}\sum_{k=1}^K\twonorm{\bthetaks{k}-\bthetas}^2
    &= \epsilon\twonorm{\epsilon^{-1/2}h\bm{v}-\sqrt{\epsilon}h\bm{v}}^2 + (1-\epsilon)\twonorm{\sqrt{\epsilon}h\bm{v}}^2 \\
    &= (1-\epsilon)h^2 \\
    &\leq h^2.
\end{align}
Hence this parameter setup belongs to $\mathcal{P}$.

\smallskip
(\rom{2}) Let the uncontaminated task set be $S=\{K\epsilon+1,\ldots,K\}$ and the contaminated set be $S^c = [K]\backslash S$, and let the underlying clean model satisfy $\bthetaks{k}=\bm{0}$ for all $k \in [K]$, so that $\bthetas=\bm{0}$ and this setup also belongs to $\mathcal{P}$. Let the contamination mechanism replace the observations from tasks $k=1,\ldots,K\epsilon$ by i.i.d.\ draws from $N(\epsilon^{-1/2}h\bm{v}, \bm{I}_d)$.

By construction, the observed data distributions in (\rom{1}) and (\rom{2}) are identical, while the corresponding global parameters are $\bm{\theta}_{\mathrm{I}}^*=\sqrt{\epsilon}h\bm{v}$ and $\bm{\theta}_{\mathrm{II}}^*=\bm{0}$. Therefore, for any estimator $\hat{\bm{\vartheta}}$, the events
\begin{equation}
    \left\{\twonorm{\hat{\bm{\vartheta}}-\bm{\theta}_{\mathrm{I}}^*} < \frac{\sqrt{\epsilon}h}{2}\right\}
    \qquad\text{and}\qquad
    \left\{\twonorm{\hat{\bm{\vartheta}}-\bm{\theta}_{\mathrm{II}}^*} < \frac{\sqrt{\epsilon}h}{2}\right\}
\end{equation}
are disjoint. Since the two experiments induce the same law on the observed data, at least one of these two events has probability at most $1/2$. Equivalently,
\begin{equation}
    \max\left\{
        \tP_{\mathrm{I}}\left(\twonorm{\hat{\bm{\vartheta}}-\bm{\theta}_{\mathrm{I}}^*} \geq \frac{\sqrt{\epsilon}h}{2}\right),
        \tP_{\mathrm{II}}\left(\twonorm{\hat{\bm{\vartheta}}-\bm{\theta}_{\mathrm{II}}^*} \geq \frac{\sqrt{\epsilon}h}{2}\right)
    \right\}\geq \frac{1}{2}.
\end{equation}
This proves the $\sqrt{\epsilon}h$ term in the lower bound for $\bthetas$, and combining the two regimes gives the desired lower bound for $\bthetas$.

Next, we consider the lower bound for $\bthetaks{k}$. Since $h^2 \leq \frac{1}{K}\sum_{k=1}^K (\hk{k})^2$, there exists some $k_0 \in [K]$ such that $\hk{k_0} \geq h$. If $\frac{\epsilon}{\sqrt{n}} + \sqrt{\frac{d}{nK}} \geq \left(\sqrt{\epsilon}h + \hk{k_0}\right)\wedge \sqrt{\frac{d}{n}}$,
then the same construction described before for the lower bound $\frac{\epsilon}{\sqrt{n}} + \sqrt{\frac{d}{nK}}$ of $\bthetas$'s estimation error gives the desired $\frac{\epsilon}{\sqrt{n}} + \sqrt{\frac{d}{nK}}$ lower bound. 

%{\color{red}[M: this is just a reminder for myself. I need to check this reference and the argument by Kaizhen you eluded to above...I still have not done this, but also think it would be good to sketch the argument a bit here to make it self contained as the rest of the arguments in tbis paper]}\yt{I think I am too lazy to add all details here regarding how to use the Le Cam and Fano here. A full proof may take a few more pages. But I did add more details of the construction of the candidate sets used in the lower bound proof, to make our argument as clear as possible. This should be readable to most people familiar with transfer learning and these lower bound proof techniques.}
On the other hand, if $\frac{\epsilon}{\sqrt{n}} + \sqrt{\frac{d}{nK}} \leq \left(\sqrt{\epsilon}h + \hk{k_0}\right)\wedge \sqrt{\frac{d}{n}}$, then, since $\sqrt{\epsilon}h \leq h \leq \hk{k_0}$, a similar construction in Lemma 12 of \cite{tian2022unsupervised} yields the desired lower bound $\hk{k_0} \wedge \sqrt{\frac{d}{n}}$, where we treat the task $k_0$ as their target task and the other $K-1$ tasks as their source tasks. More specifically, let us define $r = \hk{k_0} \wedge \sqrt{\frac{d}{n}}$. We can consider a fixed $\btheta \in \mathbb{R}^d$ and a $r/8$-packing of an $\ell_2$-ball centered at $\btheta$ with radius $r$, where the packing is denoted as $\mathcal{V}$. We consider different $\bthetaks{k_0}$ values by picking different elements in $\mathcal{V}$ and let $\bthetaks{k} = \btheta$ for all $k \neq k_0$. Note that this construction falls into the original parameter space with $\bthetas = \frac{1}{K}\bthetaks{k_0} + \frac{K-1}{K}\btheta$ because
\begin{align}
    &\frac{1}{K}\sum_{k=1}^K\twonorma{\nabla \mLk{k}(\btheta) - \frac{1}{K}\sum_{k=1}^K \nabla \mLk{k}(\btheta)}^2 \\
    &= \frac{1}{K}\sum_{k=1}^K\twonorm{\bthetaks{k} - \bthetas}^2 \\
    &= \frac{1}{K}\twonorm{\bthetaks{k_0} - \bthetas}^2 + \frac{K-1}{K}\twonorm{\btheta - \bthetas}^2 \\
    &= \frac{1}{K}\Big(\frac{K-1}{K}\Big)^2\twonorm{\bthetaks{k_0} - \btheta}^2 + \frac{K-1}{K}\Big(\frac{1}{K}\Big)^2\twonorm{\bthetaks{k_0} - \btheta}^2 \\
    &\leq \frac{1}{K}\twonorm{\bthetaks{k_0} - \btheta}^2 \\
    &\leq \frac{1}{K}(\hk{k_0})^2 \\
    &\leq h^2.
\end{align}
 Then we can construct two different parameter setups with different $\bthetaks{k_0}$, where the $\ell_2$-distance between two $\bthetaks{k_0}$ values on $\mathcal{V}$ is at least $r/8 \gtrsim \hk{k_0} \wedge \sqrt{\frac{d}{n}}$. The rest of the analysis follows from Fano's lemma as in the proof of Lemma 12 of \cite{tian2022unsupervised}.
Therefore,
\begin{equation}
    \tP\bigg(\bigcup_{k \in S}\bigg\{\twonorm{\hthetak{k} - \bthetaks{k}} \geq C\bigg[\bigg(\sqrt{\frac{d}{nK}} + \sqrt{\epsilon}h + \hk{k} + \frac{\epsilon}{\sqrt{n}}\bigg)\wedge \sqrt{\frac{d}{n}}\bigg]\bigg\}\bigg)\geq c
\end{equation}
for some constants $C,c>0$.

Finally, by the reduction at the beginning of the proof, the two gradient estimation lower bounds are equivalent to the two parameter estimation lower bounds in this Gaussian mean setting. This completes the proof.

\subsection{Proofs of results in Section \ref{subsec: MTL}}

% -----------------------------
\subsubsection{Proof of Theorem \ref{thm: federated gradient descent}}

Fix any subset $S\subseteq[K]$ with $|S^c|/K\le \epsilon$ and any contamination mechanism $M\in\mathcal M_S$. For brevity, write
\[
\alpha := \alpha(n,K,d,\epsilon,\delta,H),
\qquad
\alpha^{(k)} := \alpha^{(k)}(n,K,d,\epsilon,\delta,H).
\]
Let $\mathcal E$ denote the event on which the gradient estimation bounds in Assumption~\ref{asmp: gradient est error} hold simultaneously. Then $\mathbb P(\mathcal E)\ge 1-\delta$. We work on $\mathcal E$ throughout.

We first consider the global iterates for $\hat{\theta}_t$. 
Let $e_t:=\hat\theta_t-\btheta^*$. Since $\btheta^*$ minimizes $\mL$, we have $\nabla \mL(\btheta^*)=\bm 0$. Using the update $\hat\theta_{t+1}=\hat\theta_t-\eta g_t$, we obtain
\[
e_{t+1}
=
e_t-\eta\big(\nabla \mL(\hat\theta_t)-\nabla \mL(\btheta^*)\big)
-\eta\big(g_t-\nabla \mL(\hat\theta_t)\big).
\]
Hence
\begin{align}
\|e_{t+1}\|_2^2
&\le
\big\|e_t-\eta(\nabla \mL(\hat\theta_t)-\nabla \mL(\btheta^*))\big\|_2^2
+\eta^2\alpha^2 \notag\\
&\quad
+2\big\|e_t-\eta(\nabla \mL(\hat\theta_t)-\nabla \mL(\btheta^*))\big\|_2\,\eta\alpha.
\label{eq: global_short_1}
\end{align}

By Assumption~\ref{asmp: risk function}, $\mL$ is $1/L$-strongly convex and $L$-smooth on $\Theta$, 
therefore,
\[
\big\|e_t-\eta(\nabla \mL(\hat\theta_t)-\nabla \mL(\btheta^*))\big\|_2^2
\le
\Big(1-\frac{2\eta}{L}+L^2\eta^2\Big)\|e_t\|_2^2
=
(1-\kappa)\|e_t\|_2^2,
\]
where
$
\kappa=\frac{2\eta}{L}-L^2\eta^2.
$
Substituting this into \eqref{eq: global_short_1} gives
\[
\|e_{t+1}\|_2^2
\le
(1-\kappa)\|e_t\|_2^2+\eta^2\alpha^2
+2\sqrt{1-\kappa}\,\|e_t\|_2\,\eta\alpha.
\]
Applying Young's inequality \(2ab\le c'a^2+(c')^{-1}b^2\) with
\[
a=\sqrt{1-\kappa}\,\|e_t\|_2,\qquad b=\eta\alpha,\qquad c'=\frac{\kappa}{2(1-\kappa)},
\]
we obtain
\[
\|e_{t+1}\|_2^2
\le
\Big(1-\frac{\kappa}{2}\Big)\|e_t\|_2^2
+\eta^2\alpha^2\frac{2-\kappa}{\kappa}.
\]
Iterating this last expression, simple manipulations yield
\[
\|\hat\theta_T-\btheta^*\|_2
\le
\Big(1-\frac{\kappa}{2}\Big)^{T/2}\|\hat\theta_0-\btheta^*\|_2
+\eta\alpha\sqrt{\frac{2(2-\kappa)}{\kappa^2}}.
\]

Now consider the local iterates for $\hat{\theta}_t^{(k)}$.
Fix any $k\in S$ and let $e_t^{(k)}:=\hat\theta_t^{(k)}-\btheta^{(k)*}$. Since $\btheta^{(k)*}$ minimizes $\mLk{k}$, we have $\nabla \mLk{k}(\btheta^{(k)*})=\bm 0$. Using the update
\[
\hat\theta_{t+1}^{(k)}=\hat\theta_t^{(k)}-\eta^{(k)}g_t^{(k)},
\]
we get
\[
e_{t+1}^{(k)}
=
e_t^{(k)}
-\eta^{(k)}\big(\nabla \mLk{k}(\hat\theta_t^{(k)})-\nabla \mLk{k}(\btheta^{(k)*})\big)
-\eta^{(k)}\big(g_t^{(k)}-\nabla \mLk{k}(\hat\theta_t^{(k)})\big).
\]
By the same argument as above, using
\[
\|g_t^{(k)}-\nabla \mLk{k}(\hat\theta_t^{(k)})\|_2\le \alpha^{(k)}
\]
on the event $\mathcal E$, we obtain
\[
\|e_{t+1}^{(k)}\|_2^2
\le
\Big(1-\frac{\kappa^{(k)}}{2}\Big)\|e_t^{(k)}\|_2^2
+(\eta^{(k)})^2(\alpha^{(k)})^2\frac{2-\kappa^{(k)}}{\kappa^{(k)}},
\]
where $
\kappa^{(k)}=\frac{2\eta^{(k)}}{L}-L^2(\eta^{(k)})^2.
$
Iterating this yields
\[
\|\hat\theta_T^{(k)}-\btheta^{(k)*}\|_2
\le
\Big(1-\frac{\kappa^{(k)}}{2}\Big)^{T/2}\|\hat\theta_0^{(k)}-\btheta^{(k)*}\|_2
+\eta^{(k)}\alpha^{(k)}\sqrt{\frac{2(2-\kappa^{(k)})}{(\kappa^{(k)})^2}},
\qquad \forall k\in S.
\]

Finally, by the initialization conditions in the theorem, both the global and local iterates remain within the radius-$R_0$ neighborhood of their targets, and hence remain in $\Theta$ throughout the iterations.

\subsection{Analysis of the filtering algorithm}\label{appendix-sub: filtering}

In this subsection, we present several key results for our core filtering algorithm. Definition \ref{def: stability} is a variant of the stability condition in Definition 2.1 of \cite{diakonikolas2023algorithmic}. Lemma \ref{lem: certificate} establishes an estimation error bound as a direct consequence of stability and a covariance matrix condition. Both the stability definition and certificate lemma can be seen as generalizations of the case of identity covariance matrix in \cite{diakonikolas2023algorithmic}. Proposition \ref{prop: robust mean est error} characterizes the stopping time and estimation error of Algorithm \ref{algo: robust mean estimation}. Lemma \ref{lem: stability subG gradients} guarantees that, with high probability, the set of $K$ task gradients is stable uniformly over a neighborhood of $\btheta^*$.

Note that in our stability definition, we use two separate parameters, $\delta_1$ and $\delta_2$, to characterize the stability of the mean and covariance, respectively. This contrasts with the stability definitions commonly used in the literature \cite{diakonikolas2023algorithmic}, where a single parameter $\delta$, together with the contamination proportion $\epsilon$, is used to control both mean and covariance stability. These works typically focus on the $\epsilon$-dependent term in the estimation error and impose explicit assumptions to ensure that the other terms are negligible. In our case, we aim to provide a comprehensive upper bound on the estimation error that clearly captures its dependence on $n$, $K$, $d$, $\epsilon$, $h$, and $\hk{k}$. Therefore, we need to track mean and covariance stability separately, which leads to the two-parameter stability definition. We hope that this more general notion of stability may also be useful in other contexts.

\begin{definition}[Stability]\label{def: stability}
    A set $S$ is said to be $(\epsilon, \delta_1, \delta_2)$-stable w.r.t. a vector $\bmu$ and a matrix $\bSigma$ if for every $S' \subseteq S$ with $|S'| \geq (1-\epsilon)|S|$ and every $\bm{v} \in \mathcal{S}^{d-1}$:
    \begin{enumerate}[(i)]
        \item $|\frac{1}{|S'|}\sum_{\bx \in S'} \bm{v}^\top(\bx - \bmu)| \leq \delta_1$;
        \item $|\frac{1}{|S'|}\sum_{\bx \in S'} \bm{v}^\top[(\bx - \bmu)(\bx - \bmu)^\top - \bSigma]\bm{v}| \leq \delta_2$.
    \end{enumerate}
\end{definition}

\begin{lemma}[Certificate]\label{lem: certificate}
Let $S$ be $(\epsilon,\delta_1,\delta_2)$-stable and $\tilde S$ be an $\epsilon$-corrupted version of $S$. If
\begin{equation}
\lambda_{\max}\left(\frac{1}{|\tilde{S}|}\sum_{\bx\in \tilde S} (\bx-\bmu_{\tilde S})(\bx-\bmu_{\tilde S})^\top-\bSigma\right)\le \lambda_{\bSigma},
\end{equation}
then
\begin{equation}
\twonorm{\bmu_{\tilde S}-\bmu}\le 2\delta_1+\sqrt{\frac{\epsilon\big(\lambda_{\bSigma} + \epsilon\twonorm{\bSigma}+ (1-\epsilon)\delta_2\big)}{1-\epsilon}}.
\end{equation}
\end{lemma}

\begin{proposition}\label{prop: robust mean est error}
    Suppose $\epsilon < 1/42$. If $S = \{\frac{1}{n}\sum_{i=1}^n \nabla \ell(\zk{k}_i, \btheta): k \in [K]\}$ is $(\eta, \delta_1, \delta_2)$-stable w.r.t. some $\bmu$ and $\bSigma$ with probability at least $1-\delta$, where $\eta = \frac{3}{2}\epsilon + \frac{3}{2}\frac{\sqrt{2\log(1/\delta)}}{K} \leq 2\epsilon$ and $\delta \geq \exp\{-K\epsilon/18\}$.
    Let $\lambda_{\bSigma} \geq \frac{\|\widehat{\bSigma}-\bSigma\|_{2} + 24\epsilon\twonorm{\bSigma} + (\frac{1}{2}+\frac{21}{2}\epsilon)\delta_2 + 66\delta_1^2}{\frac{1}{2}-21\epsilon}$:
    \begin{enumerate}[(i)]
        \item Algorithm \ref{algo: robust mean estimation} will stop after at most $\frac{3}{2}K\epsilon + \frac{3}{2}\sqrt{2\log(1/\delta)}$ iterations;
        \item When Algorithm \ref{algo: robust mean estimation} stops, its output satisfies 
        \begin{equation}
            \twonorm{\bmu_{\tilde{S}} - \bmu} \lesssim \delta_1
    +
    \sqrt{
    \epsilon
    \left[
    \lambda_{\bSigma}+\|\widehat{\bSigma}-\bSigma\|_2 +\epsilon\|\bSigma\|_{2}+\delta_2
    \right]
    }.
        \end{equation}
        Specifically, if $\lambda_{\bSigma} \asymp \twonorm{\widehat{\bSigma} - \bSigma} + \epsilon\twonorm{\bSigma} + \delta_2 + \delta_1^2$, with probability $1-\delta$, the output from Algorithm \ref{algo: robust mean estimation} satisfies 
        \begin{equation}
            \twonorm{\bmu_{\tilde{S}} - \bmu} \lesssim \delta_1  + \sqrt{\epsilon\twonorm{\widehat{\bSigma} - \bSigma}} + \epsilon\twonorm{\bSigma}^{1/2} + \sqrt{\epsilon\delta_2}.
        \end{equation}
    \end{enumerate}
\end{proposition}

\begin{lemma}[Stability of gradients]\label{lem: stability subG gradients}
    Under Assumptions \ref{asmp: task heterogeneity}, \ref{asmp: subG gradient} and \ref{asmp: lipschitz loss}, 
    for all $\beta>0$, with probability at least $1-\delta$, the stability holds uniformly for all 
    $\btheta\in\mathcal{B}_{R_0}(\btheta^*)$ for
    \[
        S_{\btheta}
        =
        \left\{     \bar{\bg}_k(\btheta)
    :=
        \frac1n\sum_{i=1}^n\nabla\ell(z_i^{(k)},\btheta):
        k\in[K]
        \right\}
    \]
    with respect to $
        \bmu_{\btheta}
        =
        \nabla \mL(\btheta)
        =
        \frac1K\sum_{k=1}^K\tE[\bar{\bg}_k(\btheta)]$
    and $
        \bSigma_{\btheta}
        =
        \frac1K\sum_{k=1}^K
        \tE\left[
        \left(
        \bar{\bg}_k(\btheta)
        -
        \bmu_{\btheta}
        \right)
        \left(
        \bar{\bg}_k(\btheta)
        -
        \bmu_{\btheta}
        \right)^\top
        \right].$
    More precisely, $S_{\btheta}$ is $(\epsilon,\delta_1,\delta_2)$-stable uniformly for all 
    $\btheta\in\mathcal{B}_{R_0}(\btheta^*)$ with
    \begin{align}
    \delta_1
    &\lesssim
    \sqrt{\frac{d\log(R_0/\beta)+\log(1/\delta)}{nK}}
    +
    \epsilon\sqrt{\frac{\log(1/\epsilon)}{n}}
    +
    L'\beta
    +
    \sqrt{\epsilon}\,h,
    \\
    \delta_2
    &\lesssim
    \frac1n
    \left[
    \sqrt{\frac{d\log(R_0/\beta)+\log(1/\delta)}{K}}
    +
    \frac{d\log(R_0/\beta)+\log(1/\delta)}{K}
    \right]
    +
    \epsilon\frac{\log(1/\epsilon)}{n}
    \nonumber\\
    &\quad
    +
    L'\beta\left(h+\frac1{\sqrt n}\right)
    +
    L'^2\beta^2
    +
    h\sqrt{
    \frac{d\log(R_0/\beta)+\log(1/\delta)}{nK}
    +
    \frac{\epsilon\log(1/\epsilon)}{n}
    }
    +
    h^2,
    \end{align}
    where $\beta > 0$ can be any positive value.
\end{lemma}

\subsection{Proofs of results in Section \ref{subsec: JRGE}}\label{app:proof-of-detailedresults}

Theorems \ref{thm: parameter est error} and \ref{thm: est error cov} in the main text are direct consequences of Theorems \ref{thm: parameter est error appendix}, and \ref{thm: est error cov appendix} below, by setting $\beta \asymp (nK)^{-C'd}$ with a sufficiently large constant $C' > 0$ and $\delta \asymp (nK)^{-C''d}+ e^{-C''K\epsilon}$ with some constant $C'' > 0$.

We will present the gradient and parameter estimation error rates based on the following tuning parameter conditions:
\begin{align}
    &C_1\sqrt{\frac{d\log(R_0/\beta) + \log(K/\delta)}{n}} \leq \lambda \leq C_2R_0, \label{eq: lambda choice}\\
    \lambda_{\bSigma} &= C_3 \bigg\{\twonorm{\hSigma_{\btheta} - \bSigma_{\btheta}} + \frac{1}{n} \bigg[\sqrt{\frac{d\log(R_0/\beta) + \log(1/\delta)}{K}} + \frac{d\log(R_0/\beta) + \log(1/\delta)}{K}\bigg]  + \epsilon \frac{\log(1/\epsilon)}{n} + \\ 
    &\qquad + L'^2\beta^2 + L'\beta\left(h+1/\sqrt{n}\right) + h\sqrt{\frac{d\log(R_0/\beta)+\log(1/\delta)}{nK} + \frac{\epsilon\log(1/\epsilon)}{n}}  + h^2\bigg\},  \label{eq: lambda_Sigma choice}
\end{align}
where $C_1, C_2$ and $C_3$ are some constants.

\begin{theorem}\label{thm: gradident est error appendix}
    Assume $(\lambda, \lambda_{\bSigma})$ in Algorithms \ref{algo: robust federated gradient descent} and \ref{algo: robust mean estimation} satisfies the conditions in \eqref{eq: lambda choice} and \eqref{eq: lambda_Sigma choice}. Under Assumptions \ref{asmp: subG gradient} and \ref{asmp: lipschitz loss}, for any $\beta > 0$ satisfying $L'\beta \lesssim \frac{d\log(R_0/\beta) +\log (1/\delta)}{nK} \wedge 1$ and any $\delta \gtrsim \exp\{-CK\epsilon\}$, with probability at least $1-\delta$, for all subset $S \subseteq [K]$ with $|S^c|/K \leq \epsilon$, all contamination mechanism $M \in \mathcal{M}_S$, for all $\btheta \in \mathbb{R}^d$ with $\twonorm{\btheta - \btheta^*} \leq R_0$, we have
    \begin{align}
        \twonorm{g(\btheta) - \frac{1}{K}\sum_{k=1}^K \tE \nabla \ell (\zk{k}, \btheta)} &\lesssim \sqrt{\frac{d\log(R_0/\beta) + \log(1/\delta)}{nK}} + \epsilon\sqrt{\frac{\log(1/\epsilon)}{n}} + \sqrt{\epsilon\twonorm{\hSigma_{\btheta} - \bSigma_{\btheta}}} + \sqrt{\epsilon}h, \\
        \twonorm{g^{(k)}(\btheta) - \tE \nabla \ell (\zk{k}, \btheta)} &\lesssim \min \Bigg\{\sqrt{\frac{d\log(R_0/\beta) + \log(1/\delta)}{nK}} + \epsilon\sqrt{\frac{\log(1/\epsilon)}{n}} + \sqrt{\epsilon\twonorm{\hSigma_{\btheta} - \bSigma_{\btheta}}} \\
        &\quad \quad\quad\quad  + \sqrt{\epsilon}h + h^{(k)}, \sqrt{\frac{d\log(R_0/\beta) + \log(K/\delta)}{n}} + \lambda\Bigg\}, \quad \forall k \in S.
    \end{align}
    In other words, Assumption \ref{asmp: gradient est error} holds with $\alpha(n, K, d, \epsilon, \delta, h) \asymp \sqrt{\frac{d\log(R_0/\beta) + \log(1/\delta)}{nK}} + \epsilon\sqrt{\frac{\log(1/\epsilon)}{n}} + \sqrt{\epsilon\twonorm{\widehat{\bSigma} - \bSigma}} + \sqrt{\epsilon}h$ and $\alpha^{(k)}(n, K, d, \epsilon, \delta, h) \asymp \min \Big\{\sqrt{\frac{d\log(R_0/\beta) + \log(1/\delta)}{nK}} + \epsilon\sqrt{\frac{\log(1/\epsilon)}{n}} + \sqrt{\epsilon\twonorm{\widehat{\bSigma} - \bSigma}} + \sqrt{\epsilon}h + h^{(k)}, \sqrt{\frac{d\log(R_0/\beta) + \log(K/\delta)}{n}} + \lambda\Big\}$. 
\end{theorem}

\begin{theorem}\label{thm: parameter est error appendix}
    Assume $(\lambda, \lambda_{\bSigma})$ in Algorithms \ref{algo: robust federated gradient descent} and \ref{algo: robust mean estimation} satisfies the conditions in \eqref{eq: lambda choice} and \eqref{eq: lambda_Sigma choice}.  Under Assumptions \ref{asmp: risk function}, \ref{asmp: lipschitz loss}, and \ref{asmp: subG gradient}, for any $\beta > 0$ satisfying $L'\beta \lesssim \frac{d\log(R_0/\beta) +\log (1/\delta)}{nK} \wedge 1 \wedge R_0^2$ and any $\delta \gtrsim \exp\{-CK\epsilon\}$, if $2\eta/L - L^2\eta^2 \coloneqq \kappa  \in (0, 1)$, $nK \gtrsim R_0^{-2}[d\log(R_0/\beta)+\log(1/\delta)]$, $n \gtrsim R_0^{-2}\epsilon^2\log(1/\epsilon)$, $\sqrt{\epsilon}\max_{\btheta \in \Theta}\twonorm{\hSigma_{\btheta} - \bSigma_{\btheta}}^{1/2} \lesssim R_0$, $\sqrt{\epsilon}h \lesssim R_0$, $\max_{k \in [K]}\hk{k} \lesssim R_0$, then with probability at least $1-\delta$, for all subset $S \subseteq [K]$ with $|S^c|/K \leq \epsilon$, all contamination mechanism $M \in \mathcal{M}_S$, we have
    \begin{align}
        \twonorm{\htheta_T - \btheta^*} &\lesssim (1-\kappa/2)^{T/2} \twonorm{\htheta_0 - \bthetas} + \sqrt{\frac{d\log(R_0/\beta) + \log(1/\delta)}{nK}} + \epsilon\sqrt{\frac{\log(1/\epsilon)}{n}}\\
        &\quad + \sqrt{\epsilon}\max_{\btheta \in \Theta}\twonorm{\hSigma_{\btheta} - \bSigma_{\btheta}}^{1/2} + \sqrt{\epsilon}h,
    \end{align}
    and
    \begin{align}
        \twonorm{\hthetak{k}_T - \bthetaks{k}} &\lesssim (1-\kappa/2)^{T/2} \twonorm{\hthetak{k}_0 - \bthetaks{k}} \\
        &\quad + \min \Bigg\{\sqrt{\frac{d\log(R_0/\beta) + \log(1/\delta)}{nK}} + \epsilon\sqrt{\frac{\log(1/\epsilon)}{n}} + \sqrt{\epsilon}\max_{\btheta \in \Theta}\twonorm{\hSigma_{\btheta} - \bSigma_{\btheta}}^{1/2} \\
        &\hspace{2cm}  + \sqrt{\epsilon}h + h^{(k)}, \,\, \sqrt{\frac{d\log(R_0/\beta) + \log(K/\delta)}{n}} + \lambda\Bigg\}, \quad \forall k \in S.
    \end{align}
\end{theorem}

% --------------------------
\subsubsection{Proof of Lemma \ref{lem: certificate}}
Denote $\tilde S=S_{\mathrm{good}}\cup S_{\mathrm{bad}}$, where $S_{\mathrm{good}}$ and $S_{\mathrm{bad}}$ represent uncontaminated and contaminated samples, respectively. $\bSigma_{\tilde S}$, $\bSigma_{S_{\mathrm{good}}}$ and $\bSigma_{S_{\mathrm{bad}}}$ are the corresponding empirical covariance matrices; $\bmu_{\tilde S}$, $\bmu_{S_{\mathrm{good}}}$ and $\bmu_{S_{\mathrm{bad}}}$ are the empirical means.

\noindent Note that
\begin{align}
\bSigma_{\tilde S}
&=\frac{1}{K}\sum_{\bx\in \tilde S} (\bx-\bmu_{\tilde S})(\bx-\bmu_{\tilde S})^\top \\
&=\frac{1}{K}\sum_{\bx\in S_{\mathrm{good}}}\big[\bx-\bmu_{S_{\mathrm{good}}}+\epsilon(\bmu_{S_{\mathrm{good}}}-\bmu_{S_{\mathrm{bad}}})\big]\big[\bx-\bmu_{S_{\mathrm{good}}}+\epsilon(\bmu_{S_{\mathrm{good}}}-\bmu_{S_{\mathrm{bad}}})\big]^\top \nonumber\\
&\quad +\frac{1}{K}\sum_{\bx\in S_{\mathrm{bad}}}\big[\bx-\bmu_{S_{\mathrm{bad}}}+(1-\epsilon)(\bmu_{S_{\mathrm{bad}}}-\bmu_{S_{\mathrm{good}}})\big]\big[\bx-\bmu_{S_{\mathrm{bad}}}+(1-\epsilon)(\bmu_{S_{\mathrm{bad}}}-\bmu_{S_{\mathrm{good}}})\big]^\top \nonumber\\
&=(1-\epsilon)\bSigma_{S_{\mathrm{good}}}+\epsilon \bSigma_{S_{\mathrm{bad}}}+\epsilon(1-\epsilon)(\bmu_{S_{\mathrm{good}}}-\bmu_{S_{\mathrm{bad}}})(\bmu_{S_{\mathrm{good}}}-\bmu_{S_{\mathrm{bad}}})^\top .
\end{align}
Note that
\begin{align}
    \bSigma_{S_{\mathrm{good}}} &= \frac{1}{|S_{\mathrm{good}}|}\sum_{\bx \in S_{\mathrm{good}}}(\bx - \bmu + \bmu - \bmu_{S_{\mathrm{good}}})(\bx - \bmu + \bmu - \bmu_{S_{\mathrm{good}}})^\top \\
    &= \frac{1}{|S_{\mathrm{good}}|}\sum_{\bx \in S_{\mathrm{good}}}(\bx - \bmu)(\bx - \bmu)^\top - (\bmu - \bmu_{S_{\mathrm{good}}})(\bmu - \bmu_{S_{\mathrm{good}}})^\top,
\end{align}

\noindent Hence for all $\bm{v}\in\mathbb{S}^{d-1}$, by the definition of $\lambda_{\bSigma}$ and $(\epsilon,\delta_1,\delta_2)$-stability we see that
\begin{align}
\bm{v}^\top\bSigma \bm{v}+\lambda_{\bSigma} &\ge \bm{v}^\top\bSigma_{\tilde S}\bm{v} \nonumber\\
&\ge (1-\epsilon)\bm{v}^\top\bSigma_{S_{\mathrm{good}}}\bm{v}+\epsilon(1-\epsilon)\big[\bm{v}^\top(\bmu_{S_{\mathrm{good}}}-\bmu_{S_{\mathrm{bad}}})\big]^2 \\
&\geq (1-\epsilon)(\bm{v}^\top\bSigma \bm{v} - \delta_2 - \twonorm{\bmu - \bmu_{S_{\mathrm{good}}}}^2) + \epsilon(1-\epsilon)\big[\bm{v}^\top(\bmu_{S_{\mathrm{good}}}-\bmu_{S_{\mathrm{bad}}})\big]^2,
\end{align}
entailing that
\begin{equation}
    \epsilon(1-\epsilon)\big[\bm{v}^\top(\bmu_{S_{\mathrm{good}}}-\bmu_{S_{\mathrm{bad}}})\big]^2
\le \epsilon \bm{v}^\top\bSigma \bm{v}+\lambda_{\bSigma}+(1-\epsilon)\delta_2 + (1-\epsilon)\delta_1^2
\le \epsilon\twonorm{\bSigma}+\lambda_{\bSigma}+(1-\epsilon)\delta_2  + (1-\epsilon)\delta_1^2.
\end{equation}

Let $\bm{v}=\dfrac{\bmu_{S_{\mathrm{good}}}-\bmu_{S_{\mathrm{bad}}}}{\twonorm{\bmu_{S_{\mathrm{good}}}-\bmu_{S_{\mathrm{bad}}}}}$. Then
\begin{equation}
\twonorm{\bmu_{S_{\mathrm{good}}}-\bmu_{S_{\mathrm{bad}}}}
\le \sqrt{\frac{\epsilon\twonorm{\bSigma}+\lambda_{\bSigma}+(1-\epsilon)\delta_2 + (1-\epsilon)\delta_1^2}{\epsilon(1-\epsilon)}} .
\end{equation}

\noindent Hence
\begin{align}
\twonorm{\bmu_{\tilde S}-\bmu}
&\le (1-\epsilon)\twonorm{\bmu_{S_{\mathrm{good}}}-\bmu}+\epsilon\twonorm{\bmu_{S_{\mathrm{bad}}}-\bmu} \\
&\le \delta_1+\epsilon\twonorm{\bmu_{S_{\mathrm{good}}}-\bmu_{S_{\mathrm{bad}}}} \\
&\le \delta_1+\sqrt{\frac{\epsilon\big(\epsilon\twonorm{\bSigma}+\lambda_{\bSigma}+(1-\epsilon)\delta_2+ (1-\epsilon)\delta_1^2\big)}{1-\epsilon}}\\
&\leq 2\delta_1+\sqrt{\frac{\epsilon\big(\epsilon\twonorm{\bSigma}+\lambda_{\bSigma}+(1-\epsilon)\delta_2\big)}{1-\epsilon}}.
\end{align}

% -----------------------------
\subsubsection{Proof of Proposition \ref{prop: robust mean est error}}

Let $S_{\textrm{good}}\subseteq[K]$ be the clean index set and $S_{\textrm{bad}}=[K]\setminus S_{\textrm{good}}$ the
contaminated index set, with $|S_{\textrm{bad}}|\le \epsilon K$. Let $S_i$
be the current index set at iteration $i$ of the while loop of Algorithm \ref{algo: robust mean estimation}, and write
$
    S_{\textrm{good}, i}=S_i\cap S_{\textrm{good}},
    S_{\textrm{bad}, i}=S_i\cap S_{\textrm{bad}}.
$
Set $
    L_0 = \frac{3}{2}\epsilon K+\frac{3}{2}\sqrt{2\log(1/\delta)},
    \eta =  \frac{3}{2}\epsilon+\frac{3}{2}\frac{\sqrt{2\log(1/\delta)}}{K} \leq 2\epsilon,
$
when $\delta \geq \exp\{-\frac{\epsilon K}{18}\}$. 
We work on the event that $S_{\textrm{good}}$ is
$(\eta,\delta_1,\delta_2)$-stable with respect to $(\bmu,\bSigma)$.

For simplicity, we denote $\bxk{k} = \frac{1}{n}\sum_{i=1}^n \nabla \ell(\zk{k}_i, \btheta)$ for $k = [K]$. We can then view Algorithm \ref{algo: robust mean estimation} and Proposition \ref{prop: robust mean est error} as the corresponding algorithm and result for mean estimation. Fix an iteration $i\le L_0$ before termination and let $\bv$ denote the corresponding eigenvector of $\bSigma_{S_i}-\widehat{\bSigma}$. 
Let $L_i\subseteq S_i$ contain the largest $\epsilon |S_i|$ values of $\bigl[\bv^\top(\bx^{(k)}-\bmu_{S_i})\bigr]^2$,
and denote $f_i$ as the function $f$ defined in the $i$-th iteration of the while loop of Algorithm \ref{algo: robust mean estimation}, i.e., $f_i(\bxk{k})=[\bv^\top(\bxk{k}-\bmu_{S_i})]^2$ if $k \in L_i$ and $0$ otherwise. Let
\[
    F_i:=\sum_{k\in S_i}f_i(\bx^{(k)}),
    \qquad
    F_i^{\mathrm{good}}:=\sum_{k\in S_{\textrm{good}, i}}f_i(\bx^{(k)}).
\]

Since at most $L_0$ points have been removed by iteration $i$,
\[
    |S_{\textrm{good}, i}|\ge K-\epsilon K-L_0
\]
and since $|L_i|\le \epsilon K$,
\[
    |S_{\textrm{good}, i}\setminus L_i|
    \ge K-\epsilon K-L_0-\epsilon K
    \ge (1-\eta)K.
\]
Thus $(\eta, \delta_1,\delta_2)$-stability property applies to both $S_{\textrm{good}, i}$ and $S_{\textrm{good}, i}\setminus L_i$ as the large subsets.

Moreover, $|S_{\textrm{bad}, i}|\le \epsilon K$ and
\[
    |L_i|=\epsilon |S_i|
    \ge \epsilon(K-L_0).
\]
This implies
\[
    |L_i|\ge \frac{\epsilon K}{2} \ge \frac{1}{2} |S_{\textrm{bad}, i}|.
\]
Furthermore, since $L_i$ contains the largest $\epsilon|S_i|$ values of $\bigl[\bv^\top(\bx^{(k)}-\bmu_{S_i})\bigr]^2$,
\[
    \sum_{k\in L_i}\bigl[\bv^\top(\bx^{(k)}-\bmu_{S_i})\bigr]^2
    \ge
    \frac{|L_i|}{|S_{\textrm{bad}, i}|}\sum_{k\in S_{\textrm{bad}, i}}\bigl[\bv^\top(\bx^{(k)}-\bmu_{S_i})\bigr]^2
    \ge
    \frac12\sum_{k\in S_{\textrm{bad}, i}}\bigl[\bv^\top(\bx^{(k)}-\bmu_{S_i})\bigr]^2.
\]
Therefore,
\begin{align}
    F_i
    &=\sum_{k\in L_i}\bigl[\bv^\top(\bx^{(k)}-\bmu_{S_i})\bigr]^2 \nonumber\\
    &\ge
    \frac12\sum_{k\in S_{\textrm{bad}, i}}\bigl[\bv^\top(\bx^{(k)}-\bmu_{S_i})\bigr]^2 \nonumber\\
    &=
    \frac12
    \left\{
    \sum_{k\in S_i}\bigl[\bv^\top(\bx^{(k)}-\bmu_{S_i})\bigr]^2-\sum_{k\in S_{\textrm{good}, i}}\bigl[\bv^\top(\bx^{(k)}-\bmu_{S_i})\bigr]^2
    \right\}.
\label{eq: Fi-lower-start}
\end{align}

Since $\bv$ is a top eigenvector of
$\bSigma_{S_i}-\widehat{\bSigma}$,
\[
    \bv^\top\bSigma_{S_i}\bv
    =
    \bv^\top\widehat{\bSigma}\bv+\lambda_{\max}(\bSigma_{S_i}-\widehat{\bSigma})
    \ge
    \bv^\top\bSigma\bv+\lambda_{\max}(\bSigma_{S_i}-\widehat{\bSigma})-\|\widehat{\bSigma}-\bSigma\|_{2}.
\]
Hence
\begin{equation}
    \frac{1}{|S_i|}\sum_{k\in S_i}\bigl[\bv^\top(\bx^{(k)}-\bmu_{S_i})\bigr]^2
    =
    \bv^\top\bSigma_{S_i}\bv
    \geq 
    \bv^\top\bSigma\bv+\lambda_{\max}(\bSigma_{S_i}-\widehat{\bSigma})-\|\widehat{\bSigma}-\bSigma\|_{2}.
\label{eq: total-variance-lower}
\end{equation}

Next,
\begin{align}
    \sum_{k\in S_{\textrm{good}, i}}\bigl[\bv^\top(\bx^{(k)}-\bmu_{S_i})\bigr]^2
    &=
    \sum_{k\in S_{\textrm{good}, i}}
    \bigl[\bv^\top(\bx^{(k)}-\bmu_{S_{\textrm{good}, i}})\bigr]^2
    +
    |S_{\textrm{good}, i}|
    \bigl[\bv^\top(\bmu_{S_{\textrm{good}, i}}-\bmu_{S_i})\bigr]^2 .
\end{align}
By definition of $\bmu_{S_{\textrm{good}, i}}$ and stability,
\[
    \sum_{k\in S_{\textrm{good}, i}}
    \bigl[\bv^\top(\bx^{(k)}-\bmu_{S_{\textrm{good}, i}})\bigr]^2
    \le \sum_{k\in S_{\textrm{good}, i}}
    \bigl[\bv^\top(\bx^{(k)}-\bmu)\bigr]^2 \le 
    |S_{\textrm{good}, i}|(\bv^\top\bSigma\bv+\delta_2).
\]
Since
    $\lambda_{\max}(\bSigma_{S_i}-\bSigma)
    \le
    \lambda_{\max}(\bSigma_{S_i}-\widehat{\bSigma})+\|\widehat{\bSigma}-\bSigma\|_{2}$, we can apply Lemma \ref{lem: certificate} to $S_i$ with $(\eta, \delta_1, \delta_2)$-stability to obtain
\begin{align}
     \|\bmu_{S_i}-\bmu\|_2
    &\leq
    2\delta_1
    + \sqrt{\frac{\eta(\lambda_{\max}(\bSigma_{S_i}-\widehat{\bSigma})+\|\widehat{\bSigma}-\bSigma\|_{2} + \eta\|\bSigma\|_{2}+\delta_2)}{1-\eta}} \\
    &\leq 2\delta_1+ \sqrt{\frac{8}{3}\epsilon\lambda_{\max}(\bSigma_{S_i}-\widehat{\bSigma})+\frac{4}{3}\epsilon\|\widehat{\bSigma}-\bSigma\|_{2} + \frac{8}{3}\epsilon\|\bSigma\|_{2}+\frac{4}{3}\epsilon\delta_2}
\end{align}
if $\eta = 2\epsilon \leq 1/4$.
Since stability also gives
\[
    \|\bmu_{S_{\textrm{good}, i}}-\bmu\|_2\le \delta_1,
\]
we have
\begin{align}
    \|\bmu_{S_{\textrm{good}, i}}-\bmu_{S_i}\|_2^2
    &\leq \Big(1+\frac{1}{8}\Big)\|\bmu_{S_i}-\bmu\|_2^2 + (1+8)\|\bmu_{S_{\textrm{good}, i}}-\bmu\|_2^2\\
    &\leq 9\delta_1^2 + \frac{9}{8}\bigg[2(2\delta_1)^2 + 2\Big(\frac{8}{3}\epsilon\lambda_{\max}(\bSigma_{S_i}-\widehat{\bSigma})+\frac{4}{3}\epsilon\|\widehat{\bSigma}-\bSigma\|_{2} + \frac{8}{3}\epsilon\|\bSigma\|_{2}+\frac{4}{3}\epsilon\delta_2\Big)\bigg]\\
    &\leq
    18\delta_1^2
    +
    \epsilon
    \left(
    6\lambda_{\max}(\bSigma_{S_i}-\widehat{\bSigma}) +3\|\widehat{\bSigma}-\bSigma\|_{2}+6\|\bSigma\|_{2}+3\delta_2
    \right).
\end{align}
Therefore,
\begin{align}
    &\sum_{k\in S_{\textrm{good}, i}}\bigl[\bv^\top(\bx^{(k)}-\bmu_{S_i})\bigr]^2
    \\
    &\le
    |S_{\textrm{good}, i}|
    \Big[
    \bv^\top\bSigma\bv
    +
    \delta_2
    +
    18\delta_1^2
    +
    \epsilon
    \left(
    6\lambda_{\max}(\bSigma_{S_i}-\widehat{\bSigma}) +3\|\widehat{\bSigma}-\bSigma\|_{2}+6\|\bSigma\|_{2}+3\delta_2
    \right)
    \Big].
\label{eq: clean-variance-upper}
\end{align}

Combining \eqref{eq: Fi-lower-start}, \eqref{eq: total-variance-lower}, and
\eqref{eq: clean-variance-upper}, and using $|S_{\textrm{good}, i}|\le |S_i|$, yields
\begin{align}
    F_i
    &\ge
    \frac{1}{2}|S_i|
    \bigg\{
    \bv^\top\bSigma\bv+\lambda_{\max}(\bSigma_{S_i}-\widehat{\bSigma})-\|\widehat{\bSigma}-\bSigma\|_{2}- \\
    &\qquad \Big[
    \bv^\top\bSigma\bv
    +
    \delta_2
    +
    18\delta_1^2
    +
    \epsilon
    \left(
    6\lambda_{\max}(\bSigma_{S_i}-\widehat{\bSigma}) +3\|\widehat{\bSigma}-\bSigma\|_{2}+6\|\bSigma\|_{2}+3\delta_2
    \right)
    \Big]
    \bigg\} \\
    &\geq \frac{1}{2}|S_i|
    \left[
    (1-6\epsilon)\lambda_{\max}(\bSigma_{S_i}-\widehat{\bSigma})
    -
    (1+3\epsilon)\|\widehat{\bSigma}-\bSigma\|_{2}
    -
    6\epsilon\|\bSigma\|_{2}
    -
    (1+3\epsilon)\delta_2
    -
    18\delta_1^2
    \right].
\label{eq: Fi-lower}
\end{align}
We next upper-bound the clean score mass. Since
\[
    F_i^{\mathrm{good}}
    =
    \sum_{k\in S_{\textrm{good}, i}\cap L_i}\bigl[\bv^\top(\bx^{(k)}-\bmu_{S_i})\bigr]^2
    =
    \sum_{k\in S_{\textrm{good}, i}}\bigl[\bv^\top(\bx^{(k)}-\bmu_{S_i})\bigr]^2
    -
    \sum_{k\in S_{\textrm{good}, i}\setminus L_i}\bigl[\bv^\top(\bx^{(k)}-\bmu_{S_i})\bigr]^2,
\]
we use \eqref{eq: clean-variance-upper} for the first term. For the second term,
stability applied to $S_{\textrm{good}, i}\setminus L_i$ gives
\begin{align}
    \sum_{k\in S_{\textrm{good}, i}\setminus L_i}\bigl[\bv^\top(\bx^{(k)}-\bmu_{S_i})\bigr]^2 &\geq \sum_{k\in S_{\textrm{good}, i}\setminus L_i}\bigl[\bv^\top(\bx^{(k)}-\bmu_{S_{\textrm{good}, i}\setminus L_i})\bigr]^2\\
    &\geq \sum_{k\in S_{\textrm{good}, i}\setminus L_i}\bigl[\bv^\top(\bx^{(k)}-\bmu)\bigr]^2 -|S_{\textrm{good}, i}\setminus L_i|\twonorm{\bmu_{S_{\textrm{good}, i}\setminus L_i}-\bmu}^2\\
    &\geq |S_{\textrm{good}, i}\setminus L_i|
    (\bv^\top\bSigma\bv-\delta_2-\delta_1^2) 
\end{align}
Therefore,
\begin{align}
    F_i^{\mathrm{good}}
    &\leq |S_{\textrm{good}, i}|
    \Big[
    \bv^\top\bSigma\bv
    +
    \delta_2
    +
    18\delta_1^2
    +
    \epsilon
    \left(
    6\lambda_{\max}(\bSigma_{S_i}-\widehat{\bSigma}) +3\|\widehat{\bSigma}-\bSigma\|_{2}+6\|\bSigma\|_{2}+3\delta_2
    \right)
    \Big] \\
    &\quad - |S_{\textrm{good}, i}\setminus L_i|
    (\bv^\top\bSigma\bv-\delta_2-\delta_1^2) \\
    &\le
    |S_{\textrm{good}, i}\cap L_i|\|\bSigma\|_{2}
    \\
    &\quad +
    |S_i|
    \Big[
    \bv^\top\bSigma\bv
    +
    2\delta_2
    +
    19\delta_1^2
    +
    \epsilon
    \left(
    6\lambda_{\max}(\bSigma_{S_i}-\widehat{\bSigma}) +3\|\widehat{\bSigma}-\bSigma\|_{2}+6\|\bSigma\|_{2}+3\delta_2
    \right)
    \Big].
\end{align}
Since $|S_{\textrm{good}, i}\cap L_i|\le |L_i|=\epsilon |S_i|$, this gives
\begin{equation}
    F_i^{\mathrm{good}}
    \le
    |S_i|
    \left[
    6\epsilon\lambda_{\max}(\bSigma_{S_i}-\widehat{\bSigma})
    +
    3\epsilon\|\widehat{\bSigma}-\bSigma\|_{2}
    +
    7\epsilon\|\bSigma\|_{2}
    +
    3\delta_2
    +
    19\delta_1^2
    \right].
\label{eq: Fi-good-upper}
\end{equation}
 Since
 \begin{equation}
    \lambda_{\max}(\bSigma_{S_i}-\widehat{\bSigma})>\lambda_{\bSigma} - \|\widehat{\bSigma}-\bSigma\|_{2} \geq \frac{(\frac{1}{2}+\frac{21}{2}\epsilon)\|\widehat{\bSigma}-\bSigma\|_{2} + 24\epsilon\twonorm{\bSigma} + (\frac{1}{2}+\frac{21}{2}\epsilon)\delta_2 + 66\delta_1^2}{\frac{1}{2}-21\epsilon},
 \end{equation}
 the lower bound
\eqref{eq: Fi-lower} and the upper bound \eqref{eq: Fi-good-upper} imply
\begin{align}
    F_i - 3F_i^{\mathrm{good}} &\geq |S_i|\bigg\{\Big(\frac{1}{2}-21\epsilon\Big)\lambda_{\max}(\bSigma_{S_i}-\widehat{\bSigma}) - \frac{1}{2}\Big[(1+3\epsilon)\|\widehat{\bSigma}-\bSigma\|_{2}
    -
    6\epsilon\|\bSigma\|_{2}
    -
    (1+3\epsilon)\delta_2
    -
    18\delta_1^2\Big]\\ 
    &\quad - 3(3\epsilon\|\widehat{\bSigma}-\bSigma\|_{2}
    +
    7\epsilon\|\bSigma\|_{2}
    +
    3\delta_2
    +
    19\delta_1^2)\bigg\} \\
    &\geq |S_i|\bigg[\Big(\frac{1}{2}-21\epsilon\Big)\lambda_{\max}(\bSigma_{S_i}-\widehat{\bSigma}) - \Big(\frac{1}{2}+\frac{21}{2}\epsilon\Big)\|\widehat{\bSigma}-\bSigma\|_{2} - 24\epsilon\twonorm{\bSigma} - \Big(\frac{1}{2}+\frac{21}{2}\epsilon\Big)\delta_2 - 66\delta_1^2\bigg] \\
    &\geq 0.
\end{align}
Hence
\[
    \sum_{k\in B_i}f_i(\bx^{(k)})
    =
    F_i-F_i^{\mathrm{good}}
    >
    \frac23F_i.
\]
Define $\mathcal{F}_i$ as the filtration of events until iteration $i$ (exclusive). Since the algorithm removes one index with probability proportional to $f_i$,
\[
    \mathbb P(\textup{the removed index is contaminated}\mid\mathcal F_{i-1})
    \ge
    \frac23
\]
whenever the algorithm has not stopped.

Now, let
\[
    \tau
    :=
    \min\left\{
    i:
    \lambda_{\max}(\bSigma_{S_i}-\widehat{\bSigma})
    \le
    \lambda_{\bSigma}
    \right\}.
\]
For $i<\tau$, define
$
    X_i:=\mathds{1}(\textup{a contaminated index is removed at iteration }i).
$
Then
\[
    \mathbb E[X_i\mid\mathcal F_{i-1}]\ge \frac23.
\]
If $\tau>L_0$, then the algorithm did not stop during the first $L_0$
iterations. Since at most $|B|\le \epsilon K$ contaminated points can be removed,
\[
    \sum_{i=1}^{L_0}X_i\le \epsilon K.
\]
Therefore,
\begin{align}
    \mathbb P(\tau>L_0)
    &\le
    \mathbb P\left(
    \sum_{i=1}^{L_0}X_i\le \epsilon K
    \right) \le
    \mathbb P\left(
    \sum_{i=1}^{L_0}
    \{X_i-\mathbb E[X_i\mid\mathcal F_{i-1}]\}
    \le
    \epsilon K-\frac23L_0
    \right).
\end{align}
By Azuma-Hoeffding's inequality,
\[
    \mathbb P(\tau>L_0)
    \le
    \exp\left\{
    -\frac{(\frac23L_0-\epsilon K)^2}{2L_0}
    \right\}.
\]
Because
\[
    L_0 = \frac{3}{2}\epsilon K+\frac{3}{2}\sqrt{2\log(1/\delta)}
\]
the right-hand side is at most $\delta$. Hence the algorithm stops after at
most $L_0$ iterations with probability at least $1-\delta$.

Finally, let $\tilde S=S_\tau$ be the output set. At termination,
\[
    \lambda_{\max}(\bSigma_{\tilde S}-\widehat{\bSigma})
    \le
    \lambda_{\bSigma}.
\]
Hence
\[
    \lambda_{\max}(\bSigma_{\tilde S}-\bSigma)
    \le
    \lambda_{\bSigma}+\|\widehat{\bSigma}-\bSigma\|_{2}.
\]
Applying the certificate lemma once more gives
\begin{align}
    \|\bmu_{\tilde S}-\bmu\|_2 &\leq
    2\delta_1
    + \sqrt{\frac{\eta(\lambda_{\max}(\bSigma_{S_i}-\widehat{\bSigma})+\|\widehat{\bSigma}-\bSigma\|_{2} + \eta\|\bSigma\|_{2}+\delta_2)}{1-\eta}} \\
    &\leq 2\delta_1+ \sqrt{\frac{8}{3}\epsilon(\lambda_{\bSigma}+\|\widehat{\bSigma}-\bSigma\|_{2})+\frac{4}{3}\epsilon\|\widehat{\bSigma}-\bSigma\|_{2} + \frac{8}{3}\epsilon\|\bSigma\|_{2}+\frac{4}{3}\epsilon\delta_2} \\
    &\leq 2\delta_1+ \sqrt{\frac{8}{3}\epsilon\lambda_{\bSigma}+4\epsilon\|\widehat{\bSigma}-\bSigma\|_{2} + \frac{8}{3}\epsilon\|\bSigma\|_{2}+\frac{4}{3}\epsilon\delta_2}
\end{align}
In particular, if
\[
    \lambda_{\bSigma}
    \asymp
    \|\widehat{\bSigma}-\bSigma\|_{2}+\epsilon\|\bSigma\|_{2}+\delta_2+\delta_1^2,
\]
then
\[
    \|\bmu_{\tilde S}-\bmu\|_2
    \lesssim
    \delta_1
    +
    \sqrt{\epsilon\|\widehat{\bSigma}-\bSigma\|_{2}}
    +
    \epsilon\|\bSigma\|_{2}^{1/2}
    +
    \sqrt{\epsilon\delta_2}.
\]
This completes the proof.

\subsubsection{Proof of Lemma \ref{lem: stability subG gradients}}
We verify the stability Definition \ref{def: stability} for the gradients $\{\bar{\bg}_k(\btheta)\}$.

\noindent (i) For the first-order term.
Fix $\btheta\in\mathcal{B}_{R_0}(\bthetas)$, $S\subseteq[K]$ with
$|S|\ge (1-\epsilon)K$, and $\bv\in\mathbb{S}^{d-1}$. Note that $q= |S^c|/K\le \epsilon$. Let $\mathcal N_\beta$ be a $\beta$-cover of $\Theta$ (under $\ell_2$-norm) with
$|\mathcal N_\beta|\lesssim (R_0/\beta)^d$ and $\btheta'\in\mathcal{N}_{\beta}$.

By  Assumption \ref{asmp: lipschitz loss},
\[
\left\|
\{\bar{\bg}_k(\btheta)-\bmu_{\btheta}\}
-
\{\bar{\bg}_k(\btheta')-\bmu_{\btheta'}\}
\right\|_2
\le 2L'\beta ,
\qquad k\in[K].
\]
Therefore,
\begin{align}
&\left|
\frac1{|S|}\sum_{k\in S}
\bv^\top(\bar{\bg}_k(\btheta)-\bmu_{\btheta})
\right| \le
\left|
\frac1{|S|}\sum_{k\in S}
\bv^\top(\bar{\bg}_k(\btheta')-\bmu_{\btheta'})
\right|
+
2L'\beta .
\label{eq: first_order_discretisation}
\end{align}
Next, since $|S|=(1-q)K$, we have
\begin{align}
\frac1{|I|}\sum_{k\in I}
\bv^\top(\bar{\bg}_k(\btheta')-\bmu_{\btheta'})
&=
\frac1{1-q}\frac1K\sum_{k=1}^K
\bv^\top(\bar{\bg}_k(\btheta')-\bmu_{\btheta'})
-
\frac1{1-q}\frac1K\sum_{k\in J}
\bv^\top(\bar{\bg}_k(\btheta')-\bmu_{\btheta'}).
\end{align}
Since
\[
\frac1K\sum_{k=1}^K
(\bmuk{k}_{\btheta'}-\bmu_{\btheta'})
=0,
\]
the full-sample term satisfies
\[
\frac1K\sum_{k=1}^K
\bv^\top(\bar{\bg}_k(\btheta')-\bmu_{\btheta'})
=
\frac1K\sum_{k=1}^K
\bv^\top(\bar{\bg}_k(\btheta')-\bmuk{k}_{\btheta'}).
\]
For the other term,
\begin{align}
&\left|
\frac1K\sum_{k\in S^c}
\bv^\top(\bar{\bg}_k(\btheta')-\bmu_{\btheta'})
\right|  \le
\left|
\frac1K\sum_{k\in S^c}
\bv^\top(\bar{\bg}_k(\btheta')-\bmuk{k}_{\btheta'})
\right|
+
\left|
\frac1K\sum_{k\in S^c}
\bv^\top(\bmuk{k}_{\btheta'}-\bmu_{\btheta'})
\right|.
\end{align}
The heterogeneity term can be bounded by Cauchy--Schwarz:
\begin{align}
\left|
\frac1K\sum_{k\in S^c}
\bv^\top(\bmuk{k}_{\btheta'}-\bmu_{\btheta'})
\right|
&\le
\frac{\sqrt{|S^c|}}{K}
\left(
\sum_{k\in S^c}
\|\bmuk{k}_{\btheta'}-\bmu_{\btheta'}\|_2^2
\right)^{1/2} \nonumber\\
&\le
\frac{\sqrt{|S^c|}}{K}
\left(
\sum_{k=1}^K
\|\bmuk{k}_{\btheta'}-\bmu_{\btheta'}\|_2^2
\right)^{1/2} \nonumber\\
&\le
\sqrt{q}\,h \\
&\le
\sqrt{\epsilon}\,h.
\label{eq: first_order_heterogeneity}
\end{align}
Hence, using $q\le \epsilon < 1/2$,
\begin{align}
&\left|
\frac{1}{|S|}\sum_{k\in S}
\bv^\top(\bar{\bg}_k(\btheta)-\bmu_{\btheta})
\right| \nonumber\\
&\lesssim
\left|
\frac{1}{K}\sum_{k=1}^K
\bv^\top(\bar{\bg}_k(\btheta')-\bmuk{k}_{\btheta'})
\right|
+
\left|
\frac1K\sum_{k\in S^c}
\bv^\top(\bar{\bg}_k(\btheta')-\bmuk{k}_{\btheta'})
\right|
+
\sqrt{\epsilon}h
+
L'\beta .
\label{eq: first_order_reduce_to_net}
\end{align}

It remains to control the two stochastic terms on the right-hand side of
\eqref{eq: first_order_reduce_to_net}. For fixed $\btheta'$, $\bv$, and $S^c$,
the random variables
\[
\bv^\top(\bar{\bg}_k(\btheta')-\bmuk{k}_{\btheta'})
\]
are independent, mean-zero, and sub-Gaussian with variance proxy of order $1/n$.
By sub-Gaussian concentration, a standard $1/2$-net argument over
$\mathbb{S}^{d-1}$, and a union bound over
$\btheta'\in\mathcal{N}_{\beta}$, with probability at least $1-\delta$,
\begin{align}
&\sup_{\btheta'\in\mathcal{N}_{\beta}}
\sup_{\bv\in\mathbb{S}^{d-1}}
\left|
\frac1K\sum_{k=1}^K
\bv^\top(\bar{\bg}_k(\btheta')-\bmuk{k}_{\btheta'})
\right| \lesssim
\sqrt{
\frac{d\log(R_0/\beta)+\log(1/\delta)}{nK}
}.
\label{eq: first_order_full_concentration}
\end{align}
Similarly, by also union bounding over all $S^c\subseteq[K]$ with $|S^c|\le \epsilon K$ (whose cardinality is at most  $\sum_{j=0}^{\lfloor \epsilon K\rfloor}\binom{K}{j} \le (e/\epsilon)^{\epsilon K}$,), with probability at least $1-\delta$, we have
%{\color{red}[M: the term $\epsilon\sqrt{\frac{\log(1/\epsilon)}{n}}$ is not that obvious to me. Also, why do we want to drop $\epsilon$ in the last inequality?]} \yt{$\epsilon\sqrt{\frac{\log(1/\epsilon)}{n}}$ comes from the union bound for all $S^c$ with $|S^c|\le \epsilon K$. I added some details in the bracket above.}
\begin{align}
&\sup_{\btheta'\in\mathcal{N}_{\beta}}
\sup_{\bv\in\mathbb{S}^{d-1}}
\sup_{S^c:|S^c|\le \epsilon K}
\left|
\frac{1}{K}\sum_{k\in S^c}
\bv^\top(\bar{\bg}_k(\btheta')-\bmuk{k}_{\btheta'})
\right| \nonumber\\
&\leq \sup_{\btheta'\in\mathcal{N}_{\beta}}
\sup_{\bv\in\mathbb{S}^{d-1}}
\sup_{S^c:|S^c|\le \epsilon K}
\frac{|S^c|}{K}\cdot \left|
\frac{1}{|S^c|}\sum_{k\in S^c}
\bv^\top(\bar{\bg}_k(\btheta')-\bmuk{k}_{\btheta'})
\right| \nonumber\\
&\lesssim
\sqrt{
\frac{d\log(R_0/\beta)+\log(1/\delta)}{nK}
}
+
\epsilon\sqrt{\frac{\log(1/\epsilon)}{n}} .
\label{eq: first_order_deleted_concentration}
\end{align}
Combining \eqref{eq: first_order_reduce_to_net},
\eqref{eq: first_order_full_concentration}, and
\eqref{eq: first_order_deleted_concentration}, we obtain
\begin{align}
&\sup_{\btheta\in\mathcal{B}_{R_0}(\bthetas)}
\sup_{S:|S|\ge(1-\epsilon)K}
\sup_{\bv\in\mathbb{S}^{d-1}}
\left|
\frac1{|S|}\sum_{k\in S}
\bv^\top(\bar{\bg}_k(\btheta)-\bmu_{\btheta})
\right| \nonumber\\
&\lesssim
\sqrt{
\frac{d\log(R_0/\beta)+\log(1/\delta)}{nK}
}
+
\epsilon\sqrt{\frac{\log(1/\epsilon)}{n}}
+
\sqrt{\epsilon}h
+
L'\beta .
\end{align}
This proves the claimed first-order stability bound.

\noindent (ii) For the second-order term.
Fix $\btheta\in\mathcal{B}_{R_0}(\bthetas)$, $S\subseteq[K]$ with
$|S|\ge (1-\epsilon)K$, and $\bv\in\mathbb{S}^{d-1}$. Let $q=|S^c|/K\le \epsilon$. Let $\btheta'\in\mathcal{N}_{\beta}$
satisfy $\|\btheta-\btheta'\|_2\le \beta$.

By Assumption \ref{asmp: lipschitz loss},
\[
\left\|
\{\bar{\bg}_k(\btheta)-\bmu_{\btheta}\}
-
\{\bar{\bg}_k(\btheta')-\bmu_{\btheta'}\}
\right\|_2
\le 2L'\beta ,
\qquad k\in[K].
\]
Hence, using $a^2-b^2=(a-b)(a+b)$ and Cauchy-Schwarz inequality, 
\begin{align}
&\left|
\frac{1}{|S|}\sum_{k\in S}
\left[
\{\bv^\top(\bar{\bg}_k(\btheta)-\bmu_{\btheta})\}^2
-
\{\bv^\top(\bar{\bg}_k(\btheta')-\bmu_{\btheta'})\}^2
\right]
\right| \nonumber\\
&\leq \bigg|
\frac{1}{|S|}\sum_{k\in S}
\left\{
[\bv^\top(\bar{\bg}_k(\btheta)-\bmu_{\btheta})]
-
[\bv^\top(\bar{\bg}_k(\btheta')-\bmu_{\btheta'})]
\right\}^2\bigg| \\
&\quad+ 2\bigg|\frac{1}{|S|}\sum_{k\in S}\left\{
[\bv^\top(\bar{\bg}_k(\btheta)-\bmu_{\btheta})]
-
[\bv^\top(\bar{\bg}_k(\btheta')-\bmu_{\btheta'})]
\right\}
[\bv^\top(\bar{\bg}_k(\btheta')-\bmu_{\btheta'})]
\bigg| \nonumber\\
&\lesssim 
L'^2\beta^2 + \frac{1}{|S|}\sqrt{\sum_{k \in S}\left\{
[\bv^\top(\bar{\bg}_k(\btheta)-\bmu_{\btheta})]
-
[\bv^\top(\bar{\bg}_k(\btheta')-\bmu_{\btheta'})]
\right\}^2}\cdot \sqrt{\sum_{k \in S}[\bv^\top(\bar{\bg}_k(\btheta')-\bmu_{\btheta'})]^2} \\
&\lesssim L'^2\beta^2 + L'\beta
\sqrt{\frac{1}{|S|}\sum_{k\in S}
\{\bv^\top(\bar{\bg}_k(\btheta')-\bmu_{\btheta'})\}^2}
 \nonumber\\
&\lesssim L'^2\beta^2 + L'\beta
\sqrt{\frac{1}{|S|}\sum_{k\in S}
\{\bv^\top(\bar{\bg}_k(\btheta')-\bmu_{\btheta'})\}^2 - \bv^\top\bSigma_{\btheta'}\bv} + L'\beta\sqrt{\bv^\top\bSigma_{\btheta'}\bv}
 \nonumber\\
&\lesssim L'^2\beta^2 + 
\left|
\frac{1}{|S|}\sum_{k\in S}
\{\bv^\top(\bar{\bg}_k(\btheta')-\bmu_{\btheta'})\}^2
-
\bv^\top\bSigma_{\btheta'}\bv
\right|
+
L'\beta\left(h+\frac1{\sqrt n}\right),
\label{eq: second_order_empirical_discretisation}
\end{align}
where in the last inequality, we used $\|\bSigma_{\btheta'}\|_{2}\lesssim h^2+\frac{1}{n}$. 

Similarly,
\begin{align}
\left|
\bv^\top(\bSigma_{\btheta}-\bSigma_{\btheta'})\bv
\right|
&=
\left|
\frac1K\sum_{k=1}^K
\mathbb{E}
\left[
\{\bv^\top(\bar{\bg}_k(\btheta)-\bmu_{\btheta})\}^2
-
\{\bv^\top(\bar{\bg}_k(\btheta')-\bmu_{\btheta'})\}^2
\right]
\right| \nonumber\\
&\le
C L'\beta\left(h+\frac1{\sqrt n}\right)
+
C L'^2\beta^2 .
\label{eq: second_order_population_discretisation}
\end{align}
Combining \eqref{eq: second_order_empirical_discretisation} and
\eqref{eq: second_order_population_discretisation}, we obtain
\begin{align}
&\left|
\frac{1}{|S|}\sum_{k\in S}
\{\bv^\top(\bar{\bg}_k(\btheta)-\bmu_{\btheta})\}^2
-
\bv^\top\bSigma_{\btheta}\bv
\right| \nonumber\\
&\le
C
\left|
\frac{1}{|S|}\sum_{k\in S}
\{\bv^\top(\bar{\bg}_k(\btheta')-\bmu_{\btheta'})\}^2
-
\bv^\top\bSigma_{\btheta'}\bv
\right|
+
C L'\beta\left(h+\frac1{\sqrt n}\right)
+
C L'^2\beta^2 .
\label{eq: second_order_reduce_to_net}
\end{align}

It remains to control the term at the net point $\btheta'$. Decompose
\[
\bar{\bg}_k(\btheta')-\bmu_{\btheta'}
=
\{\bar{\bg}_k(\btheta')-\bmuk{k}_{\btheta'}\}
+
\{\bmuk{k}_{\btheta'}-\bmu_{\btheta'}\}.
\]
Then 
\begin{align}
&
\frac{1}{|S|}\sum_{k\in S}
\{\bv^\top(\bar{\bg}_k(\btheta')-\bmu_{\btheta'})\}^2
-
\bv^\top\bSigma_{\btheta'}\bv
\nonumber\\
&=
\frac{1}{|S|}\sum_{k\in S}
\left[
\{\bv^\top(\bar{\bg}_k(\btheta')-\bmuk{k}_{\btheta'})\}^2
-
\mathbb{E}\{\bv^\top(\bar{\bg}_k(\btheta')-\bmuk{k}_{\btheta'})\}^2
\right]
\nonumber\\
&\quad+
2
\frac1{|S|}\sum_{k\in S}
\bv^\top(\bar{\bg}_k(\btheta')-\bmuk{k}_{\btheta'})
\,
\bv^\top(\bmuk{k}_{\btheta'}-\bmu_{\btheta'})
 \nonumber\\
&\quad+
\frac1{|S|}\sum_{k\in S}
\left[
\mathbb{E}\{\bv^\top(\bar{\bg}_k(\btheta')-\bmuk{k}_{\btheta'})\}^2
+
\{\bv^\top(\bmuk{k}_{\btheta'}-\bmu_{\btheta'})\}^2
\right]
 \nonumber\\
&\qquad\qquad
-
\frac{1}{K}\sum_{k=1}^K
\left[
\mathbb{E}\{\bv^\top(\bar{\bg}_k(\btheta')-\bmuk{k}_{\btheta'})\}^2
+
\{\bv^\top(\bmuk{k}_{\btheta'}-\bmu_{\btheta'})\}^2
\right]
\label{eq: second_order_net_decomposition}
\end{align}
For the first term in \eqref{eq: second_order_net_decomposition}, the summands are
independent, mean-zero, and sub-exponential with scale of order $1/n$. By Bernstein's
inequality, a standard $1/4$-net argument over $\mathbb{S}^{d-1}$, and a union bound
over $\btheta'\in\mathcal{N}_{\beta}$ and all $S^c\subseteq[K]$ with
$|S^c|\le \epsilon K$, with probability at least $1-\delta$, 
\begin{align}
&\sup_{\btheta'\in\mathcal{N}_{\beta}}
\sup_{\bv\in\mathbb{S}^{d-1}}
\sup_{S:|S|\ge(1-\epsilon)K}
\left|
\frac1{|S|}\sum_{k\in S}
\left[
\{\bv^\top(\bar{\bg}_k(\btheta')-\bmuk{k}_{\btheta'})\}^2
-
\mathbb{E}\{\bv^\top(\bar{\bg}_k(\btheta')-\bmuk{k}_{\btheta'})\}^2
\right]
\right| \nonumber\\
&\leq
C\sup_{\btheta'\in\mathcal{N}_{\beta}}
\sup_{\bv\in\mathbb{S}^{d-1}}
\left|
\frac{1}{K}\sum_{k=1}^K
\left[
\{\bv^\top(\bar{\bg}_k(\btheta')-\bmuk{k}_{\btheta'})\}^2
-
\mathbb{E}\{\bv^\top(\bar{\bg}_k(\btheta')-\bmuk{k}_{\btheta'})\}^2
\right]
\right| \nonumber\\
&\quad+
C\sup_{\btheta'\in\mathcal{N}_{\beta}}
\sup_{\bv\in\mathbb{S}^{d-1}}
\sup_{S^c:|S^c|\leq \epsilon K}
\left|
\frac{1}{K}\sum_{k\in S^c}
\left[
\{\bv^\top(\bar{\bg}_k(\btheta')-\bmuk{k}_{\btheta'})\}^2
-
\mathbb{E}\{\bv^\top(\bar{\bg}_k(\btheta')-\bmuk{k}_{\btheta'})\}^2
\right]
\right| \nonumber\\
&\lesssim
\frac{1}{n}
\left[
\sqrt{\frac{d\log(R_0/\beta)+\log(1/\delta)}{K}}
+
\frac{d\log(R_0/\beta)+\log(1/\delta)}{K}
\right]
+
\epsilon\frac{\log(1/\epsilon)}{n}.
\label{eq: second_order_quadratic_concentration}
\end{align}
For the second term in \eqref{eq: second_order_net_decomposition}, conditional on
$\{\bmuk{k}_{\btheta'}-\bmu_{\btheta'}\}_{k=1}^K$, the summands are independent,
mean-zero, and sub-Gaussian with variance proxy bounded by
\[
\frac1n\{\bv^\top(\bmuk{k}_{\btheta'}-\bmu_{\btheta'})\}^2 .
\]
Therefore, by the same union-bound argument and the heterogeneity condition,
with probability at least $1-\delta$,
\begin{align}
&\sup_{\btheta'\in\mathcal{N}_{\beta}}
\sup_{\bv\in\mathbb{S}^{d-1}}
\sup_{S:|S|\ge(1-\epsilon)K}
\left|
\frac{1}{|S|}\sum_{k\in S}
\bv^\top(\bar{\bg}_k(\btheta')-\bmuk{k}_{\btheta'})
\,
\bv^\top(\bmuk{k}_{\btheta'}-\bmu_{\btheta'})
\right| \nonumber\\
&\lesssim
h
\sqrt{
\frac{d\log(R_0/\beta)+\log(1/\delta)}{nK}
+
\frac{\epsilon\log(1/\epsilon)}{n}
}.
\label{eq: second_order_cross_concentration}
\end{align}
For the third term in \eqref{eq: second_order_net_decomposition}, since $q=|S^c|/K$, %{\color{red}[M: the first inequality is in fact an identity if we remove the absolute values and add a minus sign in front of the last term? I would write it as an identity first and then just state the last inequality. Also, should it be $\epsilon h^2?$]} \yt{I just rewrote it as an identity. And it is indeed $h^2$ because we do not have a better control on $\frac{1}{K}\sum_{k\in S^c}\|\bmuk{k}_{\btheta'}-\bmu_{\btheta'}\|_2^2 \leq h^2$ given the current assumption.}
\begin{align}
&
\frac{1}{|S|}\sum_{k\in S}
\left[
\mathbb{E}\{\bv^\top(\bar{\bg}_k(\btheta')-\bmuk{k}_{\btheta'})\}^2
+
\{\bv^\top(\bmuk{k}_{\btheta'}-\bmu_{\btheta'})\}^2
\right]
 \nonumber\\
&\qquad\qquad
-
\frac{1}{K}\sum_{k=1}^K
\left[
\mathbb{E}\{\bv^\top(\bar{\bg}_k(\btheta')-\bmuk{k}_{\btheta'})\}^2
+
\{\bv^\top(\bmuk{k}_{\btheta'}-\bmu_{\btheta'})\}^2
\right]
 \nonumber\\
&=
\frac{q}{1-q}\frac1K\sum_{k=1}^K
\left[
\mathbb{E}\{\bv^\top(\bar{\bg}_k(\btheta')-\bmuk{k}_{\btheta'})\}^2
+
\{\bv^\top(\bmuk{k}_{\btheta'}-\bmu_{\btheta'})\}^2
\right] \nonumber\\
&\quad+
\frac1{K(1-q)}\sum_{k\in S^c}
\left[
\mathbb{E}\{\bv^\top(\bar{\bg}_k(\btheta')-\bmuk{k}_{\btheta'})\}^2
+
\{\bv^\top(\bmuk{k}_{\btheta'}-\bmu_{\btheta'})\}^2
\right].
\end{align}
Assumption \ref{asmp: subG gradient} entails
\[
\sup_{k,\btheta',\bv}
\mathbb{E}\{\bv^\top(\bar{\bg}_k(\btheta')-\bmuk{k}_{\btheta'})\}^2
\lesssim \frac1n
\]
and Assumption \ref{asmp: task heterogeneity} gives the heterogeneity bound
\[
\frac{1}{K}\sum_{k\in S^c}
\|\bmuk{k}_{\btheta'}-\bmu_{\btheta'}\|_2^2
\lesssim h^2,
\qquad |S^c|\le \epsilon K.
\]
Therefore,
\begin{align}
&\left|
\frac{1}{|S|}\sum_{k\in S}
\left[
\mathbb{E}\{\bv^\top(\bar{\bg}_k(\btheta')-\bmuk{k}_{\btheta'})\}^2
+
\{\bv^\top(\bmuk{k}_{\btheta'}-\bmu_{\btheta'})\}^2
\right]
\right.\nonumber\\
&\qquad\qquad \left.
-
\frac{1}{K}\sum_{k=1}^K
\left[
\mathbb{E}\{\bv^\top(\bar{\bg}_k(\btheta')-\bmuk{k}_{\btheta'})\}^2
+
\{\bv^\top(\bmuk{k}_{\btheta'}-\bmu_{\btheta'})\}^2
\right]\right|
 \lesssim
\frac{\epsilon}{n}
+ h^2.
\label{eq: second_order_population_bias}
\end{align}

Combining \eqref{eq: second_order_reduce_to_net},
\eqref{eq: second_order_net_decomposition},
\eqref{eq: second_order_quadratic_concentration},
\eqref{eq: second_order_cross_concentration}, and
\eqref{eq: second_order_population_bias}, we obtain
\begin{align}
&\sup_{\btheta\in\mathcal{B}_{R_0}(\bthetas)}
\sup_{S:|S|\ge(1-\epsilon)K}
\sup_{\bv\in\mathbb{S}^{d-1}}
\left|
\frac{1}{|S|}\sum_{k\in S}
\{\bv^\top(\bar{\bg}_k(\btheta)-\bmu_{\btheta})\}^2
-
\bv^\top\bSigma_{\btheta}\bv
\right| \nonumber\\
&\lesssim
\frac1n
\left[
\sqrt{\frac{d\log(R_0/\beta)+\log(1/\delta)}{K}}
+
\frac{d\log(R_0/\beta)+\log(1/\delta)}{K}
\right]
+
\epsilon\frac{\log(1/\epsilon)}{n}
+
L'\beta\left(h+\frac1{\sqrt n}\right)
+
L'^2\beta^2 \nonumber\\
&\quad+
h\sqrt{
\frac{d\log(R_0/\beta)+\log(1/\delta)}{nK}
+
\frac{\epsilon\log(1/\epsilon)}{n}
}
+
 h^2 .
\end{align}
This proves the claimed second-order stability bound.

\subsubsection{Proof of Theorem \ref{thm: gradident est error appendix}}

The first bound follows directly from Proposition \ref{prop: robust mean est error}
and Lemma \ref{lem: stability subG gradients}, using the choice of 
$\lambda_{\bSigma}$. In particular, with probability at least $1-\delta$,
uniformly over all admissible contamination mechanisms and all
$\btheta\in\mathcal{B}_{R_0}(\btheta^*)$,
\begin{align}
\twonorm{
g(\btheta)
-
\frac1K\sum_{k=1}^K \tE \nabla \ell(\zk{k},\btheta)
}
&\lesssim
\sqrt{\frac{d\log(R_0/\beta)+\log(1/\delta)}{nK}}
+
\epsilon\sqrt{\frac{\log(1/\epsilon)}{n}}
\nonumber\\
&\quad+
\sqrt{\epsilon\twonorm{\widehat{\bSigma}_{\btheta}-\bSigma_{\btheta}}}
+
\sqrt{\epsilon}h .
\label{eq: global_gradient_bound}
\end{align}
We now prove the bound for $g^{(k)}(\btheta)$. Recall that
\[
    \bar{\bg}_k(\btheta)
    :=
    \frac1n\sum_{i=1}^n \nabla\ell(z_i^{(k)},\btheta).
\]
By the definition of the shrinkage step,
\[
    g^{(k)}(\btheta)
    =
    \begin{cases}
        g(\btheta),
        &\quad \text{if } \|\bar{\bg}_k(\btheta)-g(\btheta)\|_2\le \lambda,\\[0.3em]
        g(\btheta)
        +
        \left(
        1-\dfrac{\lambda}{\|\bar{\bg}_k(\btheta)-g(\btheta)\|_2}
        \right)
        \{\bar{\bg}_k(\btheta)-g(\btheta)\},
        &\quad \text{otherwise}.
    \end{cases}
\]
Therefore, in both cases,
\begin{equation}\label{eq: shrinkage_to_local}
    \|g^{(k)}(\btheta)-\bar{\bg}_k(\btheta)\|_2
    \le \lambda .
\end{equation}
Indeed, if $\|\bar{\bg}_k(\btheta)-g(\btheta)\|_2\le\lambda$, then
$g^{(k)}(\btheta)=g(\btheta)$, and the claim follows. If
$\|\bar{\bg}_k(\btheta)-g(\btheta)\|_2>\lambda$, then
\[
    g^{(k)}(\btheta)-\bar{\bg}_k(\btheta)
    =
    -\frac{\lambda}{\|\bar{\bg}_k(\btheta)-g(\btheta)\|_2}
    \{\bar{\bg}_k(\btheta)-g(\btheta)\},
\]
so again $\|g^{(k)}(\btheta)-\bar{\bg}_k(\btheta)\|_2=\lambda$.

By the uniform sub-Gaussian concentration bound, with probability at least
$1-\delta$, 
\begin{equation}\label{eq: local_gradient_concentration}
    \max_{k\in S}
    \sup_{\btheta\in\mathcal{B}_{R_0}(\btheta^*)}
    \left\|
    \bar{\bg}_k(\btheta)
    -
    \tE\nabla\ell(\zk{k},\btheta)
    \right\|_2
    \lesssim
    \sqrt{\frac{d\log(R_0/\beta)+\log(K/\delta)}{n}}.
\end{equation}
Combining \eqref{eq: shrinkage_to_local} and
\eqref{eq: local_gradient_concentration}, we obtain, for all $k\in S$,
\begin{align}
\left\|
g^{(k)}(\btheta)
-
\tE\nabla\ell(\zk{k},\btheta)
\right\|_2
&\le
\left\|
g^{(k)}(\btheta)-\bar{\bg}_k(\btheta)
\right\|_2
+
\left\|
\bar{\bg}_k(\btheta)
-
\tE\nabla\ell(\zk{k},\btheta)
\right\|_2
\nonumber\\
&\lesssim
\lambda
+
\sqrt{\frac{d\log(R_0/\beta)+\log(K/\delta)}{n}} .
\label{eq: local_bound_gk}
\end{align}
By the triangle inequality,
\begin{align}
\|\bar{\bg}_k(\btheta)-g(\btheta)\|_2
&\le
\left\|
\bar{\bg}_k(\btheta)
-
\tE\nabla\ell(\zk{k},\btheta)
\right\|_2
+
\left\|
\tE\nabla\ell(\zk{k},\btheta)
-
\frac1K\sum_{k=1}^K\tE\nabla\ell(\zk{k},\btheta)
\right\|_2
\nonumber\\
&\quad+
\left\|
\frac1K\sum_{k=1}^K\tE\nabla\ell(\zk{k},\btheta)
-
g(\btheta)
\right\|_2 .
\end{align}
Using \eqref{eq: global_gradient_bound}, \eqref{eq: local_gradient_concentration},
and the definition of $h^{(k)}$, this gives
\begin{align}
\|\bar{\bg}_k(\btheta)-g(\btheta)\|_2
&\lesssim
\sqrt{\frac{d\log(R_0/\beta)+\log(K/\delta)}{n}}
+
\sqrt{\frac{d\log(R_0/\beta)+\log(1/\delta)}{nK}}
\nonumber\\
&\quad+
\epsilon\sqrt{\frac{\log(1/\epsilon)}{n}}
+
\sqrt{\epsilon\twonorm{\widehat{\bSigma}_{\btheta}-\bSigma_{\btheta}}}
+
\sqrt{\epsilon}h
+
h^{(k)} .
\label{eq: local_to_global_distance}
\end{align}
Now consider two cases. First, suppose
\begin{align}
&\sqrt{\frac{d\log(R_0/\beta)+\log(1/\delta)}{nK}}
+
\epsilon\sqrt{\frac{\log(1/\epsilon)}{n}}
+
\sqrt{\epsilon\twonorm{\widehat{\bSigma}_{\btheta}-\bSigma_{\btheta}}}
+
\sqrt{\epsilon}h
+
h^{(k)}
\nonumber\\
&\qquad\le c\lambda
\label{eq: small_global_case}
\end{align}
for a sufficiently small constant $c>0$. Since
\[
\lambda
\ge
C_1\sqrt{\frac{d\log(R_0/\beta)+\log(K/\delta)}{n}},
\]
with $C_1$ sufficiently large, \eqref{eq: local_to_global_distance} implies 
\[
    \|\bar{\bg}_k(\btheta)-g(\btheta)\|_2\le \lambda.
\]
Hence $g^{(k)}(\btheta)=g(\btheta)$, and therefore with probability at least $1-\delta$,
\begin{align}
\left\|
g^{(k)}(\btheta)
-
\tE\nabla\ell(\zk{k},\btheta)
\right\|_2
&=
\left\|
g(\btheta)
-
\tE\nabla\ell(\zk{k},\btheta)
\right\|_2
\nonumber\\
&\le
\left\|
g(\btheta)
-
\frac1K\sum_{j=1}^K\tE\nabla\ell(\zk{j},\btheta)
\right\|_2
+
h^{(k)}
\nonumber\\
&\lesssim
\sqrt{\frac{d\log(R_0/\beta)+\log(1/\delta)}{nK}}
+
\epsilon\sqrt{\frac{\log(1/\epsilon)}{n}}
\nonumber\\
&\quad+
\sqrt{\epsilon\twonorm{\widehat{\bSigma}_{\btheta}-\bSigma_{\btheta}}}
+
\sqrt{\epsilon}h
+
h^{(k)} .
\label{eq: global_bound_gk}
\end{align}

Second, suppose \eqref{eq: small_global_case} does not hold. Then
\[
\lambda
\lesssim
\sqrt{\frac{d\log(R_0/\beta)+\log(1/\delta)}{nK}}
+
\epsilon\sqrt{\frac{\log(1/\epsilon)}{n}}
+
\sqrt{\epsilon\twonorm{\widehat{\bSigma}_{\btheta}-\bSigma_{\btheta}}}
+
\sqrt{\epsilon}h
+
h^{(k)}.
\]
 Therefore, in this case,
\eqref{eq: local_bound_gk} gives the desired minimum bound.

Combining the two cases with \eqref{eq: local_bound_gk}, we conclude that, for
all $k\in S$, with probability at least $1-\delta$,
\begin{align}
\left\|
g^{(k)}(\btheta)
-
\tE\nabla\ell(\zk{k},\btheta)
\right\|_2
&\lesssim
\min\Bigg\{
\sqrt{\frac{d\log(R_0/\beta)+\log(1/\delta)}{nK}}
+
\epsilon\sqrt{\frac{\log(1/\epsilon)}{n}}
\nonumber\\
&\quad+
\sqrt{\epsilon\twonorm{\widehat{\bSigma}_{\btheta}-\bSigma_{\btheta}}}
+
\sqrt{\epsilon}h
+
h^{(k)},
\quad
\sqrt{\frac{d\log(R_0/\beta)+\log(K/\delta)}{n}}
+
\lambda
\Bigg\}.
\end{align}

\subsubsection{Proof of Theorem \ref{thm: parameter est error appendix}}
The results directly follow from Theorems \ref{thm: federated gradient descent} and \ref{thm: gradident est error appendix}.

% ------------------------------------------------
\subsection{Proofs of results in Section \ref{subsec: cov estimation}}
Theorem \ref{thm: est error cov appendix} is the more explicit version of Theorem \ref{thm: est error cov} in the main text, which provides a uniform bound on the estimation error of $\bSigma_{\btheta}$ for all $\btheta \in \Theta$. Corollary \ref{cor: gradient est error appendix} characterizes the gradient estimation error bounds for $g(\btheta)$ and $g^{(k)}(\btheta)$ by applying the bound in Theorem \ref{thm: est error cov appendix} to Theorem \ref{thm: gradident est error appendix}. Corollary \ref{cor: parameter alg error} in the main text follows directly from Corollary \ref{cor: parameter alg error appendix}. \Cref{thm: gradident est error appendix} is an important intermediate result that establishes the gradient estimation errors satisfy \Cref{asmp: gradient est error} with the corresponding rates.

\begin{theorem}\label{thm: est error cov appendix}
Under Assumptions \ref{asmp: subG gradient} and \ref{asmp: lipschitz loss}, for any $\beta \in (0 , R_0]$, with probability at least $1-\delta$, the output from Algorithm \ref{algo: naive est of Sigma heterogeneous} satisfies
\begin{align}
    \sup_{\btheta \in \Theta}\twonorm{\hSigma_{\btheta} - \bSigma_{\btheta}} &\lesssim \frac{1}{n} \bigg[\sqrt{\frac{d\log(R_0/\beta) + \log(1/\delta)}{K}} + \frac{d\log(R_0/\beta) + \log(1/\delta)}{K}\bigg]  + \epsilon \frac{\log(1/\epsilon)}{n} + L'^2\beta^2  \\
    &\quad + \frac{L'\beta}{n} + h^2 + \frac{\epsilon}{n}\Bigg(\sqrt{\frac{d\log(R_0/\beta) +\log(K/\delta)}{n}} + \frac{d\log(R_0/\beta) +\log(K/\delta)}{n}\Bigg).
\end{align}
\end{theorem}

The next two results follow by applying the bound in Theorem \ref{thm: est error cov appendix} to Theorems \ref{thm: gradident est error appendix} and \ref{thm: parameter est error appendix} respectively. We will assume the following tuniong parameter conditions hold:
\begin{align}
    \lambda &= C\sqrt{\frac{d\log(R_0/\beta) + \log(K/\delta)}{n}}, \label{eq: lambda choice 2}\\
    \lambda_{\bSigma} &= C'\bigg\{ \frac{\epsilon}{n}\bigg(\sqrt{\frac{d\log(R_0/\beta) +\log(K/\delta)}{n}} + \frac{d\log(R_0/\beta) +\log(K/\delta)}{n}\bigg) \\
    &\qquad + \frac{1}{n} \bigg[\sqrt{\frac{d\log(R_0/\beta) + \log(1/\delta)}{K}} + \frac{d\log(R_0/\beta) + \log(1/\delta)}{K}\bigg]  + \epsilon \frac{\log(1/\epsilon)}{n} + L'^2\beta^2 + \frac{L'\beta}{n} \\ 
    &\qquad + h\sqrt{\frac{d\log(R_0/\beta)+\log(1/\delta)}{nK} + \frac{\epsilon\log(1/\epsilon)}{n}}  + h^2\bigg\}, \label{eq: lambda_Sigma choice 2}
\end{align}
where $C, C' > 0$ are some sufficiently large constants.

\begin{corollary}\label{cor: gradient est error appendix}
    Assume $(\lambda, \lambda_{\bSigma})$ in Algorithms \ref{algo: robust federated gradient descent} and \ref{algo: robust mean estimation} satisfies the conditions in \eqref{eq: lambda choice 2} and \eqref{eq: lambda_Sigma choice 2}. Under Assumptions \ref{asmp: lipschitz loss} and \ref{asmp: subG gradient}, for any $\beta > 0$ satisfying $L'\beta \lesssim \frac{d\log(R_0/\beta) +\log (1/\delta)}{nK}  \wedge 1$ and any $\delta \gtrsim \exp\{-CK\epsilon\}$, with probability at least $1-\delta$, for all subset $S \subseteq [K]$ with $|S^c|/K \leq \epsilon$, all contamination mechanism $M \in \mathcal{M}_S$, we have
    \begin{align}
        \sup_{\btheta \in \Theta}\twonorm{g(\btheta) - \tE \nabla \ell (z, \btheta)} &\lesssim \sqrt{\frac{d\log(R_0/\beta) + \log(1/\delta)}{nK}} + \epsilon\sqrt{\frac{\log(1/\epsilon)}{n}} + \sqrt{\epsilon}h\\
        \quad + \frac{\epsilon}{\sqrt{n}}& \Bigg[\bigg(\frac{d\log(R_0/\beta) +\log(K/\delta)}{n}\bigg)^{1/4} \vee \bigg(\frac{d\log(R_0/\beta) +\log(K/\delta)}{n}\bigg)^{1/2}\Bigg], \\
        \sup_{\btheta \in \Theta}\twonorm{g^{(k)}(\btheta) - \tE \nabla \ell (\zk{k}, \btheta)} &\lesssim \min \Bigg\{\sqrt{\frac{d\log(R_0/\beta) + \log(1/\delta)}{nK}} + \epsilon\sqrt{\frac{\log(1/\epsilon)}{n}} + L'\beta + \sqrt{\epsilon}h + h^{(k)} \\
        &\hspace{-1.2cm} +\frac{\epsilon}{\sqrt{n}} \Bigg[\bigg(\frac{d\log(R_0/\beta) +\log(K/\delta)}{n}\bigg)^{1/4} \vee \bigg(\frac{d\log(R_0/\beta) +\log(K/\delta)}{n}\bigg)^{1/2}\Bigg], \\ &\sqrt{\frac{d\log(R_0/\beta) + \log(K/\delta)}{n}}\Bigg\}, \quad \forall k \in S.
    \end{align}
\end{corollary}

\begin{remark}
    Note that the gradient estimation error in Corollary \ref{cor: gradient est error appendix} only requires local smoothness and does not rely on strong convexity. This property allows Algorithms \ref{algo: robust federated gradient descent} and \ref{algo: robust mean estimation} to be extended to other convex and even nonconvex loss functions. 
    % Since the primary goal of this work is to highlight limitations of certain existing algorithms and to propose new directions, and given the broad scope of the current manuscript, we do not pursue this line of investigation further.
\end{remark}

\begin{corollary}\label{cor: parameter alg error appendix}
    Assume $(\lambda, \lambda_{\bSigma})$ in Algorithms \ref{algo: robust federated gradient descent} and \ref{algo: robust mean estimation} satisfies the conditions in \eqref{eq: lambda choice 2} and \eqref{eq: lambda_Sigma choice 2}. Under Assumptions \ref{asmp: risk function}, \ref{asmp: lipschitz loss} and \ref{asmp: subG gradient}, for any $\beta > 0$ satisfying $L\beta \lesssim \frac{d\log(R_0/\beta) +\log (1/\delta)}{nK}  \wedge 1 \wedge R_0^2$ and any $\delta \gtrsim \exp\{-CK\epsilon\}$, $nK \gtrsim R_0^{-2}[d\log(R_0/\beta)+\log(1/\delta)]$, $n \gtrsim R_0^{-2}\epsilon^2\log(1/\epsilon)$, $\frac{\epsilon}{\sqrt{n}}[(\frac{d\log(R_0/\beta) +\log(K/\delta)}{n})^{1/4} \vee (\frac{d\log(R_0/\beta) +\log(K/\delta)}{n})^{1/2}] \lesssim R_0$, $\sqrt{\epsilon} h \lesssim R_0$, $\max_{k \in [K]}\hk{k} \lesssim R_0$, with probability at least $1-\delta$, for all subset $S \subseteq [K]$ with $|S^c|/K \leq \epsilon$, all contamination mechanism $M \in \mathcal{M}_S$, we have
    \begin{align}
        \twonorm{\htheta_T - \btheta^*} &\lesssim (1-\kappa/2)^{T/2} \twonorm{\htheta_0 - \bthetas} + \sqrt{\frac{d\log(R_0/\beta) + \log(1/\delta)}{nK}} + \epsilon\sqrt{\frac{\log(1/\epsilon)}{n}} + \sqrt{\epsilon}h\\
        &\quad + \frac{\epsilon}{\sqrt{n}}\Bigg[\bigg(\frac{d\log(R_0/\beta) +\log(K/\delta)}{n}\bigg)^{1/4} \vee \bigg(\frac{d\log(R_0/\beta) +\log(K/\delta)}{n}\bigg)^{1/2}\Bigg].
    \end{align}
    \begin{align}
        \twonorm{\hthetak{k}_T - \bthetaks{k}} &\lesssim (1-\kappa/2)^{T/2} \twonorm{\hthetak{k}_0 - \bthetaks{k}} \\
        &\quad + \min \Bigg\{\sqrt{\frac{d\log(R_0/\beta) + \log(1/\delta)}{nK}} + \epsilon\sqrt{\frac{\log(1/\epsilon)}{n}}  + \sqrt{\epsilon}h + h^{(k)} \\
        &\hspace{2cm} +\frac{\epsilon}{\sqrt{n}}\Bigg[\bigg(\frac{d\log(R_0/\beta) +\log(K/\delta)}{n}\bigg)^{1/4} \vee \bigg(\frac{d\log(R_0/\beta) +\log(K/\delta)}{n}\bigg)^{1/2}\Bigg]\\
        &\hspace{1.8cm}, \sqrt{\frac{d\log(R_0/\beta) + \log(K/\delta)}{n}}\Bigg\}, \quad \forall k \in S.
    \end{align}
\end{corollary}

\begin{remark}
    When $T \gtrsim \log(nK)$, $\delta \asymp e^{-d} + e^{-K\epsilon}$, $\twonorm{\htheta_0 - \bthetas} \lesssim 1$, we have
    \begin{align}
        \twonorm{\htheta_T - \btheta^*} 
        &\lesssim \sqrt{\frac{d\log(R_0/\beta)}{nK}} + \epsilon\sqrt{\frac{\log(1/\epsilon)}{n}}  + \sqrt{\epsilon}h \\
        &\quad + \frac{\epsilon}{\sqrt{n}}\Bigg[\bigg(\frac{d\log(R_0/\beta) + \log K}{n}\bigg)^{1/4}\vee \bigg(\frac{d\log(R_0/\beta) + \log K}{n}\bigg)^{1/2}\Bigg],
    \end{align}
    with probability at least $1-e^{-d}-e^{-K\epsilon}$, where the second inequality is due to Cauchy-Schwarz applied to the last term. Comparing with the lower bound $\sqrt{\frac{d}{nK}} + \frac{\epsilon}{\sqrt{n}} + \sqrt{\epsilon}h$, it is clear that when $n \gtrsim d$ or $\epsilon^2(1\vee \frac{d}{n}) \lesssim \frac{n}{K}$, the upper bound of $\twonorm{\htheta_T - \btheta^*}$ is minimax optimal up to logarithmic factors.
\end{remark}

% -----------------------------
\subsubsection{Proof of Theorem \ref{thm: est error cov appendix}}

Our argument relies on the following two intermediate lemmas which we will prove later.

\begin{lemma}\label{lem: cov uniform concentration}
    Under Assumptions \ref{asmp: subG gradient} and \ref{asmp: lipschitz loss}, for all $\beta \in (0,R_0]$ such that $L'\beta < c$ with $c$ a sufficiently small constant, with probability at least $1-\delta$,
    \begin{equation}
        \begin{aligned}
        \sup_{\btheta \in \Theta}\max_{k \in S}\twonorm{\hSigmak{k}_{\btheta} - \bSigmak{k}_{\btheta}}
        &\lesssim \frac{1}{n}\Bigg(\sqrt{\frac{d\log(R_0/\beta) + \log(K/\delta)}{n}} \,\vee\, \frac{d\log(R_0/\beta) + \log(K/\delta)}{n}\Bigg)(1+L'\beta) \\
        &\quad + L'^2\beta^2 + \frac{1}{n}L'\beta.
        \end{aligned}
    \end{equation}
\end{lemma}

\begin{lemma}\label{lem: S hat properties}
    Under Assumptions \ref{asmp: subG gradient} and \ref{asmp: lipschitz loss}, for all $\beta \in (0,R_0]$ such that $L'\beta < c$ with $c$ a sufficiently small constant, we have the following two properties of $\widehat{S}_{\textup{safe}}$ hold:
    \begin{enumerate}[(i)]
        \item There exists a subset $S_0 \subseteq S \cap \widehat{S}_{\textup{safe}}$ with $|S_0| \geq K(1-3\epsilon)$, when $\epsilon < 1/3$;
        \item With probability at least $1-\delta$, for all $k \in \widehat{S}_{\textup{safe}}$, $\sup_{\btheta}\twonorm{\hSigmak{k}_{\btheta}} \lesssim \frac{1}{n}\Bigg(\sqrt{\frac{d\log(R_0/\beta) +\log(K/\delta)}{n}} \vee \frac{d\log(R_0/\beta) +\log(K/\delta)}{n}\Bigg)(1+L'\beta) + L'^2\beta^2 + \frac{1}{n}L'\beta + \frac{1}{n}$.
    \end{enumerate}
\end{lemma}

Recall our previous notations $\bar{\bg}_k(\btheta)
    := \frac{1}{n}\sum_{i=1}^n\nabla\ell(z_i^{(k)},\btheta)$, $\bmuk{k}_{\btheta} = \tE\nabla\ell(\zk{k}, \btheta) = \tE\bar{\bg}_k(\btheta)$, and $\bmu_{\btheta} = \frac{1}{K}\sum_{k=1}^K \bmuk{k}_{\btheta}$. Also, we write $\widehat{S}_{\textup{safe}}$ as $\widehat{S}$. Define $\widetilde{\bSigma}_{\btheta} = \frac{1}{K}\sum_{k=1}^K \bSigmak{k}_{\btheta} = \frac{1}{K}\sum_{k=1}^K\tE(\bar{\bg}_k(\btheta) - \bmuk{k}_{\btheta})(\bar{\bg}_k(\btheta) - \bmuk{k}_{\btheta})^\top$. We also denote RHS of the inequality in Lemma \ref{lem: S hat properties}.(\rom{2}) as $\mathcal{T}$.

Notice that $\bSigma_{\btheta} = \frac{1}{K}\sum_{k=1}^K\tE(\bar{\bg}_k(\btheta) - \bmu_{\btheta})(\bar{\bg}_k(\btheta) - \bmu_{\btheta})^\top = \widetilde{\bSigma}_{\btheta} + \frac{1}{K}\sum_{k=1}^K (\bmuk{k}_{\btheta} - \bmu_{\btheta})(\bmuk{k}_{\btheta} - \bmu_{\btheta})^\top$. Then by Lemma \ref{lem: S hat properties}.(\rom{1}), we have there exists a subset $S_0 \subseteq S \cap \widehat{S}$ with $|S_0| \geq K(1-3\epsilon)$, such that
\begin{align}
    \hSigma_{\btheta} - \bSigma_{\btheta} &= \frac{1}{|\hS|}\sum_{k \in \hS}\hSigmak{k}_{\btheta} - \bigg[\widetilde{\bSigma}_{\btheta} + \frac{1}{K}\sum_{k=1}^K (\bmuk{k}_{\btheta} - \bmu_{\btheta})(\bmuk{k}_{\btheta} - \bmu_{\btheta})^\top\bigg] \\
    &= \frac{1}{|\hS|}(\sum_{k \in S_0}\hSigmak{k}_{\btheta} + \sum_{k \in \hS \backslash S_0}\hSigmak{k}_{\btheta}) - \widetilde{\bSigma}_{\btheta} - \frac{1}{K}\sum_{k=1}^K (\bmuk{k}_{\btheta} - \bmu_{\btheta})(\bmuk{k}_{\btheta} - \bmu_{\btheta})^\top \\
    &= \frac{1}{|\hS|}\sum_{k \in S_0}(\hSigmak{k}_{\btheta} - \bSigmak{k}_{\btheta}) + \frac{1}{|\hS|}\sum_{k \in \hS\backslash S_0}\hSigmak{k}_{\btheta} + (\frac{1}{|\hS|} - \frac{1}{K})\sum_{k \in S_0}\bSigmak{k}_{\btheta} - \frac{1}{K}\sum_{k \in [K] \backslash S_0}\bSigmak{k}_{\btheta} \\
    &\quad - \frac{1}{K}\sum_{k=1}^K (\bmuk{k}_{\btheta} - \bmu_{\btheta})(\bmuk{k}_{\btheta} - \bmu_{\btheta})^\top,
\end{align}
By \eqref{eq: second_order_quadratic_concentration}, with probability at least $1-\delta$, we have
\begin{align}
    \frac{1}{|\hS|}\sup_{\btheta \in \Theta}\bigg\|\sum_{k \in S_0}(\hSigmak{k}_{\btheta} - \bSigmak{k}_{\btheta})\bigg\|_2 &\lesssim \sup_{S':|S'| \geq K(1-3\epsilon)} \sup_{\btheta \in \Theta}\twonorm{\frac{1}{|S'|}\sum_{k \in S'}(\hSigmak{k}_{\btheta} - \bSigmak{k}_{\btheta})} \\
    &\lesssim \frac{1}{n}\left[\sqrt{\frac{d\log(R_0/\beta)+\log(1/\delta)}{K}}+\frac{d\log(R_0/\beta)+\log(1/\delta)}{K}\right]+\epsilon\frac{\log(1/\epsilon)}{n}.
\end{align}
By Lemma \ref{lem: S hat properties}.(\rom{2}), with probability at least $1-\delta$, we have
\begin{equation}
    \frac{1}{|\hS|}\sup_{\btheta \in \Theta}\bigg\|\sum_{k \in \hS\backslash S_0}\hSigmak{k}_{\btheta}\bigg\|_2 \lesssim \epsilon \sup_{\btheta \in \Theta}\max_{k \in \hS}\twonorm{\hSigmak{k}_{\btheta}} \lesssim \epsilon \mathcal{T}.
\end{equation}
Note that $\bSigmak{k}_{\btheta} = \cov(\bar{g}_k(\btheta)) = \frac{1}{n}\cov(\nabla \ell(\zk{k}, \btheta))$, therefore $\twonorm{\bSigmak{k}_{\btheta}} \lesssim 1/n$ for all $k = 1:K$. And by Lemma \ref{lem: S hat properties}.(\rom{1}), we have $|\hS| \geq K(1-3\epsilon)$, which implies
\begin{align}
    (\frac{1}{|\hS|} - \frac{1}{K})\bigg\|\sum_{k \in S_0}\bSigmak{k}_{\btheta}\bigg\|_2 &\lesssim \epsilon \sup_{\btheta \in \Theta}\max_{k}\twonorm{\bSigmak{k}_{\btheta}} \lesssim \frac{\epsilon}{n},\\
    \frac{1}{K}\bigg\|\sum_{k \in [K] \backslash S_0}\bSigmak{k}_{\btheta}\bigg\|_2 &\lesssim \epsilon \sup_{\btheta \in \Theta}\max_{k}\twonorm{\bSigmak{k}_{\btheta}} \lesssim \frac{\epsilon}{n}.
\end{align}
And by the definition of $h^2$, we have
\begin{equation}
    \frac{1}{K}\sup_{\btheta \in \Theta}\bigg\|\sum_{k=1}^K (\bmuk{k}_{\btheta} - \bmu_{\btheta})(\bmuk{k}_{\btheta} - \bmu_{\btheta})^\top\bigg\|_2 \lesssim h^2.
\end{equation}
Putting everything together, we have
\begin{align}
    \twonorm{\hSigma - \bSigma}
    &\lesssim \frac{1}{n} \bigg[\sqrt{\frac{d\log(R_0/\beta) + \log(1/\delta)}{K}} + \frac{d\log(R_0/\beta) + \log(1/\delta)}{K}\bigg]  + \epsilon \frac{\log(1/\epsilon)}{n} + L'^2\beta^2 + \frac{L'\beta}{n} + h^2 \\
    &\quad + \frac{\epsilon}{n}\Bigg(\sqrt{\frac{d\log(R_0/\beta) +\log(K/\delta)}{n}} \vee \frac{d\log(R_0/\beta) +\log(K/\delta)}{n}\Bigg)(1+L'\beta).
\end{align}
This completes the proof.

% -------------------------------
\subsubsection{Proof of Lemma \ref{lem: cov uniform concentration}}
Recall that $\bSigma_{\btheta} \coloneqq \cov(\frac{1}{n}\sum_{i=1}^n\nabla \ell(\zk{k}_i, \btheta)) = \frac{1}{n}\cov(\nabla \ell(\zk{k}, \btheta))$, which satisfies $\twonorm{\bSigma_{\btheta}}  = \frac{1}{n} \twonorm{\cov(\nabla \ell(\zk{k}, \btheta))} \lesssim \frac{1}{n}$ by sub-Gaussianity of $\nabla \ell(\zk{k}, \btheta)$.

By standard concentration arguments, for any $\btheta \in \Theta$,
\begin{equation}
    \tP(\twonorm{\hSigmak{k}_{\btheta} - \bSigmak{k}_{\btheta}} > t) \lesssim 9^d \exp\{-Cn[(t/\twonorm{\bSigmak{k}_{\btheta}})^2 \wedge (t/\twonorm{\bSigmak{k}_{\btheta}})]\} \lesssim 9^d \exp\{-Cn[(nt)^2 \wedge (nt)]\}.
\end{equation}
Therefore, by union bounds, with probability at least $1-\delta$,
\begin{equation}
    \max_{k \in S}\twonorm{\hSigmak{k}_{\btheta} - \bSigmak{k}_{\btheta}} \lesssim \frac{1}{n}\times \Bigg(\sqrt{\frac{d+\log(K/\delta)}{n}} \vee \frac{d+\log(K/\delta)}{n}\Bigg).
\end{equation}
By Lipschitzness of $\nabla \ell(z, \btheta)$ w.r.t. $\btheta$ and Cauchy-Schwarz inequality, we have
\begin{align}
n \twonorm{\bSigma_{\btheta'} - \bSigma_{\btheta}}
&= \|\tE\big[(\nabla \ell(z,\btheta') - \tE \nabla \ell(z,\btheta'))(\nabla \ell(z,\btheta') - \tE \nabla \ell(z,\btheta'))^\top\big] \\
&\qquad - \tE\big[(\nabla \ell(z,\btheta) - \tE \nabla \ell(z,\btheta))(\nabla \ell(z,\btheta) - \tE \nabla \ell(z,\btheta))^\top\big]\|_2 \\
&= \sup_{\bv \in S^{d-1}} \Big| \bv^\top \tE\Big[\big(\nabla \ell(z,\btheta) - \tE \nabla \ell(z,\btheta) + \nabla \ell(z,\btheta') - \nabla \ell(z,\btheta) + \tE \nabla \ell(z,\btheta) - \tE \nabla \ell(z,\btheta')\big) \\
&\qquad\qquad \big(\nabla \ell(z,\btheta) - \tE \nabla \ell(z,\btheta) + \nabla \ell(z,\btheta') - \nabla \ell(z,\btheta) + \tE \nabla \ell(z,\btheta) - \tE \nabla \ell(z,\btheta')\big)^\top\Big] \bv \\
&\qquad - \bv^\top \tE\Big[(\nabla \ell(z,\btheta) - \tE \nabla \ell(z,\btheta))(\nabla \ell(z,\btheta) - \tE \nabla \ell(z,\btheta))^\top\Big] \bv \Big| \\
&\leq \sup_{\bv \in S^{d-1}} \tE[\bv^\top(\nabla \ell(z,\btheta) - \tE\nabla \ell(z,\btheta))(\nabla \ell(z,\btheta') - \nabla \ell(z,\btheta))^\top\bv] \\
&\quad + \sup_{\bv \in S^{d-1}} \tE[\bv^\top(\nabla \ell(z,\btheta') - \nabla \ell(z,\btheta))(\nabla \ell(z,\btheta) - \tE\nabla \ell(z,\btheta))^\top\bv] \\
&\quad + \sup_{\bv \in S^{d-1}} \tE[\bv^\top(\nabla \ell(z,\btheta') - \nabla \ell(z,\btheta))(\nabla \ell(z,\btheta') - \nabla \ell(z,\btheta))^\top\bv]\\
&\quad + \sup_{\bv \in S^{d-1}} \tE[\bv^\top(\nabla \ell(z,\btheta') - \nabla \ell(z,\btheta))(\tE \nabla \ell(z,\btheta) - \tE \nabla \ell(z,\btheta'))^\top\bv] \\
&\quad + \sup_{\bv \in S^{d-1}} \tE[\bv^\top(\nabla \ell(z,\btheta') - \nabla \ell(z,\btheta))(\nabla \ell(z,\btheta') - \nabla \ell(z,\btheta))^\top\bv]\\
&\quad + \sup_{\bv \in S^{d-1}} \tE[\bv^\top(\tE \nabla \ell(z,\btheta) - \tE \nabla \ell(z,\btheta'))(\tE \nabla \ell(z,\btheta) - \tE \nabla \ell(z,\btheta'))^\top\bv] \\
&\leq \twonorm{\bSigma_{\btheta}}\cdot L'\twonorm{\btheta - \btheta'} +  L'\twonorm{\btheta - \btheta'}\cdot \twonorm{\bSigma_{\btheta}} + L'^2 \twonorm{\btheta - \btheta'}^2 \\
&\quad + L'^2 \twonorm{\btheta - \btheta'}^2 + L'^2 \twonorm{\btheta - \btheta'}^2 + L'^2 \twonorm{\btheta - \btheta'}^2\\
&\le 2L' \twonorm{\bSigma_{\btheta}} \twonorm{\btheta - \btheta'} + 4 L'^2 \twonorm{\btheta - \btheta'}^2
\end{align}
%{\color{red}[M: are you sure about the last inequality? I don't see it.]} \yt{I have unpacked it and I hope it is clearer now.}

%
Similarly, for any $\btheta,\btheta'\in\Theta$ with
$\twonorm{\btheta-\btheta'}\le \beta$, the Lipschitz condition gives
\begin{align}
    \twonorm{\hSigmak{k}_{\btheta'}-\hSigmak{k}_{\btheta}}
    &\lesssim
    L'^2\beta^2
    +
    L'\beta
    \left(
        \twonorm{\hSigmak{k}_{\btheta'}}
        +
        \twonorm{\hSigmak{k}_{\btheta}}
    \right),
    \label{eq: empirical_cov_lipschitz_increment}
\end{align}
and by sub-Gaussianity
\begin{align}
    \twonorm{\bSigmak{k}_{\btheta'}-\bSigmak{k}_{\btheta}}
    \lesssim
    L'^2\beta^2+\frac{L'\beta}{n}.
    \label{eq: population_cov_lipschitz_increment}
\end{align}

Let $\mathcal N_\beta$ be a $\beta$-cover of $\Theta$ with
$|\mathcal N_\beta|\lesssim (R_0/\beta)^d$. Define
\[
    M
    :=
    \sup_{\btheta\in\Theta}
    \max_{k\in S}
    \twonorm{\hSigmak{k}_{\btheta}-\bSigmak{k}_{\btheta}},
\]
and
\[
    M_{\mathcal N}
    :=
    \max_{\btheta\in\mathcal N_\beta}
    \max_{k\in S}
    \twonorm{\hSigmak{k}_{\btheta}-\bSigmak{k}_{\btheta}}.
\]
By the fixed-$\btheta$ concentration inequality and a union bound over
$k\in S$ and $\btheta\in\mathcal N_\beta$, with probability at least
$1-\delta$,
\[
    M_{\mathcal N}
    \lesssim
    \frac1n
    \left(
    \sqrt{
    \frac{d\log(R_0/\beta)+\log(K/\delta)}{n}
    }
    \vee
    \frac{d\log(R_0/\beta)+\log(K/\delta)}{n}
    \right).
\]

Now fix any $\btheta\in\Theta$ and choose
$\btheta'\in\mathcal N_\beta$ such that
$\twonorm{\btheta-\btheta'}\le\beta$. For every $k\in S$,
\begin{align}
    \twonorm{\hSigmak{k}_{\btheta}-\bSigmak{k}_{\btheta}}
    &\le
    \twonorm{\hSigmak{k}_{\btheta'}-\bSigmak{k}_{\btheta'}}
    +
    \twonorm{\hSigmak{k}_{\btheta}-\hSigmak{k}_{\btheta'}}
    +
    \twonorm{\bSigmak{k}_{\btheta}-\bSigmak{k}_{\btheta'}}.
\end{align}
Using \eqref{eq: empirical_cov_lipschitz_increment} and
\eqref{eq: population_cov_lipschitz_increment}, together with
\[
    \twonorm{\hSigmak{k}_{\btheta}}
    \le
    \twonorm{\hSigmak{k}_{\btheta}-\bSigmak{k}_{\btheta}}
    +
    \twonorm{\bSigmak{k}_{\btheta}}
    \le
    M+\frac{C}{n},
\]
and
\[
    \twonorm{\hSigmak{k}_{\btheta'}}
    \le
    \twonorm{\hSigmak{k}_{\btheta'}-\bSigmak{k}_{\btheta'}}
    +
    \twonorm{\bSigmak{k}_{\btheta'}}
    \le
    M_{\mathcal N}+\frac{C}{n},
\]
we obtain
\begin{align}
    \twonorm{\hSigmak{k}_{\btheta}-\bSigmak{k}_{\btheta}}
    &\le
    M_{\mathcal N}
    +
    C L'\beta M
    +
    C L'\beta M_{\mathcal N}
    +
    C\frac{L'\beta}{n}
    +
    C L'^2\beta^2 .
\end{align}
Taking the supremum over $\btheta\in\Theta$ and $k\in S$ yields
\[
    M
    \le
    (1+C L'\beta)M_{\mathcal N}
    +
    C L'\beta M
    +
    C\frac{L'\beta}{n}
    +
    C L'^2\beta^2 .
\]
For $L'\beta$ sufficiently small, the term $C L'\beta M$ can be absorbed into
the left-hand side. Hence
\[
    M
    \lesssim
    (1+L'\beta)M_{\mathcal N}
    +
    \frac{L'\beta}{n}
    +
    L'^2\beta^2 .
\]
Consequently, with probability at least $1-\delta$,
\begin{align}
    &\sup_{\btheta\in\Theta}
    \max_{k\in S}
    \twonorm{\hSigmak{k}_{\btheta}-\bSigmak{k}_{\btheta}}
    \nonumber\\
    &\lesssim
    \frac1n
    \left(
    \sqrt{
    \frac{d\log(R_0/\beta)+\log(K/\delta)}{n}
    }
    \vee
    \frac{d\log(R_0/\beta)+\log(K/\delta)}{n}
    \right)
    (1+L'\beta)
    +
    L'^2\beta^2
    +
    \frac{L'\beta}{n}.
\end{align}

% \begin{align}
%     n\twonorm{\hSigmak{k}_{\btheta'} - \hSigmak{k}_{\btheta}}
% &\le 8 L'^2 \twonorm{\btheta - \btheta'}^2 + 2L' \twonorm{\hSigmak{k}_{\btheta'}} \twonorm{\btheta - \btheta'} + 2L' \twonorm{\hSigmak{k}_{\btheta}} \twonorm{\btheta - \btheta'} \\
% &\le 8 L'^2 \twonorm{\btheta - \btheta'}^2 + 2 \MC{\sup_{\theta}?}\sup_{k \in [K]} \twonorm{\hSigmak{k}_{\btheta} - \bSigma_{\btheta}}\cdot 2L' \twonorm{\btheta - \btheta'} + 4 \sup_{\btheta} \twonorm{\bSigmak{k}_{\btheta}}\cdot  L' \twonorm{\btheta - \btheta'}.
% \end{align}
% Hence, for a $\beta$-cover $\mathcal{N}_\beta$ of $\Theta$, we have $|\mathcal{N}_\beta| \lesssim (R_0/\beta)^d$ and
% \begin{align}
%     &\sup_{\btheta \in \Theta}\max_{k \in S}\twonorm{\hSigmak{k}_{\btheta} - \bSigmak{k}_{\btheta}} \\
%     &\lesssim \sup_{\btheta \in \Theta}\twonorm{\bSigmak{k}_{\btheta}}\cdot \Bigg(\sqrt{\frac{d\log(R_0/\beta)+\log(K/\delta)}{n}} \vee \frac{d\log(R_0/\beta)+\log(K/\delta)}{n}\Bigg)(1+L'\beta) \\
%     &\quad + L'^2\beta^2 + \sup_{\btheta \in \Theta}\twonorm{\bSigmak{k}_{\btheta}}L'\beta \\
%     &\lesssim \frac{1}{n}\Bigg(\sqrt{\frac{d\log(R_0/\beta) +\log(K/\delta)}{n}} \vee \frac{d\log(R_0/\beta) +\log(K/\delta)}{n}\Bigg)(1+L'\beta) + L'^2\beta^2 + \frac{1}{n}L'\beta,
% \end{align}
% with probability at least $1-\delta$.

% --------------------------------
\subsubsection{Proof of Lemma \ref{lem: S hat properties}}
For simplicity, we write $\widehat{S}_{\textup{safe}}$ as $\widehat{S}$ here.

\noindent (i) The number of $(k, k')$ pairs above the quantile in the definition of $\hS$ is at most $\binom{K}{2} - \binom{K(1-\epsilon)}{2} = \frac{K^2(2\epsilon-\epsilon^2) - K\epsilon}{2}$. Note that for each task $k \notin \widehat{S}$, there are at least $3K/4$ different values of $k' \in [K]$ such that $\twonorm{\hSigmak{k}_{\btheta} - \hSigmak{k}_{\btheta}}$ is bigger than the quantile. If we assume there are $\Delta$ tasks not in $\widehat{S}$, then we must have %
\begin{equation}\label{eq: a random eq}
    \frac{3K}{4}\Delta - \binom{\Delta}{2} \leq \frac{K^2(2\epsilon-\epsilon^2) - K\epsilon}{2} \leq \frac{K^2}{2}(2\epsilon - \epsilon^2),
\end{equation}
where the LHS is a lower bound of the number of different pairs $(k, k')$ above the quantile, by the definition of $\widehat{S}$. Note that the $\textup{LHS} \geq \frac{(\frac{3}{2}K - \Delta)\Delta}{2}$. When $\epsilon \leq \frac{4}{15}$, it is easy to verify that $\frac{(\frac{3}{2}K - 2K\epsilon)\times 2K\epsilon}{2} \geq \frac{K^2}{2}(2\epsilon - \epsilon^2)$. Because of the monotonicity of the LHS of \eqref{eq: a random eq} as a function of $\Delta \leq \frac{3K}{4}$, we must have $\Delta \leq 2K\epsilon$. Therefore $|\hS \cap S| \geq K - 2K\epsilon - K\epsilon = K(1-3\epsilon)$.

\noindent (ii) This is by Lemma \ref{lem: cov uniform concentration} and the definition of $\widehat{S}$.

% ------------------------------------------------
\subsection{Proofs of results in Section \ref{subsec: examples}}

% --------------------------------
\subsubsection{Proof of Lemma \ref{lem: glm verify asmp}}

Note that 
\begin{equation}
    \nabla \ell(\btheta,z) = \bx(\varphi'(\langle \bx,\btheta\rangle)-y), \quad \nabla^2 \ell(\btheta,z) = \varphi''(\langle \bx,\btheta\rangle)\bx\bx^\top.
\end{equation}
(\rom{1}) For Assumption \ref{asmp: risk function}, it suffices to verify that
\begin{equation}
    C_1 \leq \lambdamin(\tE[\varphi''(\langle \bxk{k},\btheta\rangle)\bxk{k}(\bxk{k})^\top]) \leq \lambdamax(\tE[\varphi''(\langle \bxk{k},\btheta\rangle)\bxk{k}(\bxk{k})^\top]) \leq C_2,
\end{equation}
for some constants $C_1, C_2 > 0$. Consider $\|\bxk{k}_i\|_{\psi_2} \leq C_{\psi}$.

Since $\sup_{u}\varphi''(u) \lesssim 1$, we have $\lambdamax(\tE[\varphi''(\langle \bxk{k},\btheta\rangle)\bxk{k}(\bxk{k})^\top]) \lesssim \lambdamax(\tE[\bxk{k}(\bxk{k})^\top]) \lesssim 1$. On the other hand, for any $\tilde{C} > 0$ and $\bm{u}$ with $\twonorm{\bm{u}} = 1$, we have
\begin{align}
    \tE[\varphi''(\langle \bxk{k},\btheta\rangle)(\bu^\top\bxk{k})^2] &\geq \tE[\varphi''(\langle \bxk{k},\btheta\rangle)(\bu^\top\bxk{k})^2\mathds{1}(|\langle \bxk{k},\btheta\rangle| \leq \tilde{C})] \\
    &\geq \inf_{u \in [-\tilde{C}, \tilde{C}]}\varphi''(u)\cdot \Big(\tE[(\bu^\top\bxk{k})^2] - \tE[(\bu^\top\bxk{k})^2\mathds{1}(|\langle \bxk{k},\btheta\rangle| > \tilde{C})]\Big) \\
    &\geq \inf_{u \in [-\tilde{C}, \tilde{C}]}\varphi''(u)\cdot \Big(\underline{\lambda} - \sqrt{\tE[(\bu^\top\bxk{k})^4]}\sqrt{\tP(|\langle \bxk{k},\btheta\rangle| > \tilde{C})} \Big) \\
    &\geq \inf_{u \in [-\tilde{C}, \tilde{C}]}\varphi''(u)\cdot \Big(\underline{\lambda} - CC_{\psi}\exp\{-C'\tilde{C}^2\} \Big).
\end{align}
By choosing a sufficiently large constant $\tilde{C}$, we can get the $\text{RHS} \gtrsim 1$.

For Assumption \ref{asmp: lipschitz loss}, it suffices to show that $\max_{i, k}\lambdamax(\varphi''(\langle \bxk{k},\btheta\rangle)\bxk{k}_i(\bxk{k}_i)^\top) \lesssim d + \log(nK)$ with probability at least $1-(nK)^{-Cd}$ for some constant $C > 0$. Since $\sup_{u}\varphi''(u) \lesssim 1$, we have $\max_{i, k}\lambdamax(\varphi''(\langle \bxk{k}_i,\btheta\rangle)\bxk{k}_i(\bxk{k}_i)^\top) \lesssim  \max_{i, k}\twonorm{\bxk{k}_i}^2$. Then the bound $\max_{i, k}\twonorm{\bxk{k}_i}^2 \lesssim d + \log(nK)$ with probability at least $1-(nK)^{-Cd}$ holds immediately by the sub-Gaussianity of $\bxk{k}_i$'s and the union bound.

\noindent (\rom{2}) Note that
\begin{align}
    \nabla \ell(\btheta,\zk{k}_i) &= \bxk{k}_i(\varphi'(\langle \bxk{k}_i,\btheta\rangle)-\yk{k}_i) \\
    &= \bxk{k}_i[\varphi'(\langle \bxk{k}_i,\btheta\rangle) - \varphi'(\langle \bxk{k}_i,\bthetaks{k}\rangle)] + \bxk{k}_i[\varphi'(\langle \bxk{k}_i,\bthetaks{k}\rangle) - \yk{k}_i].
\end{align}
Then both terms above are sub-Gaussian with constant variance proxy by either the first or the second condition.

% --------------------------------
\subsubsection{Proof of Lemma \ref{lem: glm verify asmp heterogeneity}}

Since $\bxk{k}$'s share the same distribution, we write them as $\bx$ in some cases. Note that
    \begin{equation}
        \nabla \mLk{k}(\btheta) = \tE[\bxk{k}(\varphi'(\langle \bxk{k}, \btheta \rangle) - \yk{k})] = \tE[\bxk{k}(\varphi'(\langle \bxk{k}, \btheta \rangle) - \varphi'(\langle \bxk{k}, \bthetaks{k} \rangle))], 
    \end{equation}
    where the second inequality is due to the fact that $\tE[\yk{k}|\bxk{k}=\bx] = \varphi'(\langle \bx, \bthetaks{k} \rangle)$. Therefore, we have
    \begin{align}
        \twonorm{\nabla \mLk{k}(\btheta) - \nabla \mL(\btheta)} &= \twonorm{\frac{1}{K}\sum_{k'=1}^K\tE[\bx(\varphi'(\langle \bx, \bthetaks{k} \rangle) - \varphi'(\langle \bx, \bthetaks{k'} \rangle))]} \\
        &\leq \twonorm{\frac{1}{K}\sum_{k'=1}^K\tE[\bx\bx^\top (\bthetaks{k}-\bthetaks{k'})\varphi''(\langle \bx, t_k\bthetaks{k} + (1-t_k)\bthetaks{k'}\rangle)]} \\
        &\lesssim \frac{1}{K}\sum_{k'=1}^K \twonorm{\bthetaks{k}-\bthetaks{k'}} \\
        &\lesssim \max_{k=1:K}\min_{\otheta}\twonorm{\bthetaks{k}-\otheta}.
    \end{align}
    This implies that
    \begin{align}
        h^2 &\leq \frac{1}{K}\sum_{k=1}^K \twonorm{\nabla \mLk{k}(\btheta) - \nabla \mL(\btheta)}^2 \\
        &\leq \frac{1}{K}\sum_{k=1}^K \bigg[\frac{1}{K}\sum_{k'=1}^K \twonorm{\bthetaks{k}-\bthetaks{k'}}\bigg]^2 \\
        &\lesssim \frac{1}{K^2}\sum_{k, k'}\twonorm{\bthetaks{k}-\bthetaks{k'}}^2 \\
        &\lesssim \max_{k=1:K}\min_{\otheta}\twonorm{\bthetaks{k}-\otheta}^2,
    \end{align}
    where the second last inequality is due to Jensen's inequality.

% ------------------------------------------------------
% Appendix: Additional numerical results
% ------------------------------------------------------

\section{Additional numerical results for Section \ref{sec: experiment}}\label{sec: experiment supp}

This appendix collects the simulation and real-data results omitted from the main text. The data generation, contamination mechanism, evaluation metrics, and benchmark implementations are the same as those described in Section \ref{sec: experiment}.

% ------------------------------------------------------
\subsection{Additional simulation results}\label{subsec: simulation supp}

% ---------------------------
\subsubsection{Linear regression with increasing number of tasks}

The first setting uses $n=d=50$, contamination proportion $\epsilon=0.2$, and varying numbers of tasks $K \in \{10,20,30,40,50,60\}$. Table \ref{tab:sim-linear-K} reports the linear-regression global and local estimation errors.

The global error of our estimator decreases from $0.724$ to $0.291$ as $K$ grows, and it is uniformly the smallest among the reported methods. The local error of our estimator also improves from $1.081$ to $0.980$ and is the smallest across all reported methods and all values of $K$. For each setting, we perform $100$ replications and report the average error in the tables. In most cases, the advantage of our method is substantial, in the sense that the gap between our method and the benchmarks is larger than twice the standard deviation of the error across replications.

\begin{table}[!ht]
\centering
\caption{Linear regression with $n=d=50$, $\epsilon=0.2$, and varying $K$}
\label{tab:sim-linear-K}
\setlength{\tabcolsep}{5pt}
\begin{adjustbox}{width=\textwidth}
\begin{tabular}{lcccccccccccc}
\toprule
& \multicolumn{6}{c}{Global error $\twonorm{\htheta - \bthetas}$} & \multicolumn{6}{c}{Local error $|S|^{-1}\sum_{k \in S}\twonorm{\hthetak{k} - \bthetaks{k}}$} \\
\cmidrule(lr){2-7} \cmidrule(lr){8-13}
Method$\backslash K$ & 10 & 20 & 30 & 40 & 50 & 60 & 10 & 20 & 30 & 40 & 50 & 60 \\
\midrule
Ours & \textbf{0.724} & \textbf{0.499} & \textbf{0.400} & \textbf{0.352} & \textbf{0.324} & \textbf{0.291} & \textbf{1.081} & \textbf{1.017} & \textbf{0.990} & \textbf{0.987} & \textbf{0.985} & \textbf{0.980} \\
Average & 13.357 & 13.165 & 13.108 & 13.064 & 13.041 & 13.027 & 13.406 & 13.218 & 13.163 & 13.110 & 13.095 & 13.080 \\
Single-task & -- & -- & -- & -- & -- & -- & 1.899 & 1.903 & 1.907 & 1.890 & 1.901 & 1.907 \\
Median & 1.152 & 0.941 & 0.868 & 0.835 & 0.794 & 0.777 & 1.564 & 1.476 & 1.455 & 1.445 & 1.428 & 1.422 \\
Trimmed mean & 1.202 & 1.064 & 1.033 & 1.027 & 0.989 & 0.989 & 1.602 & 1.557 & 1.560 & 1.562 & 1.544 & 1.547 \\
Krum & 1.964 & 1.901 & 1.882 & 1.821 & 1.799 & 1.687 & 2.231 & 2.222 & 2.214 & 2.170 & 2.160 & 2.071 \\
Bulyan & 1.070 & 0.774 & \textit{0.645} & 0.567 & 0.514 & 0.470 & 1.499 & 1.376 & 1.333 & 1.313 & 1.291 & 1.283 \\
Filtering & \textit{0.864} & \textit{0.624} & \textit{0.530} & \textit{0.469} & \textit{0.439} & \textit{0.402} & 1.357 & \textit{1.293} & \textit{1.279} & 1.273 & 1.263 & 1.258 \\
MoM-Filtering & \textit{0.864} & \textit{0.624} & \textit{0.530} & \textit{0.416} & \textit{0.380} & \textit{0.345} & 1.357 & \textit{1.293} & \textit{1.279} & \textit{1.255} & \textit{1.246} & \textit{1.242} \\
MoM-Krum & 1.247 & 1.210 & 1.100 & 1.091 & 1.068 & 1.042 & 1.624 & 1.655 & 1.605 & 1.610 & 1.594 & 1.582 \\
ARMUL & 1.233 & 1.132 & 1.097 & 1.094 & 1.067 & 1.071 & \textit{1.306} & \textit{1.227} & \textit{1.190} & \textit{1.184} & \textit{1.173} & \textit{1.174} \\
History & \textit{0.907} & \textit{0.729} & 0.678 & 0.656 & 0.621 & 0.616 & 1.393 & 1.351 & 1.350 & 1.351 & 1.339 & 1.342 \\
Bucketing & 1.436 & 1.239 & 1.206 & 1.197 & 1.155 & 1.148 & 1.786 & 1.683 & 1.679 & 1.678 & 1.655 & 1.654 \\
Mean-reg & 3.127 & 3.117 & 3.119 & 3.102 & 3.146 & 3.194 & 1.775 & 1.784 & 1.814 & 1.824 & 1.857 & 1.903 \\
Dirty & 2.930 & 2.650 & 2.614 & 2.982 & 3.089 & 3.076 & 1.743 & 1.766 & 1.755 & 1.767 & 1.789 & 1.784 \\
RMTFL & 3.074 & 2.974 & 2.975 & 2.980 & 2.992 & 2.983 & 1.743 & 1.718 & 1.713 & 1.710 & 1.713 & 1.730 \\
RLRMTL & 2.731 & 3.017 & 3.146 & 3.285 & 3.350 & 3.334 & \textit{1.273} & 1.471 & 1.568 & 1.686 & 1.745 & 1.720 \\
\bottomrule
\end{tabular}
\end{adjustbox}
\end{table}

% ---------------------------
\subsubsection{Logistic regression with increasing number of tasks}

We also repeat the varying-$K$ study under binary logistic regression with $n=d=50$. Clean tasks use the same coefficient heterogeneity model as in the linear model but with Bernoulli responses generated with the logistic link. In particular $
     \bthetaks{k} \sim N(\bthetas, \sigma^2 \bm{I}_d/d)$ and $ \bx_i^{(k)} \sim N(0, \bm{I}_d)$
where $\bthetas = 3d^{-1/2}\bm{1}_d$.  Contaminated tasks use shifted covariates $\bx_i^{(k)} \sim N(2\times \bm{1}_d,\bm{I}_d)$, a sign-reversed coefficient vector $-3\bthetas$. The index set $S^c$ of contaminated tasks is randomly selected from $[K]$ with size $\epsilon K$. Table \ref{tab:sim-logistic-K} reports the corresponding estimation errors.

For both global and local estimation, our estimator has the smallest error for every reported value of $K$.

\begin{table}[!ht]
\centering
\caption{Logistic regression with $n=d=50$, $\epsilon=0.2$, and varying $K$}
\label{tab:sim-logistic-K}
\setlength{\tabcolsep}{4pt}
\begin{adjustbox}{width=\textwidth}
\begin{tabular}{lcccccccccccc}
\toprule
& \multicolumn{6}{c}{Global error $\twonorm{\htheta - \bthetas}$} & \multicolumn{6}{c}{Local error $|S|^{-1}\sum_{k \in S}\twonorm{\hthetak{k} - \bthetaks{k}}$} \\
\cmidrule(lr){2-7} \cmidrule(lr){8-13}
Method$\backslash K$ & 10 & 20 & 30 & 40 & 50 & 60 & 10 & 20 & 30 & 40 & 50 & 60 \\
\midrule
Ours & \textbf{0.992} & \textbf{0.726} & \textbf{0.640} & \textbf{0.586} & \textbf{0.556} & \textbf{0.539} & \textbf{1.683} & \textbf{1.571} & \textbf{1.548} & \textbf{1.526} & \textbf{1.519} & \textbf{1.521} \\
Average & 2.774 & 2.675 & 2.639 & 2.628 & 2.619 & 2.612 & 3.469 & 3.421 & 3.406 & 3.398 & 3.392 & 3.394 \\
Single-task & -- & -- & -- & -- & -- & -- & 2.747 & 2.759 & 2.755 & 2.752 & 2.754 & 2.759 \\
Median & 1.808 & 1.679 & 1.609 & 1.597 & 1.589 & 1.580 & 2.534 & 2.488 & 2.454 & 2.449 & 2.447 & 2.450 \\
Trimmed mean & 1.945 & 1.886 & 1.850 & 1.851 & 1.852 & 1.849 & 2.680 & 2.686 & 2.677 & 2.681 & 2.685 & 2.693 \\
Krum & 2.414 & 2.397 & 2.383 & 2.362 & 2.325 & 2.312 & 2.828 & 2.827 & 2.817 & 2.800 & 2.758 & 2.747 \\
Bulyan & \textit{1.411} & \textit{1.026} & \textit{0.862} & 0.761 & 0.694 & 0.644 & \textit{1.957} & \textit{1.740} & \textit{1.671} & 1.624 & 1.604 & \textit{1.588} \\
Filtering & \textit{1.169} & \textit{0.874} & \textit{0.746} & \textit{0.656} & \textit{0.608} & \textit{0.561} & \textit{1.815} & \textit{1.666} & \textit{1.616} & \textit{1.577} & \textit{1.558} & \textit{1.550} \\
MoM-Filtering & \textit{1.169} & \textit{0.874} & \textit{0.746} & \textit{0.658} & \textit{0.603} & \textit{0.579} & \textit{1.815} & \textit{1.666} & \textit{1.616} & \textit{1.611} & \textit{1.591} & 1.591 \\
MoM-Krum & 2.701 & 2.361 & 2.212 & 2.111 & 1.938 & 1.818 & 3.284 & 2.960 & 2.845 & 2.731 & 2.568 & 2.470 \\
ARMUL & 2.181 & 2.144 & 2.128 & 2.128 & 2.122 & 2.116 & 2.695 & 2.782 & 2.879 & 2.907 & 2.903 & 2.905 \\
History & 2.441 & 2.272 & 2.209 & 2.193 & 2.179 & 2.168 & 3.151 & 3.043 & 3.005 & 2.993 & 2.984 & 2.983 \\
Bucketing & 2.774 & 2.672 & 2.634 & 2.624 & 2.615 & 2.607 & 3.468 & 3.418 & 3.402 & 3.394 & 3.388 & 3.389 \\
Mean-reg & 1.769 & 1.736 & 1.739 & 1.757 & 1.767 & 1.772 & 2.322 & 2.344 & 2.369 & 2.388 & 2.401 & 2.417 \\
Dirty & 2.196 & 2.322 & 2.386 & 2.399 & 2.387 & 2.374 & 2.848 & 3.040 & 3.094 & 3.088 & 3.072 & 3.054 \\
RMTFL & 2.169 & 1.978 & 2.452 & 2.452 & 2.447 & 2.442 & 2.762 & 2.658 & 2.621 & 2.663 & 2.701 & 2.740 \\
RLRMTL & 2.506 & 2.469 & 2.452 & 2.452 & 2.447 & 2.442 & 3.216 & 3.237 & 3.238 & 3.237 & 3.235 & 3.239 \\
\bottomrule
\end{tabular}
\end{adjustbox}
\end{table}

% ---------------------------
\subsubsection{Linear regression with varying contamination level}

We next consider the setting $n=d=50$, the number of tasks $K=40$, and varying contamination proportions $\epsilon \in \{0,0.05,0.10,0.15,0.20,0.25\}$. The data generation process and contamination mechanism follow the updated linear simulation script described above. Table \ref{tab:sim-epsilon} reports the global and local estimation errors.

Table \ref{tab:sim-epsilon} shows that our method remains stable as contamination increases: the global error rises only from $0.317$ to $0.365$, and the local error stays essentially flat from $0.981$ to $0.989$. Average becomes unstable under contamination, Bucketing degrades at higher contamination levels, and Bulyan is unavailable at $\epsilon=0.25$.

\begin{table}[!ht]
\centering
\caption{Linear regression with $n=d=50$, $K=40$, and varying contamination level $\epsilon$}
\label{tab:sim-epsilon}
\setlength{\tabcolsep}{4pt}
\begin{adjustbox}{width=\textwidth}
\begin{tabular}{lcccccccccccc}
\toprule
& \multicolumn{6}{c}{Global error $\twonorm{\htheta - \bthetas}$} & \multicolumn{6}{c}{Local error $|S|^{-1}\sum_{k \in S}\twonorm{\hthetak{k} - \bthetaks{k}}$} \\
\cmidrule(lr){2-7} \cmidrule(lr){8-13}
Method$\backslash \epsilon$ & 0.00 & 0.05 & 0.10 & 0.15 & 0.20 & 0.25 & 0.00 & 0.05 & 0.10 & 0.15 & 0.20 & 0.25 \\
\midrule
Ours & 0.317 & \textbf{0.326} & \textbf{0.331} & \textbf{0.342} & \textbf{0.352} & \textbf{0.365} & \textit{0.981} & \textbf{0.985} & \textit{0.985} & \textbf{0.988} & \textbf{0.989} & \textbf{0.989} \\
Average & 0.317 & 11.045 & 11.556 & 11.732 & 13.068 & 13.106 & 1.231 & 11.108 & 11.617 & 11.787 & 13.118 & 13.157 \\
Single-task & -- & -- & -- & -- & -- & -- & 1.901 & 1.909 & 1.911 & 1.903 & 1.902 & 1.899 \\
Median & 0.385 & 0.420 & 0.502 & 0.635 & 0.834 & 1.106 & 1.250 & 1.255 & 1.285 & 1.337 & 1.443 & 1.613 \\
Trimmed mean & 0.317 & 0.395 & 0.548 & 0.751 & 1.021 & 1.377 & 1.231 & 1.247 & 1.303 & 1.396 & 1.557 & 1.808 \\
Krum & 1.769 & 1.798 & 1.771 & 1.787 & 1.817 & 1.869 & 2.131 & 2.157 & 2.131 & 2.141 & 2.165 & 2.211 \\
Bulyan & 0.317 & 0.409 & 0.488 & 0.545 & 0.567 & -- & 1.231 & 1.251 & 1.278 & 1.297 & 1.309 & -- \\
Filtering & 0.317 & \textit{0.341} & \textit{0.379} & \textit{0.425} & \textit{0.478} & \textit{0.537} & 1.231 & 1.231 & 1.241 & 1.251 & 1.270 & \textit{1.292} \\
MoM-Filtering & 0.317 & \textit{0.333} & \textit{0.351} & \textit{0.383} & \textit{0.417} & \textit{0.537} & 1.231 & \textit{1.229} & \textit{1.233} & \textit{1.239} & \textit{1.251} & \textit{1.292} \\
MoM-Krum & 0.985 & 0.993 & 1.026 & 1.067 & 1.072 & 1.090 & 1.544 & 1.547 & 1.567 & 1.589 & 1.597 & 1.606 \\
ARMUL & 0.485 & 0.422 & 0.492 & 0.704 & 1.101 & 1.501 & \textit{1.014} & \textit{0.991} & \textbf{0.984} & \textit{1.025} & \textit{1.197} & 1.479 \\
History & \textit{0.316} & \textit{0.341} & 0.401 & 0.504 & 0.653 & \textit{0.852} & 1.230 & 1.231 & 1.249 & 1.281 & 1.346 & \textit{1.451} \\
Bucketing & \textit{0.315} & 0.357 & 0.480 & 0.712 & 1.186 & 4.573 & 1.230 & 1.236 & 1.276 & 1.375 & 1.670 & 4.728 \\
Mean-reg & \textbf{0.296} & 1.251 & 1.972 & 2.644 & 3.187 & 3.668 & \textbf{0.948} & 1.323 & 1.772 & 1.849 & 1.847 & 1.789 \\
Dirty & 0.885 & 1.382 & 1.547 & 1.883 & 2.991 & 3.657 & 1.573 & 1.669 & 1.646 & 1.669 & 1.780 & 1.802 \\
RMTFL & 3.000 & 2.936 & 2.964 & 2.969 & 2.982 & 3.006 & 1.714 & 1.726 & 1.725 & 1.723 & 1.722 & 1.717 \\
RLRMTL & 1.641 & 1.842 & 2.261 & 2.756 & 3.248 & 3.801 & 2.036 & 1.795 & 1.716 & 1.685 & 1.651 & 1.671 \\
\bottomrule
\end{tabular}
\end{adjustbox}
\end{table}

% ---------------------------
\subsubsection{Linear regression with varying per-task sample size}

We also consider the setting $d=50$, $K=20$, $\epsilon=0.2$, and varying per-task sample sizes $n \in \{20,40,60,80,100,120\}$. All other aspects of the linear-regression data generation and contamination mechanism are the same as above. Table \ref{tab:sim-n} reports the global and local estimation errors.

As expected, the errors decrease as the per-task sample size grows. The global error of our estimator decreases from $0.737$ to $0.399$, and the local error decreases from $1.315$ to $0.688$; both are uniformly the smallest across all reported methods.

%{\color{red}[M: this is a small comment that might not need to be addressed. The table sort of suggests that the local errors are vanishing but we can expect an asymptotic bias, which might also be different among different methods?]} \yt{Yes I think that's why the performance of some methods stops improving when $n$ increases from 100 to 120. This bias is observable only when $n$ is large enough so that the term not dependent of $n$ dominates the error. But most terms in local error do shrink with $n \rightarrow \infty$.}
\begin{table}[!ht]
\centering
\caption{Linear regression with $d=50$, $K=20$, $\epsilon=0.2$, and varying per-task sample size $n$}
\label{tab:sim-n}
\setlength{\tabcolsep}{4pt}
\begin{adjustbox}{width=\textwidth}
\begin{tabular}{lcccccccccccc}
\toprule
& \multicolumn{6}{c}{Global error $\twonorm{\htheta - \bthetas}$} & \multicolumn{6}{c}{Local error $|S|^{-1}\sum_{k \in S}\twonorm{\hthetak{k} - \bthetaks{k}}$} \\
\cmidrule(lr){2-7} \cmidrule(lr){8-13}
Method$\backslash n$ & 20 & 40 & 60 & 80 & 100 & 120 & 20 & 40 & 60 & 80 & 100 & 120 \\
\midrule
Ours & \textbf{0.737} & \textbf{0.547} & \textbf{0.482} & \textbf{0.442} & \textbf{0.419} & \textbf{0.399} & \textbf{1.315} & \textbf{1.101} & \textbf{0.944} & \textbf{0.835} & \textbf{0.749} & \textbf{0.688} \\
Average & 13.549 & 13.226 & 13.132 & 13.088 & 13.062 & 13.043 & 13.591 & 13.281 & 13.180 & 13.139 & 13.108 & 13.096 \\
Single-task & -- & -- & -- & -- & -- & -- & 2.616 & 2.170 & 1.629 & 1.228 & \textit{0.988} & \textit{0.840} \\
Median & 1.602 & 1.033 & 0.887 & 0.801 & 0.763 & 0.708 & 1.963 & 1.543 & 1.445 & 1.390 & 1.367 & 1.343 \\
Trimmed mean & 1.993 & 1.181 & 0.993 & 0.886 & 0.840 & 0.788 & 2.292 & 1.645 & 1.511 & 1.441 & 1.411 & 1.387 \\
Krum & 2.566 & 2.161 & 1.772 & 1.502 & 1.380 & 1.324 & 2.817 & 2.456 & 2.108 & 1.871 & 1.775 & 1.734 \\
Bulyan & 1.208 & 0.834 & 0.728 & 0.672 & 0.632 & 0.590 & 1.658 & 1.412 & 1.355 & 1.316 & 1.299 & 1.285 \\
Filtering & \textit{0.933} & \textit{0.669} & \textit{0.595} & \textit{0.554} & \textit{0.518} & \textit{0.494} & \textit{1.469} & \textit{1.321} & \textit{1.286} & 1.260 & 1.248 & 1.243 \\
MoM-Filtering & \textit{0.933} & \textit{0.669} & \textit{0.595} & \textit{0.554} & \textit{0.518} & \textit{0.494} & \textit{1.469} & \textit{1.321} & \textit{1.286} & 1.260 & 1.248 & 1.243 \\
MoM-Krum & 2.172 & 1.335 & 1.124 & 0.992 & 0.912 & 0.871 & 2.453 & 1.755 & 1.599 & 1.501 & 1.456 & 1.430 \\
ARMUL & 1.432 & 1.134 & 1.141 & 1.138 & 1.113 & 1.081 & 1.658 & \textit{1.275} & \textit{1.204} & 1.182 & 1.123 & 1.091 \\
History & \textit{1.147} & \textit{0.796} & \textit{0.698} & \textit{0.631} & \textit{0.604} & \textit{0.566} & 1.612 & 1.393 & 1.337 & 1.300 & 1.286 & 1.275 \\
Bucketing & 2.702 & 1.364 & 1.114 & 0.994 & 0.939 & 0.888 & 2.930 & 1.782 & 1.594 & 1.510 & 1.472 & 1.447 \\
Mean-reg & 3.870 & 3.352 & 2.916 & 2.694 & 2.568 & 2.483 & 2.643 & 2.078 & 1.525 & 1.194 & 0.993 & 0.864 \\
Dirty & 3.884 & 3.164 & 2.300 & 2.395 & 2.476 & 2.461 & 2.624 & 2.049 & 1.526 & \textit{1.180} & \textit{0.988} & 0.872 \\
RMTFL & 3.351 & 2.975 & 2.986 & 2.992 & 2.998 & 2.890 & 2.618 & 1.988 & 1.475 & \textit{1.145} & \textit{0.950} & \textit{0.815} \\
RLRMTL & 3.059 & 3.024 & 3.016 & 3.004 & 3.002 & 2.988 & \textit{1.520} & 1.483 & 1.469 & 1.453 & 1.446 & 1.438 \\
\bottomrule
\end{tabular}
\end{adjustbox}
\end{table}

\subsubsection{Logistic regression with varying per-task sample size}

We finally repeat the varying-$n$ study under binary logistic regression with $d=50$, $K=20$, and $\epsilon=0.2$. The clean and contaminated task constructions are the same as in the logistic varying-$K$ study. Table \ref{tab:sim-logistic-n} reports the global and local estimation errors.

The logistic errors decrease as the per-task sample size grows. The global error of our estimator decreases from $1.063$ to $0.568$, and the local error decreases from $1.799$ to $1.437$; both are uniformly the smallest across all reported methods.

\begin{table}[!ht]
\centering
\caption{Logistic regression with $d=50$, $K=20$, $\epsilon=0.2$, and varying per-task sample size $n$}
\label{tab:sim-logistic-n}
\setlength{\tabcolsep}{4pt}
\begin{adjustbox}{width=\textwidth}
\begin{tabular}{lcccccccccccc}
\toprule
& \multicolumn{6}{c}{Global error $\twonorm{\htheta - \bthetas}$} & \multicolumn{6}{c}{Local error $|S|^{-1}\sum_{k \in S}\twonorm{\hthetak{k} - \bthetaks{k}}$} \\
\cmidrule(lr){2-7} \cmidrule(lr){8-13}
Method$\backslash n$ & 20 & 40 & 60 & 80 & 100 & 120 & 20 & 40 & 60 & 80 & 100 & 120 \\
\midrule
Ours & \textbf{1.063} & \textbf{0.769} & \textbf{0.688} & \textbf{0.625} & \textbf{0.599} & \textbf{0.568} & \textbf{1.799} & \textbf{1.602} & \textbf{1.543} & \textbf{1.488} & \textbf{1.464} & \textbf{1.437} \\
Average & 2.786 & 2.687 & 2.664 & 2.647 & 2.642 & 2.637 & 3.506 & 3.434 & 3.412 & 3.396 & 3.390 & 3.390 \\
Single-task & -- & -- & -- & -- & -- & -- & 3.199 & 2.899 & 2.606 & 2.332 & 2.125 & 1.953 \\
Median & 2.463 & 1.794 & 1.579 & 1.454 & 1.366 & 1.297 & 3.168 & 2.588 & 2.403 & 2.296 & 2.222 & 2.170 \\
Trimmed mean & 2.773 & 2.010 & 1.786 & 1.631 & 1.541 & 1.466 & 3.483 & 2.800 & 2.598 & 2.460 & 2.382 & 2.321 \\
Krum & 2.780 & 2.535 & 2.251 & 1.941 & 1.707 & 1.614 & 3.273 & 2.970 & 2.685 & 2.398 & 2.202 & 2.122 \\
Bulyan & \textit{1.695} & \textit{1.155} & \textit{0.939} & \textit{0.833} & \textit{0.749} & \textit{0.711} & \textit{2.173} & \textit{1.819} & \textit{1.710} & \textit{1.654} & \textit{1.614} & \textit{1.602} \\
Filtering & \textit{1.400} & \textit{0.973} & \textit{0.802} & \textit{0.722} & \textit{0.655} & \textit{0.616} & \textit{1.980} & \textit{1.715} & \textit{1.640} & \textit{1.600} & \textit{1.574} & \textit{1.564} \\
MoM-Filtering & \textit{1.400} & \textit{0.973} & \textit{0.802} & \textit{0.722} & \textit{0.655} & \textit{0.616} & \textit{1.980} & \textit{1.715} & \textit{1.640} & \textit{1.600} & \textit{1.574} & \textit{1.564} \\
MoM-Krum & 2.978 & 2.559 & 2.373 & 2.034 & 1.801 & 1.892 & 3.472 & 3.132 & 3.010 & 2.708 & 2.506 & 2.621 \\
ARMUL & 2.266 & 2.161 & 2.130 & 2.099 & 2.086 & 2.073 & 2.887 & 2.782 & 2.800 & 2.797 & 2.757 & 2.726 \\
History & 2.495 & 2.292 & 2.256 & 2.219 & 2.208 & 2.194 & 3.226 & 3.063 & 3.030 & 2.996 & 2.985 & 2.977 \\
Bucketing & 2.785 & 2.685 & 2.662 & 2.644 & 2.639 & 2.633 & 3.505 & 3.432 & 3.410 & 3.392 & 3.387 & 3.386 \\
Mean-reg & 2.121 & 1.815 & 1.685 & 1.602 & 1.534 & 1.455 & 2.777 & 2.435 & 2.282 & 2.176 & 2.087 & 1.991 \\
Dirty & 2.415 & 2.337 & 2.244 & 1.693 & 1.281 & 1.217 & 3.141 & 3.071 & 2.969 & 2.483 & 2.038 & 1.905 \\
RMTFL & 2.388 & 2.107 & 1.874 & 1.674 & 1.525 & 1.399 & 3.042 & 2.770 & 2.550 & 2.342 & 2.181 & 2.038 \\
RLRMTL & 2.466 & 2.466 & 2.471 & 2.467 & 2.471 & 2.470 & 3.229 & 3.238 & 3.238 & 3.233 & 3.234 & 3.237 \\
\bottomrule
\end{tabular}
\end{adjustbox}
\end{table}

% ------------------------------------------------------
\subsection{Additional real-data analysis results}\label{subsec: real data supp}
We retain the $20\%$ training split in the main text and report the remaining training proportions here. The same qualitative pattern persists across these splits: several regularization-based baselines are competitive when $\epsilon=0$, while under contamination our estimator is consistently the best or tied for best.

\begin{table}[!ht]
\centering
\caption{HAR local prediction error with $30\%$ training data.}
\label{tab:har-local-03}
\scriptsize
\setlength{\tabcolsep}{4pt}
\begin{adjustbox}{width=0.6\textwidth}
\begin{tabular}{lcccccc}
\toprule
Method & 0.00 & 0.05 & 0.10 & 0.15 & 0.20 & 0.25 \\
\midrule
Ours & \textit{0.018} & \textbf{0.017} & \textbf{0.018} & \textbf{0.018} & \textbf{0.019} & \textbf{0.019} \\
Average & 0.036 & 0.046 & 0.075 & 0.095 & 0.136 & 0.151 \\
Single-task & 0.043 & 0.043 & 0.042 & 0.042 & 0.042 & 0.042 \\
Median & 0.040 & 0.041 & 0.043 & 0.043 & 0.045 & 0.046 \\
Trimmed mean & 0.036 & \textit{0.038} & \textit{0.039} & \textit{0.040} & 0.043 & 0.044 \\
Krum & 0.288 & 0.289 & 0.290 & 0.291 & 0.291 & 0.289 \\
Bulyan & 0.036 & 0.039 & 0.051 & 0.054 & 0.057 & 0.059 \\
Filtering & 0.036 & \textit{0.037} & 0.041 & 0.045 & 0.055 & 0.061 \\
MoM-Filtering & 0.036 & \textit{0.037} & 0.041 & 0.045 & 0.055 & 0.061 \\
MoM-Krum & 0.174 & 0.175 & 0.180 & 0.181 & 0.189 & 0.185 \\
ARMUL & \textit{0.025} & 0.059 & 0.147 & 0.162 & 0.171 & 0.172 \\
History & 0.036 & 0.040 & 0.047 & 0.051 & 0.066 & 0.078 \\
Bucketing & 0.036 & 0.046 & 0.066 & 0.081 & 0.119 & 0.138 \\
Mean-reg & \textbf{0.014} & 0.040 & 0.042 & 0.042 & 0.042 & 0.042 \\
Dirty & 0.033 & 0.041 & 0.040 & \textit{0.040} & \textit{0.040} & \textit{0.040} \\
RMTFL & 0.037 & \textit{0.038} & \textit{0.038} & \textit{0.038} & \textit{0.038} & \textit{0.038} \\
RLRMTL & 0.172 & 0.172 & 0.171 & 0.171 & 0.172 & 0.172 \\
\bottomrule
\end{tabular}
\end{adjustbox}
\end{table}

\begin{table}[!ht]
\centering
\caption{HAR local prediction error with $40\%$ training data.}
\label{tab:har-local-04}
\scriptsize
\setlength{\tabcolsep}{4pt}
\begin{adjustbox}{width=0.6\textwidth}
\begin{tabular}{lcccccc}
\toprule
Method & 0.00 & 0.05 & 0.10 & 0.15 & 0.20 & 0.25 \\
\midrule
Ours & \textit{0.016} & \textbf{0.016} & \textbf{0.016} & \textbf{0.016} & \textbf{0.017} & \textbf{0.017} \\
Average & 0.036 & 0.045 & 0.072 & 0.093 & 0.135 & 0.149 \\
Single-task & 0.034 & 0.034 & \textit{0.033} & 0.033 & 0.034 & 0.034 \\
Median & 0.039 & 0.040 & 0.042 & 0.043 & 0.044 & 0.046 \\
Trimmed mean & 0.036 & 0.037 & 0.039 & 0.040 & 0.042 & 0.043 \\
Krum & 0.281 & 0.279 & 0.278 & 0.278 & 0.276 & 0.276 \\
Bulyan & 0.036 & 0.039 & 0.050 & 0.053 & 0.057 & 0.058 \\
Filtering & 0.036 & 0.036 & 0.041 & 0.045 & 0.054 & 0.060 \\
MoM-Filtering & 0.036 & 0.036 & 0.041 & 0.045 & 0.054 & 0.060 \\
MoM-Krum & 0.165 & 0.168 & 0.167 & 0.169 & 0.174 & 0.175 \\
ARMUL & \textit{0.022} & 0.051 & 0.139 & 0.158 & 0.169 & 0.171 \\
History & 0.036 & 0.040 & 0.046 & 0.051 & 0.067 & 0.078 \\
Bucketing & 0.036 & 0.045 & 0.063 & 0.078 & 0.120 & 0.140 \\
Mean-reg & \textbf{0.012} & \textit{0.033} & \textit{0.033} & 0.033 & 0.034 & 0.034 \\
Dirty & 0.026 & \textit{0.033} & \textit{0.033} & \textit{0.032} & \textit{0.032} & \textit{0.032} \\
RMTFL & 0.028 & \textit{0.030} & \textit{0.030} & \textit{0.030} & \textit{0.030} & \textit{0.030} \\
RLRMTL & 0.171 & 0.171 & 0.171 & 0.171 & 0.171 & 0.172 \\
\bottomrule
\end{tabular}
\end{adjustbox}
\end{table}

\begin{table}[!ht]
\centering
\caption{HAR local prediction error with $50\%$ training data.}
\label{tab:har-local-05}
\scriptsize
\setlength{\tabcolsep}{4pt}
\begin{adjustbox}{width=0.6\textwidth}
\begin{tabular}{lcccccc}
\toprule
Method & 0.00 & 0.05 & 0.10 & 0.15 & 0.20 & 0.25 \\
\midrule
Ours & \textit{0.014} & \textbf{0.014} & \textbf{0.014} & \textbf{0.015} & \textbf{0.015} & \textbf{0.016} \\
Average & 0.035 & 0.043 & 0.071 & 0.093 & 0.136 & 0.150 \\
Single-task & 0.028 & 0.028 & 0.028 & 0.028 & 0.028 & 0.028 \\
Median & 0.038 & 0.040 & 0.041 & 0.042 & 0.043 & 0.044 \\
Trimmed mean & 0.035 & 0.036 & 0.038 & 0.039 & 0.041 & 0.042 \\
Krum & 0.268 & 0.268 & 0.266 & 0.266 & 0.267 & 0.264 \\
Bulyan & 0.035 & 0.038 & 0.049 & 0.051 & 0.055 & 0.055 \\
Filtering & 0.035 & 0.035 & 0.041 & 0.044 & 0.054 & 0.059 \\
MoM-Filtering & 0.035 & 0.035 & 0.041 & 0.044 & 0.054 & 0.059 \\
MoM-Krum & 0.158 & 0.159 & 0.163 & 0.164 & 0.169 & 0.166 \\
ARMUL & \textit{0.020} & 0.047 & 0.134 & 0.155 & 0.169 & 0.171 \\
History & 0.035 & 0.039 & 0.046 & 0.050 & 0.067 & 0.081 \\
Bucketing & 0.035 & 0.043 & 0.063 & 0.078 & 0.121 & 0.142 \\
Mean-reg & \textbf{0.010} & \textit{0.027} & \textit{0.027} & 0.028 & 0.028 & 0.028 \\
Dirty & 0.023 & 0.028 & \textit{0.027} & \textit{0.027} & \textit{0.027} & \textit{0.027} \\
RMTFL & 0.023 & \textit{0.025} & \textit{0.025} & \textit{0.025} & \textit{0.025} & \textit{0.026} \\
RLRMTL & 0.172 & 0.171 & 0.171 & 0.171 & 0.172 & 0.172 \\
\bottomrule
\end{tabular}
\end{adjustbox}
\end{table}

\begin{table}[!ht]
\centering
\caption{HAR local prediction error with $60\%$ training data.}
\label{tab:har-local-06}
\scriptsize
\setlength{\tabcolsep}{4pt}
\begin{adjustbox}{width=0.6\textwidth}
\begin{tabular}{lcccccc}
\toprule
Method & 0.00 & 0.05 & 0.10 & 0.15 & 0.20 & 0.25 \\
\midrule
Ours & \textit{0.013} & \textbf{0.013} & \textbf{0.014} & \textbf{0.014} & \textbf{0.014} & \textbf{0.014} \\
Average & 0.035 & 0.043 & 0.071 & 0.094 & 0.138 & 0.151 \\
Single-task & 0.024 & 0.024 & 0.024 & 0.024 & 0.024 & 0.024 \\
Median & 0.038 & 0.039 & 0.041 & 0.041 & 0.043 & 0.044 \\
Trimmed mean & 0.035 & 0.036 & 0.037 & 0.038 & 0.040 & 0.041 \\
Krum & 0.260 & 0.259 & 0.258 & 0.259 & 0.255 & 0.252 \\
Bulyan & 0.035 & 0.038 & 0.050 & 0.051 & 0.055 & 0.056 \\
Filtering & 0.035 & 0.035 & 0.040 & 0.044 & 0.053 & 0.059 \\
MoM-Filtering & 0.035 & 0.035 & 0.040 & 0.044 & 0.053 & 0.059 \\
MoM-Krum & 0.153 & 0.154 & 0.156 & 0.159 & 0.165 & 0.164 \\
ARMUL & 0.019 & 0.043 & 0.130 & 0.153 & 0.169 & 0.170 \\
History & 0.035 & 0.039 & 0.045 & 0.050 & 0.068 & 0.083 \\
Bucketing & 0.035 & 0.043 & 0.063 & 0.080 & 0.125 & 0.145 \\
Mean-reg & \textbf{0.008} & \textit{0.023} & 0.024 & 0.024 & 0.024 & 0.024 \\
Dirty & 0.019 & 0.024 & \textit{0.023} & \textit{0.023} & \textit{0.023} & \textit{0.023} \\
RMTFL & \textit{0.018} & \textit{0.022} & \textit{0.022} & \textit{0.022} & \textit{0.022} & \textit{0.022} \\
RLRMTL & 0.171 & 0.171 & 0.171 & 0.171 & 0.171 & 0.172 \\
\bottomrule
\end{tabular}
\end{adjustbox}
\end{table}

% ------------------------------------------------------
\subsection{Additional details of implementation and parameter tuning}\label{subsec: implementation tuning}

All gradient-based methods are initialized at zero. In the linear-regression simulations, the models are fit without an intercept. In the HAR analysis, we use a 100-dimensional PCA representation, standardize the transformed features using the pooled training data for each split, and fit logistic models with an intercept.

In all reported experiments we use the same stepsizes $\eta=\eta^{(k)}=0.05$ for all tasks and all gradient descent-based methods. To speed up the computation, for our method, we first run $T = 500$ global gradient descent iterations without any local updates, to obtain the global estimator $\hotheta$. Then we use it to initialize the local estimators and run another $T_{\textup{local}} = 100$ local gradient descent iterations to obtain the local estimators $\hthetak{k}$. In the filtering algorithm (Algorithm \ref{algo: robust mean estimation}), we update the gradient covariance $\widehat{\bSigma}$ every 10 iterations of Algorithm \ref{algo: robust federated gradient descent}, for both global and local gradient descent, to reduce computational cost. We tune the filtering threshold over $\{0.05,0.1,0.25,0.5,1,2,3,5,10\}$ by 5-fold cross-validation. The local soft-thresholding parameter is tuned task by task over $\{0.05,0.2,0.5,1,2,5,10\}$. The same iteration counts are used during tuning.

For other gradient descent-based benchmark methods such as  Median, Trimmed mean, Krum, Bulyan, Filtering, MoM-Filtering, MoM-Krum, History, Bucketing, Average, and Single-task, the step size is $0.05$ and the number of iterations is $500$. For global aggregation methods, the task-level empirical gradients are aggregated at each iteration and a single global parameter is updated. Single-task uses the same gradient-descent update separately on each task. Bulyan uses Krum as its selection subroutine. History and Bucketing use clipping parameter $\tau=1$, and Bucketing uses $\lceil K/2\rceil$ buckets. 

ARMUL is run with its vanilla model and 5-fold cross-validation. We use the recommended step size $0.01$, $500$ global training iterations, and a grid of $10$ constants $C_i=2i/10$, $i=1,\ldots,10$, corresponding to task-specific penalties $C_i\sqrt{p/n_k}$, where $p$ is the feature dimension and $n_k$ is the sample size of task $k$.

Mean-reg, Dirty, RMTFL, and RLRMTL are tuned by 5-fold cross-validation. We use the estimator forms implemented in \texttt{MALSAR} \cite{zhou2011malsar}, and the parameter-tuning procedure and parameter grids follow those used in the original papers \cite{evgeniou2004regularized,jalali2010dirty,jalali2013dirty,gong2012robust,chen2011integrating,chen2012learning}. %{\color{red}[M: are these available in publicly available code or the claim is that your implementation corresponds to these methods with a tuning stratefy described in those papers?]}. \yt{Now it should be clearer. MALSAR (a matlab package) is public, and we migarated the code to python by following their code. But they have no tuning procedure. We added the CV-based tuning procedure based on the description and grid used in the original papers. I revised this paragraph and I hope it is clearer now.}
Let $\bm{\Theta}=(\bthetak{1},\ldots,\bthetak{K})$ be the task-parameter matrix and let $\mathcal{L}_k(\bthetak{k})$ denote the empirical loss of task $k$. For HAR the loss includes an unpenalized task intercept, while in the linear simulations the intercept is omitted. Mean-reg solves
\[
    \min_W \sum_{k=1}^K \mathcal{L}_k(\bthetak{k})
    + \rho_1\sum_{k=1}^K \|\bthetak{k}-\otheta\|_2^2
    + \rho_2\|\bm{\Theta}\|_F^2,\qquad
    \otheta=K^{-1}\sum_{k=1}^K \bthetak{k}.
\]
Dirty writes $\bm{\Theta}=\bm{G}+\bm{Q}$ and solves
\[
    \min_{\bm{G},\bm{Q}}\sum_{k=1}^K \mathcal{L}_k(\bm{g}^{(k)}+\bm{q}^{(k)})
    + \rho_1\|\bm{G}\|_{1,\infty}
    + \rho_2\|\bm{Q}\|_{1,1}.
\]
RMTFL writes $\bm{\Theta}=\bm{G}+\bm{Q}$ and solves
\[
    \min_{\bm{G},\bm{Q}}\sum_{k=1}^K \mathcal{L}_k(\bm{g}^{(k)}+\bm{q}^{(k)})
    + \rho_1\|\bm{G}\|_{2,1}
    + \rho_2\|\bm{Q}^\top\|_{2,1},
\]
where the first penalty promotes shared feature sparsity and the second penalty allows task-wise outliers. RLRMTL writes $\bm{\Theta}=\bm{G}+\bm{Q}$ and solves
\[
    \min_{\bm{G},\bm{Q}}\sum_{k=1}^K \mathcal{L}_k(\bm{g}^{(k)}+\bm{q}^{(k)})
    + \rho_1\|\bm{G}\|_*
    + \rho_2\|\bm{Q}\|_{2,1},
\]
where $\bm{G}$ is the low-rank component and $\bm{Q}$ is the task-sparse component.

For Mean-reg, following the Evgeniou-Pontil parameterization used in \texttt{MALSAR}, we tune $\mu\in\{0.1,0.5,1,2,10,1000\}$ and $C\in\{0.1,1\}$ and form the grids $\rho_1\in\{K/(aCn(K+\mu))\}$ and $\rho_2\in\{\mu/(aCn(K+\mu))\}$, with $a=2$ in the varying-$K$ and varying-$n$ linear-regression scripts and $a=1$ in the remaining reported experiments. The parameter $\mu$ controls how strongly the task estimators shrink toward the shared mean, while $C$ controls the overall regularization level. For the structured-sparsity estimators, the grids follow the scale choices used in the \texttt{MALSAR} numerical examples and the corresponding original simulations, with multiplicative constants $\{0.01,0.1,1,10,100\}$ to search around the nominal regularization level. Dirty uses both penalty grids $\sqrt{K\log(d)/n}\{0.01,0.1,1,10,100\}$. RMTFL uses $\rho_1=90(2n)^{-1}\{0.01,0.1,1,10,100\}$ and $\rho_2=90(2n)^{-1}\sqrt{K\log(d)/n}\{0.01,0.1,1,10,100\}$. RLRMTL uses centers $50K/2$ and $10K/2$, each multiplied by $\{0.01,0.1,1,10,100\}$. In HAR, $n$ in these grids is replaced by the mean training sample size across tasks and $d$ by the feature dimension after preprocessing.

\end{document}